\documentclass{article} 
\usepackage{iclr2027_conference,times}

\usepackage{amsmath,amsfonts,bm}

\def\eqref#1{equation~\ref{#1}}

\def\1{\bm{1}}

\DeclareMathAlphabet{\mathsfit}{\encodingdefault}{\sfdefault}{m}{sl}
\SetMathAlphabet{\mathsfit}{bold}{\encodingdefault}{\sfdefault}{bx}{n}

\DeclareMathOperator*{\argmin}{arg\,min}

\usepackage{hyperref}
\usepackage{url}

\usepackage{graphicx}
\usepackage{natbib}
\usepackage{caption}
\usepackage{algorithm}
\usepackage{algorithmic}
\usepackage{subcaption}  
\usepackage{adjustbox}  
\usepackage{booktabs} 
\usepackage{amssymb}
\usepackage{bm}
\usepackage{enumitem}
\usepackage{newfloat}
\usepackage{listings}
\usepackage{amsmath}
\usepackage{amsthm}
\usepackage{color}
\usepackage{amsthm,amsmath,amssymb}
\newtheorem{thm}{Theorem}
\newtheorem{lem}{Lemma}
\newtheorem{cor}{Corollary}
\newtheorem{remark}{Remark}
\newtheorem{assumption}{Assumption}

\title{Domain-Adapted Diffusion Models for Conditional Independence Testing}

\author{Yanfeng Yang$^{12}$\thanks{Corresponding Authors: Yanfeng Yang (\texttt{yanfengyang0316@gmail.com}), Ziqi Chen (\texttt{zqchen@fem.ecnu.edu.cn}), Shuai Li (\texttt{2582005851@qq.com}), Kenji Fukumizu (\texttt{fukumizu@ism.ac.jp})}, \ Junda Zhao$^{3}$, Yijie Gao$^{3}$, Jiaqi Yang$^{3}$, Xinyu Shi$^{3}$, Ziqi Chen$^{3*}$, \\
\textbf{Shunyu Zhao}$^{12}$, \textbf{Shuai Li}$^{4*}$, \textbf{Wei Huang}$^5$, \textbf{Eshant English}$^{67}$, \textbf{Kenji Fukumizu}$^{2*}$ \\
${}^1$The Graduate University for Advanced Studies, SOKENDAI, Tokyo, Japan \\
${}^2$The Institute of Statistical Mathematics, Tokyo, Japan \\
${}^3$East China Normal University, Shanghai, China \\
${}^4$Hunan Normal University, Hunan, China \\
${}^5$Riken AIP, Tokyo, Japan \\
${}^6$The University of Tokyo, Tokyo, Japan \\
${}^7$Hasso Plattner Institute, Brandenburg, Germany
}

\iclrfinalcopy 

\begin{document}

\maketitle

\begin{abstract}
Conditional independence (CI) is a fundamental concept in statistics and machine learning. Recent advances in conditional generative modeling provide flexible tools for generative-model-based CI tests, which rely on an estimated conditional distribution to generate randomized samples. However, errors in estimating this distribution accumulate in existing Type I error bounds, and consistency of the generative estimator alone does not guarantee asymptotic Type I error control.
To address this limitation, we formulate conditional generative modeling as a domain adaptation problem and leverage auxiliary data from multiple source domains to improve estimation in the target CI testing domain.
We propose Domain-Adapted Diffusion (DA-Diff), a multi-source domain adaptation framework for conditional diffusion models based on weighted empirical risk minimization over both target and source domains.
We establish the convergence rate of DA-Diff and show how transferable source data can improve target-domain estimation through an increased effective sample size while controlling transfer bias.
Building on DA-Diff, we further propose Domain-Adapted Conditional Independence Testing (DA-CIT) and show that its Type I error satisfies $P(p \leq \alpha) \leq \alpha + o(1)$.
Experiments demonstrate that DA-Diff improved conditional generation quality compared with transfer-learning diffusion baselines, while DA-CIT provides strong Type I error control and competitive power.
\end{abstract}

\section{Introduction}
\label{sec_intro}
Conditional independence (CI) is a fundamental concept in statistics and machine learning, with broad applications to variable selection, causal structure learning, and sufficient dimension reduction \citep{pearl2009causality,spirtes2000causation,cook2026,yang_2025_cdcit,he_2025_on_hardness,wang_2025_practical_kernel,wang_2026_zeroflow,Zhang_2026_jrssb_cit}. Let $X_0\in\mathcal{X}\subseteq\mathbb{R}^{d_x}$, $Y_0\in\mathcal{Y}\subseteq\mathbb{R}^{d_y}$, and $Z_0\in\mathcal{Z}\subseteq\mathbb{R}^{d_z}$ be random variables. CI testing considers the hypotheses
\begin{equation}
\label{eq_cit_def}
    H_0: X_0 \perp \! \! \! \perp Y_0 | Z_0 \quad \text{versus} \quad H_1: X_0 \not \! \perp \! \! \! \perp Y_0 |Z_0 .
\end{equation}
A broad range of CI tests has been developed, including kernel-based,
regression-based, and metric-based approaches \citep{fukumizu2007kernel,zhang2011kernel,shah2020hardness,
Azadkia_2021_Chatterjee,Banerjee_2026_ball_div}. More recently, the conditional randomization test (CRT) \citep{candes2018panning} has become an important framework for CI testing
\citep{bellot2019conditional,
liu2022fast,Shuangning_2023_maxway,
li2024k,yang_2025_cdcit,
ren_2025_sgmcit,zhao_2026_fmcit}. A key advantage of the CRT is that, given a correctly specified conditional distribution
$P_0(\cdot| Z_0)$, it yields valid inference while allowing flexible choices of test statistics \citep{berrett2020conditional}.
When $P_0(\cdot| Z_0)$ is estimated from data, however, the validity of the resulting test depends critically on its estimation accuracy. 

Recent advances in generative modeling provide flexible tools for estimating
$P_0(\cdot| Z_0)$ \citep{goodfellow2014generative,ho2020denoising,song2021score,karras_2022_edm,Fu2024UnveilCD}, leading to several generative-model-based CI tests \citep{bellot2019conditional,shi2021double,ren_2025_sgmcit,zhao_2026_fmcit}. Among generative models, diffusion models are particularly attractive due to their strong generation quality and have also been shown to perform well for CRT \citep{yang_2025_cdcit}.  Despite these advantages, a central difficulty is that errors in conditional distribution estimation directly affect the validity of the resulting CRT. To make this dependence explicit, suppose that the target dataset contains $2n_0$ independent observations, with $n_0$ observations used for the CRT and the remaining $n_0$ observations used to estimate $P_0(\cdot| Z_0)$. Let $\widehat{P}_0(\cdot| Z_0)$ denote the estimated conditional distribution and let $p$ be the resulting CRT $p$-value. By Theorem~4 of \citet{berrett2020conditional}, under $H_0$, 
\begin{equation}
\label{eq_pval_of_CRT}
    \mathbb{P}(p \leq \alpha)
    \leq
    \alpha
    +
    \mathcal{O}\left(
        n_0 \cdot
        \mathbb{E} 
        \left[
            d_{\mathrm{TV}}
            \left(
                P_0(\cdot | Z_0),
                \widehat{P}_0(\cdot | Z_0)
            \right)
        \right]
    \right).
\end{equation}
Here, $d_{\mathrm{TV}}\!\left(
P_0(\cdot | Z_0),
\widehat{P}_0(\cdot | Z_0)
\right)$ denotes the total variation (TV) distance between the true and estimated conditional distributions of $X_0$ given $Z_0$. The second term in (\ref{eq_pval_of_CRT}) quantifies the excess Type~I error caused by conditional distribution estimation. In particular, (\ref{eq_pval_of_CRT}) guarantees asymptotic Type~I error control when $\mathbb{E} 
    \left[
        d_{\mathrm{TV}}
        \left(
            P_0(\cdot | Z_0),
            \widehat{P}_0(\cdot | Z_0)
        \right)
    \right]
    =
    o(n_0^{-1})$. Note that this is only a sufficient condition, rather than a necessary condition
for asymptotic validity.
Nevertheless, this sufficient condition is particularly stringent.
For perspective, even regular parametric estimators typically achieve only  $\mathcal{O}(n_0^{-1/2})$ \citep{Vaart_1998_linear_005}, while flexible nonparametric conditional generative models generally converge more slowly and their rates depend on $d_x$ and $d_z$. 
Thus, when $\widehat P_0(\cdot| Z_0)$ is learned only from the $n_0$ target observations reserved for conditional distribution estimation, existing convergence rates may be insufficient to make the excess term in (\ref{eq_pval_of_CRT}) vanish.

Therefore, a natural strategy is to increase the amount of data available for conditional distribution estimation. Existing theory for conditional diffusion models shows that estimation accuracy improves with the growth of sample size \citep{Fu2024UnveilCD,yang_2025_cdcit}, motivating the use of auxiliary data to improve estimation of $P_0(\cdot| Z_0)$. This argument is true when the auxiliary observations and the target observations follow the same distribution. However, in practice, auxiliary data often come from heterogeneous source domains rather than from the target distribution itself. Naively pooling such observations can introduce transfer bias that offsets the benefit of the additional samples. This raises the central question of how to selectively leverage multiple heterogeneous source domains to improve estimation of $P_0(\cdot| Z_0)$ while controlling the bias induced by source--target distribution shifts.

To solve this question, we propose the \textbf{Domain-Adapted Diffusion Model (DA-Diff)}, a multi-source domain adaptation framework for conditional diffusion models. DA-Diff learns a target-oriented conditional score model by minimizing a weighted empirical score-matching loss over the target domain and other source domains, where the source-domain weights determine how much information is transferred from each source. Our theoretical analysis reveals an explicit estimation--transfer trade-off: incorporating informative source observations increases the effective sample
size and reduces estimation error, whereas discrepancies between source and target conditional scores induce transfer bias. This decomposition provides the theoretical basis for our source-weighting criterion, which favors the statistical benefit of informative sources while penalizing source--target discrepancy. We further translate the target-domain score-matching risk bound into a TV error bound for the learned conditional distribution. Building on this result, we develop the \textbf{Domain-Adapted Conditional Independence Test (DA-CIT)}, which uses DA-Diff to generate the randomized samples required by the CRT. Under suitable source--target transferability conditions, we establish asymptotic Type I error control for DA-CIT.


Our contributions are summarized as follows:
\begin{itemize}
    \item[1.] We propose the \textbf{Domain-Adapted Diffusion Model (DA-Diff)}, the first transfer-learning framework for diffusion models based on weighted ERM under a MSDA setting. We establish its theoretical convergence rate in TV distance, providing convergence guarantees for diffusion models trained by jointly leveraging multiple heterogeneous source domains.

    \item[2.] Building on DA-Diff, we develop the \textbf{Domain-Adapted Conditional Independence Test (DA-CIT)}, which first incorporates heterogeneous external data into generative model-based CIT. We further establish theoretical conditions under which the excess Type~I error of DA-CIT vanishes asymptotically, i.e., is controlled at the rate $o(1)$.

    \item[3.] We conduct synthetic experiments to evaluate both components of the proposed framework. The DA-Diff experiments assess conditional-generation accuracy under varying numbers of heterogeneous source domains and compare against various baselines. The DA-CIT experiments evaluate both Type~I error control and test power under nonlinear data-generating mechanisms. We further evaluate DA-CIT on the multi-source Sachs dataset to demonstrate its performance on real-world data.
\end{itemize}

\section{Preliminary}

\subsection{Conditional randomization tests (CRT)}
\label{sec_crt}

Conditional randomization tests (CRT) is a general framework for testing CI \citep{candes2018panning}. Let
$D_0=\{(x_{0,i},y_{0,i},z_{0,i})\}_{i=1}^{2n_0}$
denote the target-domain dataset. We reserve the first $n_0$ observations for testing and use the remaining $n_0$ observations exclusively for conditional distribution estimation. 
Suppose that an estimator $\widehat{P}_0(\cdot| Z_0)$ of the conditional distribution $P_0(\cdot| Z_0)$ is available.
Conditioning on $\mathbf{Z}_0
:=
(z_{0,1},\ldots,z_{0,n_0})^\top
\in\mathbb{R}^{n_0\times d_z}$, we independently generate $B$ pseudo-samples from $\prod_{i=1}^{n_0}
\widehat{P}_0(\cdot| z_{0,i})$, yielding $\mathbf{X}_0^{(b)}
\in\mathbb{R}^{n_0\times d_x}, \ 
b=1,\ldots,B$. Let $\mathbb T(\cdot,\cdot,\cdot):
\mathbb{R}^{n_0\times d_x}
\times
\mathbb{R}^{n_0\times d_y}
\times
\mathbb{R}^{n_0\times d_z}
\to\mathbb{R}$ be a pre-specified test statistic, with larger values indicating stronger evidence against $H_0$. Unless otherwise specified, we use distance correlation (dCor) \citep{Maria_2008_dcor} as $\mathbb T$ in the proposed DA-CIT.  
Define $\mathbf{X}_0
=
(x_{0,1},\ldots,x_{0,n_0})^\top
\in\mathbb{R}^{n_0\times d_x}$ and $
\mathbf{Y}_0
=
(y_{0,1},\ldots,y_{0,n_0})^\top
\in\mathbb{R}^{n_0\times d_y}$. The test statistic computed from the observed data is $\mathbb T^{(0)}
:=
\mathbb T(\mathbf{X}_0,\mathbf{Y}_0,\mathbf{Z}_0)$, while for each pseudo-sample we compute $\mathbb T^{(b)}
:=
\mathbb T(\mathbf{X}_0^{(b)},\mathbf{Y}_0,\mathbf{Z}_0), \ 
b=1,\ldots,B$. The CRT compares the observed statistic $\mathbb T^{(0)}$ with the randomized statistics
$\{\mathbb T^{(b)}\}_{b=1}^B$ and computes the $p$-value as
\begin{equation}
\label{eq_crt_pval}
p
=
\frac{
1+\sum_{b=1}^B
\mathbf{1}\!\left\{
\mathbb T^{(b)}\geq \mathbb T^{(0)}
\right\}
}{
B+1
}.
\end{equation}
Under $H_0$ in (\ref{eq_cit_def}), if $\widehat P_0(\cdot| Z_0)=P_0(\cdot| Z_0)$, the observed statistic $\mathbb T^{(0)}$ and the randomized statistics $\{\mathbb T^{(b)}\}_{b=1}^B$ are exchangeable, and the CRT provides finite sample Type I error control, $\mathbb P(p\leq\alpha)\leq\alpha$ \citep{berrett2020conditional}. When $\widehat P_0(\cdot| Z_0)$ is an approximation to $P_0(\cdot| Z_0)$, the Type I error is no longer guaranteed to be bounded exactly by $\alpha$; instead, the excess Type I error can be controlled by the estimation error of the conditional distribution, as shown in~(\ref{eq_pval_of_CRT}). Under $H_1$, $\mathbb T^{(0)}$ tends to be larger than its randomized counterparts, leading to smaller $p$-values and evidence against $H_0$. 

\subsection{Diffusion model}
\label{sec_diff_model}
Diffusion models are a class of generative models that can flexibly approximate complex conditional distributions such as $P_0(\cdot| z)$
\citep{song2021score,ho2020denoising,Fu2024UnveilCD}.
A diffusion model consists of a forward diffusion process and a reverse process.
Conditioned on $Z_0=z$, let
$X_0(0)\sim P_0(\cdot| z)$.
The forward process is defined by the stochastic differential equation (SDE)
\begin{equation}
\label{eq_diff_forward}
    dX_0(t)
    =
    -\frac{1}{2}X_0(t)\,dt
    +
    dB(t),
    \qquad
    t\in[0,T],
\end{equation}
where $B(t)$ is a standard Brownian motion and $T>0$ denotes the terminal time.
By the properties of the Ornstein--Uhlenbeck (OU) process, $X_0(t)| X_0(0)
    \sim
    \mathcal{N}
    \left(
        a_t X_0(0),
        \sigma_t^2 I_{d_x}
    \right)$, where
$a_t=\exp(-t/2)$ and
$\sigma_t=\sqrt{1-\exp(-t)}$.
As $T$ becomes sufficiently large, the marginal distribution of $X_0(T)$ approaches $\mathcal{N}(0,I_{d_x})$.

The reverse process transforms samples from $N(0,I_{d_x})$ back to the data distribution $P_0(\cdot| z)$.
Specifically, it is characterized by the reverse-time SDE
\begin{equation}
\label{eq_diff_reverse}
    dX_0(t)
    =
    \left[
        -\frac{1}{2}X_0(t)
        -
        \nabla_x
        \log p_{0,t}
        \bigl(
            X_0(t)| z
        \bigr)
    \right]dt
    +
    d\overline{B}(t), \qquad  X_0(T) \sim N(0,I_{d_x}),
\end{equation}
which is simulated backward in time from $t=T$ to an early stopping time $t_0>0$.
Here,
$p_{0,t}(\cdot| z)$ denotes the conditional density of $X_0(t)$ given $Z_0=z$, and
$\overline{B}(t)$ is a Brownian motion associated with the reverse process. The conditional score function
$\nabla_x\log p_{0,t}(\cdot| z)$
appearing in (\ref{eq_diff_reverse}) is generally unknown and must therefore be estimated from data.
We estimate it using denoising score matching
\citep{Vincent_2011_score_mat,song2021score}.
Let $s:
    \mathbb{R}^{d_x}
    \times
    \mathbb{R}^{d_z}
    \times
    \mathbb{R}
    \to
    \mathbb{R}^{d_x}$ be a score network belonging to a ReLU network class $\mathcal{F}$ (defined in (\ref{eq_relu_network_class})).
For a single observation $(x,z)\in\mathcal{X}\times\mathcal{Z}$, define the denoising score-matching loss as
\begin{align}
\label{eq_diffusion_loss_single}
    \ell(x,z;s)
    =
    \frac{1}{T-t_0}
    \int_{t_0}^{T}
    \mathbb{E}_{X_t
    \sim
    \mathcal{N}(a_t x,\sigma_t^2 I_{d_x})}
    \left[
        \left\|
            s(X_t,z,t)
            +
            \left({X_t-a_t x}\right)/{\sigma_t^2}
        \right\|_2^2
    \right]
    dt.
\end{align}
We use the last $n_0$ samples in $D_0$, i.e.
$\{(x_{0,i},y_{0,i},z_{0,i})\}_{i=n_0+1}^{2n_0}$,
to estimate the target conditional distribution. The empirical
score-matching loss is
\begin{equation}
\label{eq_diffusion_loss_empirical}
    \widehat{\mathcal{L}}_0(s)
    =
    \frac{1}{n_0}
    \sum_{i=n_0+1}^{2n_0}
    \ell(x_{0,i},z_{0,i};s).
\end{equation}
The score network is then estimated by empirical risk minimization (ERM):
\begin{equation}
\label{eq_erm_s0}
    \widehat{s}_0
    =
    \underset{s\in\mathcal{F}}{\arg\min}
    \,
    \widehat{\mathcal{L}}_0(s).
\end{equation}
Replacing the unknown $\nabla_x\log p_{0,t}(X_0(t)| z)$
in (\ref{eq_diff_reverse}) with
$\widehat{s}_0(X_0(t),z,t)$
yields a learned reverse SDE.
Starting from
$X_0(T)\sim\mathcal{N}(0,I_{d_x})$
and solving this SDE backward from $T$ to $t_0$ produces conditional samples from the learned distribution, which approximates
$P_0(\cdot| Z_0)$
when $T$ is sufficiently large and $t_0$ is sufficiently small.

\subsection{Domain adaptation through weighted empirical risk minimization (ERM)}
\label{sec_da_werm}

We regard the dataset $D_0$ used for the CRT in Section~\ref{sec_crt} as the target-domain dataset.
In addition, we assume access to $K$ source domains, with datasets $D_k=\{(x_{k,i},z_{k,i})\}_{i=1}^{n_k},
 k=1,\ldots,K$, of sample sizes $n_1,\ldots,n_K$, respectively.
Compared with training a model solely on the limited target-domain data, MSDA can exploit informative samples from multiple source domains to improve target-domain performance and potentially accelerate statistical convergence
\citep{David_2006_da,David_2010_theory_da,Farahani_2021_da_Review}.
Among existing domain adaptation (DA) approaches, weighted ERM provides a simple and widely used framework for balancing information from heterogeneous domains
\citep{Konstantinov_2019_Robust_Learning,vogel_2020_werm,zhang_2026_uowq}.

For each domain $k\in\{0,1,\ldots,K\}$, let
$(X_k,Z_k)\overset{\text{i.i.d.}}{\sim} P_k$ denote random variables, where each $P_k$ is a joint distribution on
$\mathcal{X}\times\mathcal{Z}$ with density $p_k(x,z)$ and factorization: $P_k(x,z)=P_k(x| z)P_k(z)$ and $ p_k(x,z)=p_k(x| z)p_k(z)$. Here, $P_k(\cdot| z)$ and $p_k(\cdot| z)$ denote the conditional distribution and density of $X_k$ given $Z_k=z$, respectively, while $P_k(z)$ and $p_k(z)$ denote the marginal distribution and density of $Z_k$.
In particular, $P_0(\cdot| z)$ coincides with the target conditional distribution introduced in Section~\ref{sec_diff_model}.  Given $D_k$, we define the empirical diffusion loss on the $k$-th domain as
\begin{equation}
\label{eq_empirical_source_loss}
    \widehat{\mathcal{L}}_k(s)
    :=
    \frac{1}{n_k}
    \sum_{i=1}^{n_k}
    \ell(x_{k,i},z_{k,i};s),\quad k = 1,\ldots, K.
\end{equation}
Inspired by \citet{zhang_2026_uowq}, we consider the weighted ERM objective
\begin{equation}
\label{eq_weighted_erm_objective}
    \widehat{\mathcal{L}}_w(s)
    :=
    \frac{
        n_0\widehat{\mathcal{L}}_0(s)
        +
        \sum_{k=1}^{K}w_k n_k\widehat{\mathcal{L}}_k(s)
    }{
        n_0+\sum_{k=1}^{K}w_k n_k
    },
\end{equation}
where $w_k\geq 0$ denotes the weight assigned to the $k$-th source domain.
The factors $n_k$ account explicitly for the quantities of information contributed by each domain, while the weights $w_k$ control the extent to which each source domain is incorporated into training.
For a given weight vector
$w=(w_1,\ldots,w_K)^\top$,
the resulting weighted-ERM estimator is
\begin{equation}
\label{eq_weighted_erm_sw}
    \widehat{s}_w
    =
    \underset{s\in\mathcal{F}}{\arg\min}
    \widehat{\mathcal{L}}_w(s).
\end{equation}
To characterize the target-domain performance of $\widehat{s}_w$, we further define the population loss on domain $k$ as
\begin{equation}
\label{eq_population_domain_loss}
    \mathcal{L}_k(s)
    :=
    \mathbb{E}_{(X_k,Z_k)\sim P_k}
    \big[
        \ell(X_k,Z_k;s)
    \big], \quad k = 0,\ldots, K.
\end{equation}
Our goal is to choose the source-domain weights so that the resulting estimator achieves the smallest expected population loss on the target domain.
Accordingly, we define the optimal weights by
\begin{align}
\label{eq_optimal_weights_sample_sizes}
    (w_1^*,\ldots,w_K^*)^\top
    =
    \underset{\{w_k\geq 0\}_{k=1}^K}{\arg\min}
    \;
    \mathbb{E}_{D_0,\ldots,D_K}
    \left[
        \mathcal{L}_0(\widehat{s}_w)
    \right].
\end{align}
Let
$w^*:=(w_1^*,\ldots,w_K^*)^\top$.
Substituting $w^*$ into (\ref{eq_weighted_erm_objective}) yields the final weighted-ERM estimator
\begin{equation}
\label{eq_weighted_erm_sw_star}
    \widehat{s}_{w^*}
    =
    \underset{s\in\mathcal{F}}{\arg\min}
    \widehat{\mathcal{L}}_{w^*}(s).
\end{equation}
The role of the weights is to balance the statistical benefit of incorporating additional source samples against the bias induced by discrepancies between the source and target distributions.
As shown in the theoretical analysis in Section \ref{sec_thm}, under suitable conditions, informative source domains can improve the target-domain convergence rate relative to training solely on the target-domain data.  The detailed optimization method of (\ref{eq_optimal_weights_sample_sizes}) can be found in Appendix \ref{sec_optimal_weights}.

\section{Theoretical results}
\label{sec_thm}
In this section, we establish theoretical guarantees for DA-Diff and DA-CIT. Theorem~\ref{thm_score_risk_bound} bounds the target-domain score-matching risk of DA-Diff, Theorem~\ref{thm_tv_target_bound} translates this result into a TV error bound for the learned conditional distribution, and Theorem~\ref{thm_da_cit_type1} further establishes asymptotic Type~I error control for DA-CIT. Together, these results connect domain-adapted score matching to distributional approximation and, ultimately, valid CI testing. Complete proofs are provided in Appendix~\ref{sec_proof}. The complete algorithms of training DA-Diff and DA-CIT can be found in Algorithm \ref{algo_train_da_diff} and \ref{algo_da_cit}.

The proof is organized as follows. We first introduce the H\"older
function class and the ReLU network class at the beginning of
Appendix~\ref{sec_proof}. We then characterize the global population
minimizer of the weighted score-matching risk in
Appendix~\ref{app_global_minimizer}, followed by the excess-risk
decomposition in Appendix~\ref{app_excess_decomp} and the properties of the mixture
density $p_{w,t}$ in Appendix~\ref{app_property_of_pw}. The transfer bias and estimation
error are bounded separately in Appendix~\ref{app_transfer_bias} and
Appendix~\ref{app_bound_esti_error}, respectively, and are combined to prove
Theorem~\ref{thm_score_risk_bound} in Appendix~\ref{app_proof_of_th1}. The TV bound in
Theorem~\ref{thm_tv_target_bound} is proved in Appendix~\ref{app_proof_of_tv}, and the
Type~I error result in Theorem~\ref{thm_da_cit_type1} is proved in
Appendix~\ref{subsec_da_cit_type1}. The main new ingredient is the
extension of diffusion-model convergence analysis to weighted ERM over
multiple domains, which requires characterizing the mixture
distribution and separately controlling the source-induced transfer
bias and the weighted estimation error.

\subsection{Domain-adapted diffusion model (DA-Diff)}
In this section, we provide a bound on the score matching risk of DA-Diff trained via weighted ERM, together with the corresponding bound on the TV distance to the target domain.

For each domain $k\in\{0,1,\ldots,K\}$, we consider the same forward process as in~(\ref{eq_diff_forward}), initialized by $X_k(0)| Z_k=z \sim P_k(\cdot| z)$. Let $p_{k,t}(\cdot| z)$ denote the conditional density of $X_k(t)$ given $Z_k=z$, and define the (intractable) conditional score function as:
\begin{equation}
\label{eq_domain_score_def}
 s_k^*(x,z,t) := \nabla_x \log p_{k,t}(x| z), \qquad k=0,1,\ldots,K.
\end{equation}
Define the joint density of $(X_k(t),Z_k)$ by $p_{k,t}(x,z) := p_{k,t}(x|z)\,p_k(z)$. To distinguish from the denoising score matching loss in~(\ref{eq_diffusion_loss_single}), we define the score matching risk on the $k$-th domain as:
\begin{align}
\label{eq_population_score_matching_risk}
\mathcal{R}_k(s) := \frac{1}{T-t_0}\int_{t_0}^{T}  \mathbb{E}_{(X_t,Z) \sim p_{k,t}(\cdot,\cdot)}\Big[\,\| s(X_t,Z,t) - s_k^*(X_t,Z,t) \|_2^2\,\Big] \, dt.
\end{align}
According to~\cite{Vincent_2011_score_mat}, $\mathcal{R}_k(s)$ can be expressed as:
\begin{equation}
\label{eq_score_matching_risk_decomposition}
\mathcal{R}_k(s) = \mathcal{L}_k(s) + C_{\mathrm{DSM},k}.
\end{equation}
Therefore, the optimization objective in~(\ref{eq_optimal_weights_sample_sizes}) can be equivalently expressed as
$\mathbb{E}_{D_0,\ldots,D_K}\big[\mathcal{R}_0(\widehat{s}_w)\big]$.
We next establish an upper bound of the target score matching risk under the following assumptions.

\begin{assumption}[Transferability conditions]
\label{assump_transf}
Let
    $W_N := \sum_{k=1}^K w_k n_k$,
and define the weighted source marginal density of $Z$ as
    $p_{w\setminus 0}(z)
    :=
    \sum_{k=1}^K
    \frac{w_k n_k}{W_N} p_k(z)$.
Assume that there exist constants
$C_{\mathrm{dr}}\in(0,1]$ and
$C_{\mathrm{DR}}\geq 1$, such that
\begin{equation}
\label{eq_assump_ratio}
    C_{\mathrm{dr}} \cdot
    \max_{1\leq k\leq K} p_k(z)
    \leq
    p_0(z)
    \leq
    C_{\mathrm{DR}} \cdot
    p_{w\setminus 0}(z),
    \qquad
    p_0\text{-almost everywhere}.
\end{equation}
Let $\gamma>0$ be a convergence rate. For $k=1,\ldots,K$, assume that
\begin{equation}
\label{eq_assump_loss}
    \frac{1}{T-t_0}
    \int_{t_0}^{T}
    \mathbb{E}_{(X_t,Z)\sim p_{0,t}}
    \left[
        \left\|
            s_0^*(X_t,Z,t)
            -
            s_k^*(X_t,Z,t)
        \right\|_2^2
    \right]
    dt
    =
    \mathcal{O}\left(n_0^{-2\gamma}\right).
\end{equation}
\end{assumption}

\begin{assumption}[Smoothness and tail conditions]
\label{assump_smoothness_tail}
Let $\beta>0$ be the smoothness parameter.
For each $k=0,\ldots,K$, assume that there exists a function
\begin{equation}
\label{eq_assump_holder}
    f_k
    \in
    \mathcal{H}^{\beta}
    \left(
        \mathbb{R}^{d_x}\times\mathbb{R}^{d_z},
        B
    \right),
\end{equation}
where the H\"older ball
$\mathcal{H}^{\beta}(\cdot,B)$ is defined in
(\ref{eq_holder_ball}), such that the conditional density $p_k(x| z)$ admits
the representation
\begin{equation}
    p_k(x| z)
    =
    \exp\left(
        -{C_x\|x\|_2^2}/{2}
    \right)
    f_k(x,z),
    \qquad
    (x,z)\in
    \mathbb{R}^{d_x}\times\mathbb{R}^{d_z},
    \quad
    k=0,\ldots,K.
    \label{eq_joint_density_tail}
\end{equation}
Here, $B>0$ and $C_x>0$ are constants independent of $k$. 
Moreover, there exist constants
$0<C_f<\infty$, independent of $k$, such that
\begin{equation}
    \inf_{k=0,\ldots,K}
    \inf_{(x,z)\in\mathbb{R}^{d_x}\times\mathbb{R}^{d_z}}
    f_k(x,z)
    \geq
    C_f.
    \label{eq_fk_lower_bound}
\end{equation}
In addition, we assume that each marginal density $p_k(z)$ has a sub-Gaussian tail, i.e., $p_k(z)
    \lesssim
    \exp\left(
        -{C_z\|z\|_2^2}/{2}
    \right)$. 
We also assume that there exist nonnegative  functions
$q_k$ for $ k=1,\ldots,K$, s.t.
\begin{equation}
\label{eq_assump_smooth_ratio} 
    q_k \in\mathcal H^\beta(\mathbb R^{d_z},L_q) \quad \text{and} \quad p_k(z)
    =
    p_0(z)q_k(z) \quad  \text{for Lebesgue-almost every }z\in\mathbb R^{d_z}, 
\end{equation}
where $L_q>0$ is a constant. We also define $q_0(z)\equiv 1$. More discussion is in Appendix \ref{sec_transferability_interpretation}. 
\end{assumption}

\begin{thm}[Score matching risk on the target domain of DA-Diff]
\label{thm_score_risk_bound}
Define the effective sample size
\begin{equation}
\label{eq_neff_def}
    N_{\mathrm{eff}}
    :=
    \frac{(n_0+W_N)^2}
    {n_0+\sum_{k=1}^K w_k^2 n_k}.
\end{equation}
Suppose that $W_N>0$, and define the normalized source weight vector $\overline w
    := 
    \bigl(
        \frac{w_1n_1}{W_N},\ldots,\frac{w_Kn_K}{W_N}
    \bigr)^\top
    \in\mathbb{R}^K$. Further, define the source discrepancy matrix
$G\in\mathbb{R}^{K\times K}$ with entries
\begin{equation}
\label{eq_G_def}
\begin{aligned}
    G_{ij}
    :=
    \frac{1}{T-t_0}
    \int_{t_0}^{T}
    \mathbb{E}_{(X_t,Z)\sim p_{0,t}}
    \Big[
        \big\langle
        s_0^*(X_t,Z,t)-s_i^*(X_t,Z,t),
        s_0^*(X_t,Z,t)-s_j^*(X_t,Z,t)
        \big\rangle
    \Big]
    \,dt .
\end{aligned}
\end{equation}
Suppose that there is at least one $w_k >0$ and the weights satisfy
\begin{equation}
\label{eq_weight_regularity_asymptotic}
    \max\{1,w_1,\ldots,w_K\}
    \lesssim \frac{n_0+\sum_{k=1}^K w_k^2 n_k}
    {n_0+W_N}.
\end{equation}
Then, for any weight vector
$w=(w_1,\ldots,w_K)^\top$ satisfying
$w_k\geq 0$ for all $k$,
$W_N>0$, and
(\ref{eq_weight_regularity_asymptotic}),
the score-matching risk on the target domain can be bounded as
\begin{align}
    \mathbb{E}_{D_0,\ldots,D_K}
    \big[
        \mathcal{R}_0(\widehat{s}_w)
    \big]
    \lesssim\;& \widetilde{\mathcal{O}} \left(
    N_{\mathrm{eff}}^{-\frac{2\beta}{d_x+d_z+2\beta}} \right)
    \label{eq_target_risk_bound_erm}
    \\
    &+
    \frac{W_N^2}{(n_0+W_N)^2} \left[ 
    \overline{w}^\top G \overline{w} + \operatorname{diag}(G)^\top
        \overline w \right] ,
    \label{eq_target_risk_bound_bias}
\end{align}
where $\widetilde{\mathcal{O}}(\cdot)$ suppresses polylogarithmic factors in $N_{\mathrm{eff}}$. The $t_0,T$ and the parameter of $\widehat{s}_w$ is specified in Appendix \ref{sec_proof}.
\end{thm}
\begin{remark}
The upper bound in Theorem~\ref{thm_score_risk_bound} consists of two main components. 
The first term (\ref{eq_target_risk_bound_erm}) is the estimation error arising from weighted ERM and decreases as the effective sample size $N_{\mathrm{eff}}$ increases. 
The second term (\ref{eq_target_risk_bound_bias}) captures the transfer bias induced by discrepancies between the source and target domains. 
\end{remark}
\begin{remark} \label{rem_estimate_G} The matrix $G$ in~(\ref{eq_G_def}) involves the unknown scores $s_0^*,\ldots,s_K^*$ and therefore cannot be evaluated directly. For each source domain $k=1,\ldots,K$, we replace $s_k^*$ with a pretrained score network $\widehat{s}_k$ obtained from the corresponding source dataset. For the target domain, we use the weighted ERM estimator $\widehat{s}_w$ as a surrogate for $s_0^*$, since $\widehat{s}_w$ is trained to achieve a small score matching risk on the target domain. Moreover, the expectation $\mathbb E_{(X_t,Z)\sim P_{0,t}}$ (\ref{eq_G_def}) can be written as  $\mathbb E_{(X_0,Z)\sim P_0} \mathbb E_{X_t| X_0}$, where the distribution of $X_t| X_0$ is defined in Section~\ref{sec_diff_model}. The outer expectation is approximated using an independent holdout sample from the target distribution. The resulting empirical estimator of $G$ is given in (\ref{eq_G_empirical_estimator}), and its estimation consistency is established in Theorem~\ref{thm_G_estimation_consistency}. \end{remark}

Building on Theorem~\ref{thm_score_risk_bound}, we further establish a convergence bound
for the learned conditional distribution in TV distance.
Specifically, let $\widehat{P}_w(\cdot| z)$ denote the conditional
distribution induced by the reverse process (\ref{eq_diff_reverse}) with $\widehat{s}_w$.
The following theorem gives the expected TV distance between
$\widehat{P}_w(\cdot| z)$ and the target conditional distribution
$P_0(\cdot| z)$.

\begin{thm}[Target-domain TV error bound of DA-Diff]
\label{thm_tv_target_bound}
Suppose that the conditions of Theorem~\ref{thm_score_risk_bound} hold. Then,
\begin{align}
\label{eq_tv_target_bound}
&\mathbb{E}_{D_0,\ldots,D_K}
    \mathbb{E}_{Z_0}
    \left[
        d_{\mathrm{TV}}
        \left(
            P_0(\cdot| Z_0),
            \widehat{P}_w(\cdot| Z_0)
        \right)
    \right] \notag \\
& \quad =
\widetilde{\mathcal{O}}
\left(
    N_{\mathrm{eff}}^{-\frac{\beta}{d_x+d_z+2\beta}}
    +
    \frac{W_N}{n_0+W_N}
    \sqrt{  \overline{w}^\top G \overline{w} + \operatorname{diag}(G)^\top \overline w }
\right).
\end{align}
\end{thm}
The two terms in~(\ref{eq_tv_target_bound}) correspond to the estimation error and the transfer bias, respectively. To characterize a sufficient regime for consistency, suppose that the
polylogarithmic factor hidden in
$\widetilde{\mathcal O}(\cdot)$ is of order
$(\log N_{\mathrm{eff}})^q$ for some fixed $q>0$, and let $\Gamma := \frac{\beta}{d_x+d_z+2\beta}$. Under Assumption~\ref{assump_transf}, $\sqrt{
        \overline w^\top G\overline w
        +
        \operatorname{diag}(G)^\top\overline w
    } = \mathcal O(n_0^{-\gamma})$. Hence, the TV error is controlled by $(\log N_{\mathrm{eff}})^q \cdot 
    \left(
        N_{\mathrm{eff}}^{-\Gamma}
        + n_0^{-\gamma}
    \right)$. The estimation error term converges to zero whenever
$N_{\mathrm{eff}}\to\infty$.
For the transfer bias term, a sufficient condition is
$(\log N_{\mathrm{eff}})^q=o(n_0^\gamma)$.
Therefore, a simple sufficient growth regime for the consistency of
DA-Diff is
\begin{equation}
\label{eq_neff_tv_growth_condition}
    1
    \ll
    N_{\mathrm{eff}}
    \ll
    \exp(n_0^c),
    \qquad
    0<c<\frac{\gamma}{q}.
\end{equation}
This growth regime (\ref{eq_neff_tv_growth_condition}) is rather mild: it only requires $\gamma>0$ in
(\ref{eq_assump_loss}), meaning that the source-to-target score
discrepancy may decay at an arbitrarily slow polynomial rate.

\subsection{Domain-adapted conditional independence test (DA-CIT)}
\label{sec_da_cit}

\begin{thm}[Type~I error of DA-CIT]
\label{thm_da_cit_type1}
Suppose that the conditions of Theorem~\ref{thm_score_risk_bound} hold.
Let $p_w$ denote the $p$-value returned by DA-CIT.
Then, under the null hypothesis $H_0$ in~(\ref{eq_cit_def}), for any significance level $\alpha\in(0,1)$,
\begin{align}
\label{eq_da_cit_type1}
\mathbb{P}\bigl(p_w\leq\alpha\bigr)
\leq
\alpha
+
n_0\,
\widetilde{\mathcal{O}}
\left(
N_{\mathrm{eff}}^{-\frac{\beta}{d_x+d_z+2\beta}}
+
\frac{W_N}{n_0+W_N}
\sqrt{  \overline{w}^\top G \overline{w} + \operatorname{diag}(G)^\top \overline w }
\right).
\end{align}
\end{thm}
The conditions for asymptotic Type~I error control are stronger than
those required for the TV consistency of DA-Diff in
(\ref{eq_neff_tv_growth_condition}) because the TV error is accumulated
by $n_0$ times in~(\ref{eq_da_cit_type1}). 
Using the same polylogarithmic factor $(\log N_{\mathrm{eff}})^q$ as
in (\ref{eq_neff_tv_growth_condition}), the excess Type~I error is controlled by
    $n_0 \cdot (\log N_{\mathrm{eff}})^q \cdot 
    \left(
        N_{\mathrm{eff}}^{-\Gamma}
        + n_0^{-\gamma}
    \right)$. For the estimation term, since
$(\log N_{\mathrm{eff}})^q
=O(N_{\mathrm{eff}}^\zeta)$
for any sufficiently small $\zeta>0$, a sufficient condition is
    $n_0^{1/(\Gamma-\zeta)} \ll N_{\mathrm{eff}} ,
    \ 
    0<\zeta<\Gamma$. For the transfer-bias term, the additional factor $n_0$ strengthens
the upper-growth condition in
(\ref{eq_neff_tv_growth_condition}) to
    $(\log N_{\mathrm{eff}})^q
    =
    o(n_0^{\gamma-1})$. Thus, when $\gamma>1$, a simple sufficient growth regime for
asymptotic Type~I error control is
\begin{equation}
\label{eq_neff_sandwich_condition}
    n_0^{1/(\Gamma-\zeta)}
    \ll
    N_{\mathrm{eff}}
    \ll
    \exp(n_0^c),
    \qquad
    0<\zeta<\Gamma,
    \quad
    0<c<\frac{\gamma-1}{q}.
\end{equation}
Under~(\ref{eq_neff_sandwich_condition}), the excess Type~I error
vanishes, and hence
$\mathbb{P}(p_w\leq\alpha)\leq\alpha+o(1)$.
The lower bound on $N_{\mathrm{eff}}$ ensures that the estimation error
decays sufficiently fast to offset the factor $n_0$, whereas the upper
bound ensures that the polylogarithmic factor grows sufficiently slowly
so that
$n_0^{1-\gamma}(\log N_{\mathrm{eff}})^q\to0$.

By the decomposition in~(\ref{eq_score_matching_risk_decomposition}),
minimizing the objective in~(\ref{eq_optimal_weights_sample_sizes}) with
respect to the weights $w$ is equivalent to minimizing
$\mathbb{E}_{D_0,\ldots,D_K}
[\mathcal{R}_0(\widehat{s}_w)]$.
In Appendix~\ref{sec_optimal_weights}, we reformulate this weight
optimization problem into a majorization--minimization quadratic program (MM-QP) procedure and solve it accordingly.

\section{Synthetic Experiments}
\label{sec_simulation}
Section~\ref{sec_gaussian} examines the conditional distribution estimation performance of DA-Diff under heterogeneous source domains, while Section~\ref{sec_post_non_linear} evaluates the Type I error  and  power of DA-CIT.

\subsection{Gaussian mixture example}
\label{sec_gaussian}
We evaluate DA-Diff on a controlled Gaussian mixture setting \citep{Hagemann_2022_gmm,Jannis_2025_condw,yang_2026_dmdg}. The detailed data-generating mechanism is provided in Appendix~\ref{app_gmm}. We set $n_0=n_1=\cdots=n_K=500$ and study the effect of varying the number of source domains $K$.  We use the Wasserstein-2 ($W_2$) distance to evaluate the quality of the estimated conditional distribution. For each $Z_0$ in the test set, we draw 1,000 samples from the true conditional distribution $P_0(\cdot | Z_0)$ and 1,000 samples from $\widehat{P}_w(\cdot | Z_0)$ estimated by DA-Diff. We then compute the $W_2$ distance between these two empirical distributions. We repeat this procedure over 200 test values of $Z_0$ and report the average $W_2$ distance. Additional ablation studies are provided in Appendix~\ref{app_ablation}.

Our baselines include a diffusion model trained only on the target dataset (Target-only), a diffusion model trained on the concatenation of all target and source datasets (Cat-all), and a diffusion model first trained on all available target and source data and then fine-tuned on the target dataset (Finetune). We further compare with diffusion models for transfer learning, namely Domain Guidance (DoG) \citep{zhong_2025_dog} and the Transfer Guided Diffusion Process (TGDP) \citep{ouyang_2024_tgdp}. Weighted ERM methods are also included, such as Unified Optimization of Weights and Quantities (UOWQ) \citep{zhang_2026_uowq} and Robust Learning from Untrusted Sources (Robust) \citep{Konstantinov_2019_Robust_Learning}. The results are reported in Table \ref{tab_gaussian_baselines_w2}.

\begin{table}[ht]
\centering
\caption{Performance of proposed DA-Diff, Target-only, Cat-all, Finetune, DoG, TGDP, UOWQ, and Robust under the data-generating mechanism in Appendix~\ref{app_gmm}. Experiments are repeated over 30 seeds; we report the $W_2$ distance along with one standard deviation (std). The best, second-best, and third-best results are highlighted in bold,
underlined, and $\dagger$, respectively.
}
\label{tab_gaussian_baselines_w2}
\begin{tabular}{lcccc}
\toprule
Method & $W_2$ ($K=2$) & $W_2$ ($K=5$) & $W_2$ ($K=10$) & $W_2$ ($K=20$) \\
\midrule
DA-Diff (Ours) & \textbf{0.194}(0.030) & \textbf{0.190}(0.052) & \underline{0.216}(0.060) & \underline{0.196}(0.033)\\
Target-only & 0.343(0.079) & 0.343(0.079) & 0.343(0.079) & 0.343(0.079) \\
Cat-all & 0.644(0.134) & 0.495(0.114) & 0.372(0.109) & 0.330(0.079) \\
Finetune & 0.384(0.063) & 0.283(0.056) & \textbf{0.193}(0.043) & \textbf{0.168}(0.036) \\
DoG \citeyearpar{zhong_2025_dog} & {0.311}$^\dagger$(0.057) & \underline{0.240}(0.068) & {0.218}$^\dagger$(0.062) & 0.222$^\dagger$(0.069) \\
TGDP \citeyearpar{ouyang_2024_tgdp} & 0.663(0.151) & 0.466(0.110) & 0.372(0.089) & 0.331(0.077) \\
UOWQ \citeyearpar{zhang_2026_uowq} & 0.484(0.108) & 0.571(0.284) & 0.613(0.182) & 0.742(0.237)\\
Robust \citeyearpar{Konstantinov_2019_Robust_Learning} & \underline{0.230}(0.056) & {0.282}$^\dagger$(0.093) & 0.260(0.078) & 0.227(0.066)\\
\bottomrule
\end{tabular}

\end{table}

Table~\ref{tab_gaussian_baselines_w2} compares the methods under different $K$. DA-Diff achieves the lowest mean $W_2$ error for $K=2,5$, and the second-lowest error for $K=10,20$, where Finetune performs best. Across all four settings, DA-Diff substantially reduces the error relative to Target-only. In contrast, Cat-all performs worse than Target-only for $K=2,5,10$ and yields only a modest improvement at $K=20$, illustrating that simply pooling heterogeneous source data can lead to negative transfer. Compared with DoG and TGDP, DA-Diff attains lower $W_2$ error for every $K$. TGDP is also worse than Target-only for $K=2,5,10$. UOWQ degrades as $K$ increases, whereas Robust remains relatively stable but is consistently worse than DA-Diff. Additional experiments in Table~\ref{tab_gaussian_baselines_w2_appendix} show a similar pattern at larger sample sizes: when the per-domain sample size is $1000$ or $1500$, DA-Diff outperforms Target-only, Cat-all, and Finetune for all considered values of $K$. 

\subsection{Post-nonlinear example}
\label{sec_post_non_linear}
We evaluate the performance of DA-CIT on the post-nonlinear model \citep{yang_2025_cdcit,scetbon2022asymptotic}, a commonly used benchmark for nonlinear conditional independence testing. The detailed data-generating procedure is provided in Appendix~\ref{app_post_nonlinear}. 

We fix the significance level at $\alpha=0.05$ and repeat each experiment 100 times, reporting both Type I error and testing power. We consider one target domain and 10 source domains, with $n_0=\cdots=n_{10}=500$, and vary the conditioning dimension $d_z$ from 5 to 100. To further assess calibration and its variability, Appendix~\ref{app_ablation} compares the empirical distribution of the $p$-values over the 100 repetitions with the $U(0,1)$ distribution and reports the corresponding Kolmogorov--Smirnov (KS) statistic. Additional ablation studies are also provided in Appendix~\ref{app_ablation}.

We compare DA-CIT with several representative CI tests. KCIT \citep{zhang2011kernel} performs CI testing in a reproducing kernel Hilbert space. CDCIT \citep{yang_2025_cdcit} estimates the conditional distribution using half of the target data and applies the CRT to the remaining half. SGMCIT \citep{ren_2025_sgmcit} estimates the conditional distribution via sliced score matching, while RBPT \citep{polo_2023_rbpt} is designed to improve robustness to regression-model misspecification. ECCIT \citep{pan_2026_eccit} calibrates the $p$-values of the Generalized Covariance Measure (GCM; \citealp{shah2020hardness}) to mitigate testing miscalibration. CRT$^*$ \citep{zhang_2026_crt_star} leverage heterogeneous external and unlabeled data to improve the CRT; however, their approach relies on a sparse linear model for the conditional distribution, whereas our diffusion-based framework accommodates substantially more general nonlinear and non-Gaussian conditional distributions.


\begin{figure}[htb]
    \centering
    \includegraphics[width=0.80\linewidth]{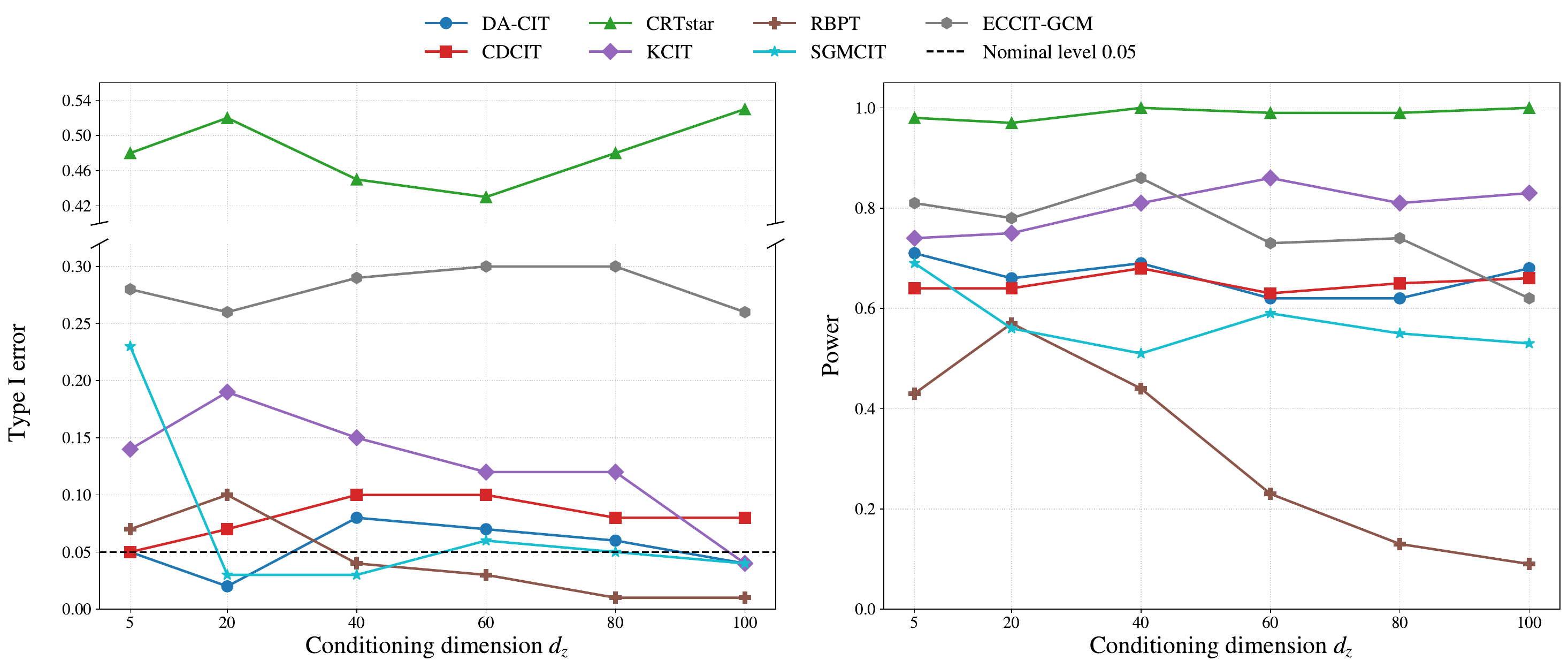}
    \caption{Type I error and power of each CI testing under different $d_z$. }
    \label{fig_cit_results_main}
\end{figure}

Figure~\ref{fig_cit_results_main} reports the Type I error and power of different
CI testing methods with varying $d_z$. Across all considered dimensions, DA-CIT consistently achieves a Type I error no larger than that of CDCIT and generally maintains it close to the nominal level, while maintaining comparable power. SGMCIT and RBPT also maintain reasonable Type I error control in higher dimensions, but with lower power than DA-CIT and CDCIT. KCIT maintains relatively high power, but its Type I error is clearly inflated in most settings. ECCIT-GCM and CRT$^*$ achieve relatively high power, but both exhibit substantial Type I error inflation across the considered dimensions. In particular, despite incorporating external data, CRT$^*$ achieves high power at the cost of severe Type I error inflation. 

\section{Real data analysis}
\label{app_real_data_sachs}

We evaluate our method on the Sachs flow-cytometry dataset \citep{sachs2005causal}, a widely used real-world benchmark for causal discovery. The dataset contains simultaneous single-cell measurements of 11 phosphorylated proteins and phospholipids involved in human T-cell signaling: Raf, Mek, Plc$\gamma$, PIP2, PIP3, Erk, Akt, PKA, PKC, P38, and Jnk.

The Sachs dataset consists of 14 sub-datasets collected under different
experimental conditions involving cellular stimulation and molecular
interventions. Each sub-dataset contains the same 11 variables and 707--927 single-cell observations. We treat each experimental condition as a separate domain rather than as an independent replicate from a common distribution, yielding a natural multi-source domain adaptation setting with heterogeneous distributions across domains. 

Specifically, we use the \texttt{cd3cd28} condition as the target domain. This condition corresponds to T-cell activation through CD3 and CD28 stimulation without an additional inhibitor or ICAM-2 stimulation. The remaining 13 conditions are treated as source domains. These source domains include: (i) CD3/CD28 stimulation combined with one additional molecular perturbation; (ii) CD3/CD28 stimulation together with ICAM-2 stimulation; (iii) PMA or $\beta_2$cAMP stimulation; and (iv) additional CD3/CD28+ICAM-2 conditions combined with molecular perturbations. Table~\ref{tab_sachs_domains} summarizes the target and source domains used in our experiments.

\begin{table*}[htb!]
\centering
\caption{The 14 domains used in the Sachs experiment. The \texttt{cd3cd28} condition is used as the target domain, while the remaining 13 conditions are treated as source domains. The domain labels are consistent with those used in our implementation and visualization.}
\label{tab_sachs_domains}
\resizebox{\textwidth}{!}{
\begin{tabular}{cclll}
\toprule
Label & Role & Dataset & Stimulation / intervention condition & sample size \\
\midrule
0  & Target & \texttt{cd3cd28}
   & CD3 + CD28 & 853 \\

\midrule
1  & Source & \texttt{b2camp}
   & $\beta_2$cAMP & 707 \\

2  & Source & \texttt{cd3cd28\_aktinhib}
   & CD3 + CD28 + Akt inhibitor & 911 \\

3  & Source & \texttt{cd3cd28\_g0076}
   & CD3 + CD28 + G\"o6976 & 723 \\

4  & Source & \texttt{cd3cd28\_ly}
   & CD3 + CD28 + LY294002 & 848 \\

5  & Source & \texttt{cd3cd28\_psitect}
   & CD3 + CD28 + Psitectorigenin & 810 \\

6  & Source & \texttt{cd3cd28\_u0126}
   & CD3 + CD28 + U0126 & 799 \\

7  & Source & \texttt{cd3cd28icam2}
   & CD3 + CD28 + ICAM-2 & 902 \\

8  & Source & \texttt{cd3cd28icam2\_aktinhib}
   & CD3 + CD28 + ICAM-2 + Akt inhibitor & 899 \\

9  & Source & \texttt{cd3cd28icam2\_g0076}
   & CD3 + CD28 + ICAM-2 + G\"o6976 & 753 \\

10 & Source & \texttt{cd3cd28icam2\_ly}
   & CD3 + CD28 + ICAM-2 + LY294002 & 927 \\

11 & Source & \texttt{cd3cd28icam2\_psit}
   & CD3 + CD28 + ICAM-2 + Psitectorigenin & 868 \\

12 & Source & \texttt{cd3cd28icam2\_u0126}
   & CD3 + CD28 + ICAM-2 + U0126 & 759 \\

13 & Source & \texttt{pma}
   & PMA & 913 \\
\bottomrule
\end{tabular}
}
\end{table*}

To visualize the distributional heterogeneity across domains, we project samples from all domains into two dimensions using t-SNE. Figure~\ref{fig_sachs_tsne} shows the resulting representation, where each color corresponds to one experimental domain and label 0 denotes the target domain. The target samples are broadly distributed over the main high-density region of the embedding and substantially overlap with samples from several source domains, suggesting that a subset of the source domains shares similar distributional characteristics with the target. At the same time, several source domains exhibit visibly shifted structures from the main sample population. This motivates selectively exploiting source information rather than naively pooling all source observations.

Following \citet{li2024k}, we select 50 CI and 40 non-CI $(X_0,Y_0,Z_0)$ triples from the reference causal graph of the Sachs dataset \citep{sachs2005causal}. In all triples, $X$ and $Y$ are one-dimensional, while the dimension of $Z$ ranges from 1 to 9. Treating CI as the positive class, we evaluate each method using precision, recall, and F1 score, defined as $\mathrm{TP}/(\mathrm{TP}+\mathrm{FP})$, $\mathrm{TP}/(\mathrm{TP}+\mathrm{FN})$, and $2\,\mathrm{Precision}\,\mathrm{Recall}/(\mathrm{Precision}+\mathrm{Recall})$, respectively.

We compare DA-CIT with the six baselines introduced in Section~\ref{sec_post_non_linear} on the Sachs dataset. Except for DA-CIT and CRT$^{*}$, which are able to leverage external source data directly, we evaluate each baseline under both the Target-only and Cat-all settings. To improve the robustness of the results, we repeat the experiment five times with different random seeds for each CI or non-CI triple, and determine the final prediction by majority vote over the five runs.

\begin{figure}[htb!]
\centering
\includegraphics[
width=\linewidth
]{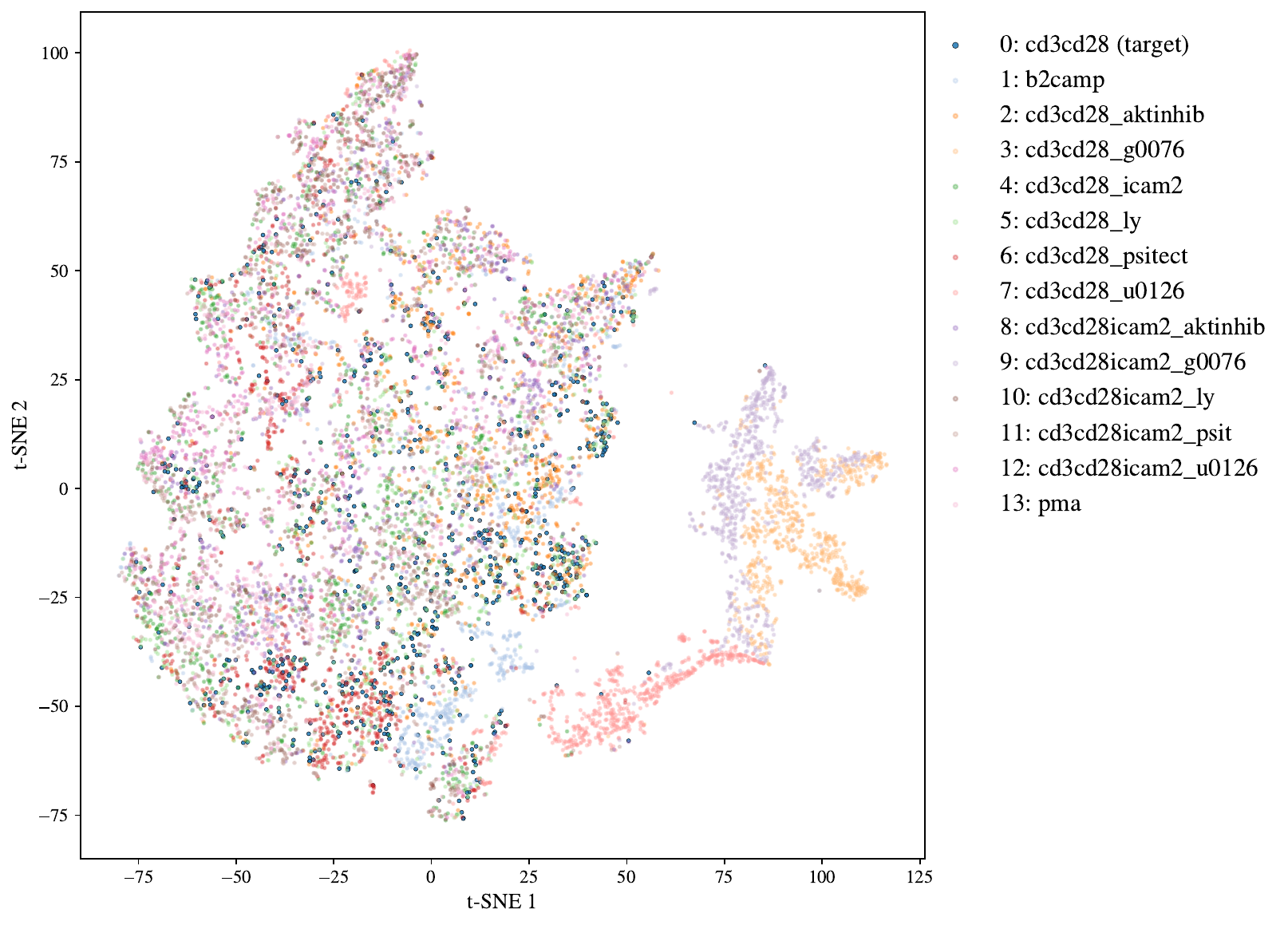}
\caption{
t-SNE visualization of the 14 domains in the Sachs dataset.
Each color represents a different stimulation or intervention
condition. The \texttt{cd3cd28} dataset is used as
the target domain, and the remaining 13 datasets are used as
source domains.
}
\label{fig_sachs_tsne}
\end{figure}

\begin{table}[htb!]
\centering
\caption{Performance of different CI testing methods on the Sachs dataset. Each entry reports \textit{Target-only / Cat-all}. DA-CIT and CRT$^{*}$ report a single result because both methods directly leverage external source data, making the Target-only/Cat-all distinction unnecessary. The best, second-best, and third-best results are highlighted in bold,
underlined, and $\dagger$, respectively.}
\label{tab_sachs_cit_results}

\begin{tabular}{lccc}
\toprule
Method & Precision & Recall & F-score \\
\midrule
DA-CIT (Ours)
& 0.738$^\dagger$
& \underline{0.960}
& \textbf{0.834} \\

CRT$^{*}$ \citeyearpar{zhang_2026_crt_star}
& \textbf{0.764}
& 0.260
& 0.388 \\

CDCIT \citeyearpar{yang_2025_cdcit}
& 0.671 / 0.692
& \textbf{0.980} / 0.540
& 0.796$^\dagger$ / 0.606 \\

KCIT \citeyearpar{zhang2011kernel}
& 0.725 / \underline{0.750}
& 0.900$^\dagger$ / 0.060
& \underline{0.803} / 0.111 \\

RBPT \citeyearpar{polo_2023_rbpt}
& 0.620 / 0.625
& \textbf{0.980 }/ 0.600
& 0.759 / 0.612 \\

ECCIT-GCM \citeyearpar{pan_2026_eccit}
& \textbf{0.764} / 0.733
& 0.520 / 0.220
& 0.619 / 0.338 \\

SGMCIT \citeyearpar{ren_2025_sgmcit}
& 0.721 / 0.285
& 0.880 / 0.040
& 0.792 / 0.070 \\
\bottomrule
\end{tabular} 
\end{table}

Table~\ref{tab_sachs_cit_results} summarizes the results. DA-CIT achieves the highest F1 score among the compared methods, while maintaining both high precision and high recall. 
In comparison, CRT$^{*}$ achieves high precision but suffers from
substantially lower recall, while methods such as CDCIT and RBPT attain
high recall at the cost of lower precision. KCIT also performs competitively
under the Target-only setting, but its performance deteriorates markedly
when all domains are concatenated.

For most competing methods, Cat-all performs worse than the corresponding Target-only setting, particularly in terms of recall and F1 score. This observation
is consistent with our motivation that naively pooling heterogeneous
source and target data may introduce substantial transfer bias. In
contrast, DA-CIT selectively exploits external source information and
achieves the best F1 score while maintaining both high precision and high
recall.

\section{Conclusion}
\label{sec_conclusion}

In this work, we address the fundamental difficulty of conditional distribution estimation in the CRT when target-domain data are limited and heterogeneous source-domain data are available.
To address this problem, we
propose DA-Diff, to the best of our knowledge, the first 
transfer learning method for diffusion models built on weighted ERM. We establish a
target-domain TV error bound for DA-Diff, which explicitly characterizes
the trade-off between estimation error and transfer bias. Building on
DA-Diff, we further propose DA-CIT for CI testing.
Under suitable assumptions, we show that the TV error of DA-Diff
converges to zero and that the excess Type~I error of DA-CIT also
vanishes asymptotically. We evaluate DA-Diff through simulation studies on conditional
distribution estimation, and assess DA-CIT on both simulated and
real-world data. The empirical results show that DA-Diff can effectively
leverage heterogeneous source domains to improve conditional
distribution estimation, while DA-CIT provides reliable Type~I error control and maintains competitive power. A current limitation is
the relatively high computational cost, mainly due to source-model
pretraining and repeated weight optimization during weighted ERM
training. Improving the computational efficiency of DA-Diff and DA-CIT
is therefore an important direction for future work.







\bibliography{iclr2027_conference}
\bibliographystyle{iclr2027_conference}

\clearpage

\appendix

\section{Related work}
We have provided a detailed review of the literature on CI tests in Section~\ref{sec_intro}. In this section, we focus on related work on diffusion models for transfer learning, domain adaptation, and weighted ERM.

\subsection{Diffusion models for transfer learning}

Diffusion models are a class of generative models that have achieved remarkable success in various fields, including image generation \citep{ho2020denoising,song2021score}, scientific discovery \citep{Hoogeboom_2022_e3diff,xu_2022_geodiff,yang_2026_dmdg}, engineering problem \citep{shysheya_2024_diffpde,YANG_2025_mixoil}, time-series forecasting \citep{ye_2025_nsdiff,yang_2026_cwgen}, and language modeling \citep{sahoo_2024_MDLM,Nie_2025_llada}. 

These studies demonstrate the broad applicability of diffusion models across a wide range of fields. However, most existing diffusion models are trained on data from a single domain or from pooled datasets, and relatively limited attention has been paid to how heterogeneous source domains can be selectively leveraged to improve generation in a target domain. This motivates our study of multi-source domain adaptation for diffusion models.

Several works have studied transfer learning specifically for diffusion
models. \citet{zhong_2025_dog} propose Domain Guidance (DoG), which
combines a pre-trained source model and a target-domain fine-tuned model
by classifier-free guidance \citep{ho_2021_cfg}. Although effective, DoG requires both models to be
evaluated at each sampling step, increasing the inference cost, and is
primarily designed for transferring a single pre-trained source model
rather than selectively exploiting multiple heterogeneous source domains.

\citet{ouyang_2024_tgdp} propose the Transfer Guided Diffusion Process
(TGDP), which expresses the target score as the pre-trained source score
augmented by a guidance term estimated from the source--target density
ratio. In practice, this density ratio is estimated through a domain
classifier. However, reliable density-ratio estimation can be challenging
when target data are limited or when the source and target distributions
differ substantially, and estimation errors can directly propagate to the
guidance term. TGDP therefore introduces additional regularization to
improve the training and calibration of the guidance network.
Moreover, similar to DoG, TGDP mainly considers transfer from a single
source model and does not address source selection and weighting in the
multi-source setting.

In contrast, DA-Diff directly incorporates multiple source datasets into
diffusion-model training through weighted ERM. It does not require explicit
density-ratio estimation and automatically balances the statistical benefit
of additional source samples against source--target distribution shifts.

Unlike DoG and TGDP, which mainly focus on transferring knowledge from a
single pre-trained source model, \citet{cheng_2025_thm_finetune} explicitly
consider transfer learning from multiple source domains and provide
theoretical guarantees for diffusion-model finetuning. Their approach,
however, relies on a different transfer mechanism based on learning a shared
representation across domains. Their analysis decomposes the target error
into the error of target-domain fine-tuning, the error of learning the shared
representation from source domains, and a task-mismatch term. In particular,
source data mainly improve the estimation of the shared representation, while
the target score model is still fine-tuned using target-domain samples.

In comparison, DA-Diff directly incorporates source samples into the training
of the target model through weighted ERM. Consequently, the estimation term
in our bound depends on the effective sample size $N_{\mathrm{eff}}$, so that
transferable source samples can directly improve the convergence rate of the
target conditional distribution estimator. At the same time, our bound in Theorem \ref{thm_tv_target_bound}
explicitly characterizes the transfer bias through the source weights, the
total source contribution $W_N/(n_0+W_N)$, and the source discrepancy matrix
$G$, providing a quantitative trade-off between exploiting additional source
samples and avoiding negative transfer. Moreover, our analysis does not
require the existence of a shared dimension-reducing representation across
domains.

\subsection{Domain adaptation and Weighted ERM}

Weighted ERM is a classical approach to domain adaptation (DA) that learns a model by minimizing a weighted combination of
empirical losses across domains \citep{vogel_2020_werm}. Rather than treating all datasets
equally, it adjusts their contributions to exploit useful source
information while mitigating negative transfer caused by distribution
shift.

Two settings closely related to our approach are target-as-mixture
adaptation and learning from potentially untrusted sources
\citep{Mansour_2008_ms,
Konstantinov_2019_Robust_Learning}.
In the former, the target distribution is represented or approximated by
a convex combination of source distributions.
In the latter, a trusted target reference dataset is used to assess
source relevance and downweight irrelevant or corrupted sources.
Our method is related to both perspectives: it constructs a weighted
training mixture and adaptively controls source contributions, without
requiring the target distribution to be exactly a mixture of the source
distributions.

Foundational theoretical contributions to DA include
\citet{David_2006_da}, who establish target-risk bounds
highlighting the trade-off between source prediction error and
distributional discrepancy in the learned representation.
For the multi-source setting, \citet{Mansour_2008_ms}
provide guarantees for distribution-weighted combinations of source
predictors when the target input distribution is a mixture of source
distributions.
Building on this framework, \citet{Mansour_2009_renyi} introduce
R\'enyi divergence into the analysis and extend the guarantees to
arbitrary target distributions by quantifying their discrepancy from
suitable mixtures of the source distributions.

A particularly relevant approach is Robust Learning from Untrusted Sources
\citep{Konstantinov_2019_Robust_Learning}. The method uses a small trusted target dataset to evaluate the reliability of multiple source domains, and then assigns larger weights to sources that are more similar to the target domain. The resulting weights are used in a weighted ERM objective to reduce the effect of unreliable or negatively transferable sources. However, its theoretical analysis assumes a uniformly bounded loss, which is not directly satisfied by the denoising score-matching loss in (\ref{eq_diffusion_loss_single}). Therefore, extending this framework to diffusion models requires additional treatment of the unbounded loss and a separate analysis of the generated distribution.

\citet{Turrisi_2022_ot_msda} propose multi-source domain adaptation via
weighted joint distributions optimal transport (MSDA-WJDOT).
Their method jointly learns source weights and a target predictor by
minimizing an optimal-transport discrepancy between a weighted mixture
of source joint distributions and a target joint distribution constructed
using predicted labels, thereby accounting for discrepancies in both
feature and label spaces.

More recently, \citet{zhang_2026_uowq} propose Unified Optimization of
Weights and Quantities (UOWQ), which jointly optimizes source weights and
transfer quantities under weighted maximum likelihood estimation.
UOWQ is developed in a parametric setting, where source--target discrepancies
are characterized through differences between the parameters of the
corresponding distributions, rather than directly comparing general
nonparametric distributions.

Our work is inspired by the risk-driven weighting strategy of UOWQ
\citep{zhang_2026_uowq}, but differs substantially in both its formulation
and applicability. UOWQ is developed under a parametric model and weighted maximum likelihood estimation. It represents each domain by a parameter vector and assumes that source and target domains are close in parameter space. While the basic transferability assumption is expressed through the $\ell_2$ norm between source and target parameters, the theoretical risk bound and source-weight optimization further use a Fisher-information-weighted version of this $\ell_2$ norm. Therefore, UOWQ fundamentally relies on a common parametric representation across domains, rather than directly comparing the underlying nonparametric distributions or learned functions. This parameter-space formulation requires
the domain-specific models to share a common and identifiable
parameterization, which can be restrictive for modern neural networks.
In particular, neural-network parameters are generally non-identifiable:
for example, the two ReLU networks
$a\,\mathrm{ReLU}(bx+c)+d$ and
$\mathrm{ReLU}(abx+ac)+d$ represent exactly the same function for $a>0$,
but have different parameter vectors $(a,b,c,d)$ and $(1,ab,ac,d)$. Consequently, parameter discrepancies
may reflect differences in parameterization rather than genuine differences
between the learned distributions, and such symmetries may also lead to
singular Fisher information. Moreover, the parametric formulation of UOWQ
typically requires the models trained on different domains to have the same
parameter structures, making it difficult to directly compare models with
different neural-network architectures. These assumptions are particularly
restrictive for diffusion models, which are commonly trained by denoising
score matching rather than regular maximum likelihood estimation.
In contrast, DA-Diff defines source discrepancies directly in the function space
through differences between the output of the score functions. Therefore, its
discrepancy measure is invariant to parameter reparameterization and does not
require aligned parameter vectors or identical network architectures across
domains; it only requires the models to produce score outputs of the same
dimension. This enables DA-Diff to accommodate substantially more flexible
and potentially nonparametric generative models while directly measuring the
quantities relevant to diffusion-model training.

\section{Interpretation of the Assumption \ref{assump_transf} and \ref{assump_smoothness_tail}}
\label{sec_transferability_interpretation}

\paragraph{Transferability assumptions} Assumption~\ref{assump_transf} characterizes the transferability of the source domains from two complementary perspectives: the \textbf{overlap} of the marginal densities of the $Z$ and the \textbf{similarity} of the conditional scores of $X$ given $Z$.

First, (\ref{eq_assump_ratio}) requires sufficient overlap between the target and source distributions of $Z$. The lower bound
$$
C_{\mathrm{dr}}\max_{1\leq k\leq K}p_k(z)\leq p_0(z)
$$
prevents any individual source domain from assigning arbitrarily large probability mass to regions that are rarely represented in the target domain. Equivalently, it uniformly controls the source-to-target density ratios,
$$
\frac{p_k(z)}{p_0(z)}
\leq
C_{\mathrm{dr}}^{-1},
\qquad k=1,\ldots,K.
$$
The upper bound
$$
p_0(z)\leq C_{\mathrm{DR}}p_{w\setminus 0}(z)
$$
requires the weighted collection of source domains to sufficiently cover the target distribution. In particular, it rules out target regions with substantial probability mass but essentially no source observations. Since
$p_{w\setminus 0}(z)\leq \max_k p_k(z)$, these two inequalities together also imply
$$
C_{\mathrm{dr}}p_{w\setminus 0}(z)
\leq
p_0(z)
\leq
C_{\mathrm{DR}}p_{w\setminus 0}(z),
$$
so the target and weighted source marginals of $Z$ are comparable up to constants. This overlap condition allows information learned from the source domains to be transferred to the target domain without introducing instability caused by extreme covariate shift. An illustration showing cases that satisfy or violate the conditions can be found in Figure~\ref{fig_transferability}.

\begin{figure}[t]
    \centering

    \begin{subfigure}[t]{0.48\linewidth}
        \centering
        \includegraphics[width=\linewidth]{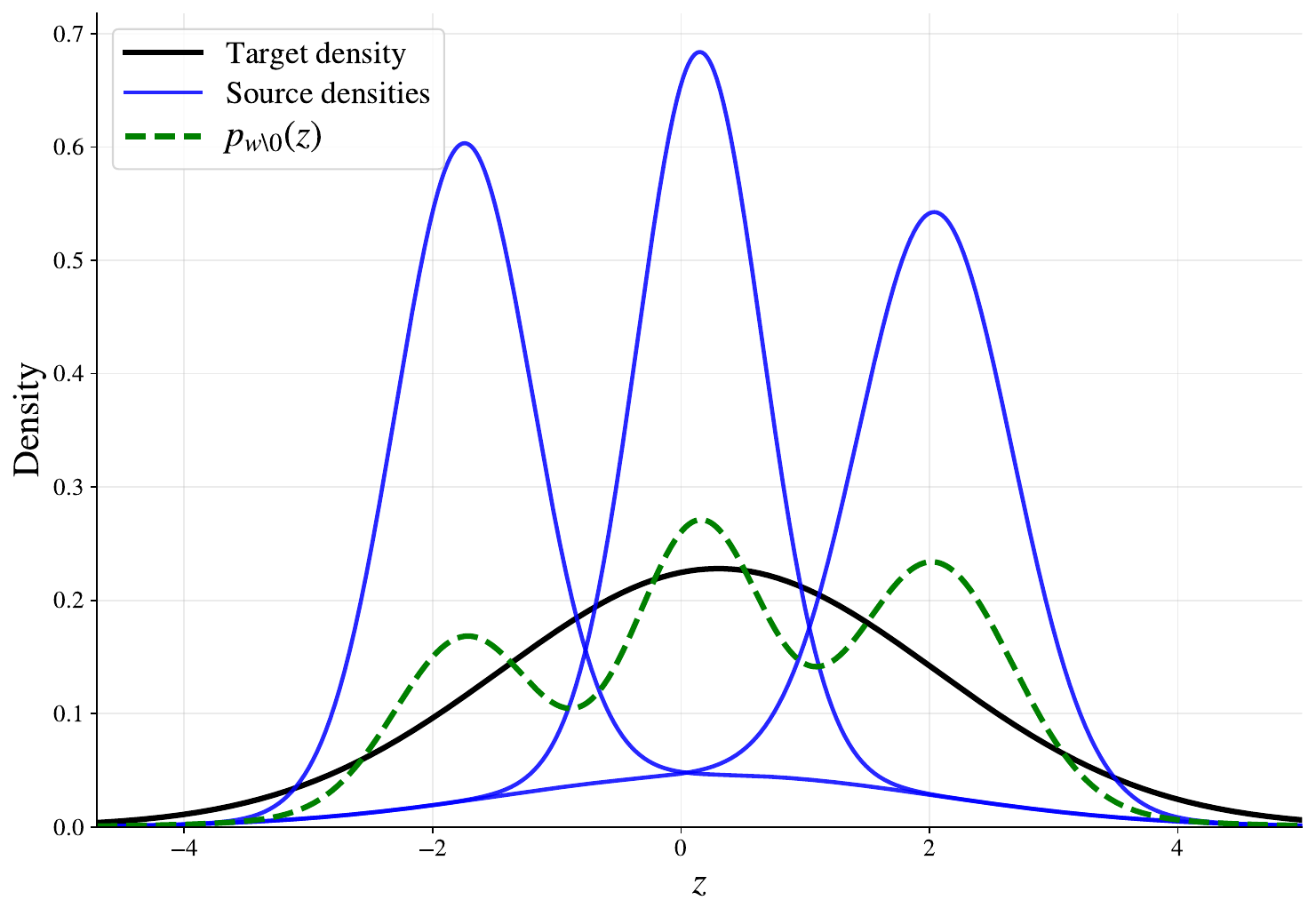}
    \end{subfigure}
    \hfill
    \begin{subfigure}[t]{0.48\linewidth}
        \centering
        \includegraphics[width=\linewidth]{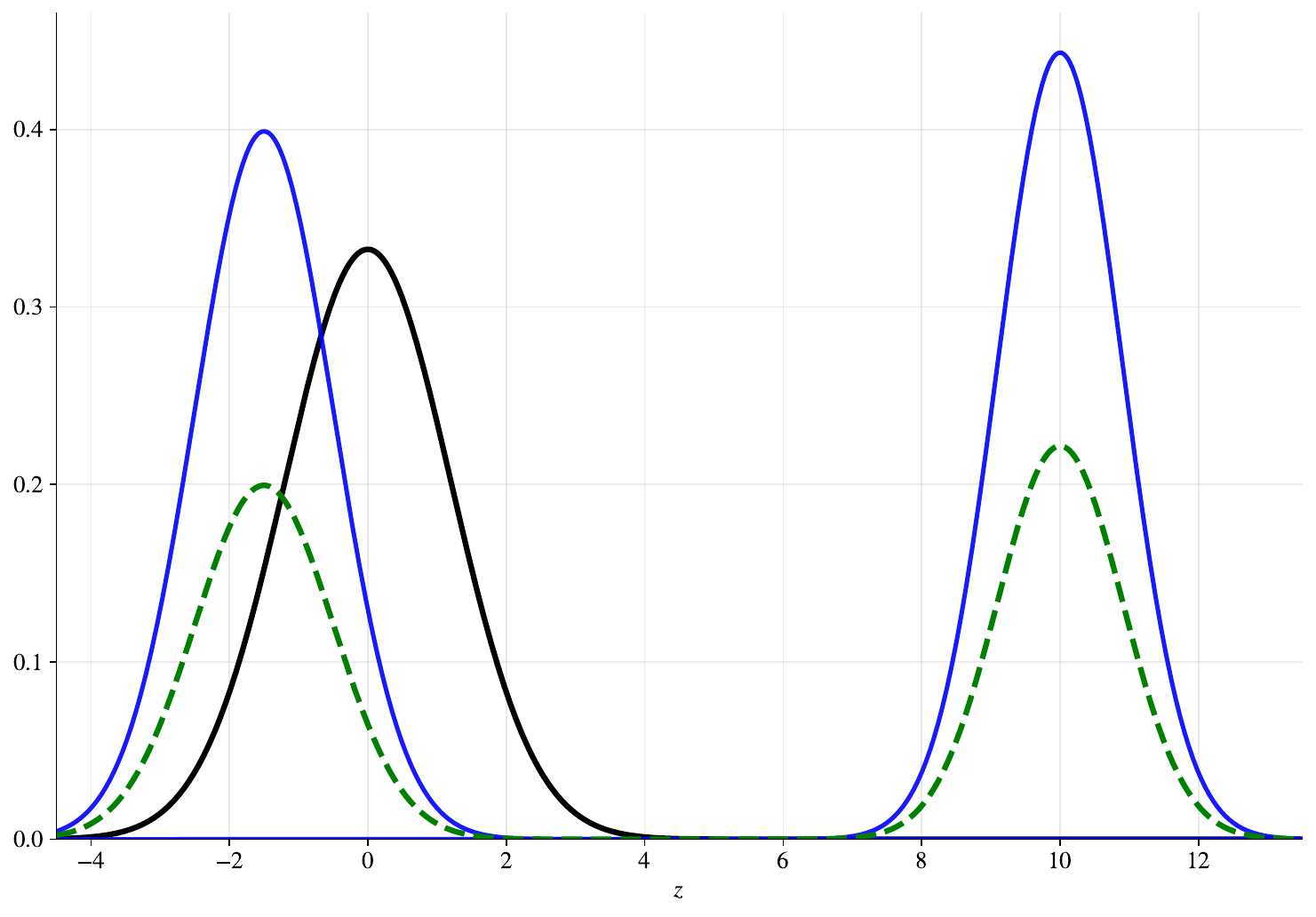}
    \end{subfigure}

    \caption{\textbf{Left panel:} The source densities are sufficiently covered by the
    target density $\left(
C_{\mathrm{dr}}\max_{1\leq k\leq K}p_k(z)\leq p_0(z)\right)
$, while their weighted mixture
    $p_{w\setminus 0}$ also covers the target density $\left(
p_0(z)\leq C_{\mathrm{DR}}p_{w\setminus 0}(z)
\right)$; hence both
    density-ratio conditions can hold uniformly.
    \textbf{Right panel:} One source is located far from the target, while the
    remaining source does not sufficiently cover the target
    distribution. Consequently, the weighted source density fails to
    uniformly dominate $p_0$, violating the transferability condition.}
    \label{fig_transferability}
\end{figure}

Second, (\ref{eq_assump_loss}) controls the discrepancy between the target and source conditional distributions through their diffusion scores. Specifically, (\ref{eq_assump_loss}) requires each transferable source domain to have a conditional score that is sufficiently close to the target score under the target distribution. The parameter $\gamma$ quantifies the degree of transferability: a larger $\gamma$ corresponds to a faster decay of the source--target discrepancy and hence a smaller transfer bias.

\paragraph{Smoothness and tail assumptions} Assumption~\ref{assump_smoothness_tail} imposes standard smoothness and tail regularity conditions on the target and source distributions. It does not assume a specific parametric form for $p_k(X| Z)$ and allows flexible nonlinear and non-Gaussian relationships between $X$ and $Z$.

Assumption (\ref{eq_joint_density_tail})
requires only Gaussian-type tail decay in $x$. The uniform bounds
$C_f\leq f_k(x,z)\leq B$ prevent the density and its logarithmic derivatives from becoming degenerate. Similarly, the sub-Gaussian tail condition on $p_k(z)$ controls the probability of extreme values of $Z$.

\paragraph{Density-ratio regularity.}
Condition~(\ref{eq_assump_smooth_ratio}) complements the overlap
condition~(\ref{eq_assump_ratio}). While
(\ref{eq_assump_ratio}) controls the magnitude of the source-to-target
density ratios, it does not control their regularity as functions of
$z$. The H\"older regularity of $q_k$ is used to establish the
H\"older regularity of the mixture conditional-density factor in
$(x,z)$ in Lemma~\ref{thm_mixture_distribution_properties}, which is
needed for the subsequent score-approximation analysis.

The factorization $p_k(z)=p_0(z)q_k(z)$ avoids explicitly dividing by
$p_0(z)$ and therefore allows $p_0(z)$ to vanish. It also requires the
source marginal to be absolutely continuous with respect to the target
marginal. Together, (\ref{eq_assump_ratio}) and
(\ref{eq_assump_smooth_ratio}) control, respectively, the overlap and
the smoothness of the marginal density ratios.

Overall, these smoothness and density-ratio regularity conditions are
mainly technical assumptions for approximation and truncation arguments.
They exclude highly irregular or extremely heavy-tailed distributions,
while still covering a broad class of smooth nonparametric conditional
distributions.

\section{Optimization of the Weighted-ERM Weights}
\label{sec_optimal_weights}

The optimal weighted-ERM weights defined in
(\ref{eq_optimal_weights_sample_sizes}) do not admit a closed-form
solution. We therefore use the target score-matching risk bound in
Theorem~\ref{thm_score_risk_bound} to construct a surrogate weight
selection criterion. 
Define the normalized domain weights
\begin{equation}
\label{eq_define_pi}
    \pi_0
    :=
    \frac{n_0}{n_0+W_N},
    \qquad
    \pi_k
    :=
    \frac{w_kn_k}{n_0+W_N},
    \quad k=1,\ldots,K.
\end{equation}
Let
    $\boldsymbol{\pi}
    := (\pi_0,\pi_1,\ldots,\pi_K)^\top
    \in\Delta^K$,
where
$\Delta^K
    :=
    \left\{
        \boldsymbol{\pi}\in\mathbb{R}^{K+1}:
        \boldsymbol{1}^\top\boldsymbol{\pi}=1,\;
        \boldsymbol{\pi}\geq0
    \right\}$.
We further denote
    $\boldsymbol{\pi}_S
    :=
    (\pi_1,\ldots,\pi_K)^\top$. The original weighted-ERM weights and $W_N$ can be recovered from
$\boldsymbol{\pi}$ as
\begin{equation}
\label{eq_recover_w_from_pi}
    w_k
    =
    \frac{n_0\pi_k}{n_k\pi_0}, \quad \overline w_k
    =
    \frac{\pi_k}{1-\pi_0}, \quad k=1,\ldots,K, \quad \text{ and } \quad W_N
    =
    \frac{n_0(1-\pi_0)}{\pi_0}
    .
\end{equation}

\subsection{Surrogate objective of (\ref{eq_target_risk_bound_erm}) and the weight regularity constraint (\ref{eq_weight_regularity_asymptotic})}
\label{sec_surrogate_obj}
\paragraph{Reparameterization from $w$ to $ \boldsymbol \pi$} Define
    $D_N
    :=
    \operatorname{diag}
    \left(
        n_0^{-1},
        n_1^{-1},
        \ldots,
        n_K^{-1}
    \right)$.
Then the inverse effective sample size can be written as
\begin{align}
\label{eq_Neff_pi}
    N_{\mathrm{eff}}^{-1}
    &=
    \frac{
        n_0+\sum_{k=1}^K w_k^2n_k
    }{
        (n_0+W_N)^2
    }
    =
    \boldsymbol{\pi}^\top
    D_N
    \boldsymbol{\pi}.
\end{align}
For notational convenience, define 
    $h(\boldsymbol{\pi})
    :=
    \boldsymbol{\pi}^\top
    D_N
    \boldsymbol{\pi}$. 
The theoretical estimation term in
Theorem~\ref{thm_score_risk_bound} involves
$N_{\mathrm{eff}}^{-2\Gamma}$, where $\Gamma=\frac{\beta}{d_x+d_z+2\beta}$. 
Since the H\"older smoothness parameter $\beta$ is unknown in practice,
we replace $2\Gamma$ by a tuning exponent $\rho \in(0,1)$. The selection of $\rho$ is discussed in
Section~\ref{subsec_select_Gamma}.

Next, let
    $g
    :=
    \operatorname{diag}(G)
    =
    (G_{11},\ldots,G_{KK})^\top $.
The two transfer-bias terms in~(\ref{eq_target_risk_bound_bias}) can be
rewritten as
\begin{equation}
\label{eq_G_transfer_pi}
    \frac{W_N^2}{(n_0+W_N)^2}
    \overline w^\top G\overline w
    =
    \boldsymbol{\pi}_S^\top
    G
    \boldsymbol{\pi}_S \quad \text{and} \quad 
    \frac{W_N^2}{(n_0+W_N)^2}
    g^\top\overline w
    =
    \frac{1}{2}
    \boldsymbol{\pi}_S^\top
    \left(
        \boldsymbol{1}g^\top
        +
        g\boldsymbol{1}^\top
    \right)
    \boldsymbol{\pi}_S .
\end{equation}
\paragraph{Calibration of the Estimation--Bias trade off} The bound in~(\ref{eq_target_risk_bound_erm})--(\ref{eq_target_risk_bound_bias})
is specified only up to multiplicative constants and polylogarithmic factors.
Hence, the relative scales of the estimation term and the two transfer-bias
terms are not determined by the theoretical bound. We therefore introduce
three positive calibration parameters
$\boldsymbol{\lambda} = (\lambda_{\mathrm{est}},\lambda_G,\lambda_{\mathrm{diag}})^\top$. For a
fixed $\rho\in(0,1)$, consider the surrogate objective
\begin{equation}
\label{eq_surrogate_pi_rho}
\begin{aligned}
    F_{\rho,\boldsymbol{\lambda}}(\boldsymbol{\pi})
    := 
    \lambda_{\mathrm{est}}
    h(\boldsymbol{\pi})^\rho
    +
    \lambda_G
    \boldsymbol{\pi}_S^\top
    G
    \boldsymbol{\pi}_S +
    \frac{\lambda_{\mathrm{diag}}}{2}
    \boldsymbol{\pi}_S^\top
    \left(
        \boldsymbol{1}g^\top
        +
        g\boldsymbol{1}^\top
    \right)
    \boldsymbol{\pi}_S .
\end{aligned}
\end{equation}
Since multiplying all three calibration parameters by the same positive
constant leaves the minimizer of
$F_{\rho,\boldsymbol{\lambda}}$ unchanged, only their relative magnitudes
matter. Without loss of generality, we normalize them as
\begin{equation}
\label{eq_lambda_simplex}
    \lambda_{\mathrm{est}}
    +
    \lambda_G
    +
    \lambda_{\mathrm{diag}}
    =1,
    \qquad
    \lambda_{\mathrm{est}},
    \lambda_G,
    \lambda_{\mathrm{diag}}>0.
\end{equation}
Thus, although three calibration parameters are introduced, the
parameterization has only two degrees of freedom. For notational convenience, define
\begin{equation}
\label{eq_AG_calibrated}
    A_G(\boldsymbol{\lambda})
    :=
    \lambda_G G
    +
    \frac{\lambda_{\mathrm{diag}}}{2}
    \left(
        \boldsymbol{1}g^\top
        +
        g\boldsymbol{1}^\top
    \right),
\end{equation}
so that
\begin{equation}
    F_{\rho,\boldsymbol{\lambda}}(\boldsymbol{\pi})
    =
    \lambda_{\mathrm{est}}
    h(\boldsymbol{\pi})^\rho
    +
    \boldsymbol{\pi}_S^\top
    A_G(\boldsymbol{\lambda})
    \boldsymbol{\pi}_S .
\end{equation}

In practice, we select the $\boldsymbol{\lambda}$ over a finite simplex grid. Specifically,
for an integer $M_\lambda\geq 3$, define
\begin{equation}
\label{eq_lambda_grid}
    \mathcal G_\lambda
    :=
    \left\{
        \left(
            \frac{j_1}{M_\lambda},
            \frac{j_2}{M_\lambda},
            \frac{j_3}{M_\lambda}
        \right):
        j_1,j_2,j_3\in\mathbb N_+,\quad
        j_1+j_2+j_3=M_\lambda
    \right\}.
\end{equation}
We then select
$(\lambda_{\mathrm{est}},\lambda_G,\lambda_{\mathrm{diag}})
\in\mathcal G_\lambda$
according to the gradient alignment criterion introduced in Section \ref{subsec_select_Gamma}.

\paragraph{Numerical Enforcement of the Regularity Condition} The regularity constraint
(\ref{eq_weight_regularity_asymptotic}) requires
    $\max\{1,w_1,\ldots,w_K\}
    \lesssim
    \frac{
        n_0+\sum_{k=1}^K w_k^2n_k
    }{
        n_0+W_N
    }$.
To enforce this condition numerically, we introduce a fixed constant
$C_{\mathrm{reg}}\geq1$ and impose the constraint
\begin{equation}
\label{eq_regularity_finite_w}
    \max\{1,w_1,\ldots,w_K\}
    \leq
    C_{\mathrm{reg}}
    \frac{
        n_0+\sum_{k=1}^K w_k^2n_k
    }{
        n_0+W_N
    },
\end{equation}
which is equivalent to
\begin{equation}
\label{eq_regularity_pi}
    \frac{\pi_k}{n_k}
    \leq
    C_{\mathrm{reg}}
    h(\boldsymbol{\pi}),
    \qquad
    k=0,\ldots,K.
\end{equation}
The value of $C_{\mathrm{reg}}$ is selected from a finite candidate grid,
with the detailed selection procedure described in
Section~\ref{subsec_select_Gamma}. Hence, for each fixed $\rho$, we consider
\begin{align}
\label{eq_pi_optimization_original}
    \boldsymbol{\pi}^* \in \argmin_{\boldsymbol{\pi}}
    \quad&
    F_{\rho,\boldsymbol{\lambda}}(\boldsymbol{\pi})
    \\
    \text{s.t.}\quad&
    \boldsymbol{\pi}\in\Delta^K,
    \nonumber\\
    &
    \frac{\pi_k}{n_k}
    \leq
    C_{\mathrm{reg}}
    h(\boldsymbol{\pi}),
    \quad
    k=0,\ldots,K. \notag
\end{align}
In implementation, we additionally impose
$\pi_0\geq\epsilon_\pi$ for a small
$\epsilon_\pi>0$ to ensure that the recovered weights in
(\ref{eq_recover_w_from_pi}) remain finite.

\subsection{Majorization--Minimization quadratic program (MM-QP) update for fixed hyperparameters}
\label{subsec_mm_qp_pi}

The optimization problem~(\ref{eq_pi_optimization_original}) is generally nonconvex for
three reasons: the fractional-power term $h^{\rho}$, the possibly
indefinite matrix $A_G(\boldsymbol{\lambda})$, and the quadratic constraints in
(\ref{eq_regularity_pi}). We handle these terms using a
Majorization--Minimization quadratic program (MM-QP) procedure.

\paragraph{Majorization of $h(\boldsymbol{\pi})^\rho$} Suppose that a feasible iterate
$\boldsymbol{\pi}^{(\kappa)}$ is available and define
\begin{equation}
\label{eq_h_kappa_pi}
    h^{(\kappa)}
    :=
    h(\boldsymbol{\pi}^{(\kappa)})
    =
    (\boldsymbol{\pi}^{(\kappa)})^\top
    D_N
    \boldsymbol{\pi}^{(\kappa)}.
\end{equation}

Since $u\mapsto u^{\rho}$ is concave for $0<\rho<1$,
\begin{equation}
\label{eq_fractional_majorization_pi}
    h^{\rho}
    \leq
    (h^{(\kappa)})^{\rho}
    +
    a_\rho^{(\kappa)}
    (h-h^{(\kappa)}),
\end{equation}
where
\begin{equation}
\label{eq_a_kappa_pi}
    a_\rho^{(\kappa)}
    :=
    \rho
    (h^{(\kappa)})^{\rho-1}.
\end{equation}
Therefore, up to an additive constant independent of
$\boldsymbol{\pi}$, the nonconstant part of the majorizer of
$h(\boldsymbol{\pi})^\rho$ is
\begin{equation}
\label{eq_estimation_majorizer_pi}
    a_\rho^{(\kappa)}
    \boldsymbol{\pi}^\top
    D_N
    \boldsymbol{\pi}.
\end{equation}
Importantly, this construction only uses the concavity of the scalar
map $u\mapsto u^\rho$ and does not require the composition
$\boldsymbol{\pi}\mapsto h(\boldsymbol{\pi})^\rho$ to be concave.

\paragraph{Majorization of $\boldsymbol{\pi}_S^\top
    A_G(\boldsymbol{\lambda})
    \boldsymbol{\pi}_S$}
Consider the spectral decomposition
\[
    A_G(\boldsymbol{\lambda})
    =
    V\Xi V^\top,
    \qquad
    \Xi
    =
    \operatorname{diag}(\xi_1,\ldots,\xi_K),
\]
where $\xi_1,\ldots,\xi_K$ are the eigenvalues of
$A_G(\boldsymbol{\lambda})$.
Define
\[
    \Xi_+
    :=
    \operatorname{diag}
    \bigl(
        \max\{\xi_j,0\}
    \bigr),
    \qquad
    \Xi_-
    :=
    \operatorname{diag}
    \bigl(
        \max\{-\xi_j,0\}
    \bigr),
\]
and
\begin{equation}
\label{eq_AG_DC}
    A^+_G(\boldsymbol{\lambda})
    :=
    V\Xi_+V^\top \succeq0,
    \qquad
    A^-_G(\boldsymbol{\lambda})
    :=
    V\Xi_-V^\top \succeq0.
\end{equation}
Then
\begin{equation}
    A_G(\boldsymbol{\lambda})
    =
    A_G^+(\boldsymbol{\lambda})
    -
    A_G^-(\boldsymbol{\lambda}),
\end{equation}
and hence
\begin{equation}
\begin{aligned}
    \boldsymbol{\pi}_S^\top
    A_G(\boldsymbol{\lambda})
    \boldsymbol{\pi}_S
    = 
    \boldsymbol{\pi}_S^\top
    A_G^+(\boldsymbol{\lambda})
    \boldsymbol{\pi}_S -
    \boldsymbol{\pi}_S^\top
    A_G^-(\boldsymbol{\lambda})
    \boldsymbol{\pi}_S.
\end{aligned}
\end{equation}
Since
$-\boldsymbol{\pi}_S^\top
A_G^-(\boldsymbol{\lambda})
\boldsymbol{\pi}_S$
is concave, its first-order Taylor expansion at
$\boldsymbol{\pi}_S^{(\kappa)}$ provides the upper bound
\begin{align}
\label{eq_AG_negative_majorization}
    -
    \boldsymbol{\pi}_S^\top
    A_G^-(\boldsymbol{\lambda})
    \boldsymbol{\pi}_S
    \leq
    -
    2
    (\boldsymbol{\pi}_S^{(\kappa)})^\top
    A_G^-(\boldsymbol{\lambda})
    \boldsymbol{\pi}_S
    +
    (\boldsymbol{\pi}_S^{(\kappa)})^\top
    A_G^-(\boldsymbol{\lambda})
    \boldsymbol{\pi}_S^{(\kappa)}.
\end{align}
Therefore, up to an additive constant independent of
$\boldsymbol{\pi}_S$, the transfer-bias term
$\boldsymbol{\pi}_S^\top
A_G(\boldsymbol{\lambda})
\boldsymbol{\pi}_S$
is majorized by
\begin{equation}
    \boldsymbol{\pi}_S^\top
    A_G^+(\boldsymbol{\lambda})
    \boldsymbol{\pi}_S
    -
    2
    (\boldsymbol{\pi}_S^{(\kappa)})^\top
    A_G^-(\boldsymbol{\lambda})
    \boldsymbol{\pi}_S.
\end{equation}

\paragraph{Inner approximation of the regularity constraint
(\ref{eq_weight_regularity_asymptotic}).}
Because
$h(\boldsymbol{\pi})=\boldsymbol{\pi}^\top D_N\boldsymbol{\pi}$
is convex, its first-order Taylor expansion at
$\boldsymbol{\pi}^{(\kappa)}$ provides the global lower bound
\begin{equation}
\label{eq_h_affine_lower}
    h(\boldsymbol{\pi})
    \geq
    2
    (\boldsymbol{\pi}^{(\kappa)})^\top
    D_N
    \boldsymbol{\pi}
    -
    h^{(\kappa)}.
\end{equation}
We therefore replace
(\ref{eq_regularity_pi}) by the stronger inner constraints
\begin{equation}
\label{eq_regularity_inner_pi}
    \frac{\pi_k}{n_k}
    \leq
    C_{\mathrm{reg}}
    \left[
        2
        (\boldsymbol{\pi}^{(\kappa)})^\top
        D_N
        \boldsymbol{\pi}
        -
        h^{(\kappa)}
    \right],
    \qquad
    k=0,\ldots,K.
\end{equation}
Since
    $2
    (\boldsymbol{\pi}^{(\kappa)})^\top
    D_N
    \boldsymbol{\pi}
    -
    h^{(\kappa)}
    \leq
    h(\boldsymbol{\pi})$,
any point satisfying
(\ref{eq_regularity_inner_pi}) also satisfies the original regularity
constraint (\ref{eq_regularity_pi}). Moreover, the lower bound is tight at the current iterate,
since
    $2
    (\boldsymbol{\pi}^{(\kappa)})^\top
    D_N
    \boldsymbol{\pi}^{(\kappa)}
    -
    h^{(\kappa)}
    =
    h^{(\kappa)}$,
so the current feasible iterate remains feasible for the inner
approximation.

Combining the above bounds, the MM update for a fixed $\rho$ is the
convex quadratic program
\begin{align}
\label{eq_final_pi_mm_qp}
    \boldsymbol{\pi}^{(\kappa+1)}
    \in
    \underset{\boldsymbol{\pi}}{\arg\min}
    \quad&
    \lambda_{\mathrm{est}} a^{(\kappa)}_\rho
    \boldsymbol{\pi}^\top
    D_N
    \boldsymbol{\pi}
    +
    \boldsymbol{\pi}_S^\top
    A_G^+(\boldsymbol{\lambda})
    \boldsymbol{\pi}_S -
    2
    (\boldsymbol{\pi}_S^{(\kappa)})^\top
    A_G^-(\boldsymbol{\lambda})
    \boldsymbol{\pi}_S
    \nonumber\\
    \text{s.t.}\quad&
    \boldsymbol{1}^\top\boldsymbol{\pi}=1,
    \nonumber\\
    &
    \boldsymbol{\pi}\geq0,
    \qquad
    \pi_0\geq\epsilon_\pi,
    \nonumber\\
    &
    \frac{\pi_k}{n_k}
    \leq
    C_{\mathrm{reg}}
    \left[
        2
        (\boldsymbol{\pi}^{(\kappa)})^\top
        D_N
        \boldsymbol{\pi}
        -
        h^{(\kappa)}
    \right],
    \quad
    k=0,\ldots,K.
\end{align}
Since
$a_\rho^{(\kappa)}>0$,
$D_N\succ0$, and
$A_G^+\succeq0$,
the quadratic objective in
(\ref{eq_final_pi_mm_qp}) is strictly convex.
Hence, each MM subproblem admits a unique solution.

Let $Q_{\rho,\boldsymbol{\lambda}}^{(\kappa)}(\boldsymbol{\pi})$ denote the objective
in~(\ref{eq_final_pi_mm_qp}), i.e.,
\begin{align}
\label{eq_define_Q_rho}
Q_{\rho,\boldsymbol{\lambda}}^{(\kappa)}(\boldsymbol{\pi})
:={}&
\lambda_{\mathrm{est}} a_\rho^{(\kappa)}
\boldsymbol{\pi}^\top
D_N
\boldsymbol{\pi}
+
\boldsymbol{\pi}_S^\top
A_G^+(\boldsymbol{\lambda})
\boldsymbol{\pi}_S
-
2
(\boldsymbol{\pi}_S^{(\kappa)})^\top
A_G^-(\boldsymbol{\lambda})
\boldsymbol{\pi}_S.
\end{align}
The corresponding full MM majorizer differs from
$Q_{\rho,\boldsymbol{\lambda}}^{(\kappa)}$ only by the additive constant
\begin{align}
\label{eq_define_C_rho}
C_{\rho, \boldsymbol{\lambda}}^{(\kappa)}
:={}&
\lambda_{\mathrm{est}} \left[\bigl(h^{(\kappa)}\bigr)^\rho
-
a^{(\kappa)}_\rho h^{(\kappa)} \right]
+
(\boldsymbol{\pi}_S^{(\kappa)})^\top
A_G^-(\boldsymbol{\lambda})
\boldsymbol{\pi}_S^{(\kappa)}.
\end{align}
Hence,
\begin{equation}
\label{eq_full_majorization}
F_{\rho, \boldsymbol{\lambda}}(\boldsymbol{\pi})
\leq
Q_{\rho , \boldsymbol{\lambda}}^{(\kappa)}(\boldsymbol{\pi})
+
C_{\rho, \boldsymbol{\lambda}}^{(\kappa)},
\end{equation}
with equality at
$\boldsymbol{\pi}=\boldsymbol{\pi}^{(\kappa)}$.
Since $C_{\rho, \boldsymbol{\lambda}}^{(\kappa)}$ is independent of
$\boldsymbol{\pi}$, minimizing the full majorizer is equivalent to
solving~(\ref{eq_final_pi_mm_qp}). Moreover, the inner approximation of the regularity constraint
contains the current feasible iterate
$\boldsymbol{\pi}^{(\kappa)}$.
Therefore, the construction guarantees monotonic descent of
$F_{\rho, \boldsymbol{\lambda}}$:
\begin{align}
    F_{\rho, \boldsymbol{\lambda}}(\boldsymbol{\pi}^{(\kappa+1)})
    &\leq
    Q_{\rho, \boldsymbol{\lambda}}^{(\kappa)}
    (\boldsymbol{\pi}^{(\kappa+1)})
    +
    C_{\rho, \boldsymbol{\lambda}}^{(\kappa)}
    \nonumber\\
    &\leq
    Q_{\rho, \boldsymbol{\lambda}}^{(\kappa)}
    (\boldsymbol{\pi}^{(\kappa)})
    +
    C_{\rho, \boldsymbol{\lambda}}^{(\kappa)}
    \nonumber\\
    &=
    F_{\rho, \boldsymbol{\lambda}}(\boldsymbol{\pi}^{(\kappa)}).
\end{align}

\paragraph{Initialization.}
A guaranteed feasible initialization is obtained by setting
$w_k^{(0)}=1$ for all source domains, which gives
\begin{equation}
\label{eq_pi_initialization}
    \pi_k^{(0)}
    =
    \frac{n_k}
    {\sum_{j=0}^K n_j},
    \qquad
    k=0,\ldots,K.
\end{equation}
For this initialization,
\[
    \frac{\pi_k^{(0)}}{n_k}
    =
    \frac{1}{\sum_{j=0}^K n_j}
    =
    h(\boldsymbol{\pi}^{(0)}),
\]
and hence the original regularity constraint is satisfied for every
$C_{\mathrm{reg}}\geq1$.


\subsection{Selection of tuning parameters by gradient alignment}
\label{subsec_select_Gamma}

The theoretical exponent
$2\beta/(d_x+d_z+2\beta)$ in (\ref{eq_target_risk_bound_erm})
depends on the unknown H\"older smoothness parameter $\beta$.
As mentioned in Section~\ref{sec_surrogate_obj}, we therefore treat
$\rho$ as a tuning parameter and select it from a finite candidate set
\[
    \mathcal G_\rho
    =
    \{\rho_1,\ldots,\rho_J\}
    \subset(0,1),
\]
where the upper bound $1$ corresponds to the limiting case
$\beta\to\infty$. Similarly, the constant $C_{\mathrm{reg}}$ in the weight regularity
constraint is unknown in practice. We therefore select it from
\[
    \mathcal G_C
    =
    \{C_1,\ldots,C_L\}
    \subset(1,C_{\max}],
\]
where $C_{\max}<\infty$ is a fixed constant independent of the sample
sizes. The calibration parameters
$\boldsymbol{\lambda}
=
(\lambda_{\mathrm{est}},\lambda_G,\lambda_{\mathrm{diag}})^\top$
introduced in~(\ref{eq_surrogate_pi_rho}) are also selected from the
finite simplex grid $\mathcal G_\lambda$ defined in
(\ref{eq_lambda_grid}). Write
\[
    \mathcal G_\lambda
    =
    \{
        \boldsymbol{\lambda}^{(1)},
        \ldots,
        \boldsymbol{\lambda}^{(R)}
    \},
    \qquad
    R
    :=
    |\mathcal G_\lambda|
    =
    \binom{M_\lambda-1}{2}.
\]
Although $\boldsymbol{\lambda}$ contains three components, the
normalization in~(\ref{eq_lambda_simplex}) leaves only two degrees of
freedom.

We do not select these tuning parameters by directly comparing the
values of the surrogate objective
$F_{\rho,\boldsymbol{\lambda}}$ across different candidates. In
particular, changing $\rho$ changes the scale of
$h(\boldsymbol{\pi})^\rho$, while changing
$\boldsymbol{\lambda}$ changes the relative scales of the estimation
and transfer-bias terms. Instead, we jointly select
$\rho$, $C_{\mathrm{reg}}$, and $\boldsymbol{\lambda}$ according to
the predicted first-order improvement of the target validation loss.

Let $\theta$ denote the parameters of the current score network.
For domain $k$, denote its empirical training loss by
$\widehat{\mathcal L}_k(\theta)$. Define their gradient by
\begin{equation}
\label{eq_domain_gradient}
    \boldsymbol{u}_k
    :=
    \nabla_\theta
    \widehat{\mathcal L}_k(\theta),
    \qquad
    k=0,\ldots,K.
\end{equation}
Let
\begin{equation}
\label{eq_target_val_gradient}
    \boldsymbol{u}_0^{\mathrm{val}}
    :=
    \nabla_\theta
    \widehat{\mathcal L}_{0,\mathrm{val}}(\theta)
\end{equation}
denote the gradient of the target validation loss.

For each candidate triple
    $\left( \rho_j,C_l,\boldsymbol{\lambda}^{(m)} \right)
    \in
    \mathcal G_\rho
    \times
    \mathcal G_C
    \times
    \mathcal G_\lambda$,
we solve~(\ref{eq_final_pi_mm_qp}) with
$\rho=\rho_j$,
$C_{\mathrm{reg}}=C_l$, and
$\boldsymbol{\lambda}=\boldsymbol{\lambda}^{(m)}$
to obtain
\[
    \widehat{\boldsymbol{\pi}}_{j,l,m}
    =
    \left(
        \widehat{\pi}_{0,j,l,m},
        \ldots,
        \widehat{\pi}_{K,j,l,m}
    \right)^\top.
\]
The resulting weighted training loss is
\begin{equation}
\label{eq_weighted_loss_rho}
    \widehat{\mathcal L}_{j,l,m}(\theta)
    :=
    \sum_{k=0}^K
    \widehat{\pi}_{k,j,l,m}
    \widehat{\mathcal L}_k(\theta),
\end{equation}
whose gradient is
\begin{equation}
\label{eq_weighted_gradient_rho}
    \boldsymbol{u}_{j,l,m}
    :=
    \nabla_\theta
    \widehat{\mathcal L}_{j,l,m}(\theta)
    =
    \sum_{k=0}^K
    \widehat{\pi}_{k,j,l,m}
    \boldsymbol{u}_k.
\end{equation}
To motivate the selection criterion, consider a single gradient step
with learning rate $\eta>0$,
\[
    \theta_{j,l,m}^{+}
    =
    \theta-\eta\boldsymbol{u}_{j,l,m}.
\]
A first-order Taylor expansion of the target validation loss gives
\begin{align}
\label{eq_gradient_alignment_taylor}
    \widehat{\mathcal L}_{0,\mathrm{val}}
    (\theta_{j,l,m}^{+})
    &=
    \widehat{\mathcal L}_{0,\mathrm{val}}
    (\theta-\eta\boldsymbol{u}_{j,l,m})
    \nonumber\\
    &=
    \widehat{\mathcal L}_{0,\mathrm{val}}(\theta)
    -
    \eta
    \left\langle
        \boldsymbol{u}_0^{\mathrm{val}},
        \boldsymbol{u}_{j,l,m}
    \right\rangle
    +
    O(\eta^2).
\end{align}
Hence, to first order, a larger positive inner product
\[
    \left\langle
        \boldsymbol{u}_0^{\mathrm{val}},
        \boldsymbol{u}_{j,l,m}
    \right\rangle
\]
corresponds to a larger predicted reduction in the target validation
loss. We therefore jointly select the tuning parameters according to
\begin{equation}
\label{eq_select_rho_gradient_alignment}
\begin{aligned}
    \left(
        \widehat{\rho},
        \widehat C_{\mathrm{reg}},
        \widehat{\boldsymbol{\lambda}}
    \right)
    \in
    \underset{
        \substack{
            \rho_j\in\mathcal G_\rho,
            C_l\in\mathcal G_C,
            \boldsymbol{\lambda}^{(m)}
            \in\mathcal G_\lambda
        }
    }{\arg\max}
    \;
    \left\langle
        \boldsymbol{u}_0^{\mathrm{val}},
        \boldsymbol{u}_{j,l,m}
    \right\rangle .
\end{aligned}
\end{equation}
Because the weighted gradient is linear in the domain weights,
(\ref{eq_select_rho_gradient_alignment}) can be evaluated without
performing an additional backward pass for every candidate triple.
Define the domain-wise target-alignment scores
\begin{equation}
\label{eq_domain_alignment_score}
    c_k
    :=
    \left\langle
        \boldsymbol{u}_0^{\mathrm{val}},
        \boldsymbol{u}_k
    \right\rangle,
    \qquad
    k=0,\ldots,K,
\end{equation}
and let
\[
    c
    :=
    (c_0,\ldots,c_K)^\top.
\]
Then
\begin{align}
\label{eq_alignment_linear_pi}
    \left\langle
        \boldsymbol{u}_0^{\mathrm{val}},
        \boldsymbol{u}_{j,l,m}
    \right\rangle
    &=
    \left\langle
        \boldsymbol{u}_0^{\mathrm{val}},
        \sum_{k=0}^K
        \widehat{\pi}_{k,j,l,m}
        \boldsymbol{u}_k
    \right\rangle
    \nonumber\\
    &=
    \sum_{k=0}^K
    \widehat{\pi}_{k,j,l,m}
    \left\langle
        \boldsymbol{u}_0^{\mathrm{val}},
        \boldsymbol{u}_k
    \right\rangle
    \nonumber\\
    &=
    c^\top
    \widehat{\boldsymbol{\pi}}_{j,l,m}.
\end{align}
Consequently, the tuning parameters can equivalently be selected as
\begin{equation}
\label{eq_select_rho_linear_score}
\boxed{
\begin{aligned}
    \left(
        \widehat{\rho},
        \widehat C_{\mathrm{reg}},
        \widehat{\boldsymbol{\lambda}}
    \right)
    \in
    \underset{
        \substack{
            \rho_j\in\mathcal G_\rho,
            C_l\in\mathcal G_C,
            \boldsymbol{\lambda}^{(m)}
            \in\mathcal G_\lambda
        }
    }{\arg\max}
    \;
    c^\top
    \widehat{\boldsymbol{\pi}}_{j,l,m}.
\end{aligned}
}
\end{equation}
This criterion has a direct domain-adaptation interpretation.
If
\[
    c_k
    =
    \left\langle
        \boldsymbol{u}_0^{\mathrm{val}},
        \boldsymbol{u}_k
    \right\rangle
    >0,
\]
then a gradient step using domain $k$ is locally aligned with a
descent direction of the target validation loss. Conversely,
$c_k<0$ indicates a locally conflicting training direction and
potential negative transfer. Hence,
(\ref{eq_select_rho_linear_score}) favors tuning-parameter
configurations whose MM-QP solutions place larger weights on domains
whose training gradients are beneficial to the target domain.

Since
$\widehat C_{\mathrm{reg}}\leq C_{\max}$ and $C_{\max}$ is fixed
independently of the sample sizes, the selected weights continue to
satisfy the weight regularity condition with a sample-size-independent
constant.

After selecting
$\left(\widehat\rho,\widehat C_{\mathrm{reg}},
\widehat{\boldsymbol{\lambda}}\right)$,
we use the corresponding MM-QP solution
\[
    \widehat{\boldsymbol{\pi}}
    :=
    \widehat{\boldsymbol{\pi}}_{
        \widehat\rho,
        \widehat C_{\mathrm{reg}},
        \widehat{\boldsymbol{\lambda}}
    }
\]
for weighted-ERM training, and recover the original source weights via
\begin{equation}
\label{eq_final_w_recovery}
    \widehat w_k
    =
    \frac{
        n_0\widehat\pi_k
    }{
        n_k\widehat\pi_0
    },
    \qquad
    k=1,\ldots,K.
\end{equation}
In practice, neither the tuning parameters nor the MM-QP weights are
updated at every epoch. We first warm up the score model using only
target-domain data, then jointly select
$(\widehat\rho,\widehat C_{\mathrm{reg}},
\widehat{\boldsymbol{\lambda}})$
by gradient alignment and solve the corresponding MM-QP for
$\widehat{\boldsymbol{\pi}}$. The model is subsequently trained for
several epochs with these weights fixed, after which the tuning
parameters and weights are updated again.

\paragraph{Computational complexity.}
For a fixed triple
$(\rho,C_{\mathrm{reg}},\boldsymbol{\lambda})$, each MM iteration
requires the eigendecomposition in~(\ref{eq_AG_DC}) and the solution
of a convex quadratic program of dimension $K+1$ in
(\ref{eq_final_pi_mm_qp}). Both operations have computational
complexity $\mathcal O(K^3)$, so the overall complexity of each MM
iteration remains $\mathcal O(K^3)$.

Let
\[
    J=|\mathcal G_\rho|,
    \qquad
    L=|\mathcal G_C|,
    \qquad
    R=|\mathcal G_\lambda|.
\]
The MM-QP weight selection across all candidate triples therefore costs
$\mathcal O(JLRK^3)$ per MM iteration. If the number of MM iterations
is treated as fixed, the overall cost of each tuning-parameter
selection step is also $\mathcal O(JLRK^3)$. Moreover, since
$A_G(\boldsymbol{\lambda})$ depends only on
$\boldsymbol{\lambda}$, its eigendecomposition can be reused across
all $(\rho,C_{\mathrm{reg}})$ pairs associated with the same
$\boldsymbol{\lambda}$. Since $K$ is typically much smaller than the
number of score-network parameters, this optimization cost is
negligible relative to score-network training.

Moreover, the gradient-alignment criterion does not require one
backward pass for each candidate triple
$(\rho,C_{\mathrm{reg}},\boldsymbol{\lambda})$.
The domain gradients
$\boldsymbol{u}_0,\ldots,\boldsymbol{u}_K$ and the target validation
gradient $\boldsymbol{u}_0^{\mathrm{val}}$ are computed only once.
The alignment score for each candidate triple is then obtained from
the inexpensive inner product
$c^\top\widehat{\boldsymbol{\pi}}_{j,l,m}$.

Detailed training behavior of DA-Diff and MM-QP can be found in Appendix \ref{sec_train_behavior}.

\section{Data generation mechanism of the synthetic experiments in Section \ref{sec_simulation}}

\subsection{Gaussian mixture data}
\label{app_gmm}
We first introduce the data-generating mechanism for the Gaussian mixture data used to evaluate DA-Diff and other transfer-learning methods in Section \ref{sec_gaussian}.

Let $X_0\in\mathbb{R}^{5}$ follow a Gaussian mixture
$X_0 \sim \sum_{i=1}^{10} 0.1 \cdot \mathcal N(\mu_i,\Sigma_i)$, where $\mathcal N$ is the density of the normal distribution. We consider the linear transform:
\begin{equation*}
    Z_0 = AX_0 + \sigma\,\varepsilon, \qquad \varepsilon\sim N(0,I_5),
\end{equation*}
where $A \in \mathbb{R}^{5 \times 5}$ is fixed and $\sigma>0$. Conditioning on $Z_0=z$, the posterior remains a Gaussian mixture with analytically available parameters:
\begin{equation} \label{eq_exact_poster_GM}
{\textstyle
    P_0(\cdot| Z_0=z) \propto \sum_{i=1}^{10} \widetilde{\pi}_i \cdot \mathcal N(\widetilde{\mu}_i ,\widetilde{\Sigma}_i), }
\end{equation}
where $\widetilde{\Sigma}_i = \bigl(\Sigma_i^{-1} + A^\top A/\sigma^2\bigr)^{-1}$,
$\widetilde{\mu}_i  = \widetilde{\Sigma}_i\bigl(\Sigma_i^{-1}\mu_i + A^\top z/\sigma^2\bigr)$, and the mixture weights satisfy
$\widetilde{\pi}_i =  0.1 \cdot \exp [ 0.5(\widetilde{\mu}_i^\top  \widetilde{\Sigma}_i^{-1} \widetilde{\mu}_i    -   \mu_i^\top  \Sigma_i^{-1} \mu_i) ] / \sqrt{|\Sigma_i|}$.

We draw $\mu_i\sim U([-1.5,1.5]^5)$ and set $\Sigma_i = 0.01 I_5$ for $i=1,\ldots,10$, $A = 0.25\,\mathbf{1}\mathbf{1}^\top + 0.25 I_5$, and $\sigma = 0.5$. We generate three $K$ domains. The sample sizes are set to $n_0 = n_1=\cdots=n_K = 500$. Let $d^{(k)}$ denote the shift strength of the $k$-th source. We consider five near-source domains, with shift strengths lying on the linear interpolation segment from 0.1 toward 0.2, and five far-source domains, with shift strengths lying on the segment from 0.2 toward 0.5. For each source, we introduce a controlled distribution shift by perturbing the mixture parameters as:
\begin{align}
    \mu_i^{(k)} &= \mu_i + d^{(k)}\,\xi_{i}^{(k)}, \qquad \xi_{i}^{(k)} \sim N(0,I_5), \notag \\
    \Sigma_i^{(k)} &= \Sigma_i + \left[ d^{(k)}\, \mathrm{diag}(u_{i}^{(k)}) \right]^2, \qquad u_{i}^{(k)}\sim U([0,1]^5), \notag \\
    A^{(k)}_{ij} &= A_{ij} + d^{(k)}\,\psi_{ij}^{(k)}, \qquad \psi_{ij}^{(k)} \sim N(0,1), \notag \\
    \sigma^{(k)} & =  \max\{\sigma + d^{(k)}  \gamma^{(k)}, 10^{-4}\}, \qquad \gamma^{(k)}\sim N(0,1).
\end{align}

\subsection{Post-nonlinear model}
\label{app_post_nonlinear}
In this subsection, we describe the data-generating mechanism for the synthetic experiments used to evaluate DA-CIT and the competing conditional independence tests. Under the null hypothesis, the data-generating mechanism follows the fork structure $X \leftarrow Z \rightarrow Y$. Under the alternative, conditional
dependence is introduced through a direct effect of $Y$ on $X$. We construct multiple source domains by perturbing the covariate distribution and the conditional mechanisms relating $Z$ to $X$ and $Y$, with the transferability of each source controlled by its shift strength $d^{(k)}$. 

For the target domain, we generate $
Z_0\sim N(0,I_{d_z}),
$ and independently draw $
\beta_X,\beta_Y\sim N(0,I_{d_z}).
$ Given $Z_0$, we generate
\begin{align}
    Y_0
     =
    g_Y\left(
    Z_0^\top {\beta_Y}/{\|\beta_Y\|_2}
    +0.5\cdot\epsilon_Y
    \right), \quad 
    X_0
    =
    g_X\left(
    Z_0^\top {\beta_X}/{\|\beta_X\|_2}
    +0.5\cdot \epsilon_X
    \right) + h_1 \cdot Y_0.
\end{align}
Here, $h_1\geq 0$ controls the strength of the conditional dependence between $X_0$ and $Y_0$ given $Z_0$. When $h_1=0$, $X_0$ and $Y_0$ are CI given $Z_0$. As $h_1$ increases, the conditional dependence between $X_0$ and $Y_0$ becomes stronger, resulting in a larger departure from conditional independence.  Moreover, $
\epsilon_X,\epsilon_Y\overset{\mathrm{i.i.d.}}{\sim} N(0,1),
$ and the functions $g_X$ and $g_Y$ are independently and uniformly sampled from $ \texttt{func\_set} = 
\left\{
x,\,
x^2,\,
\cos(x),\,
\sin(x),\,
\tanh(x),\,
\exp(-|x|)
\right\}.
$ 

For source domain $k$, let $d^{(k)}$ denote its shift strength. We perturb the data generating  mechanism through the regression mechanism, nonlinear functions, noise scale, and marginal distribution of $Z$. Specifically, the regression mechanism is shifted as
\begin{equation*}
\beta_j^{(k)}
=
\beta_j+d^{(k)}\xi_{j}^{(k)},
\qquad
\xi_{j}^{(k)} \sim N(0,I_{d_z}),
\qquad j\in \{X,Y\}.
\end{equation*}
The noise scale is perturbed according to
\begin{equation}
\sigma^{(k)}
=
0.5 +d^{(k)} u^{(k)},
\qquad
u^{(k)}\sim U(0,1).
\end{equation}
For the nonlinear functions, with probability $1-\exp(-d^{(k)})$, both
$g_X^{(k)}$ and $g_Y^{(k)}$ are independently resampled from \texttt{func\_set};
otherwise, the target functions $g_X$ and $g_Y$ are retained.

We also include a covariate shift in $Z$. For each source domain,
\begin{equation}
Z_k
\sim
N
\left(
\mu^{(k)},
\operatorname{diag}\bigl((s^{(k)}_{Z})^2\bigr)
\right),
\end{equation}
where
\begin{align}
\mu^{(k)}
&=
0.1d^{(k)} \cdot \eta^{(k)}, \quad 
\eta^{(k)}
\sim N(0,I_{d_z}),\notag \\
s^{(k)}_{Z}
&=
\exp(0.1d^{(k)} \cdot \nu^{(k)}), \quad  \nu^{(k)}
\sim N(0,I_{d_z}),
\end{align}
with the exponential applied element-wise. The source variables are then generated as
\begin{align}
Y_k
=
g_Y^{(k)}
\left(
Z_k^\top
{\beta_Y^{(k)}}/{\|\beta_Y^{(k)}\|_2}
+
\sigma^{(k)} \cdot \epsilon_{Y,k}
\right), \quad 
X_k
=
g_X^{(k)}
\left(
Z_k^\top
{\beta_X^{(k)}}/{\|\beta_X^{(k)}\|_2}
+
\sigma^{(k)} \cdot \epsilon_{X,k}
\right) + h_1 \cdot Y_k,\
\end{align}
where $
\epsilon_{X,k},\epsilon_{Y,k}
\overset{\mathrm{i.i.d.}}{\sim}
\mathcal{N}(0,1)$. We use seven near sources and three far sources. The shift strengths of the near sources are linearly interpolated from $0.01$ to $0.1$, while those of the far sources are linearly interpolated from $0.2$ to $0.5$.

\section{Ablation and sensitivity studies}
\label{app_ablation}

In this section, we conduct ablation and sensitivity analyses of DA-Diff and DA-CIT using the data-generation mechanisms described in Appendices~\ref{app_gmm} and~\ref{app_post_nonlinear}. UOWQ has apparent negative transfer performance in Table \ref{tab_gaussian_baselines_w2}, thus we skip the experiments of UOWQ.

\subsection{Sensitivity analysis on Gaussian mixture data}
We study the sensitivity of DA-Diff to the target and source sample sizes, the number of source domains, and several key hyperparameters.  

\begin{table}[ht]
\centering
\caption{
Performance of the proposed DA-Diff, Target-only, Cat-all, Finetune,
DoG~\citep{zhong_2025_dog}, TGDP~\citep{ouyang_2024_tgdp},
 and Robust~\citep{Konstantinov_2019_Robust_Learning}
under the data-generating mechanism in Section~\ref{sec_gaussian}, with sample size $n_0 = \cdots = n_K = 1000, 1500, 2000$.
We report the $W_2$ distance along with one standard deviation (std).
}
\label{tab_gaussian_baselines_w2_appendix}
\begin{tabular}{c l cccc}
\toprule
Sample size & Method
& $W_2$ ($K=2$)
& $W_2$ ($K=5$)
& $W_2$ ($K=10$)
& $W_2$ ($K=20$) \\
\midrule

1000
& DA-Diff (Ours)
& \textbf{0.142}(0.039)
& \textbf{0.140}(0.037)
& \textbf{0.133}(0.030)
& \textbf{0.134}(0.027) \\

& Target-only
& 0.376(0.070)
& 0.376(0.070)
& 0.376(0.070)
& 0.376(0.070) \\

& Cat-all
& 0.440(0.112)
& 0.371(0.087)
& 0.334(0.088)
& 0.308(0.063) \\

& Finetune & 0.216$^\dagger$(0.056) & \underline{0.203}(0.057) & \underline{0.173}(0.045) & \underline{0.164}(0.040)\\

& DoG \citeyearpar{zhong_2025_dog} & 0.285(0.103) & 0.231(0.066) & 0.216(0.064) & 0.181(0.046) \\

& TGDP \citeyearpar{ouyang_2024_tgdp} & 0.562(0.135) & 0.413(0.116) & 0.389(0.104) & 0.328(0.074) \\
& Robust \citeyearpar{Konstantinov_2019_Robust_Learning} & \underline{0.175}(0.055) & 0.219$^\dagger$(0.062) & 0.208$^\dagger$(0.070) & 0.176$^\dagger$(0.047) \\

\midrule

1500
& DA-Diff (Ours)
& \textbf{0.111}(0.025)
& \textbf{0.113}(0.022)
& \textbf{0.106}(0.024)
& \textbf{0.110}(0.027) \\

& Target-only
& 0.259(0.047)
& 0.259(0.047)
& 0.259(0.047)
& 0.259(0.047) \\

& Cat-all
& 0.336(0.098)
& 0.339(0.073)
& 0.322(0.081)
& 0.316(0.071) \\

& Finetune & 0.226(0.052) & \underline{0.154}(0.037) & \underline{0.135}(0.035) & \underline{0.121}(0.032) \\

& DoG \citeyearpar{zhong_2025_dog} & 0.197$^\dagger$(0.058) & 0.164$^\dagger$(0.048) & 0.147$^\dagger$(0.037) & 0.138$^\dagger$(0.032) \\

& TGDP \citeyearpar{ouyang_2024_tgdp} & 0.558(0.131) & 0.428(0.118) & 0.370(0.096) & 0.337(0.083)\\
& Robust \citeyearpar{Konstantinov_2019_Robust_Learning} & \underline{0.150}(0.050) & 0.200(0.062) & 0.176(0.053) & 0.178(0.068) \\

\midrule

2000
& DA-Diff (Ours)
& \textbf{0.095}(0.022)
& \textbf{0.098}(0.027)
& \textbf{0.090}(0.022)
& \underline{0.097}(0.022) \\

& Target-only
& 0.123$^\dagger$(0.018)
& 0.123(0.018)
& 0.123(0.018)
& 0.123(0.018) \\

& Cat-all
& 0.291(0.075)
& 0.335(0.114)
& 0.330(0.080)
& 0.310(0.069) \\

& Finetune & \underline{0.105}(0.017) & \underline{0.103}(0.041) & \underline{0.092}(0.019) & \textbf{0.085}(0.020) \\

& DoG \citeyearpar{zhong_2025_dog} & 0.165(0.057) & 0.121$^\dagger$(0.035) & 0.113$^\dagger$(0.028) & 0.109$^\dagger$(0.031) \\

& TGDP \citeyearpar{ouyang_2024_tgdp} & 0.482(0.128) & 0.421(0.121) & 0.370(0.093) & 0.328(0.081) \\
& Robust \citeyearpar{Konstantinov_2019_Robust_Learning} & 0.123(0.037) & 0.202(0.060) & 0.162(0.043) & 0.161(0.054) \\

\bottomrule
\end{tabular}
\end{table}

Table~\ref{tab_gaussian_baselines_w2_appendix} reports the performance of DA-Diff across different sample sizes and numbers of source domains. DA-Diff achieves the best performance for all considered values of $K$ when the sample size is 1000 or 1500, and remains among the top-performing
methods when the sample size is 2000. These results suggest that DA-Diff is
relatively robust to changes in sample size and the number of source domains.

\begin{figure}[htb!]
    \centering
    \includegraphics[width=\linewidth]
    {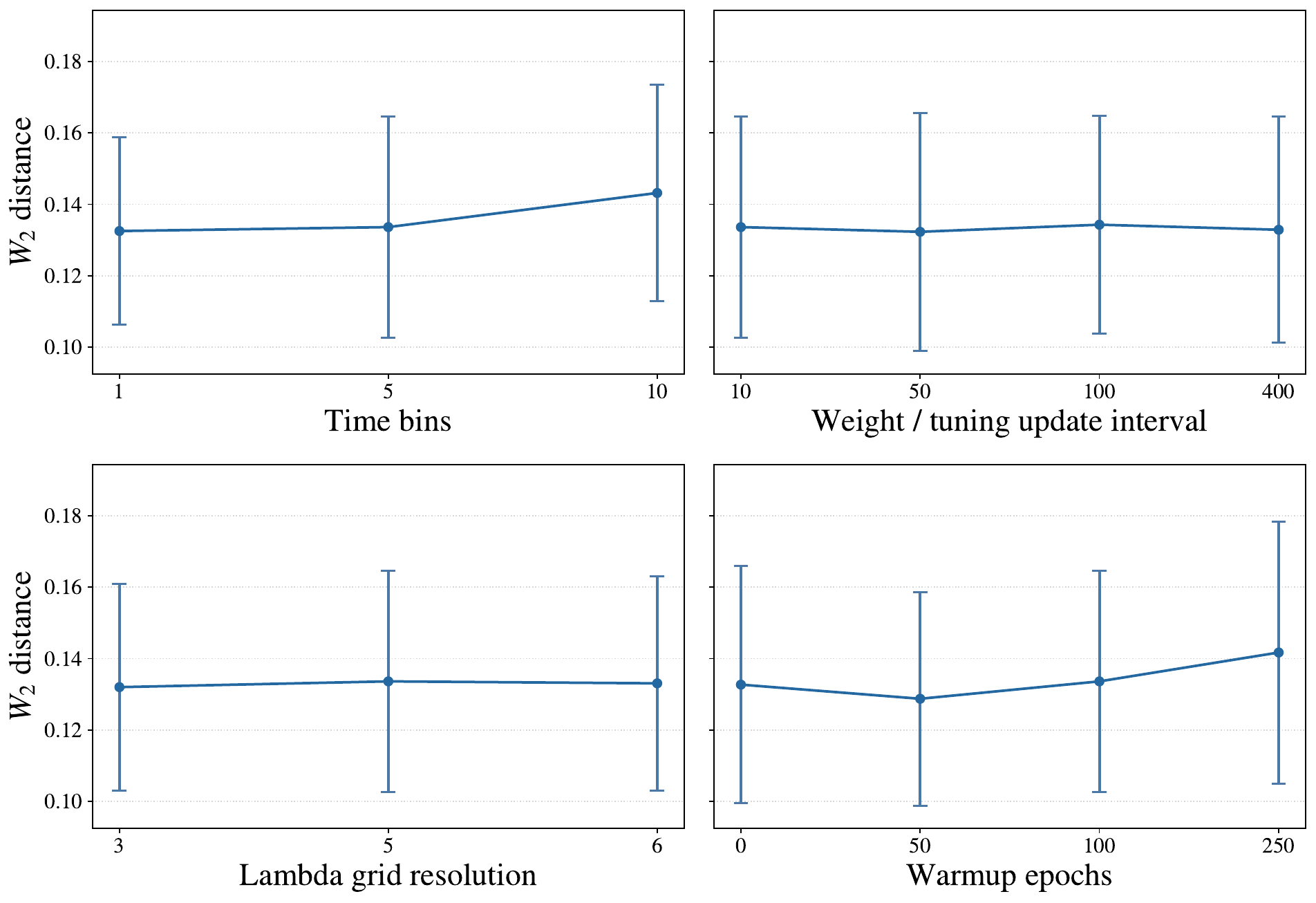}
    \caption{
    Sensitivity study of DA-Diff on the Gaussian mixture data.
    We vary the number of time bins $B_T$,
    the weight/tuning update interval $\texttt{update\_every}$, the resolution of the $\mathcal{G}_\lambda$, and the number of warmup epochs $e_w$,
    and evaluate the performance using the $W_2$ distance.
    }
    \label{fig_da_diff_ablation}
\end{figure}

Figure~\ref{fig_da_diff_ablation} further examines the sensitivity of DA-Diff to four hyperparameters: the number of time bins $B_T$ in Section~\ref{sec_train_behavior}, the weight/tuning update interval $\texttt{update\_every}$ in Algorithm \ref{algo_train_da_diff}, the grid resolution of $\mathcal{G}_\lambda$ in~(\ref{eq_lambda_grid}), and the number of warmup epochs $e_w$ in Algorithm \ref{algo_train_da_diff}. The $W_2$ distance varies only moderately across the considered hyperparameter values, with largely overlapping error bars. Overall, DA-Diff shows limited sensitivity to these hyperparameter choices.

\subsection{Sensitivity analysis on the post-nonlinear model}

We further investigate the sensitivity of DA-CIT to the number of source domains, source--target similarity, the target and source sample sizes, and the strength of conditional dependence under $H_1$. Following the setting in Section~\ref{sec_post_non_linear}, we fix $d_z=20$ throughout and vary one factor at a time while keeping the remaining settings at their default values. Unless otherwise specified, we use $K=10$ source
domains, consisting of seven near sources and three far sources, and set $n_0=n_1=\cdots=n_K=500$. The construction of near/far sources is introduced in Appendix \ref{app_post_nonlinear}. The parameter $h_1$ controls the strength of conditional dependence under $H_1$, with larger values corresponding to stronger departures from conditional independence. Specifically, we consider
\begin{align} 
    & K \in \{2,5,10,20 \}, \notag \\ 
    & \text{Near sources: Far sources} \in \{ 10:0, 7:3, 3:7, 0:10 \}, \notag \\
    & n_0=n_1=\cdots=n_K\in \{ 500, 1000, 1500, 2000\}, \notag \\ 
    & n_1=\cdots=n_K=500, \quad n_0 \in \{ 400, 500, 1000, 1500\}, \notag \\ 
    & n_0=500, \quad n_1=\cdots=n_K \in \{ 500, 1000, 2000\}, \notag \\ 
    & h_1\in \{ 0.05, 0.1, 0.2, 0.3, 0.5\}. \notag 
\end{align}



\begin{figure}[htb!]
    \centering
    \includegraphics[width=0.8\linewidth]
    {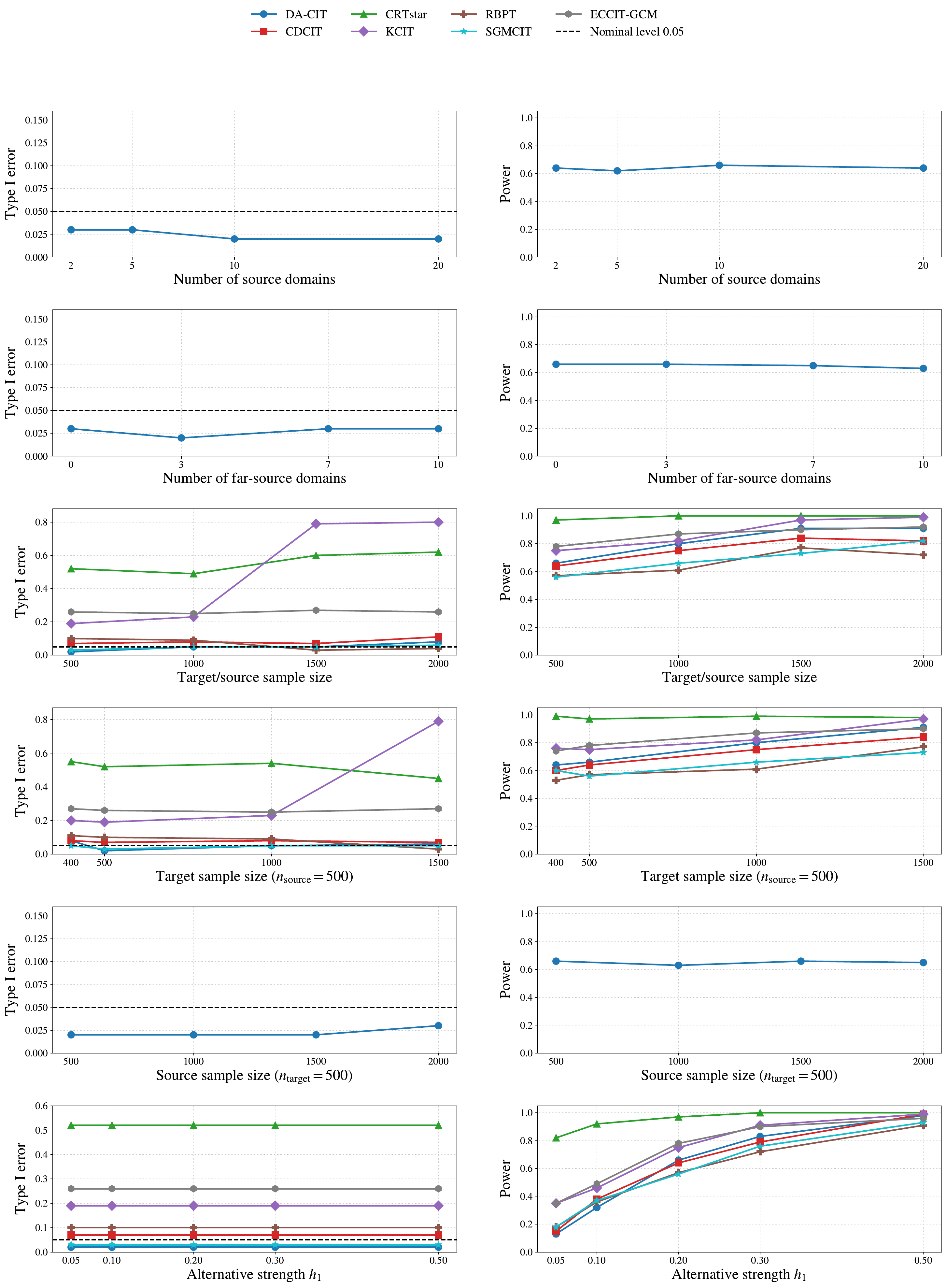}
    \caption{
    Sensitivity results of DA-CIT on the post-nonlinear data with
    $d_z=20$. We vary the number of source domains, the proportion of
    near and far sources, the target and source sample sizes, and the
    alternative strength $h_1$. 
    }
    \label{fig_cit_ablation}
\end{figure}

\begin{figure}[htb!]
    \centering
    \includegraphics[width=0.8\linewidth]
    {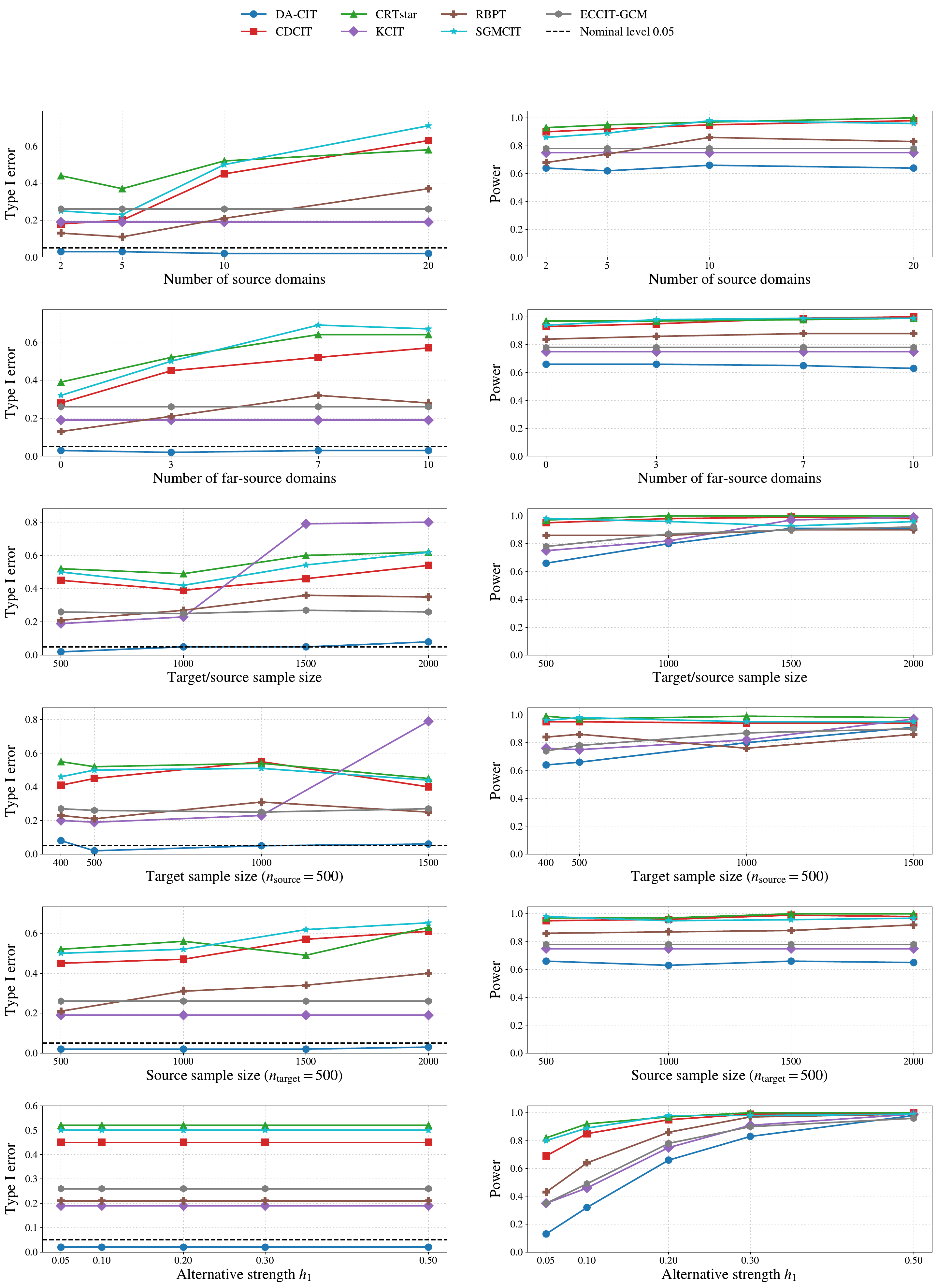}
    \caption{
    Ablation results on the post-nonlinear data with $d_z=20$ under
    the Cat-all setting. We vary the number of source domains,
    the proportion of near and far sources, the target and source sample
    sizes, and the alternative strength $h_1$.
    }
    \label{fig_cit_ablation_concatenate_all}
\end{figure}

Figure~\ref{fig_cit_ablation} summarizes the sensitivity results. Overall, DA-CIT remains stable across changes in the number of source domains, the near--far source composition, and the target and source sample sizes. In most settings, DA-CIT maintains Type~I error at or below
the nominal level and consistently yields lower Type~I error than CDCIT. At the same time, DA-CIT achieves higher power than CDCIT and SGMCIT in most settings. As $h_1$ increases, the power of DA-CIT increases, indicating that stronger conditional dependence under $H_1$ is more readily tested. These results suggest that DA-CIT can effectively leverage heterogeneous source-domain information while maintaining reliable Type~I error control.

For comparison, we also consider a naive pooling strategy in which the target and source samples are directly concatenated before CI testing. As shown in Figure~\ref{fig_cit_ablation_concatenate_all}, this strategy generally leads to substantial Type~I error inflation for the competing methods. In contrast, DA-CIT remains substantially better calibrated,
highlighting the importance of accounting for source--target heterogeneity rather than naively pooling observations across domains.

\subsection{Empirical cumulative distribution functions of CI-test $p$-values}
\label{app_ecdf}

We additionally examine the empirical distribution of the $p$-values produced by different CI testing methods under $H_0$. As a diagnostic of null calibration, we compare the empirical cumulative distribution function (ECDF) of the $p$-values with the CDF of $U(0,1)$. We also report the $p$-value of the Kolmogorov--Smirnov (KS) test, denoted by $p_{\mathrm{KS}}$, based on the 100 $p$-values obtained across repeated experiments.

\begin{figure}[htb!]
    \centering
    \includegraphics[width=\linewidth]
    {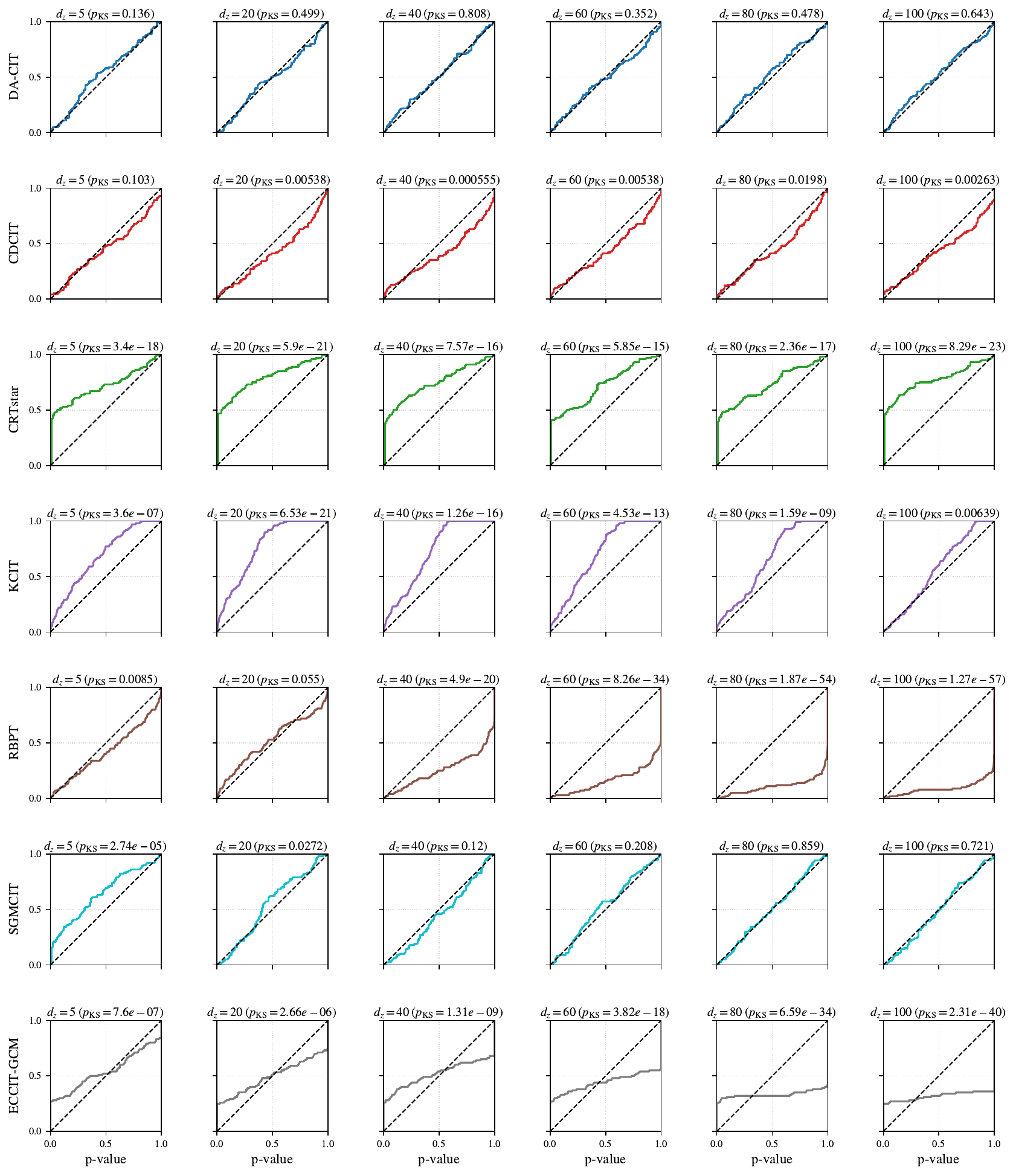}
    \caption{
    ECDFs of the $p$-values under $H_0$ for the CI testing methods
    considered in Figure~\ref{fig_cit_results_main}.
    The dashed diagonal line represents the CDF of
    $U(0,1)$, and $p_{\mathrm{KS}}$ denotes the $p$-value
    of the KS test based on 100 null $p$-values.
    }
    \label{fig_cit_pvalue_ecdf_main}
\end{figure}

\begin{figure}[htb!]
    \centering
    \includegraphics[width=0.66\linewidth]
    {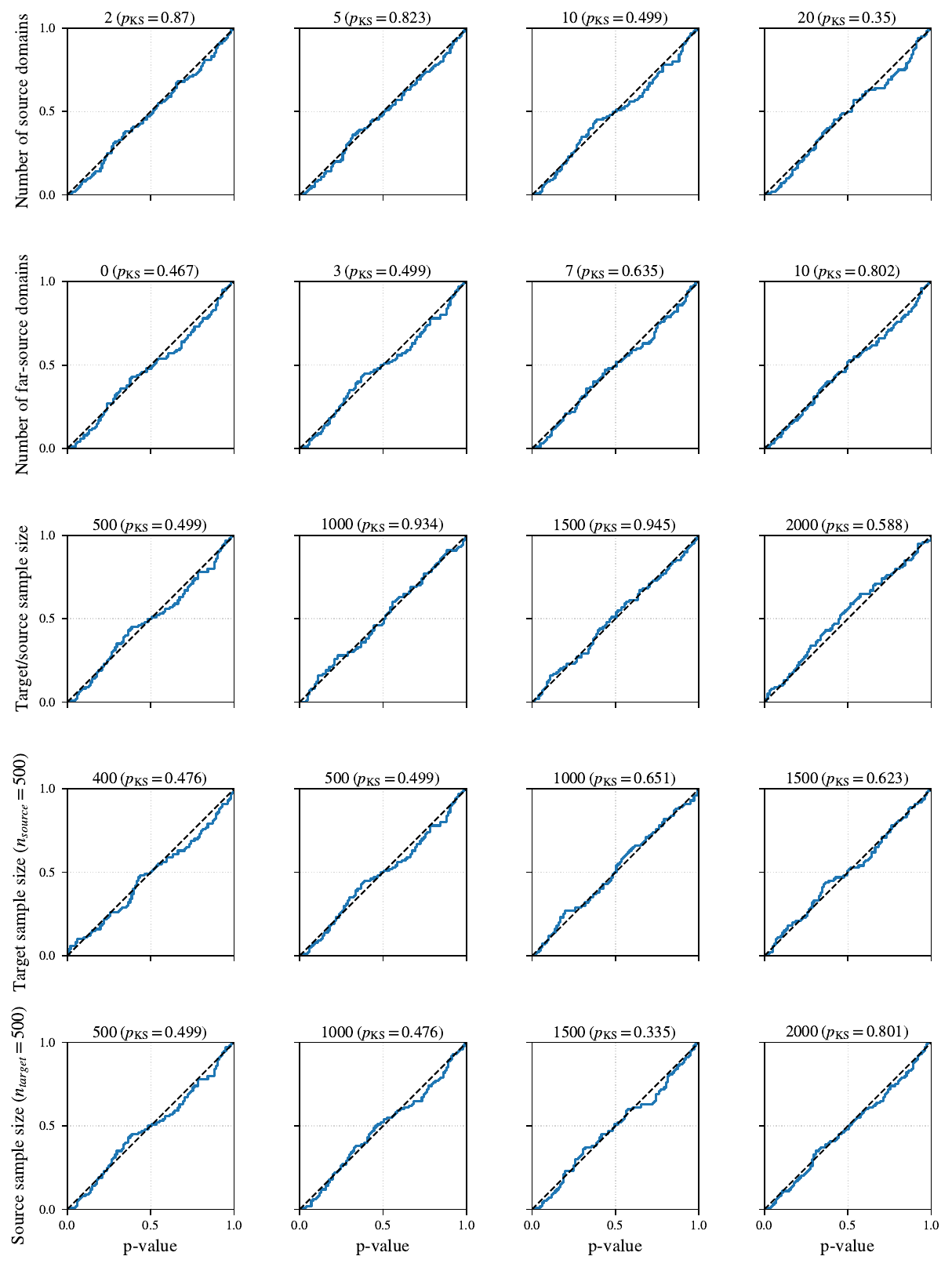}
    \caption{
    ECDFs of the DA-CIT $p$-values under $H_0$ for the ablation
    settings considered in Figures~\ref{fig_cit_ablation} and
    \ref{fig_cit_ablation_concatenate_all}, with $d_z=20$.
    The dashed diagonal line represents the CDF of
    $U(0,1)$, and $p_{\mathrm{KS}}$ denotes the $p$-value
    of the KS test based on 100 null $p$-values.
    }
    \label{fig_dacit_pvalue_ecdf_ablation}
\end{figure}

As shown in Figure~\ref{fig_cit_pvalue_ecdf_main}, the ECDFs produced by DA-CIT closely follow the CDF of $U(0,1)$ across all
considered conditioning dimensions $d_z$. Moreover, none of the corresponding KS tests rejects the null of uniformity at the $0.05$ significance level. 

Figure~\ref{fig_dacit_pvalue_ecdf_ablation} further shows that this behavior remains stable across the sensitivity settings. The ECDFs of DA-CIT remain close to the CDF of $U(0,1)$ as we vary the number of source domains, the number of far-source domains, and the target and source sample sizes. None of the KS tests rejects uniformity at the $0.05$ significance level in these settings. These results provide additional evidence that DA-CIT maintains stable null calibration across heterogeneous multi-source settings.

\section{Algorithms of DA-Diff and DA-CIT}

The algorithms of DA-Diff and DA-CIT are provided in ALgorithm \ref{algo_train_da_diff} and \ref{algo_da_cit}.

\begin{algorithm}[htb]
\caption{Training the Domain-adapted diffusion model (DA-Diff)}
\label{algo_train_da_diff}
\textbf{Input}: Dataset from the target domain $D_0 = \{(x_{0,i}, z_{0,i})\}_{i=1}^{n_0} $, datasets from $K$ source domains $D_k = \{(x_{k,i}, z_{k,i})\}_{i=1}^{n_k}$ and the corresponding pre-trained source models $\widehat{s}_k$ trained on $D_k$ (for $k = 1,\ldots, K$), the grid $\mathcal{G}_M$ of $M$, the grid $\mathcal{G}_{\rho}$ of $\rho$, the number of the warm-up epochs $e_w$, the weight-update interval \texttt{update\_every}, the total number all epochs $E$. \\
\textbf{Output}: A neural network $\widehat{s}_w$ trained by weighted ERM.
\begin{algorithmic}[1]
    \STATE Initialize a neural network $\widehat{s}_w(x,z,t)$
    \FOR{$e$ in $1,\ldots,E$}
        \STATE {\footnotesize\itshape(warm-up phase)}
        \IF{$e\leq e_w$}
            \STATE Compute $\widehat{\mathcal{L}}_0 (\widehat{s}_w) = \frac{1}{n_0} \sum_{i=1}^{n_0} \ell(x_{0,i},z_{0,i};\widehat{s}_w)$ 
            \STATE Take optimization step on $\nabla \widehat{\mathcal{L}}_0 (\widehat{s}_w) $ and update the parameters of $\widehat{s}_w$
            \STATE \hspace*{-\algorithmicindent}{\footnotesize\itshape(weighted ERM phase)}
        \ELSE
            \IF{$(e-e_w-1)\bmod \texttt{update\_every}=0$}
                \STATE Calculate $\widehat G_{ij}$ for $i,j = 1,\ldots, K$ by (\ref{eq_G_Hij_def}) and (\ref{eq_G_Hij_emp_pop}) in parallel and obtain $\widehat G$
                \STATE For each $(\rho, C_{\mathrm{reg}}, \boldsymbol{\lambda}) \in \mathcal{G}_{\rho} \times \mathcal G_C \times \mathcal G_\lambda$, iteratively solve $\widehat{\boldsymbol{\pi}}_{j,l,m}$ using the updates in (\ref{eq_final_pi_mm_qp}), executed in parallel
                \STATE Select $(\rho^*, C_{\mathrm{reg}}^*, \boldsymbol{\lambda}^*)$ that maxmizes the objective in (\ref{eq_select_rho_linear_score}) and obtain the corresponding $\boldsymbol{\pi}^*$
                \STATE Solve $w_1^*, \ldots, w_K^*$ and $W_N^*$ from  $\boldsymbol{\pi}^*$ by (\ref{eq_recover_w_from_pi})
            \ENDIF
            \STATE Compute $\widehat{\mathcal{L}}_{w}(\widehat{s}_w) = \left[ \sum_{i=1}^{n_0} \ell(x_{0,i},z_{0,i};\widehat{s}_w)  + \sum_{k=1}^{K} w_k^*  \sum_{i=1}^{n_k} \ell(x_{k,i},z_{k,i};\widehat{s}_w) \right] / (n_0 + W_N^*)$
            \STATE Take optimization step on $\nabla \widehat{\mathcal{L}}_{w}(\widehat{s}_w) $ and update the parameters of $\widehat{s}_w$
        
        \ENDIF
    \ENDFOR

    \STATE \textbf{Return} $\widehat{s}_w$
\end{algorithmic}
\end{algorithm}


\begin{algorithm}[htb]
\caption{Domain-adapted conditional independence test (DA-CIT)}
\label{algo_da_cit}
\textbf{Input}: Target dataset $D_0 = \{(x_{0,i}, y_{0,i}, z_{0,i})\}_{i=1}^{n_0}$, trained neural network $\widehat{s}_w$, the number of repetitions $B$.\\
\textbf{Output}:  The p-value $p$.
\begin{algorithmic}[1] 
\STATE Let $\mathbf{X}_0 = (x_{0,1}, \ldots, x_{0,n_0})^\top, \mathbf{Y}_0 = (y_{0,1}, \ldots, y_{0,n_0})^\top, \mathbf{Z}_0 = (z_{0,1}, \ldots, z_{0,n_0})^\top$
\STATE Compute $\mathbb T = \mathbb T(\mathbf{X}_0, \mathbf{Y}_0, \mathbf{Z}_0)$ 
\STATE Concatenate $\mathbf{Z}_0$ for $B$ times and use the reverse process in (\ref{eq_diff_reverse}) with the estimated neural network $\widehat{s}_w$ to generate $\mathbf{X}_0^{(b)}$ for $b=1,\ldots,B$ in one shot
\STATE Compute $\mathbb T^{(b)} = \mathbb T(\mathbf{X}_0^{(b)}, \mathbf{Y}_0, \mathbf{Z}_0)$ for $b=1,\ldots,B$ in parallel
\STATE Compute $p$-value: 
\begin{equation*}
    p=
\frac{
1+\sum_{b=1}^B \mathbf{1}\!\left\{
\mathbb T^{(b)} \ge \mathbb T 
\right\}
}{
1+B
}.
\end{equation*}
\STATE \textbf{Return} $p$
\end{algorithmic}
\end{algorithm}

\section{Training behavior of DA-Diff and MM-QP}
\label{sec_train_behavior}

Although our theoretical results and weighted ERM formulation are based on the entire diffusion-time interval $t\in[t_0,T]$, in practice, we can partition $[t_0,T]$ into $B_T$ time bins and perform weighted ERM separately within each bin. The intuition behind this strategy is that the degree of distribution shift across domains varies with the diffusion time. At $t=0$, the distribution shifts among different source domains are fully preserved. As the forward diffusion process progressively adds noise to the variables, these shifts are gradually attenuated, and as $t\to\infty$, the distributions across domains converge to the same limiting distribution, so that the domain shifts asymptotically vanish. Therefore, estimating separate weights $w^*(t)$ for different time bins and performing weighted ERM accordingly can better account for the time-varying transferability of the source domains than using a single set of weights over the entire diffusion-time interval. Ablation studies on the number of time bins are provided in Figure \ref{fig_da_diff_ablation}, Appendix~\ref{app_ablation}. 

We next examine the training behavior of DA-Diff trained using the MM-QP procedure proposed in Appendix~\ref{sec_optimal_weights} on the Gaussian mixture data introduced in Appendix~\ref{app_gmm}. For training stability, we consider only three source domains and set the sample sizes of the target and source domains to $
n_0=n_1=n_2=n_3=5000.
$ The shift strengths $d^{(k)}$ defined in Appendix~\ref{app_gmm} are set to $
d^{(1)}=0.01,
d^{(2)}=0.1,
d^{(3)}=0.2$. The histograms of the target and source data are shown in Figure~\ref{fig_gmm_hist}.

\begin{figure}[htb!]
    \centering
    \begin{subfigure}[t]{0.48\linewidth}
        \centering
        \includegraphics[width=\linewidth]{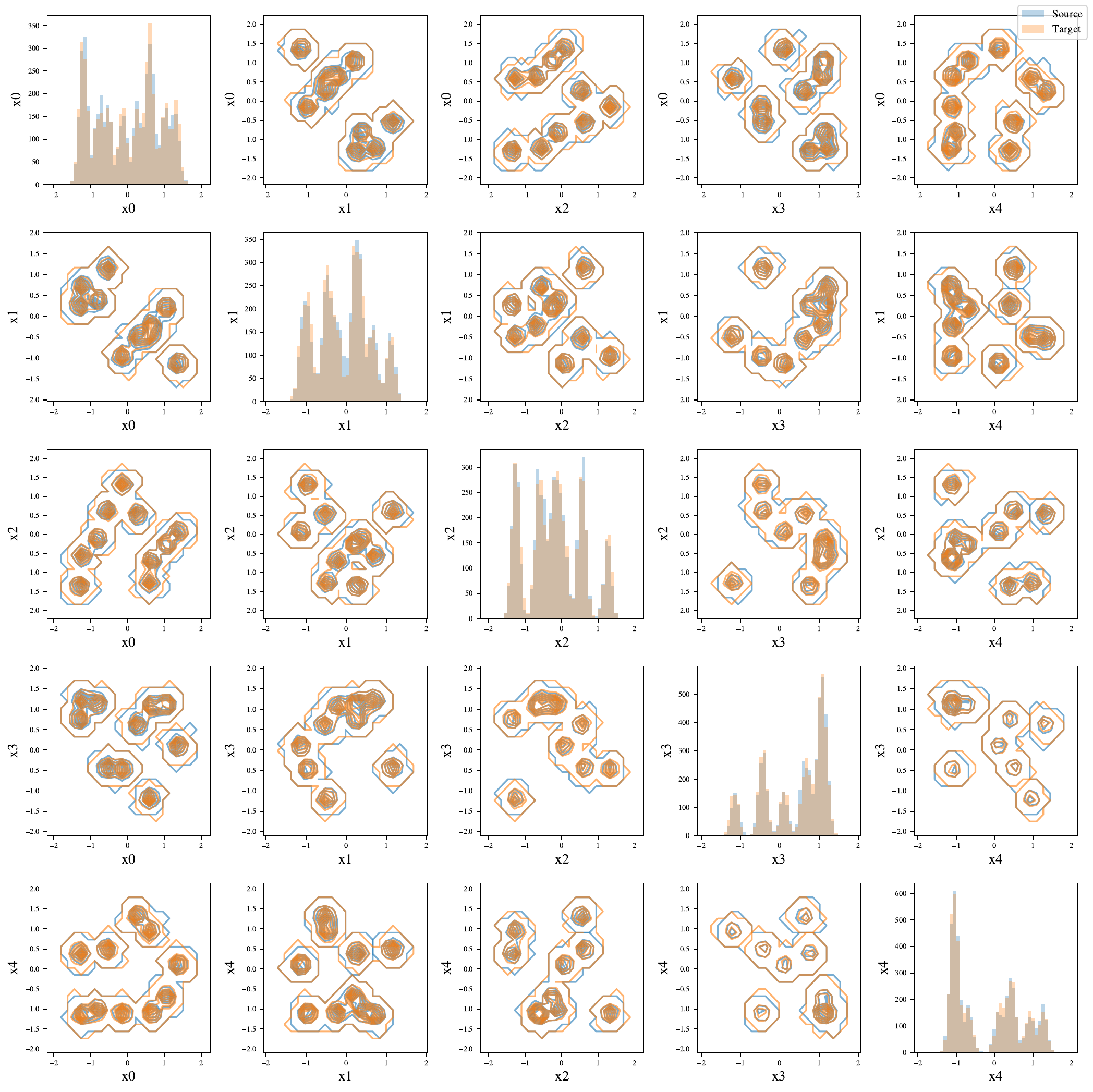}
    \end{subfigure}
    \hfill
    \begin{subfigure}[t]{0.48\linewidth}
        \centering
        \includegraphics[width=\linewidth]{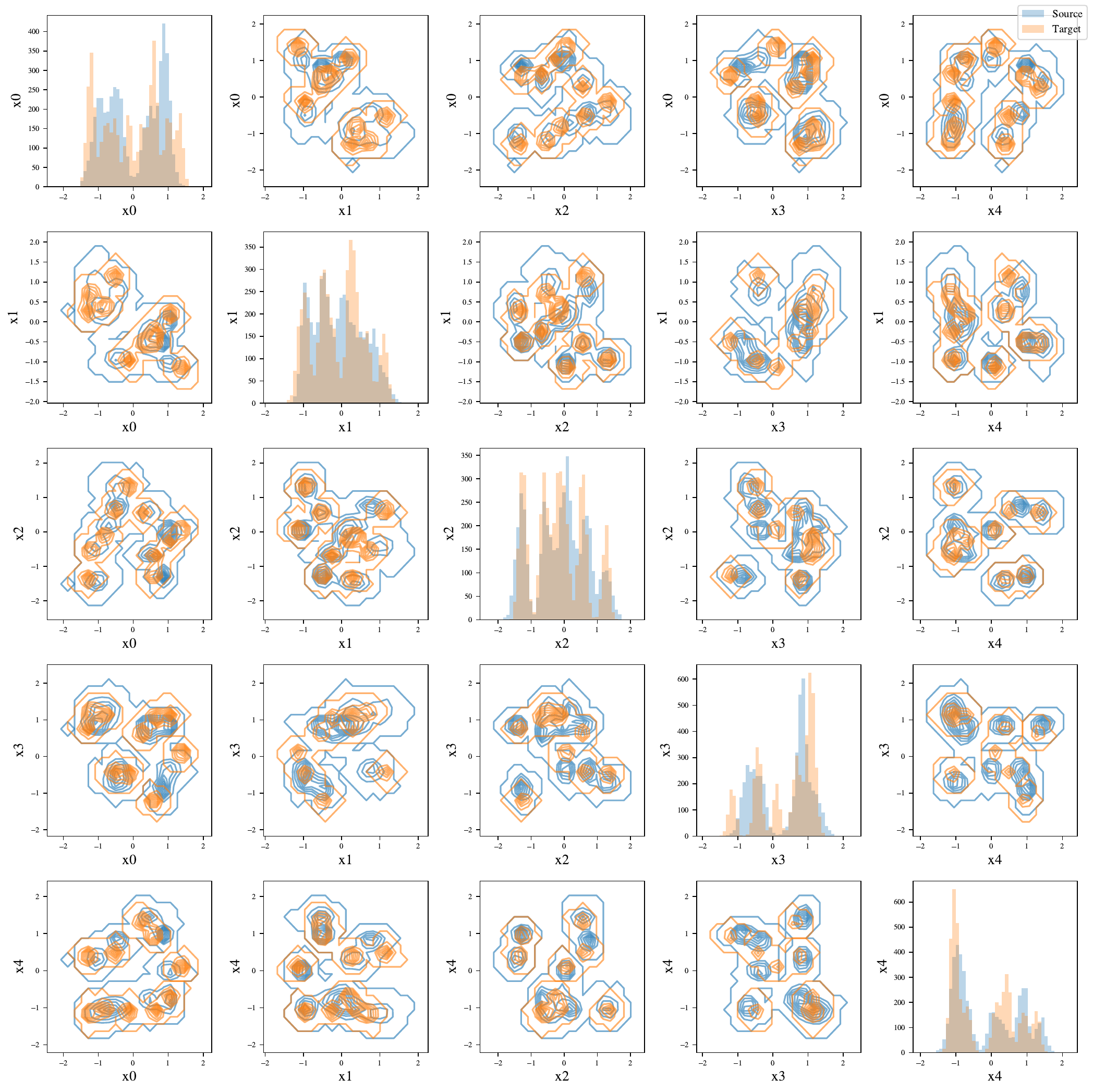}
    \end{subfigure}
    \caption{Histograms of the target and source variable $X$ in the Gaussian mixture experiment.
    The orange histograms represent the target-domain data, while the blue histograms represent the source-domain data.
    The left panel compares the target domain with the first source domain, whose shift strength is $d^{(1)}=0.01$, whereas the right panel compares the target domain with the third source domain, whose shift strength is $d^{(3)}=0.2$.
    As the shift strength increases to $d^{(3)}=0.2$, a clear distributional discrepancy between the target and source domains can be observed.
}
    \label{fig_gmm_hist}
\end{figure}

We train DA-Diff according to Algorithm~\ref{algo_train_da_diff} and report the validation losses of DA-Diff and the target-only diffusion model in Figure~\ref{fig_val_loss_da_tar}.  We set the warm-up epoch and the total epoch number to 100 and 500, respectively. We dynamically update the weights of the weighted ERM and hyperparameters every 10 epochs by MM-QP proposed in Appendix \ref{sec_optimal_weights}. 

As shown in Figure~\ref{fig_val_loss_da_tar}, the validation-loss curve of DA-Diff exhibits a similar level of smoothness to that of the target-only diffusion model.

We next visualize the weights obtained by our optimization procedure.
Figure~\ref{fig_pi_w_WN} shows the solution $\boldsymbol{\pi}$ obtained by the MM-QP procedure in Appendix~\ref{sec_optimal_weights}, together with the corresponding source weights $w_1,\ldots,w_K$ and $W_N$ recovered using~(\ref{eq_recover_w_from_pi}).
The hyperparameters $\rho$, $C_{\mathrm{reg}}$, and $\boldsymbol{\lambda}$ selected by the gradient-alignment criterion in Appendix~\ref{subsec_select_Gamma} are shown in Figure~\ref{fig_hyper_param}.
We use the candidate grids
\[
\mathcal{G}_{\rho}
=
\{0.2,0.4,0.6,0.8\},
\qquad
\mathcal{G}_{C}
=
\{2^1,\ldots,2^7\},
\]
and set $M_{\lambda}=5$ for the construction of $\mathcal{G}_{\lambda}$.

As shown in Figure~\ref{fig_pi_w_WN}, the trajectories of $\boldsymbol{\pi}$, $w$, and $W_N$ exhibit a zig-zag pattern over epochs. This is expected, since the hyperparameters, $\boldsymbol{\pi}$, $w$, and $W_N$ are dynamically updated every 10 epochs. Importantly, Figure~\ref{fig_val_loss_da_tar} shows that these periodic updates do not noticeably increase the fluctuations of the validation loss, and the overall training process remains stable.

From the left column of Figure~\ref{fig_pi_w_WN}, we observe that the mixture distribution used in weighted ERM consistently assigns the largest proportion to the target domain. Moreover, the MM-QP procedure is able to reflect the degree of distribution shift between the source and target domains and assign the source-domain weights accordingly. The behavior of $\boldsymbol{\pi}$ across different time bins is also consistent with our motivation for partitioning the diffusion-time interval: when the diffusion time is close to $t_0$, the distributional discrepancies across domains are relatively large, whereas these discrepancies become weaker as the diffusion time approaches $T$. In contrast, Figure~\ref{fig_uowq_weights} shows that the weights obtained by UOWQ \citep{zhang_2026_uowq} for diffusion-model training exhibit little differentiation across source domains, suggesting that it has difficulty identifying which sources should contribute more and which should be downweighted. This behavior is also reflected in Table~\ref{tab_gaussian_baselines_w2}, where the performance of UOWQ deteriorates as the number of source domains increases, indicating increasingly severe negative transfer. We emphasize, however, that UOWQ was originally developed for regular parametric models rather than specifically for diffusion models, so its assumptions and weighting mechanism are not directly tailored to the score-matching setting considered here.

The source weights $w$ shown in the middle column of Figure~\ref{fig_pi_w_WN} exhibit the same pattern. The right column reports the total weighted source sample size $W_N$ for each time bin. From the scale of the $y$-axis, we can see that relatively few source samples are effectively used in the early time bins, while substantially more source information is incorporated in the later time bins. This is consistent with the fact that the forward diffusion process progressively reduces the distributional discrepancies across domains.

Figure~\ref{fig_hyper_param} shows that the hyperparameters selected by the
gradient-alignment criterion vary throughout the training process.
In particular, $\rho$ switches among several candidate values as training
progresses, while $C_{\mathrm{reg}}$ remains relatively small for most epochs
but occasionally takes substantially larger values.
The three components of $\boldsymbol{\lambda}$ also change across epochs,
indicating that the relative importance assigned to the corresponding terms
in the selection criterion is adjusted dynamically during training.
These results suggest that using a single fixed set of hyperparameters
throughout training may be overly restrictive, whereas the proposed
gradient-alignment procedure can adapt the hyperparameters to the current
state of the model. Such changes in the selected hyperparameters also help explain the
piecewise-varying behavior of $\boldsymbol{\pi}$, $w$, and $W_N$ observed in
Figure~\ref{fig_pi_w_WN}.

\begin{figure}[htb]
    \centering
    \includegraphics[width=0.65\linewidth]{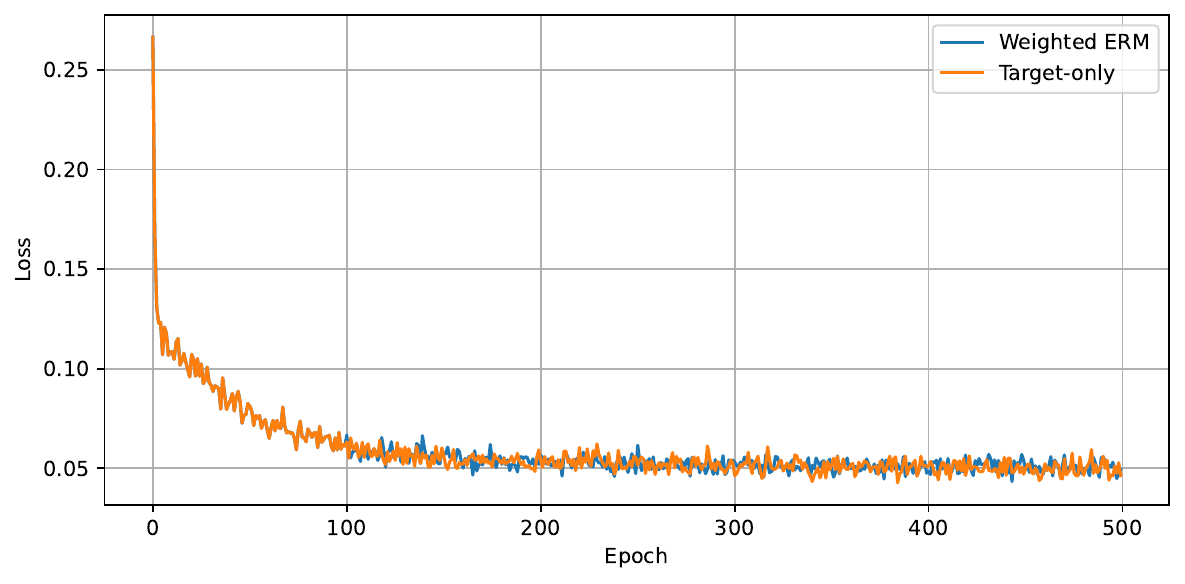}
    \caption{
    Validation loss curves of DA-Diff and the target-only diffusion model during training.
    }
    \label{fig_val_loss_da_tar}
\end{figure}

\begin{figure}[htb]
    \centering
    \includegraphics[width=0.95\linewidth]{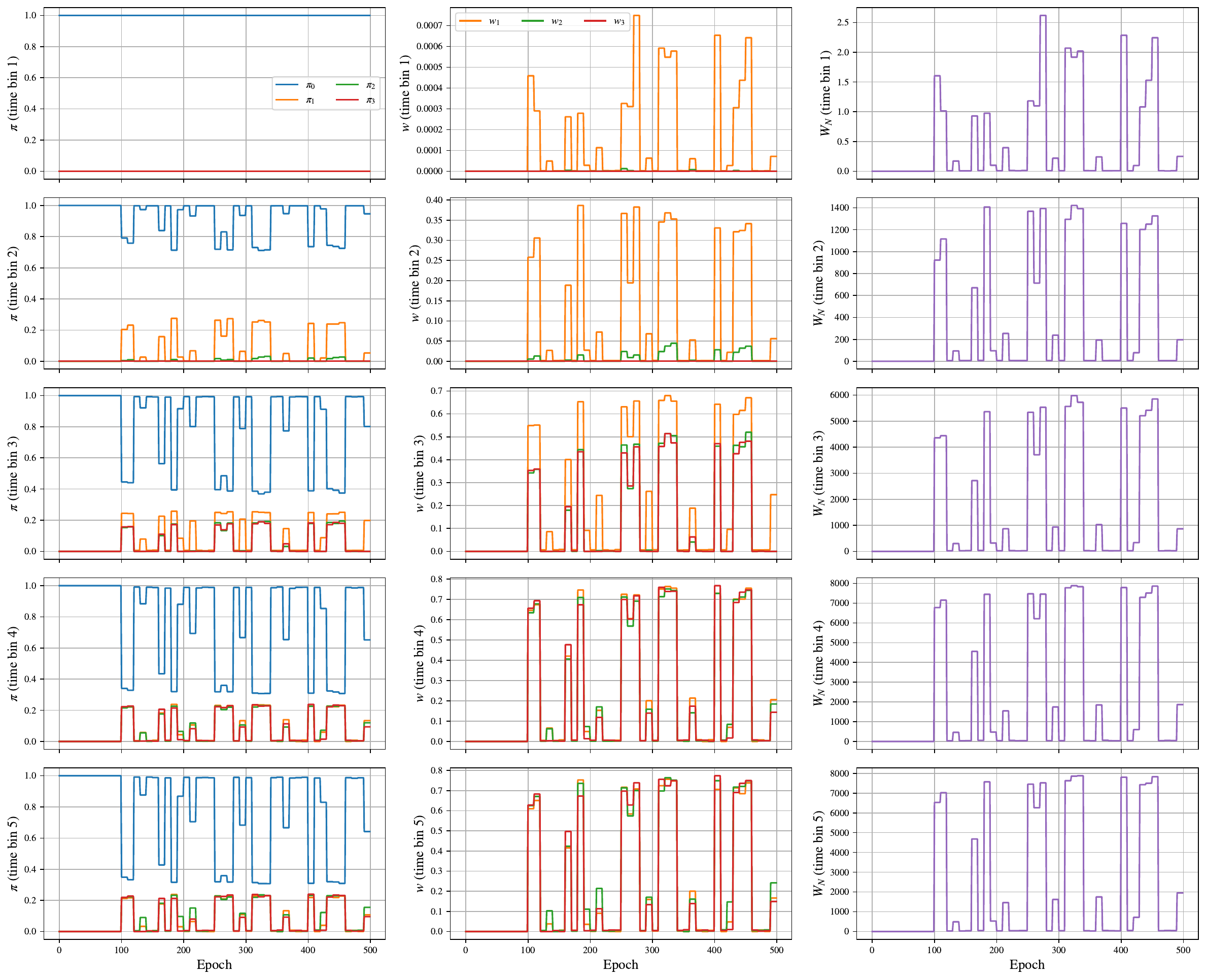}
    \caption{
    Evolution of the MM-QP solution $\boldsymbol{\pi}$ and the corresponding
    $w_1,\ldots,w_K$ and $W_N$ recovered by~(\ref{eq_recover_w_from_pi})
    during the training of DA-Diff.
    }
    \label{fig_pi_w_WN}
\end{figure}

\begin{figure}[htb]
    \centering
    \includegraphics[width=0.95\linewidth]{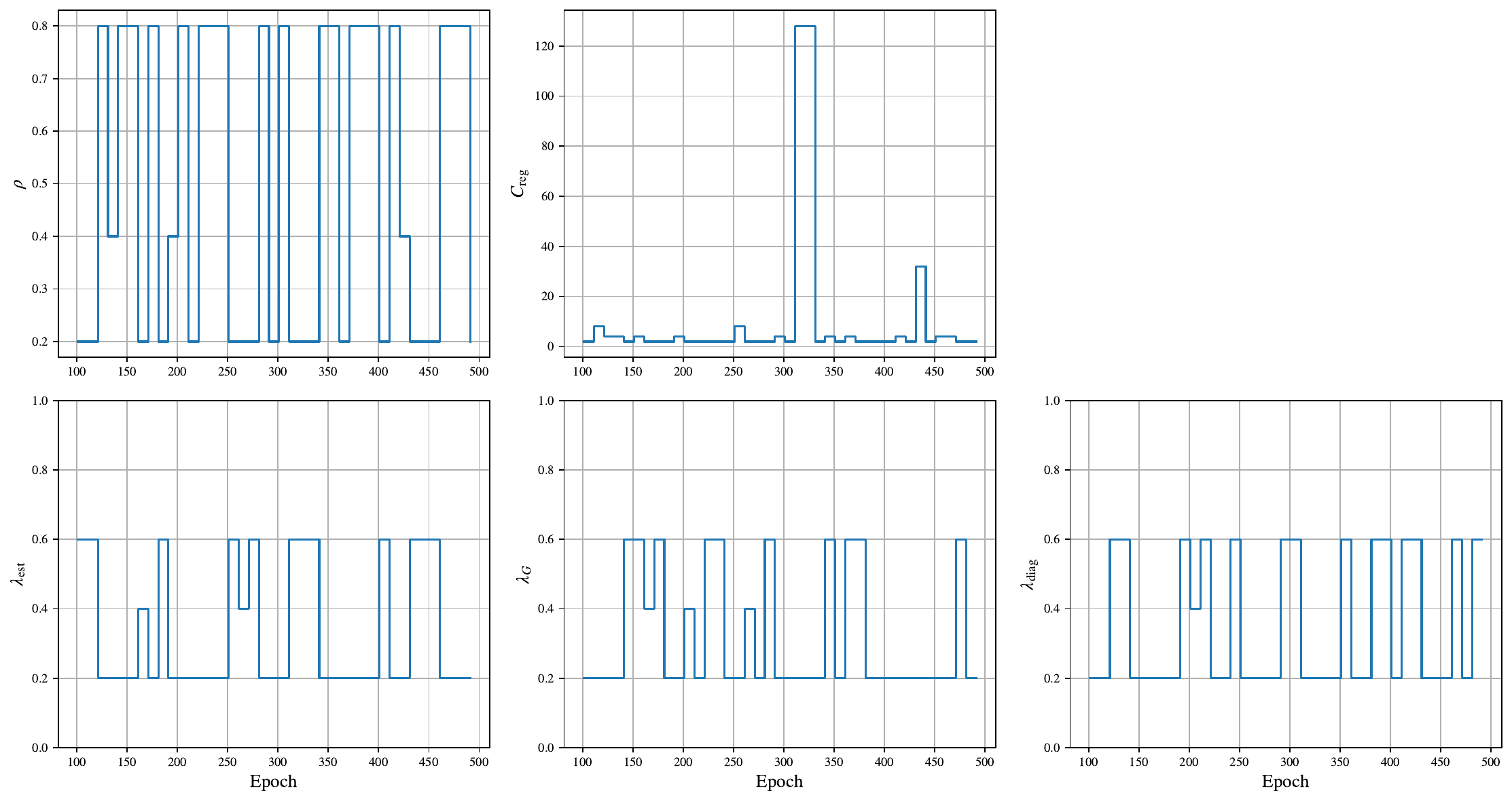}
    \caption{
    Hyperparameters $\rho$, $C_{\mathrm{reg}}$, and $\boldsymbol{\lambda}$
    selected by the gradient-alignment criterion during the training of DA-Diff.
    }
    \label{fig_hyper_param}
\end{figure}

\begin{figure}[t]
    \centering
    \begin{subfigure}[t]{0.32\linewidth}
        \centering
        \includegraphics[width=\linewidth]{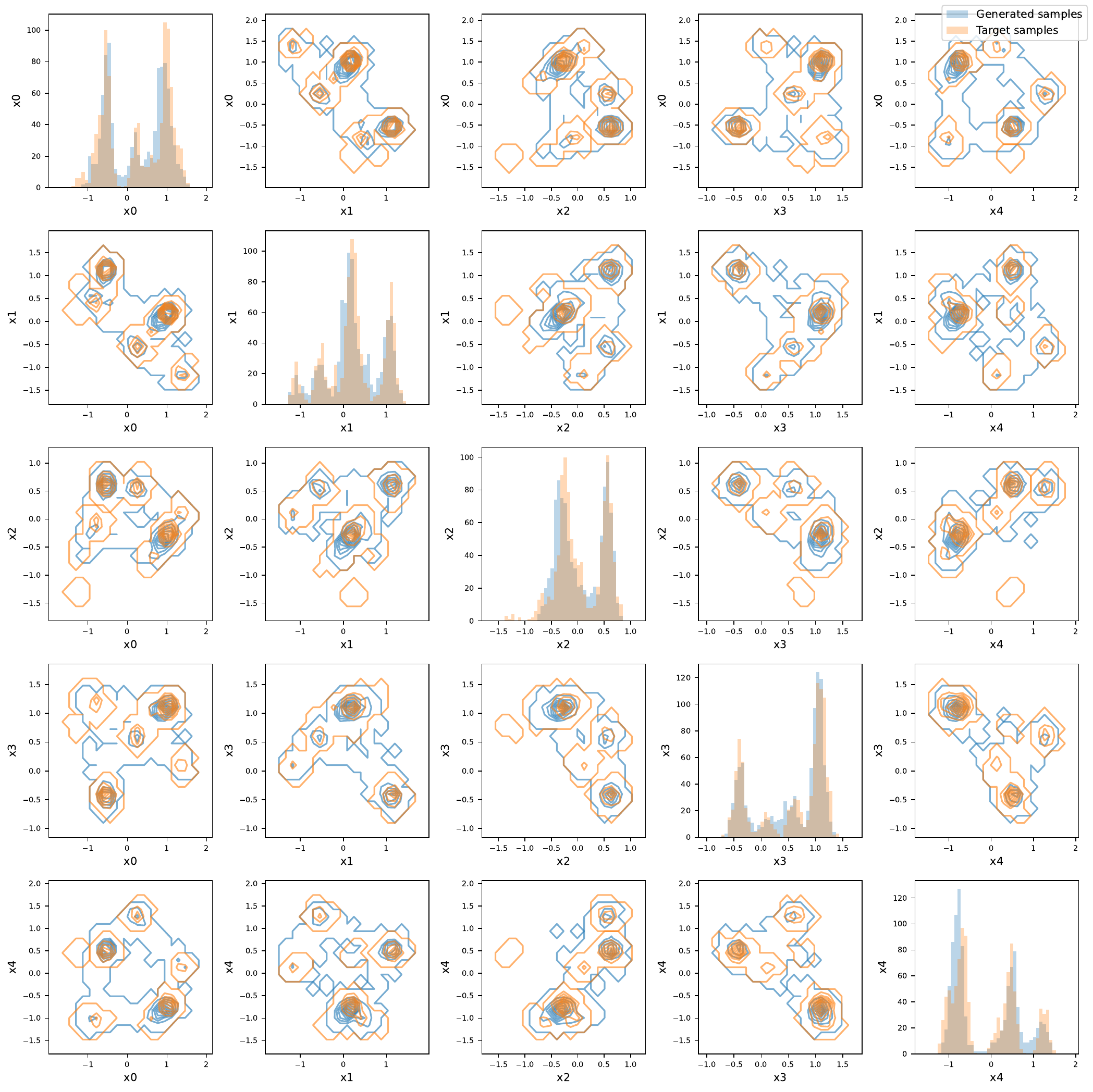}
        \caption*{Target only}
    \end{subfigure}
    \hfill
    \begin{subfigure}[t]{0.32\linewidth}
        \centering
        \includegraphics[width=\linewidth]{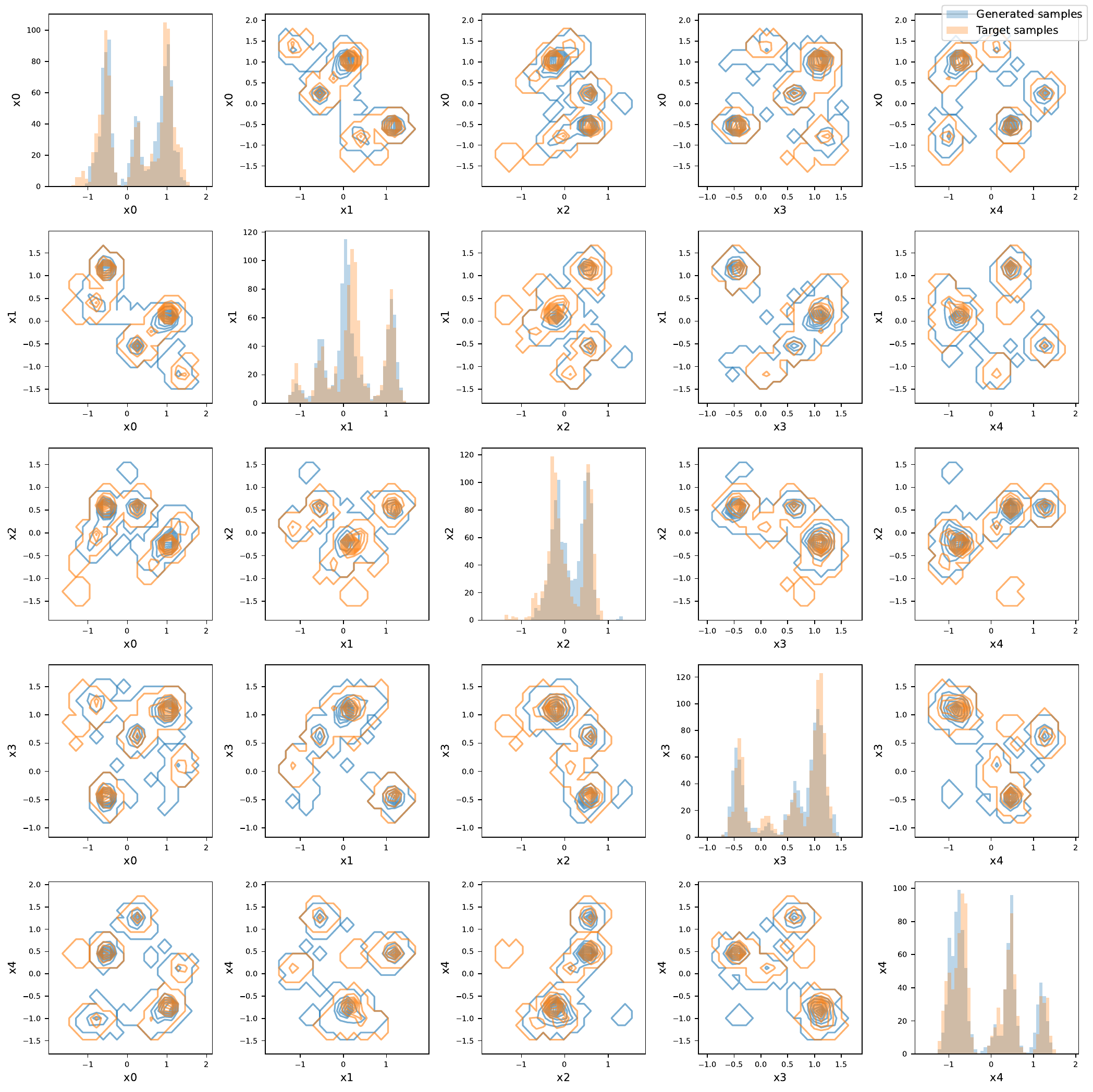}
        \caption*{Cat-all}
    \end{subfigure}
    \hfill
    \begin{subfigure}[t]{0.32\linewidth}
        \centering
        \includegraphics[width=\linewidth]{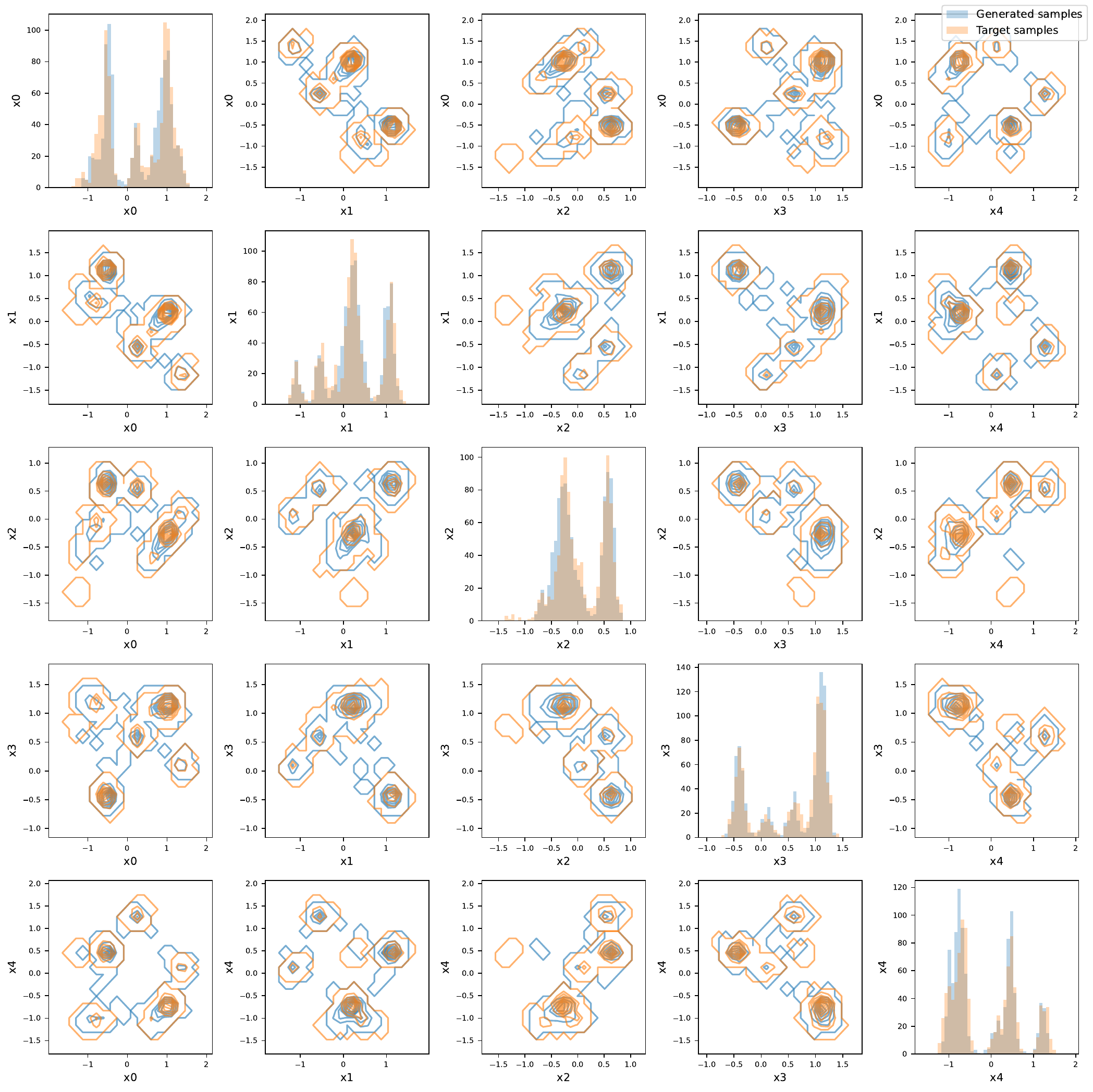}
        \caption*{DA-Diff}
    \end{subfigure}
    \caption{Generated samples in the Gaussian mixture experiment.
    From left to right, the methods are Target only, Cat-all, and DA-Diff.
    Their corresponding $W_2$ distances are $0.040$, $0.047$, and $0.023$, respectively.}
    \label{fig_gmm_generated_samples}
\end{figure}

\begin{figure}[t]
    \centering
    \includegraphics[width=0.95\linewidth]{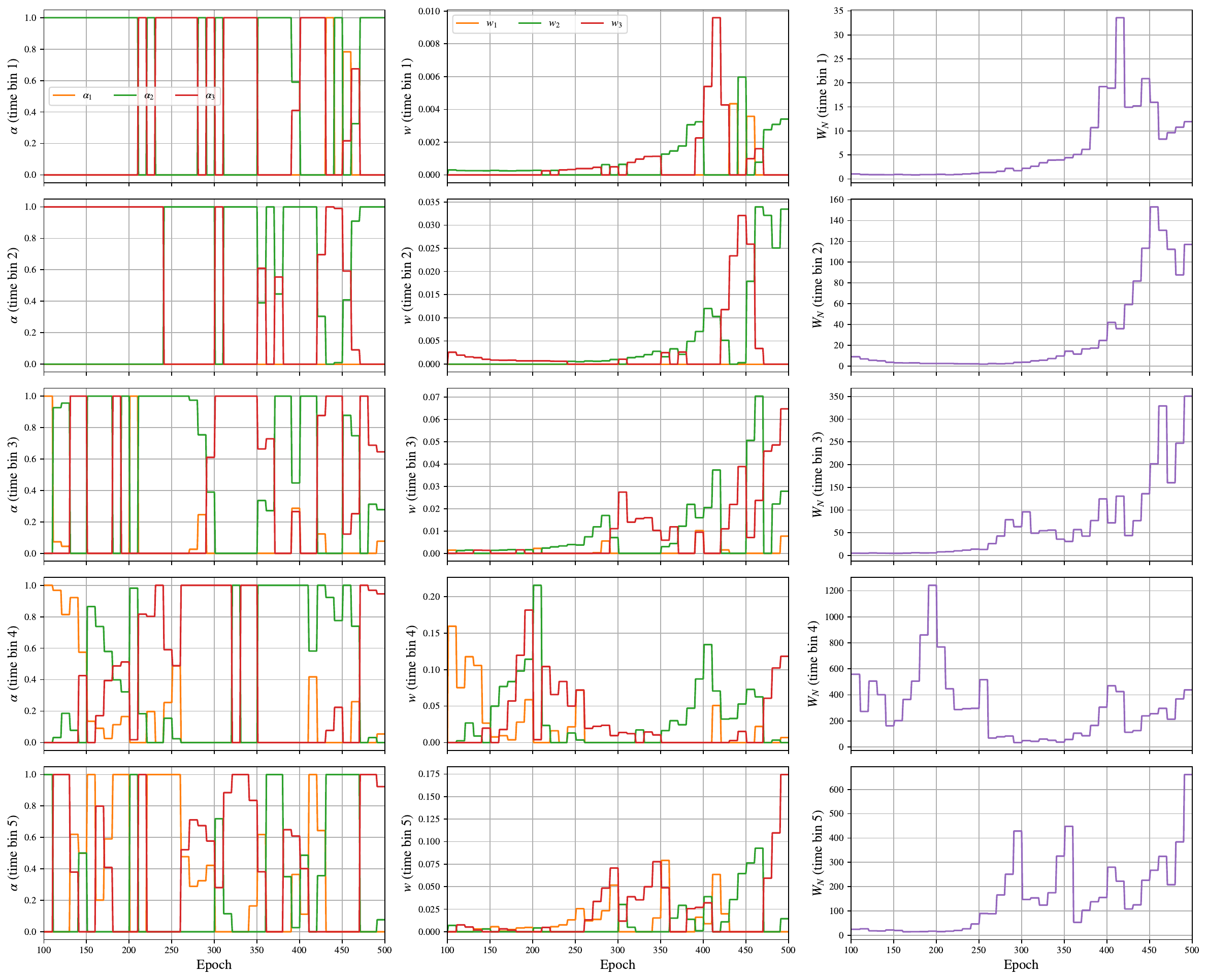}
    \caption{
    Source proportions $\alpha_1, \ldots, \alpha_K$, source weights $w_1,\ldots,w_K$, and the total weighted source sample size $W_N$ solved by UOWQ \citep{zhang_2026_uowq}.
    }
    \label{fig_uowq_weights}
\end{figure}

\section{Computational Efficiency}

We additionally provide a detailed analysis of the wall-clock runtime of
the competing CI testing methods and DA-CIT.

\begin{figure}[htb!]
    \centering
    \includegraphics[width=0.8\linewidth]
    {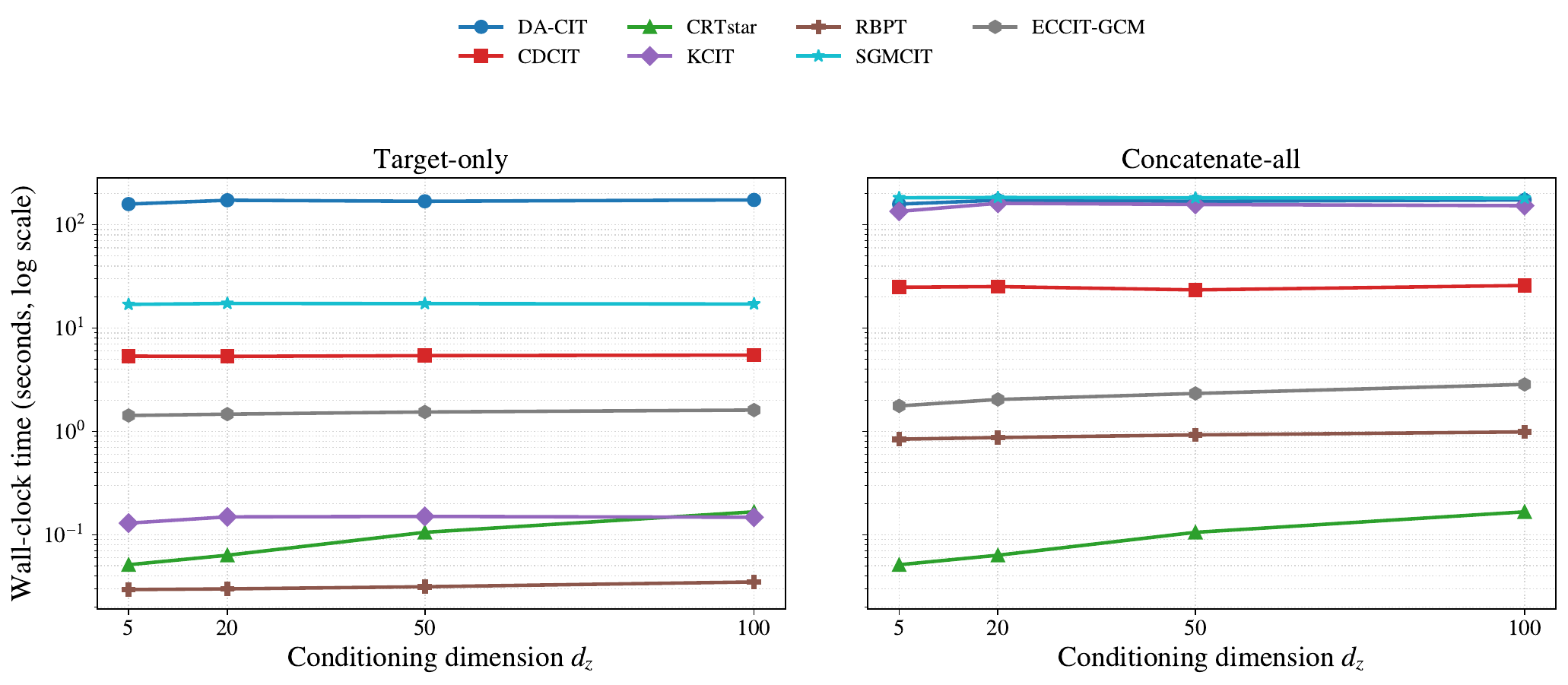}
    \caption{
    Wall-clock runtime of the CI testing methods considered in
    Figure~\ref{fig_cit_results_main} across different conditioning
    dimensions $d_z$. We report the runtime under both the Target-only
    and Concatenate-all settings. The vertical axis is shown on a
    logarithmic scale.
    }
    \label{fig_cit_wall_clock}
\end{figure}

\begin{figure}[htb!]
    \centering
    \includegraphics[width=\linewidth]
    {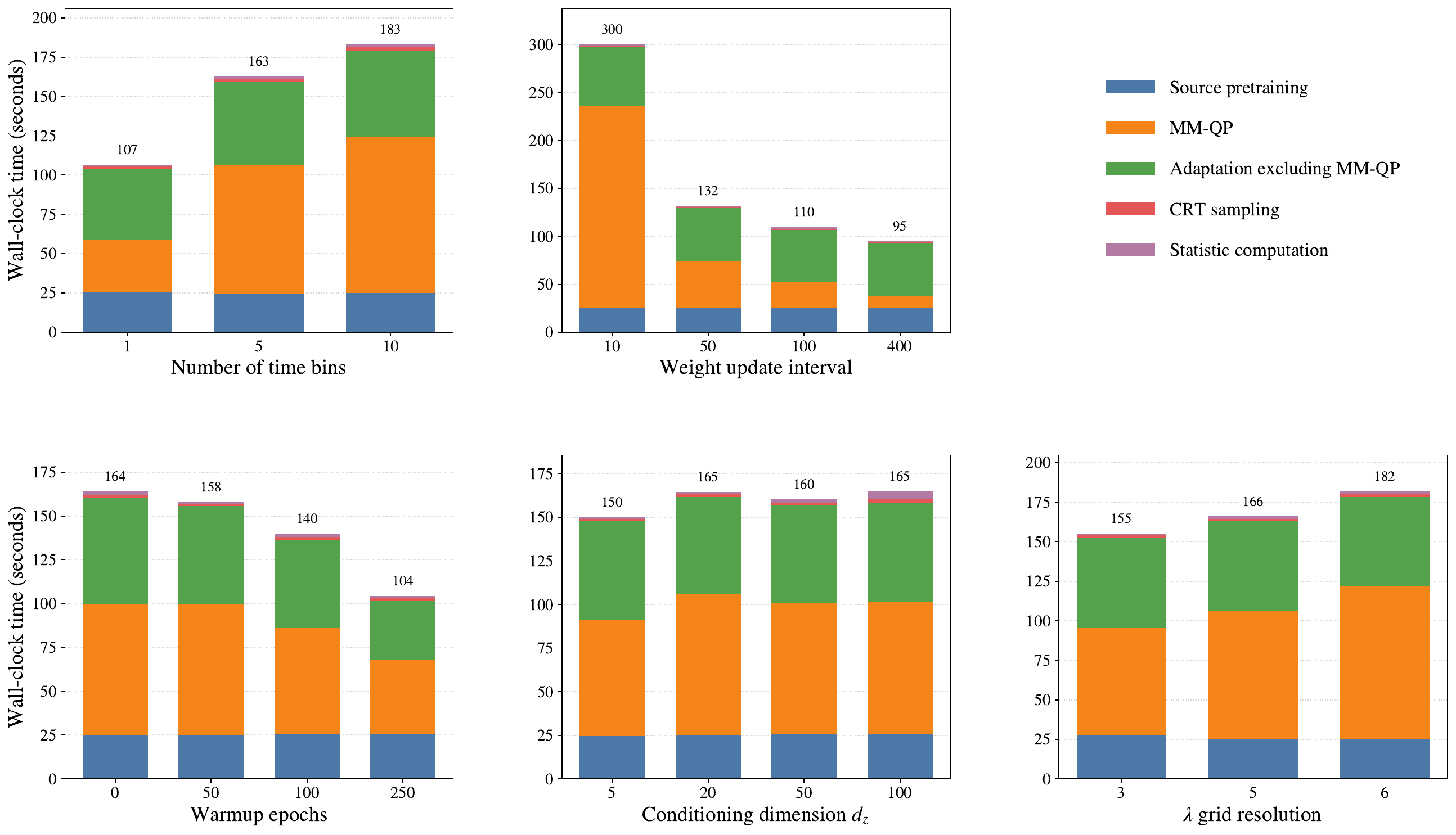}
    \caption{
    Detailed runtime decomposition of DA-CIT under the ablation settings
    considered in Figures~\ref{fig_cit_ablation} and
    \ref{fig_cit_ablation_concatenate_all}. The total runtime is
    decomposed into source pretraining, MM-QP optimization, adaptation
    excluding MM-QP, CRT sampling, and statistic computation.
    }
    \label{fig_dacit_runtime_ablation}
\end{figure}

As shown in Figure~\ref{fig_dacit_runtime_ablation}, the computational
cost of DA-CIT consists of five main components: source-model
pretraining, MM-QP optimization for the domain weights, adaptation
excluding MM-QP, CRT sampling, and statistic computation. Among these,
the adaptation stage and MM-QP optimization account for most of the
runtime, while CRT sampling and statistic computation contribute only a
small fraction. In particular, the MM-QP cost increases when weight
optimization is performed more frequently, when more time bins are used,
or when a finer $\lambda$ grid is adopted.

Although DA-CIT achieves strong CI testing performance, this improvement
comes at the cost of substantially higher computational time, as shown
in Figure~\ref{fig_cit_wall_clock}. Compared with CDCIT, the additional
computational overhead of DA-CIT mainly arises from pretraining the
source-specific score models and repeatedly solving the MM-QP problems
during weighted-ERM training. Therefore, improving the computational
efficiency of source pretraining and weight optimization is an important
direction for future work.

\section{Implementation Details}
\label{sec_implementation_details}

Unless otherwise specified, we use the following implementation for
DA-Diff throughout the experiments.
The score network is a fully connected neural network with hidden
dimension 256 and three hidden layers. 

For each source domain, we first pretrain a source-specific score network
for 500 epochs with learning rate $10^{-3}$. The domain-adaptation stage
is then trained for another 500 epochs with learning rate $10^{-3}$,
with a warm-up period of 100 epochs. We use five batches per epoch.

To allow the transferability of source domains to vary across diffusion
time, we divide the diffusion interval into 5 time bins and learn
separate domain weights for each bin. The weights are updated every
10 epochs. The hyperparameters used for weight optimization are also
re-selected every 100 epochs using the target validation set. We use
30\% of the target training data for validation.

For hyperparameter selection, we fix
\[
    \rho\in\{0.2,0.4,0.6,0.8\},
    \qquad
    C_{\mathrm{reg}}\in\{2,4,8,16,32,64,128\},
\]
and discretize the $\lambda$ search space with a grid resolution of 5.
The minimum target-domain mixture weight is set to
$\epsilon_{\pi}=10^{-6}$.

We solve the weight optimization problem using the MM-QP procedure
described in Appendix~\ref{sec_optimal_weights}. For each optimization,
we use both the pooling and target-only solutions as initializations and
retain the solution with the smaller objective value. The maximum number
of MM iterations is 50 with tolerance $10^{-6}$. Each convex QP
subproblem is solved with at most 2000 iterations and tolerance
$10^{-9}$.

All experiments were conducted on a machine equipped with an NVIDIA RTX A6000 GPU and an Intel(R) Xeon(R) Gold 6430 CPU with 128 cores.

\section{Proof of Theorems}
\label{sec_proof}

Although the first $n_0$ observations in $D_0$ are designated as the samples used for the CRT in~(\ref{eq_diffusion_loss_empirical}), the target-domain observations are i.i.d., so using the first $n_0$ samples or the last $n_0$ samples of $D_0$ for training is theoretically equivalent. Therefore, without loss of generality, throughout the theoretical analysis of the diffusion model, we take the first $n_0$ observations in $D_0$ as the target-domain training samples. 

Before we begin the proof, we first define the H\"older function class
and the ReLU neural network class $\mathcal{F}$ used throughout the analysis.

\paragraph{H\"older function class.}
Let $\beta>0$ denote the degree of smoothness, and write:
\[
    \beta = m+\gamma,
    \qquad
    m=\lfloor \beta \rfloor,
    \qquad
    \gamma\in[0,1).
\]
For a multi-index
$\boldsymbol{\nu}=(\nu_1,\ldots,\nu_d)\in\mathbb{N}_0^d$, let:
\[
    |\boldsymbol{\nu}|_1
    :=
    \sum_{j=1}^d \nu_j,
    \qquad
    \partial^{\boldsymbol{\nu}}
    :=
    \frac{
        \partial^{|\boldsymbol{\nu}|_1}
    }{
        \partial x_1^{\nu_1}\cdots\partial x_d^{\nu_d}
    }.
\]
For a function $f:\mathbb{R}^d\rightarrow\mathbb{R}$, its
$\beta$-H\"older norm is defined as:
\begin{equation}
    \label{eq_holder_norm}
    \begin{aligned}
    \|f\|_{\mathcal{H}^{\beta}(\mathbb{R}^d)}
    :=
    \max_{|\boldsymbol{\nu}|_1\leq m}
    \sup_{x\in\mathbb{R}^d}
    \left|
        \partial^{\boldsymbol{\nu}}f(x)
    \right| + \max_{|\boldsymbol{\nu}|_1=m}
    \sup_{x\neq x'}
    \frac{
        \left|
            \partial^{\boldsymbol{\nu}}f(x)
            -
            \partial^{\boldsymbol{\nu}}f(x')
        \right|
    }{
        \|x-x'\|_{\infty}^{\gamma}
    }.
    \end{aligned}
\end{equation}
where the second term is omitted when $\gamma=0$.
A function $f$ is said to be $\beta$-H\"older continuous if
$\|f\|_{\mathcal{H}^{\beta}(\mathbb{R}^d)}<\infty$.

For any $B>0$, define the H\"older ball of smoothness $\beta$ and
radius $B$ as:
\begin{equation}
    \label{eq_holder_ball}
    \mathcal{H}^{\beta}(\mathbb{R}^d,B)
    :=
    \left\{
        f:\mathbb{R}^d\rightarrow\mathbb{R}
        :
        \|f\|_{\mathcal{H}^{\beta}(\mathbb{R}^d)}
        \leq B
    \right\}.
\end{equation}

\paragraph{ReLU neural network class.}
We use neural networks to parameterize the score functions.
Let:
\[
    \sigma(u):=\max\{u,0\}
\]
denote the ReLU activation function, applied componentwise. For the
input:
\begin{equation}
    u_0
    :=
    \begin{bmatrix}
        x^\top \ z^\top \ t
    \end{bmatrix}^{\top},
\end{equation}
consider an $L$-layer feed-forward ReLU network defined recursively by:
\begin{align}
    & u_{\ell}
    =
    \sigma\left(A_{\ell}u_{\ell-1}+b_{\ell}\right),
    \qquad
    \ell=1,\ldots,L-1,
    \notag\\
    & s(x,z,t)
    =
    A_Lu_{L-1}+b_L.
    \label{eq_relu_network}
\end{align}
Here, $d_0=d_x+d_z+1, d_L=d_x,$ and $A_{\ell} \in \mathbb{R}^{d_{\ell}\times d_{\ell-1}}, b_{\ell}\in\mathbb{R}^{d_{\ell}},
    \ell=1,\ldots,L$.

For $M_t,W,\kappa,L,S>0$, define the ReLU neural network class:
\begin{equation}
    \label{eq_relu_network_class}
    \begin{aligned}
    \mathcal{F}(M_t,W,\kappa,L,S)
    :=
    \Bigg\{
        s \text{ of the form (\ref{eq_relu_network})}:
        \;&
        \max_{1\leq \ell\leq L-1}d_{\ell}\leq W,
        \\
        &
        \sup_{x,z}
        \|s(x,z,t)\|_{\infty}
        \leq M_t,
        \\
        &
        \max_{1\leq\ell\leq L}
        \left(
            \|A_{\ell}\|_{\infty}
            \vee
            \|b_{\ell}\|_{\infty}
        \right)
        \leq\kappa,
        \\
        &
        \sum_{\ell=1}^L
        \left(
            \|A_{\ell}\|_0
            +
            \|b_{\ell}\|_0
        \right)
        \leq S
    \Bigg\}.
    \end{aligned}
\end{equation}
Here, $W$ controls the maximum width of the network, $L$ denotes the
network depth, $\kappa$ bounds the magnitude of the network parameters,
and $S$ controls the total number of nonzero parameters. Moreover,
$\|\cdot\|_{\infty}$ denotes the entrywise maximum norm and
$\|\cdot\|_0$ denotes the number of nonzero entries. The quantity
$M_t$ controls the magnitude of the network output and is allowed to
depend on the diffusion time $t$.

\subsection{Global population minimizer of the weighted score matching risk}
\label{app_global_minimizer}
We define the weighted score matching risk:
\begin{equation}
\label{eq_weighted_score_matching_risk}
\mathcal{R}_w(s) := n_0\,\mathcal{R}_0(s) + \sum_{k=1}^K w_k n_k\, \mathcal{R}_k(s).
\end{equation}

\begin{thm}[Global population minimizer of the weighted score matching risk (\ref{eq_weighted_score_matching_risk})]
Recall the $\pi_k$ defined in (\ref{eq_define_pi}):
\begin{equation}
\pi_k :=
\begin{cases}
\frac{n_0}{n_0+W_N}, & k=0,\\ \\ 
\frac{w_k n_k}{n_0+W_N}, & k\in\{1,\ldots,K\}.
\end{cases}
\end{equation}
For each $t\in[t_0,T]$, define the mixture density:
\begin{equation}
\label{eq_mixture_density_pw}
 p_{w,t}(x,z) := \sum_{k=0}^{K} \pi_k\, p_{k,t}(x,z),
\end{equation}
and the posterior domain probability:
\begin{equation}
\label{eq_posterior_eta}
\eta_k(x,z,t) := \frac{\pi_k\, p_{k,t}(x,z)}{p_{w,t}(x,z)}.
\end{equation}
Then the (global) population minimizer over all measurable functions $s$:
\begin{equation}
\label{eq_optimal_weighted_score}
 s_w^o := \arg\min_{s} \mathcal{R}_w(s) = \arg\min_{s}  \sum_{k=0}^K \pi_k \mathcal{R}_{k} (s) 
\end{equation}
satisfies, for all $(x,z,t)$ with $p_{w,t}(x,z)>0$,
\begin{equation}
\label{eq_optimal_weighted_score_pointwise}
 s_w^o(x,z,t) = \sum_{k=0}^{K} \eta_k(x,z,t)\, s_k^*(x,z,t).
\end{equation}
\begin{proof}
    For notational convenience, we introduce the following inner products and associated norms on the relevant function spaces. For any measurable vector-valued functions $g_1(x,z,t)$ and $g_2(x,z,t)$, we define the inner product induced by the $k$-th domain as
    \begin{align}
    \label{eq_inner_product_Hk}
    \langle g_1, g_2 \rangle_{\mathcal{H}_k} := \frac{1}{T-t_0}\int_{t_0}^{T} \mathbb{E}_{(X_t,Z)\sim p_{k,t}(\cdot,\cdot)}\Big[\,\langle g_1(X_t,Z,t), g_2(X_t,Z,t)\rangle\,\Big] \, dt,
    \end{align}
    with the associated squared norm $\|g_1\|_{\mathcal{H}_k}^2 := \langle g_1,g_1\rangle_{\mathcal{H}_k}$. Under this definition and (\ref{eq_population_score_matching_risk}), the population risk on the $k$-th domain can be equivalently written as $\mathcal{R}_k (s) = \|s-s_k^*\|_{\mathcal{H}_k}^2$. Similarly, define the inner product under the mixture density $p_{w,t}$ as
    \begin{align}
    \label{eq_inner_product_Hw}
    \langle g_1, g_2 \rangle_{\mathcal{H}_w} := \frac{1}{T-t_0}\int_{t_0}^{T} \mathbb{E}_{(X_t,Z)\sim p_{w,t} (\cdot, \cdot)}\Big[\,\langle g_1(X_t,Z,t), g_2(X_t,Z,t)\rangle\,\Big] \, dt,
    \end{align}
    with the associated squared norm $\|g_1\|_{\mathcal{H}_w}^2 := \langle g_1,g_1\rangle_{\mathcal{H}_w}$.

    Define $\overline{\mathcal{R}}_w (s) := \sum_{k=0}^K \pi_k \mathcal{R}_{k} (s)$, we know:
    \begin{align}
        \overline{\mathcal{R}}_w (s) & = \frac{1}{T - t_0} \int_{t_0}^T \sum_{k=0}^K \pi_k \mathbb{E}_{(X_t,Z) \sim p_{k,t} (\cdot, \cdot)} \| s - s_k^* \|_2^2 \, dt \notag \\
        & = \frac{1}{T - t_0} \int_{t_0}^T  \int_{\mathcal{X} \times \mathcal{Z}}  \sum_{k=0}^K \pi_k p_{k,t}  \| s - s_k^* \|_2^2 \, dx dz dt \notag \\
        & = \frac{1}{T - t_0} \int_{t_0}^T  \int_{ \{(x,z) \in \mathcal{X} \times \mathcal{Z} | p_{w,t} (x,z) >0\} } p_{w,t} \sum_{k=0}^K \frac{\pi_k p_{k,t} }{p_{w,t}} \| s - s_k^* \|_2^2 \, dx dz dt \notag \\
        & = \frac{1}{T - t_0} \int_{t_0}^T  \int_{ \{(x,z) \in \mathcal{X} \times \mathcal{Z} | p_{w,t} (x,z) >0\} } p_{w,t} \sum_{k=0}^K \eta_k \| s - s_k^* \|_2^2 \, dx dz dt \notag \\
        & = \frac{1}{T - t_0} \int_{t_0}^T  \mathbb{E}_{ (X_t,Z) \sim p_{w,t}(\cdot, \cdot) } \left[  \sum_{k=0}^K \eta_k \| s - s_k^* \|_2^2 \right] \, dt. \label{eq_rk_ew_eta}
    \end{align}
    We next characterize the global minimizer of 
    $\overline{\mathcal{R}}_w$ over the space of measurable functions. To derive its minimizer, let $h$ be an arbitrary admissible perturbation
    and consider the perturbed function $s_\epsilon=s+\epsilon h$.
    Assuming that differentiation with respect to $\epsilon$ can be interchanged with the time integral and expectation, the first variation of $\overline{\mathcal{R}}_w(s)$ at $s$ in the direction $h$ is given by
    \begin{align}
    D\overline{\mathcal{R}}_w(s)[h]
    &:=
    \left.
    \frac{d}{d\epsilon}
    \overline{\mathcal{R}}_w(s+\epsilon h)
    \right|_{\epsilon=0} \nonumber\\
    &=
    \frac{2}{T-t_0}
    \int_{t_0}^{T}
    \mathbb{E}_{(X_t,Z)\sim p_{w,t}(\cdot, \cdot)}
    \left[
    \left\langle
    \sum_{k=0}^K
    \eta_k
    \bigl(s-s_k^*\bigr),
    h
    \right\rangle
    \right]dt.
    \label{eq_first_variation_weighted_risk}
    \end{align}
    Since $h$ is arbitrary, the first-order optimality condition
    $D\overline{\mathcal{R}}_w(s)[h]=0$ for every admissible $h$ implies:
    \begin{equation}
    \sum_{k=0}^K
    \eta_k
    \bigl(s-s_k^*\bigr)
    =0
    \end{equation}
    for almost every $t\in[t_0,T]$ and Lebesgue-almost every $(x_t,z)\in \left\{ (x,z)\in\mathcal{X}\times\mathcal{Z} | p_{w,t}(x,z)>0 \right\}$.
    Using the truth $\sum_{k=0}^K\eta_k=1$ ((\ref{eq_mixture_density_pw}) and (\ref{eq_posterior_eta})), we obtain:
    \begin{equation*}
        \boxed{s_w^o(x_t,z,t)
        =
        \sum_{k=0}^K
        \eta_k(x_t,z,t)s_k^*(x_t,z,t).}
    \end{equation*}
    Moreover, the second variation in the direction $h$ is
    \begin{align}
    D^2\overline{\mathcal{R}}_w(s)[h,h]
    &=
    \frac{2}{T-t_0}
    \int_{t_0}^{T}
    \mathbb{E}_{(X_t,Z)\sim p_{w,t}}
    \left[
    \|h(x_t,z,t)\|_2^2
    \right] \, dt \nonumber\\
    &=
    2\|h\|_{\mathcal{H}_w}^2
    \geq 0.
    \end{align}
    Therefore, $\overline{\mathcal{R}}_w (s)$ is strictly convex in $s$, and $s_w^o$ is its unique
    global minimizer almost everywhere.
    
\end{proof}

\end{thm}

\subsection{Excess-risk decomposition}
\label{app_excess_decomp}

Knowing the global minimizer $s_w^o$, we can further decompose $\overline{\mathcal{R}}_w(s)$ based on~(\ref{eq_rk_ew_eta}) as follows
\begin{align}
    \overline{\mathcal{R}}_w (s) & = \frac{1}{T - t_0} \int_{t_0}^T  \mathbb{E}_{ (X_t,Z) \sim p_{w,t}(\cdot, \cdot) } \left[  \sum_{k=0}^K \eta_k \| s - s_k^* \|_2^2 \right] \, dt \notag \\
    & = \frac{1}{T - t_0} \int_{t_0}^T  \mathbb{E}_{ (X_t,Z) \sim p_{w,t}(\cdot, \cdot) } \left[  \sum_{k=0}^K \eta_k \| s - s_w^o + s_w^o - s_k^* \|_2^2 \right] \, dt \notag \\
    & = \frac{1}{T - t_0} \int_{t_0}^T  \mathbb{E}_{ (X_t,Z) \sim p_{w,t}(\cdot, \cdot) } \Biggl[  \sum_{k=0}^K \eta_k \| s - s_w^o\|_2^2 + \sum_{k=0}^K \eta_k \|s_w^o - s_k^* \|_2^2 \notag \\
    & \hspace{13em} + 2 \sum_{k=0}^K \eta_k \left\langle s - s_w^o, s_w^o - s_k^*  \right\rangle \Biggr] \, dt. 
\end{align}
By the truth that $\sum_{k=0}^K \eta_k = 1$ and (\ref{eq_optimal_weighted_score_pointwise}), we have $\sum_{k=0}^K \eta_k \| s - s_w^o\|_2^2 = \| s - s_w^o\|_2^2$ and $\sum_{k=0}^K \eta_k \left\langle s - s_w^o, s_w^o - s_k^*  \right\rangle = \left\langle s - s_w^o, \sum_{k=0}^K \eta_k (s_w^o - s_k^*)  \right\rangle = \left\langle s - s_w^o,  s_w^o - \sum_{k=0}^K \eta_k s_k^*  \right\rangle = 0$. Therefore, $\overline{\mathcal{R}}_w (s)$ can be rewritten as follows
\begin{align}
    \overline{\mathcal{R}}_w (s) & = \frac{1}{T - t_0} \int_{t_0}^T  \mathbb{E}_{ (X_t,Z) \sim p_{w,t}(\cdot, \cdot) } \| s - s_w^o\|_2^2 \, dt \notag \\
    & \quad + \frac{1}{T - t_0} \int_{t_0}^T  \mathbb{E}_{ (X_t,Z) \sim p_{w,t}(\cdot, \cdot) } \Biggl[ \sum_{k=0}^K \eta_k \|s_w^o - s_k^* \|_2^2 \Biggr] \, dt \label{eq_decompose_Rw},
\end{align}
where the first term depends on $s$, while the second term is irrelevant to $s$.
Moreover, it holds that
\begin{equation}
    \overline{\mathcal{R}}_w (s_w^o) = \frac{1}{T - t_0} \int_{t_0}^T \mathbb{E}_{ (X_t,Z) \sim p_{w,t}(\cdot, \cdot) } \Biggl[ \sum_{k=0}^K \eta_k \|s_w^o - s_k^* \|_2^2 \Biggr] \, dt,
\end{equation}
and the excessive score matching risk is defined as
\begin{equation}
\label{eq_excessive_risk}
    \overline{\mathcal{E}}_w (s) :=\overline{\mathcal{R}}_w (s) - \overline{\mathcal{R}}_w (s_w^o) = \frac{1}{T - t_0} \int_{t_0}^T \mathbb{E}_{ (X_t,Z) \sim p_{w,t}(\cdot, \cdot) } \| s - s_w^o\|_2^2 \, dt = \| s - s_w^o \|_{\mathcal{H}_w}^2.
\end{equation}

With the definitions in~(\ref{eq_inner_product_Hk}) and~(\ref{eq_inner_product_Hw}), we can bound the score matching risk $\mathbb{E}_{D_0,\ldots,D_K}\big[\mathcal{R}_0(\widehat{s}_w)\big]$ in Theorem~\ref{thm_score_risk_bound} as follows
\begin{equation}
\label{eq_error_decomp}
\boxed{\begin{aligned}
    \mathbb{E}_{D_0,\ldots,D_K}\big[\mathcal{R}_0(\widehat{s}_w)\big] & = \mathbb{E}_{D_0,\ldots,D_K}\big[ \|\widehat{s}_w-s_0^*\|_{\mathcal{H}_0}^2 \big]  \\
    & \lesssim \mathbb{E}_{D_0,\ldots,D_K}\big[ \|\widehat{s}_w-s_w^o\|_{\mathcal{H}_0}^2 \big] + \|s_w^o-s_0^*\|_{\mathcal{H}_0}^2,
\end{aligned}}
\end{equation}
where the first term is the estimation error under the density $p_{0,t}$, and the second term is the transfer bias between the optimal minimizer $s_w^o$ of $\overline{\mathcal{R}}_w(s)$ and the true score in the target domain. 

\subsection{Properties of $p_{w,t}$}
\label{app_property_of_pw}
In this section, we unveil some properties of the mixture density $p_{w,t}$, including smoothness, its conditional density given $z$, and its score function.

\begin{lem}[Conditional-density structure and score representation of the mixture distribution]
\label{thm_mixture_distribution_properties}

Let
\[
    P_w(x,z)
    :=
    \sum_{k=0}^K \pi_k P_k(x,z),
    \qquad
    \pi_k\geq 0,
    \qquad
    \sum_{k=0}^K \pi_k=1.
\]
and its joint density is $p_{w,0}(x,z)$ defined in (\ref{eq_mixture_density_pw}) when the diffusion time $t=0$. Under Assumption~\ref{assump_smoothness_tail}, the following statements hold.

Define the marginal density of $Z$ under $P_w$ as
\begin{equation}
    p_w(z)
    :=
    \sum_{k=0}^K \pi_k p_k(z),
    \label{eq_pw_z_marginal_pi_k}
\end{equation}
and, for $p_w(z)>0$, define
\begin{equation}
    \overline{\eta}_{k}(z)
    :=
    \frac{\pi_k p_k(z)}{p_w(z)},
    \qquad
    k=0,\ldots,K.
    \label{eq_eta_bar_weighted}
\end{equation}
Then,
\begin{equation}
    \overline{\eta}_{k}(z)\geq 0,
    \qquad
    \sum_{k=0}^K \overline{\eta}_{k}(z)=1.
\end{equation}

The conditional density of $X$ given $Z=z$ under $P_w$ satisfies
\begin{equation}
    p_{w,0}(x| z)
    =
    \sum_{k=0}^K
    \overline{\eta}_{k}(z)
    p_k(x| z)
    =
    \exp\left(
        -{C_x\|x\|_2^2}/{2}
    \right)
    \widetilde f_w(x,z),
    \label{eq_mixture_conditional_density_structure}
\end{equation}
where
\begin{equation}
    \widetilde f_w(x,z)
    :=
    \sum_{k=0}^K
    \overline{\eta}_{k}(z)
    f_k(x,z)
\end{equation}
satisfies
\begin{equation}
    \widetilde f_w \in \mathcal H^\beta
    \left(
        \mathbb R^{d_x}\times\mathbb R^{d_z},
        B_{\mathrm{mix}}
    \right), \quad \sup_{x,z} |f_k (x,z) | \leq B \ \text{ for all $k$,} \quad \text{and} \quad 
    \widetilde f_w(x,z) \geq C_f,
    \label{eq_tilde_fw_bounds}
\end{equation}
where $B_{\mathrm{mix}}>0$ depends only on
$\beta,d_x,d_z,B,L_q$, and $C_{\mathrm{DR}}$.
Moreover, the marginal density $p_w(z)$ has a sub-Gaussian tail:
\begin{equation}
    p_w(z)
    \lesssim
    \exp\left(
        -{C_z\|z\|_2^2}/{2}
    \right).
    \label{eq_pw_z_subgaussian}
\end{equation}

For any $t\in[t_0,T]$, the conditional density under the forward
diffusion satisfies
\begin{equation}
    p_{w,t}(x| z)
    =
    \sum_{k=0}^K
    \overline{\eta}_{k}(z)
    p_{k,t}(x| z).
    \label{eq_pwt_conditional_mixture}
\end{equation}
Furthermore, the population minimizer $s_w^o$ of the weighted
score-matching risk satisfies
\begin{equation}
    s_w^o(x,z,t)
    =
    \nabla_x\log p_{w,t}(x| z)
    \label{eq_sw_conditional_score}
\end{equation}
for almost every $(x,z,t)$.
\end{lem}

\begin{proof}
By definition,
\begin{equation}
    p_{w,0}(x,z)
    =
    \sum_{k=0}^K
    \pi_k p_k(x,z)
    =
    \sum_{k=0}^K
    \pi_k p_k(z)p_k(x| z).
\end{equation}
Integrating with respect to $x$ gives
\begin{equation}
    p_w(z)
    =
    \sum_{k=0}^K
    \pi_k p_k(z).
\end{equation}
Therefore, for every $z$ such that $p_w(z)>0$,
\begin{align}
    p_{w,0}(x| z)
    &=
    \frac{p_{w,0}(x,z)}{p_w(z)}
    \notag\\
    &=
    \sum_{k=0}^K
    \frac{\pi_k p_k(z)}{p_w(z)}
    p_k(x| z)
    \notag\\
    &=
    \sum_{k=0}^K
    \overline{\eta}_{k}(z)
    p_k(x| z).
    \label{eq_pw0_conditional_mixture_proof}
\end{align}

By Assumption~\ref{assump_smoothness_tail},
\[
    p_k(x| z)
    =
    \exp\left(
        -{C_x\|x\|_2^2}/{2}
    \right)
    f_k(x,z).
\]
Hence,
\begin{align}
    p_{w,0}(x| z)
    &=
    \exp\left(
        -{C_x\|x\|_2^2}/{2}
    \right)
    \sum_{k=0}^K
    \overline{\eta}_{k}(z)
    f_k(x,z)
    \notag\\
    &=
    \exp\left(
        -{C_x\|x\|_2^2}/{2}
    \right)
    \widetilde f_w(x,z),
\end{align}
which proves  (\ref{eq_mixture_conditional_density_structure}). A direct result of assumption (\ref{eq_assump_holder}) and definition in (\ref{eq_holder_norm}) is that $\sup_{x,z} |f_k (x,z) | \leq B$ for all $k$.  Since
$\overline{\eta}_{k}(z)\geq0$,
$\sum_{k=0}^K\overline{\eta}_{k}(z)=1$, and
$C_f\leq f_k(x,z)\leq B$, we have
\begin{equation}
    C_f
    \leq
    \widetilde f_w(x,z)
    \leq
    B,
\end{equation}
which proves the second and third statements of  (\ref{eq_tilde_fw_bounds}). 

We next establish the H\"older regularity of $\widetilde f_w (x,z)$ in (\ref{eq_tilde_fw_bounds}). Recall the definition of $q_k, k=0,\ldots, K$ in (\ref{eq_assump_smooth_ratio}). We define 
\begin{equation}
    q_w(z)
    :=
    \sum_{k=0}^K
    \pi_k q_k(z).
    \label{eq_qw_definition}
\end{equation}
Then, for Lebesgue-almost every $z$,
\begin{align}
    p_w(z)
    &=
    \sum_{k=0}^K
    \pi_kp_k(z)
    \notag\\
    &=
    p_0(z)
    \sum_{k=0}^K
    \pi_kq_k(z)
    \notag\\
    &=
    p_0(z)q_w(z).
    \label{eq_pw_p0_qw}
\end{align}
In particular, this identity remains meaningful even at points
where $p_0(z)=0$, since no density ratio is formed there. Further, define
\begin{equation}
    F_w(x,z)
    :=
    \sum_{k=0}^K
    \pi_k q_k(z)f_k(x,z).
    \label{eq_Fw_definition}
\end{equation}
Whenever $p_w(z)>0$, we necessarily have $p_0(z)>0$ by
(\ref{eq_pw_p0_qw}), and hence
\begin{align}
    \widetilde f_w(x,z)
    &=
    \sum_{k=0}^K
    \frac{\pi_kp_k(z)}{p_w(z)}
    f_k(x,z)
    \notag\\
    &=
    \sum_{k=0}^K
    \frac{\pi_kq_k(z)}{q_w(z)}
    f_k(x,z)
    \notag\\
    &=
    \frac{F_w(x,z)}{q_w(z)}.
    \label{eq_fw_Fw_over_qw}
\end{align}
We first control the H\"older regularity of $q_w$.
By the subadditivity of the H\"older norm, (\ref{eq_assump_smooth_ratio}) and
$\sum_{k=0}^K\pi_k=1$,
\begin{align}
    \|q_w\|_{\mathcal H^\beta(\mathbb R^{d_z})}
    &\leq
    \sum_{k=0}^K
    \pi_k
    \|q_k\|_{\mathcal H^\beta(\mathbb R^{d_z})}
    \notag\\
    &\leq
    L_q.
    \label{eq_qw_holder_bound}
\end{align}
We next derive a lower bound for $q_w$ on the support of
the target marginal distribution.
By the definitions in (\ref{eq_pw_z_marginal_pi_k}) and $p_{w\setminus0}(z)$ in Assumption \ref{assump_transf}, we have 
\begin{equation}
    p_w(z)
    =
    \pi_0p_0(z)
    +
    (1-\pi_0)p_{w\setminus0}(z).
\end{equation}
By the marginal overlap condition
(\ref{eq_assump_ratio}),
\begin{equation}
    p_{w\setminus0}(z)
    \geq
    C_{\mathrm{DR}}^{-1}p_0(z)
\end{equation}
for $p_0$-almost every $z$.
Consequently,
\begin{align}
    p_w(z)
    &\geq
    \left[
        \pi_0
        +
        \frac{1-\pi_0}{C_{\mathrm{DR}}}
    \right]
    p_0(z)
    \notag\\
    &\geq
    C_{\mathrm{DR}}^{-1}p_0(z).
    \label{eq_pw_lower_p0}
\end{align}
Combining
(\ref{eq_pw_p0_qw}) and
(\ref{eq_pw_lower_p0}) yields
\begin{equation}
    q_w(z)
    \geq C_{\mathrm{DR}}^{-1},
    \label{eq_qw_lower_bound_support}
\end{equation}
for $p_0$-almost every $z$.

We next control the H\"older regularity of $F_w$.
Regarding $q_k(z)$ as a function of $(x,z)$ that is constant
in $x$, the standard product inequality for H\"older spaces gives
\begin{equation}
    \|gh\|_{\mathcal H^\beta(\mathbb R^{d_x+d_z})}
    \leq
    C_{\beta,d_x,d_z}
    \|g\|_{\mathcal H^\beta(\mathbb R^{d_z})}
    \|h\|_{\mathcal H^\beta(\mathbb R^{d_x+d_z})},
    \label{eq_holder_product_inequality}
\end{equation}
where $C_{\beta,d_x,d_z} >0$ is a constant.  
Therefore, by the definition of $F_w$ in (\ref{eq_Fw_definition}),
\begin{align}
    \|F_w\|_{\mathcal H^\beta(\mathbb R^{d_x+d_z})}
    &\leq
    \sum_{k=0}^K
    \pi_k
    \|q_k f_k\|_{\mathcal H^\beta(
        \mathbb R^{d_x+d_z})}
    \notag\\
    &\leq
    C_{\beta,d_x,d_z}
    \sum_{k=0}^K
    \pi_k
    \|q_k\|_{\mathcal H^\beta(\mathbb R^{d_z})}
    \|f_k\|_{\mathcal H^\beta(
        \mathbb R^{d_x+d_z})}
    \notag\\
    &\leq
    C_{\beta,d_x,d_z}L_qB.
    \label{eq_Fw_holder_bound}
\end{align}
The representation in
(\ref{eq_fw_Fw_over_qw}) is only required on the support of
$p_w$. At points where $p_w(z)=0$, the conditional distribution
is not uniquely determined, and therefore it can be modified
without changing the joint distribution. Let
$\chi:\mathbb R\to[0,1]$ be a fixed infinitely differentiable
function satisfying
\begin{equation}
    \chi(u)
    =
    0
    \quad\text{for }u\leq C_{\mathrm{DR}}^{-1}/2,
    \qquad
    \chi(u)
    =
    1
    \quad\text{for }u\geq C_{\mathrm{DR}}^{-1}.
    \label{eq_cutoff_chi}
\end{equation}
Define
\begin{equation}
    h(u)
    :=
    \begin{cases}
        \chi(u)/u,
        & u>0,\\
        0,
        & u\leq0.
    \end{cases}
    \label{eq_cutoff_h}
\end{equation}
Since $\chi$ vanishes on a neighborhood of zero,
$h$ is infinitely differentiable with uniformly bounded
derivatives of every finite order. We then define, for every
$(x,z)\in\mathbb R^{d_x}\times\mathbb R^{d_z}$,
\begin{equation}
    \widetilde f_w^{\mathrm{ext}}(x,z)
    :=
    h(q_w(z))F_w(x,z)
    +
    \bigl[1-\chi(q_w(z))\bigr]f_0(x,z).
    \label{eq_fw_smooth_extension}
\end{equation}
For $p_w$-almost every $z$, we have $p_0(z)>0$ and,
by (\ref{eq_qw_lower_bound_support}),
$q_w(z)\geq C_{\mathrm{DR}}^{-1}$.
Hence
\[
    \chi(q_w(z))=1,
    \qquad
    h(q_w(z))=\frac{1}{q_w(z)}.
\]
Therefore,
\begin{align}
    \widetilde f_w^{\mathrm{ext}}(x,z)
    &=
    \frac{F_w(x,z)}{q_w(z)}
    \notag\\
    &=
    \widetilde f_w(x,z)
\end{align}
for $p_w$-almost every $z$.
Thus
$\widetilde f_w^{\mathrm{ext}}$
defines a valid version of the conditional-density factor. For simplicity, we henceforth denote this version again by
$\widetilde f_w$.

We next verify that the uniform lower and upper bounds are
preserved. When $q_w(z)>0$,
(\ref{eq_fw_smooth_extension}) can be rewritten as
\begin{align}
    \widetilde f_w^{\mathrm{ext}}(x,z)
    &=
    \sum_{k=0}^K
    \chi(q_w(z))
    \frac{\pi_kq_k(z)}{q_w(z)}
    f_k(x,z)
    \notag\\
    &\quad+
    \bigl[1-\chi(q_w(z))\bigr]f_0(x,z).
    \label{eq_fw_extension_convex_combination}
\end{align}
All coefficients in
(\ref{eq_fw_extension_convex_combination})
are nonnegative and their sum equals
\begin{align}
    \chi(q_w(z))
    \sum_{k=0}^K
    \frac{\pi_kq_k(z)}{q_w(z)}
    +
    1-\chi(q_w(z))
    =
    1.
\end{align}
Hence
$\widetilde f_w^{\mathrm{ext}}$
is a convex combination of
$f_0,\ldots,f_K$.
When $q_w(z)=0$,
(\ref{eq_fw_smooth_extension}) reduces to
\[
    \widetilde f_w^{\mathrm{ext}}(x,z)
    =
    f_0(x,z).
\]
Since
\[
    C_f
    \leq
    f_k(x,z)
    \leq
    B,
    \qquad
    k=0,\ldots,K,
\]
we conclude that
\begin{equation}
    C_f
    \leq
    \widetilde f_w^{\mathrm{ext}}(x,z)
    \leq
    B
\end{equation}
for all $(x,z)$.

It remains to establish the joint H\"older regularity.
By (\ref{eq_qw_holder_bound}),
$q_w\in\mathcal H^\beta(\mathbb R^{d_z},L_q)$.
Since $\chi$ and $h$ are fixed smooth functions with bounded
derivatives up to every finite order, the standard composition
inequality for H\"older spaces implies
\begin{align}
    \|\chi\circ q_w\|_{\mathcal H^\beta(\mathbb R^{d_z})}
    &\leq
    C_\chi,
    \label{eq_chi_qw_holder}
    \\
    \|h\circ q_w\|_{\mathcal H^\beta(\mathbb R^{d_z})}
    &\leq
    C_h,
    \label{eq_h_qw_holder}
\end{align}
where $C_\chi$ and $C_h$ depend only on
$\beta,d_z,L_q$, and $C_{\mathrm{DR}}$. 
Combining
(\ref{eq_Fw_holder_bound}),
(\ref{eq_chi_qw_holder}),
(\ref{eq_h_qw_holder}),
and the H\"older product inequality yields
\begin{align}
    \|
        \widetilde f_w^{\mathrm{ext}}
    \|_{\mathcal H^\beta(
        \mathbb R^{d_x+d_z})}
    &\leq
    C_{\beta,d_x,d_z}
    \Big(
        \|h\circ q_w\|_{\mathcal H^\beta}
        \|F_w\|_{\mathcal H^\beta} +
        \|1-\chi\circ q_w\|_{\mathcal H^\beta}
        \|f_0\|_{\mathcal H^\beta}
    \Big)
    \notag\\
    &\leq
    B_{\mathrm{mix}},
    \label{eq_fw_ext_holder_final}
\end{align}
where
    $B_{\mathrm{mix}}
    =
    C\left(
        \beta,
        d_x,
        d_z,
        B,
        L_q,
        C_{\mathrm{DR}}
    \right)
    <\infty.$
Importantly, $B_{\mathrm{mix}}$ does not depend on
$K$, the sample sizes, or the particular mixture weights.
Therefore,
\begin{equation}
    \widetilde f_w^{\mathrm{ext}}
    \in
    \mathcal H^\beta
    \left(
        \mathbb R^{d_x}\times\mathbb R^{d_z},
        B_{\mathrm{mix}}
    \right).
\end{equation}
Since
$\widetilde f_w^{\mathrm{ext}}$
coincides with the original conditional-density factor
$p_w$-almost everywhere, the first statement of (\ref{eq_tilde_fw_bounds}) is proven. 


The sub-Gaussian tail of $p_w(z)$ follows directly from
Assumption~\ref{assump_smoothness_tail}:
\begin{align}
    p_w(z)
    &=
    \sum_{k=0}^K
    \pi_k p_k(z)
    \notag\\
    &\lesssim
    \sum_{k=0}^K
    \pi_k
    \exp\left(
        -{C_z\|z\|_2^2}/{2}
    \right)
    \notag\\
    &\lesssim
    \exp\left(
        -{C_z\|z\|_2^2}/{2}
    \right),
\end{align}
which proves (\ref{eq_pw_z_subgaussian}). We next consider the forward diffusion. Since the diffusion process is
applied only to $X$ and leaves $Z$ unchanged,
\begin{equation}
    p_{k,t}(x,z)
    =
    p_k(z)p_{k,t}(x| z),
    \qquad
    k=0,\ldots,K.
\end{equation}
Consequently,
\begin{align}
    p_{w,t}(x,z)
    &=
    \sum_{k=0}^K
    \pi_k p_{k,t}(x,z)
    \notag\\
    &=
    \sum_{k=0}^K
    \pi_k p_k(z)p_{k,t}(x| z).
\end{align}
Since the marginal distribution of $Z$ is unchanged by the forward
diffusion,
\[
    \int_{\mathbb R^{d_x}}
    p_{w,t}(x,z)\,dx
    =
    p_w(z).
\]
Therefore,
\begin{align}
    p_{w,t}(x| z)
    &=
    \frac{p_{w,t}(x,z)}{p_w(z)}
    \notag\\
    &=
    \sum_{k=0}^K
    \frac{\pi_k p_k(z)}{p_w(z)}
    p_{k,t}(x| z)
    \notag\\
    &=
    \sum_{k=0}^K
    \overline{\eta}_{k}(z)
    p_{k,t}(x| z),
\end{align}
which proves (\ref{eq_pwt_conditional_mixture}). Recall that
\begin{equation}
    s_k^*(x,z,t)
    :=
    \nabla_x\log p_{k,t}(x| z).
\end{equation}
Since $\overline{\eta}_{k}(z)$ does not depend on $x$,
differentiating (\ref{eq_pwt_conditional_mixture}) with respect to
$x$ gives
\begin{align}
    \nabla_x p_{w,t}(x| z)
    &=
    \sum_{k=0}^K
    \overline{\eta}_{k}(z)
    \nabla_x p_{k,t}(x| z)
    \notag\\
    &=
    \sum_{k=0}^K
    \overline{\eta}_{k}(z)
    p_{k,t}(x| z)
    s_k^*(x,z,t).
\end{align}
Therefore,
\begin{align}
    \nabla_x\log p_{w,t}(x| z)
    &=
    \frac{
        \nabla_x p_{w,t}(x| z)
    }{
        p_{w,t}(x| z)
    }
    \notag\\
    &=
    \frac{
        \sum_{k=0}^K
        \overline{\eta}_{k}(z)
        \nabla_x p_{k,t}(x| z)
    }{
        p_{w,t}(x| z)
    }
    \notag\\
    &=
    \frac{
        \sum_{k=0}^K
        \overline{\eta}_{k}(z)
        p_{k,t}(x| z)
        s_k^*(x,z,t)
    }{
        p_{w,t}(x| z)
    }
    \notag\\
    &=
    \sum_{k=0}^K
    \frac{
        \overline{\eta}_{k}(z)
        p_{k,t}(x| z)
    }{
        p_{w,t}(x| z)
    }
    s_k^*(x,z,t)
    \notag\\
    &=
    \sum_{k=0}^K
    \eta_{k,w,t}(x,z)
    s_k^*(x,z,t),
    \label{eq_pwt_score_mixture}
\end{align}
where
\begin{align}
    \eta_{k,w,t}(x,z)
    &:=
    \frac{
        \overline{\eta}_{k}(z)p_{k,t}(x| z)
    }{
        p_{w,t}(x| z)
    }
    \notag\\
    &=
    \frac{
        \pi_k p_k(z)p_{k,t}(x| z)
    }{
        p_w(z)p_{w,t}(x| z)
    }
    \notag\\
    &=
    \frac{
        \pi_k p_{k,t}(x,z)
    }{
        p_{w,t}(x,z)
    }
    \notag\\
    &=
    \eta_k(x,z,t).
    \label{eq_eta_joint_form}
\end{align}
It follows from
(\ref{eq_optimal_weighted_score_pointwise}) that
\begin{equation}
    \nabla_x\log p_{w,t}(x| z)
    =
    \sum_{k=0}^K
    \eta_k(x,z,t)s_k^*(x,z,t)
    =
    s_w^o(x,z,t),
\end{equation}
which proves (\ref{eq_sw_conditional_score}). 

\end{proof}

\subsection{Bounding the transfer bias
$\|s_w^o-s_0^*\|_{\mathcal H_0}^2$}
\label{app_transfer_bias}

\begin{thm}[The bound of transfer bias]
    Under Assumption \ref{assump_transf} and \ref{assump_smoothness_tail}, the transfer bias $\|s_w^o-s_0^*\|_{\mathcal H_0}^2$ can be bounded as
    \begin{equation}
        \|s_w^o-s_0^*\|_{\mathcal H_0}^2
        \lesssim
        \frac{W_N^2}{(n_0+W_N)^2}
        \left[
            \overline w^\top
            G
            \overline w
            +
            \operatorname{diag}(G)^\top
            \overline w
        \right].
    \end{equation}
\end{thm}

\begin{proof}
Recall that in Assumption \ref{assump_transf} and Theorem \ref{thm_score_risk_bound}, we defined
\[
    W_N
    :=
    \sum_{k=1}^K w_k n_k,
    \qquad
    \overline w_k
    :=
    \frac{w_k n_k}{W_N},
    \qquad
    \overline w
    :=
    (\overline w_1,\ldots,\overline w_K)^\top,
\]
so that
\[
    \overline w_k\geq 0,
    \qquad
    \sum_{k=1}^K \overline w_k=1.
\]
Define
\begin{equation}
    \Delta_{k,t}(x_t,z)
    :=
    s_k^*(x_t,z,t)
    -
    s_0^*(x_t,z,t),
    \qquad
    k=1,\ldots,K.
    \label{eq_delta_k_t}
\end{equation}
According to the expression of the population-optimal weighted score
$s_w^o$ in~(\ref{eq_optimal_weighted_score_pointwise}), we have
\begin{align}
    s_w^o
    &=
    \frac{
        n_0 p_{0,t}s_0^*
        +
        \sum_{k=1}^K
        w_k n_k p_{k,t}s_k^*
    }{
        n_0 p_{0,t}
        +
        \sum_{k=1}^K
        w_k n_k p_{k,t}
    }.
    \label{eq_swo_transfer}
\end{align}
For $p_{0,t}(x_t,z)>0$, define the joint density ratio
\begin{equation}
    R_{k,t}(x_t,z)
    :=
    \frac{
        p_{k,t}(x_t,z)
    }{
        p_{0,t}(x_t,z)
    }
    =
    \frac{p_k(z)}{p_0(z)}
    \frac{
        p_{k,t}(x_t| z)
    }{
        p_{0,t}(x_t| z)
    }.
    \label{eq_joint_density_ratio_transfer}
\end{equation}
Then,
\begin{equation}
    s_w^o-s_0^*
    =
    \frac{
        W_N
        \sum_{k=1}^K
        \overline w_k
        R_{k,t}\Delta_{k,t}
    }{
        n_0
        +
        W_N
        \sum_{k=1}^K
        \overline w_kR_{k,t}
    }.
    \label{eq_swo_minus_s0_transfer}
\end{equation}

We first derive a uniform lower bound for the denominator.
By Assumption~\ref{assump_smoothness_tail} and Lemma \ref{thm_mixture_distribution_properties},
\[
    p_k(x| z)
    =
    \exp\left(
        -{C_x\|x\|_2^2}/{2}
    \right)
    f_k(x,z),
\]
where
$C_f\leq f_k(x,z)\leq B$.
Therefore,
\begin{equation}
    \frac{C_f}{B}
    p_0(x| z)
    \leq
    p_k(x| z)
    \leq
    \frac{B}{C_f}
    p_0(x| z).
    \label{eq_conditional_ratio_bound_transfer}
\end{equation}
Since the forward diffusion is applied only to $X$, integrating
(\ref{eq_conditional_ratio_bound_transfer}) against the nonnegative
transition kernel preserves the inequalities. Hence, for every
$t\in[t_0,T]$,
\begin{gather} 
    \textcolor{white}{ i.e } \frac{C_f}{B}
    p_{0,t}(x_t| z)
    \leq
    p_{k,t}(x_t| z)
    \leq
    \frac{B}{C_f}
    p_{0,t}(x_t| z), \text{ i.e. } \notag \\
    \textcolor{white}{ i.e } \frac{C_f}{B}
    \leq
    \frac{
        p_{k,t}(x_t| z)
    }{
        p_{0,t}(x_t| z)
    }
    \leq
    \frac{B}{C_f}. \textcolor{white}{ i.e.e } 
    \label{eq_diffused_conditional_ratio_bound}
\end{gather}

Using (\ref{eq_joint_density_ratio_transfer}) and
(\ref{eq_diffused_conditional_ratio_bound}), we obtain
\begin{align}
    \sum_{k=1}^K
    \overline w_kR_{k,t}(x_t,z)
    &=
    \sum_{k=1}^K
    \overline w_k
    \frac{p_k(z)}{p_0(z)}
    \frac{
        p_{k,t}(x_t| z)
    }{
        p_{0,t}(x_t| z)
    }
    \notag\\
    &\geq
    \frac{C_f}{B}
    \sum_{k=1}^K
    \overline w_k
    \frac{p_k(z)}{p_0(z)}
    \notag\\
    &=
    \frac{C_f}{B}
    \frac{
        p_{w\setminus0}(z)
    }{
        p_0(z)
    }
    \notag\\
    &\geq
    \frac{C_f}{B C_{\mathrm{DR}}},
    \label{eq_weighted_joint_ratio_lower}
\end{align}
where the last inequality follows from
Assumption~\ref{assump_transf}.
Consequently,
\begin{align}
    n_0
    +
    W_N
    \sum_{k=1}^K
    \overline w_kR_{k,t}
    &\geq
    n_0
    +
    \frac{C_f}{B C_{\mathrm{DR}}} W_N
    \notag\\
    &\gtrsim
    n_0+W_N,
    \label{eq_transfer_denominator_lower}
\end{align}
where the hidden constant depends only on
$C_{\mathrm{DR}},C_f,$ and $B$.
It follows from (\ref{eq_swo_minus_s0_transfer}) that
\begin{align}
    \|s_w^o-s_0^*\|_{\mathcal H_0}^2
    &\lesssim
    \frac{W_N^2}{(n_0+W_N)^2}
    \left\|
        \sum_{k=1}^K
        \overline w_k
        R_{k,t}\Delta_{k,t}
    \right\|_{\mathcal H_0}^2.
    \label{eq_transfer_bias_first_bound}
\end{align}
We next bound the term involving the joint density ratios.
By (\ref{eq_assump_ratio}), (\ref{eq_joint_density_ratio_transfer}), and
(\ref{eq_diffused_conditional_ratio_bound}),
\begin{equation}
    0
    \leq
    R_{k,t}(x_t,z)
    \leq
    \frac{B}{C_f C_{\mathrm{dr}}},
\end{equation}
uniformly over $k,x_t,z,$ and $t$. Since all constants on the
right-hand side are independent of the sample sizes and the weights, and $B/C_f C_{\mathrm{dr}}$ is larger than 1,
we have
\begin{equation}
    |R_{k,t}(x_t,z)-1|
    \lesssim 1.
    \label{eq_R_minus_one_bound}
\end{equation}
Now decompose
\begin{align}
    \sum_{k=1}^K
    \overline w_kR_{k,t}\Delta_{k,t}
    &=
    \sum_{k=1}^K
    \overline w_k\Delta_{k,t}
    +
    \sum_{k=1}^K
    \overline w_k
    (R_{k,t}-1)\Delta_{k,t}.
    \label{eq_transfer_decomposition}
\end{align}
Using
$\|a+b\|_2^2\leq2\|a\|_2^2+2\|b\|_2^2$,
we obtain
\begin{align}
    \left\|
        \sum_{k=1}^K
        \overline w_kR_{k,t}\Delta_{k,t}
    \right\|_{\mathcal H_0}^2
    \lesssim
    \left\|
        \sum_{k=1}^K
        \overline w_k\Delta_{k,t}
    \right\|_{\mathcal H_0}^2 +
    \left\|
        \sum_{k=1}^K
        \overline w_k
        (R_{k,t}-1)\Delta_{k,t}
    \right\|_{\mathcal H_0}^2.
    \label{eq_transfer_decomposition_bound}
\end{align}
By the definition of the source discrepancy matrix $G$ in
(\ref{eq_G_def}),
\begin{equation}
    \left\|
        \sum_{k=1}^K
        \overline w_k\Delta_{k,t}
    \right\|_{\mathcal H_0}^2
    =
    \overline w^\top
    G
    \overline w.
    \label{eq_transfer_quadratic_G}
\end{equation}
For the second term, since
$\overline w_k\geq0$ and
$\sum_{k=1}^K\overline w_k=1$, Jensen's inequality gives, pointwise,
\begin{align}
    \left\|
        \sum_{k=1}^K
        \overline w_k
        (R_{k,t}-1)\Delta_{k,t}
    \right\|_2^2
    &\leq
    \sum_{k=1}^K
    \overline w_k
    (R_{k,t}-1)^2
    \|\Delta_{k,t}\|_2^2
    \notag\\
    &\lesssim
    \sum_{k=1}^K
    \overline w_k
    \|\Delta_{k,t}\|_2^2,
    \label{eq_transfer_jensen_diag}
\end{align}
where the last inequality follows from
(\ref{eq_R_minus_one_bound}).
Taking expectation under $p_{0,t}$ and integrating over
$t\in[t_0,T]$ yield
\begin{align}
    \left\|
        \sum_{k=1}^K
        \overline w_k
        (R_{k,t}-1)\Delta_{k,t}
    \right\|_{\mathcal H_0}^2
    &\lesssim
    \sum_{k=1}^K
    \overline w_k
    \|\Delta_{k,t}\|_{\mathcal H_0}^2
    \notag\\
    &=
    \sum_{k=1}^K
    \overline w_kG_{kk}
    \notag\\
    &=
    \operatorname{diag}(G)^\top
    \overline w.
    \label{eq_transfer_diag_G}
\end{align}
Combining
(\ref{eq_transfer_decomposition_bound}),
(\ref{eq_transfer_quadratic_G}), and
(\ref{eq_transfer_diag_G}), we obtain
\begin{equation}
    \left\|
        \sum_{k=1}^K
        \overline w_kR_{k,t}\Delta_{k,t}
    \right\|_{\mathcal H_0}^2
    \lesssim
    \overline w^\top
    G
    \overline w
    +
    \operatorname{diag}(G)^\top
    \overline w.
    \label{eq_weighted_ratio_discrepancy_bound}
\end{equation}
Substituting
(\ref{eq_weighted_ratio_discrepancy_bound}) into
(\ref{eq_transfer_bias_first_bound}) gives
\begin{equation}
\boxed{
    \|s_w^o-s_0^*\|_{\mathcal H_0}^2
    \lesssim
    \frac{W_N^2}{(n_0+W_N)^2}
    \left[
        \overline w^\top
        G
        \overline w
        +
        \operatorname{diag}(G)^\top
        \overline w
    \right].
}
\label{eq_final_transfer_bias_general}
\end{equation}
Here, the hidden constant depends only on
$C_{\mathrm{dr}},C_{\mathrm{DR}},C_f,$ and $B$, and is independent of
$n_0$, $W_N$, and $\overline w$.

\end{proof}

\subsection{Bounding the estimation error $\mathbb{E}_{D_0,\ldots,D_K}
\big[
\|\widehat{s}_w-s_w^o\|_{\mathcal{H}_0}^2
\big]$}
\label{app_bound_esti_error}

\subsubsection{Turn $\mathcal{H}_0$ to $\mathcal{H}_w$}
\label{sec_turn_H0_to_Hw}
In this section, we relate
$\mathbb{E}_{D_0,\ldots,D_K}
\big[
\|\widehat{s}_w-s_w^o\|_{\mathcal{H}_0}^2
\big]$
to the corresponding error measured under $\mathcal{H}_w$.
\begin{lem}[Change of measure from $\mathcal H_0$ to $\mathcal H_w$]
    Under Assumption \ref{assump_transf} and \ref{assump_smoothness_tail}, we have 
    \begin{equation}
        \mathbb{E}_{D_0,\ldots,D_K}
    \left[
        \|\widehat{s}_w-s_w^o\|_{\mathcal{H}_0}^2
    \right]
    \lesssim
    \mathbb{E}_{D_0,\ldots,D_K}
    \left[
        \|\widehat{s}_w-s_w^o\|_{\mathcal{H}_w}^2
    \right].
    \end{equation}
\end{lem}

\begin{proof}    
First, Assumption~\ref{assump_smoothness_tail} implies a uniform overlap
between the conditional distributions across domains. In particular,
since
\[
    p_k(x| z)
    =
    \exp\left(
        -{C_x\|x\|_2^2}/{2}
    \right)
    f_k(x,z)
\]
and $C_f \leq f_k(x,z) \leq B$, we have
\begin{equation}
    p_k(x| z)
    \geq
    \frac{C_f}{B}
    p_0(x| z),
    \qquad
    k=1,\ldots,K.
    \label{eq_conditional_density_overlap}
\end{equation}
Since the forward diffusion is applied only to $X$, integrating
(\ref{eq_conditional_density_overlap}) against the transition kernel
preserves the inequality. Hence, for every $t\in[t_0,T]$,
\begin{equation}
    p_{k,t}(x_t| z)
    \geq
    \frac{C_f}{B}
    p_{0,t}(x_t| z).
    \label{eq_conditional_diffused_overlap}
\end{equation}
Recall that in Assumption \ref{assump_transf}, we defined
\[
    p_{w\setminus 0}(z)
    =
    \sum_{k=1}^K
    \frac{w_k n_k}{W_N}p_k(z).
\]
Combining (\ref{eq_conditional_diffused_overlap}) with
Assumption~\ref{assump_transf}, we obtain
\begin{align}
    \sum_{k=1}^K
    \frac{w_k n_k}{W_N}
    p_{k,t}(x_t,z)
    &=
    \sum_{k=1}^K
    \frac{w_k n_k}{W_N}
    p_k(z)p_{k,t}(x_t| z)
    \notag\\
    &\geq
    \frac{C_f}{B}
    p_{0,t}(x_t| z)
    \sum_{k=1}^K
    \frac{w_k n_k}{W_N}p_k(z)
    \notag\\
    &=
    \frac{C_f}{B}
    p_{0,t}(x_t| z)
    p_{w\setminus 0}(z)
    \notag\\
    &\geq
    \frac{C_f}{B C_{\mathrm{DR}}}
    p_{0,t}(x_t,z).
    \label{eq_source_mixture_overlap}
\end{align}
Recall that
\[
    p_{w,t}(x_t,z)
    =
    \frac{
        n_0p_{0,t}(x_t,z)
        +
        \sum_{k=1}^K w_k n_k p_{k,t}(x_t,z)
    }{
        n_0+W_N
    }.
\]
Using (\ref{eq_source_mixture_overlap}), we have
\begin{align}
    p_{w,t}(x_t,z)
    &\geq
    \frac{
        n_0
        +
        W_N \cdot C_f/C_{\mathrm{DR}} B
    }{
        n_0+W_N
    }
    p_{0,t}(x_t,z)
    \notag\\
    &\geq
    \frac{C_f}{C_{\mathrm{DR}}B}
    p_{0,t}(x_t,z),
    \label{eq_pw_p0_lower_bound}
\end{align}
where the last inequality follows from
$C_f/C_{\mathrm{DR}}B \leq 1$.
Therefore,
\begin{equation}
    \frac{p_{0,t}(x_t,z)}
         {p_{w,t}(x_t,z)}
    \leq
    \frac{C_{\mathrm{DR}}B}{C_f}.
    \label{eq_p0_pw_upper_bound}
\end{equation}

For any measurable function $g$, by the definitions of
$\mathcal{H}_0$ and $\mathcal{H}_w$, we have
\begin{align}
    \|g\|_{\mathcal{H}_0}^2
    &=
    \frac{1}{T-t_0}
    \int_{t_0}^T
    \int
    \|g(x_t,z,t)\|_2^2
    p_{0,t}(x_t,z)
    \,dx_t\,dz\,dt
    \notag\\
    &=
    \frac{1}{T-t_0}
    \int_{t_0}^T
    \int
    \|g(x_t,z,t)\|_2^2
    \frac{p_{0,t}(x_t,z)}
         {p_{w,t}(x_t,z)}
    p_{w,t}(x_t,z)
    \,dx_t\,dz\,dt
    \notag\\
    &\leq
    \frac{C_{\mathrm{DR}}B}{C_f}
    \|g\|_{\mathcal{H}_w}^2.
    \label{eq_H0_Hw_norm_conversion}
\end{align}

Applying (\ref{eq_H0_Hw_norm_conversion}) to
$g=\widehat{s}_w-s_w^o$ and taking expectation with respect to
$D_0,\ldots,D_K$ yield
\begin{equation}
\boxed{
    \mathbb{E}_{D_0,\ldots,D_K}
    \left[
        \|\widehat{s}_w-s_w^o\|_{\mathcal{H}_0}^2
    \right]
    \leq
    \frac{C_{\mathrm{DR}}B}{C_f}
    \mathbb{E}_{D_0,\ldots,D_K}
    \left[
        \|\widehat{s}_w-s_w^o\|_{\mathcal{H}_w}^2
    \right]
    \lesssim
    \mathbb{E}_{D_0,\ldots,D_K}
    \left[
        \|\widehat{s}_w-s_w^o\|_{\mathcal{H}_w}^2
    \right].
}
\label{eq_H0_Hw_estimation_error}
\end{equation}

Hence, any convergence rate established for the estimation error under
$\mathcal{H}_w$ is preserved under $\mathcal{H}_0$.

\end{proof}

\subsubsection{Decomposition of the expected excess risk $\mathbb{E}_{D_0,\ldots,D_K} \bigl[ \overline{\mathcal{E}}_w(\widehat{s}_w)\bigr]$}
Recall that in (\ref{eq_excessive_risk}), the excess risk is defined as
\begin{equation*}
    \overline{\mathcal{E}}_w(s)
    =
    \overline{\mathcal{R}}_w(s)
    -
    \overline{\mathcal{R}}_w(s_w^o).
\end{equation*}
By the definition of $\overline{\mathcal{R}}_w(s)$ defined above (\ref{eq_rk_ew_eta}), the relationship between ${\mathcal{R}}_k(s)$ and ${\mathcal{L}}_k(s)$ in (\ref{eq_score_matching_risk_decomposition}), and the definition of ${\mathcal{L}}_k(s)$ (\ref{eq_population_domain_loss}), we have:
\begin{align}
    \overline{\mathcal{E}}_w(s)
    & =
    \sum_{k=0}^K \pi_k \left[\mathcal{R}_{k} (s) - \mathcal{R}_{k} (s_w^o)\right] \notag \\
    & =\sum_{k=0}^K \pi_k \left[\mathcal{L}_{k} (s) - \mathcal{L}_{k} (s_w^o)\right] \notag \\
    & = \sum_{k=0}^K \pi_k \mathbb{E}_{(x,z)\sim P_k}\big[\ell(x,z;s) - \ell(x,z;s_w^o)\big] \notag \\
    & = \sum_{k=0}^K \pi_k \mathbb{E}_{D_k}\left[\frac{1}{n_k} \sum_{i=1}^{n_k} \ell(x_{k,i},z_{k,i};s) - \frac{1}{n_k} \sum_{i=1}^{n_k}\ell(x_{k,i},z_{k,i};s_w^o)\right]. \label{eq_excess_risk_decomposed}
\end{align}
We introduce an independent ghost sample for each domain. Specifically, let
\begin{equation}
    D_k'
    :=
    \left\{
        (x_{k,i}',z_{k,i}')
    \right\}_{i=1}^{n_k},
    \qquad
    (x_{k,i}',z_{k,i}')
    \overset{\mathrm{i.i.d.}}{\sim} P_k,
    \qquad
    k=0,\ldots,K,
\end{equation}
where $D_k'$ is independent of $D_k$, and the collection
$D'=(D_0',\ldots,D_K')$ is independent of the training data
$D=(D_0,\ldots,D_K)$.

For any $s$, define the excess loss
\begin{equation}
    h_s(x,z)
    :=
    \ell(x,z;s)
    -
    \ell(x,z;s_w^o).
    \label{eq_definition_hs}
\end{equation}
Then, conditional on the training data $D$, $\widehat{s}_w$ is fixed, and
\begin{align}
    \overline{\mathcal{E}}_w(\widehat{s}_w)
    &=
    \sum_{k=0}^K
    \pi_k
    \mathbb{E}_{(X,Z)\sim P_k}
    \left[
        h_{\widehat{s}_w}(X,Z)
    \right]
    \notag\\
    &=
    \mathbb{E}_{D'}
    \left[
        \sum_{k=0}^K
        \pi_k
        \frac{1}{n_k}
        \sum_{i=1}^{n_k}
        h_{\widehat{s}_w}
        (x_{k,i}',z_{k,i}')
        \,\middle|\, D
    \right].
    \label{eq_excess_risk_ghost_sample}
\end{align}
Consequently,
\begin{equation}
    \mathbb{E}_{D}
    \left[
        \overline{\mathcal{E}}_w(\widehat{s}_w)
    \right]
    =
    \mathbb{E}_{D}
    \mathbb{E}_{D'}
    \left[
        \sum_{k=0}^K
        \pi_k
        \frac{1}{n_k}
        \sum_{i=1}^{n_k}
        h_{\widehat{s}_w}
        (x_{k,i}',z_{k,i}')
    \right].
    \label{eq_expected_excess_risk_ghost}
\end{equation}

To control the contribution from the tail region, for some truncation
level $R>0$, define the truncated loss as:
\begin{equation}
    \ell^{\mathrm{tr}}(x,z;s)
    :=
    \ell(x,z;s)
    \mathbf{1}
    \left\{
        \|x\|_{\infty}\leq R, \  \|z\|_{\infty}\leq R
    \right\},
    \label{eq_truncated_loss}
\end{equation}
and the corresponding truncated excess loss as
\begin{align}
    h_s^{\mathrm{tr}}(x,z)
    &:=
    \ell^{\mathrm{tr}}(x,z;s)
    -
    \ell^{\mathrm{tr}}(x,z;s_w^o)
    \notag\\
    &=
    h_s(x,z)
    \mathbf{1}
    \left\{
        \|x\|_{\infty}\leq R , \  \|z\|_{\infty}\leq R
    \right\}.
    \label{eq_truncated_excess_loss}
\end{align}

We next define the empirical excess losses evaluated on the training
samples and the ghost samples, respectively. Let:
\begin{equation}
    H_1
    :=
    \sum_{k=0}^K
    \pi_k
    \frac{1}{n_k}
    \sum_{i=1}^{n_k}
    h_{\widehat{s}_w}
    (x_{k,i},z_{k,i}),
    \label{eq_H1}
\end{equation}
and
\begin{equation}
    H_1^{\mathrm{tr}}
    :=
    \sum_{k=0}^K
    \pi_k
    \frac{1}{n_k}
    \sum_{i=1}^{n_k}
    h_{\widehat{s}_w}^{\mathrm{tr}}
    (x_{k,i},z_{k,i}).
    \label{eq_H1_tr}
\end{equation}
Similarly, define
\begin{equation}
    H_2
    :=
    \sum_{k=0}^K
    \pi_k
    \frac{1}{n_k}
    \sum_{i=1}^{n_k}
    h_{\widehat{s}_w}
    (x_{k,i}',z_{k,i}'),
    \label{eq_H2}
\end{equation}
and
\begin{equation}
    H_2^{\mathrm{tr}}
    :=
    \sum_{k=0}^K
    \pi_k
    \frac{1}{n_k}
    \sum_{i=1}^{n_k}
    h_{\widehat{s}_w}^{\mathrm{tr}}
    (x_{k,i}',z_{k,i}').
    \label{eq_H2_tr}
\end{equation}

By (\ref{eq_expected_excess_risk_ghost}), we have
\begin{equation}
    \mathbb{E}_{D}
    \left[
        \overline{\mathcal{E}}_w(\widehat{s}_w)
    \right]
    =
    \mathbb{E}_{D}
    \mathbb{E}_{D'}
    \left[
        H_2
    \right].
\end{equation}
Adding and subtracting
$H_2^{\mathrm{tr}}$, $H_1^{\mathrm{tr}}$, and $H_1$ yields the
decomposition
\begin{align}
    \mathbb{E}_{D}
    \left[
        \overline{\mathcal{E}}_w(\widehat{s}_w)
    \right]
    &=
    \mathbb{E}_{D}
    \mathbb{E}_{D'}
    \left[
        H_2-H_2^{\mathrm{tr}}
    \right]
    \notag\\
    &\quad+
    \mathbb{E}_{D}
    \left[
        \mathbb{E}_{D'}
        \left[
            H_2^{\mathrm{tr}}
        \right]
        -
        H_1^{\mathrm{tr}}
    \right]
    \notag\\
    &\quad+
    \mathbb{E}_{D}
    \left[
        H_1^{\mathrm{tr}}-H_1
    \right] \notag \\
    &\quad+
    \mathbb{E}_{D}
    \left[
        H_1
    \right]
    \notag\\
    &:= A_2+B+A_1+C,
    \label{eq_four_term_decomposition}
\end{align}
where
\begin{align}
    A_2
    &:=
    \mathbb{E}_{D}
    \mathbb{E}_{D'}
    \left[
        H_2-H_2^{\mathrm{tr}}
    \right],
    \label{eq_A2_definition}\\
    B
    &:=
    \mathbb{E}_{D}
    \left[
        \mathbb{E}_{D'}
        \left[
            H_2^{\mathrm{tr}}
        \right]
        -
        H_1^{\mathrm{tr}}
    \right],
    \label{eq_B_definition}\\
    A_1
    &:=
    \mathbb{E}_{D}
    \left[
        H_1^{\mathrm{tr}}-H_1
    \right],
    \label{eq_A1_definition}\\
    C
    &:=
    \mathbb{E}_{D}
    \left[
        H_1
    \right].
    \label{eq_C_definition}
\end{align}

\subsubsection{Bounding $A_1$ and $A_2$}
\paragraph{Uniform bound on the score-matching loss.}
We first state the following uniform bound on the score-matching loss.

\begin{lem}[Uniform bound of $\ell(\cdot, \cdot, s)$;
{\cite{Fu2024UnveilCD}, Lemma~D.1 and Appendix~D.6.5}]
\label{lem_uniform_loss_bound_unbounded}

Suppose that the network parameters
$M_t,W,\kappa,L,S$ are configured according to
Proposition~C.4 of \cite{Fu2024UnveilCD}, and let
\begin{equation}
    m_t:=\frac{M_t}{\sqrt{\log N}},
    \qquad
    M_{\ell}:=\int_{t_0}^{T}m_t^2\,dt, \label{eq_mt_Ml}
\end{equation}
where $N$ is a sufficiently large number. 
Then, uniformly over
$s\in\mathcal{F}(M_t,W,\kappa,L,S)$ defined in
(\ref{eq_relu_network_class}) and
$(x,z)\in\mathbb{R}^{d_x}\times\mathbb{R}^{d_z}$,
\begin{equation}
    0
    \leq
    \ell^{\mathrm{tr}}(x,z;s)
    \leq
    \ell(x,z;s)
    \lesssim
    M_{\ell}.
    \label{eq_uniform_loss_bound_unbounded}
\end{equation}

Moreover, under the network configuration of \citet{Fu2024UnveilCD}'s 
Proposition~C.4,
\[
    m_t\lesssim \frac{1}{\sigma_t}.
\]
Hence, if
$t_0=N^{-C_{\sigma}}$ and $T=C_{\alpha}\log N$, then
\begin{equation}
    M_{\ell}
    \lesssim
    \int_{t_0}^{T}\frac{1}{\sigma_t^2}\,dt
    =
    \mathcal{O}
    \left(
        \log\frac{1}{t_0}
    \right)
    =
    \mathcal{O}(\log N).
    \label{eq_Mell_unbounded_order}
\end{equation}
\end{lem}

\begin{lem}[Truncation bound of ReLU network $s\in\mathcal{F}$;
{\cite{Fu2024UnveilCD}, (D.33) and Appendix~D.6.5}]
\label{lem_truncation_bound}

Suppose that Assumption~\ref{assump_smoothness_tail} holds for
$P_0,\ldots,P_K$, and let the jointly truncated loss be defined as
in (\ref{eq_truncated_loss}), namely,
\[
    \ell^{\mathrm{tr}}(x,z;s)
    :=
    \ell(x,z;s)
    \mathbf{1}
    \left\{
        \|x\|_\infty\leq R,\,
        \|z\|_\infty\leq R
    \right\}.
\]
Then, uniformly over $k=0,\ldots,K$ and $s\in\mathcal{F}$,
\begin{align}
    \mathbb{E}_{(X,Z)\sim P_k}
    \left[
        \left|
            \ell(X,Z;s)
            -
            \ell^{\mathrm{tr}}(X,Z;s)
        \right|
    \right]
    &\lesssim
    \left[
        \exp(-C_xR^2)
        +
        \exp(-C_zR^2)
    \right]
    R M_{\ell}
    \notag\\
    &\lesssim
    \exp(-C_{\mathrm{sg}}R^2)
    R M_{\ell},
    \label{eq_lemma_loss_truncation_bound}
\end{align}
where
\[
    C_{\mathrm{sg}}
    :=
    \min\{C_x,C_z\}>0,
\]
up to a universal multiplicative constant, and
$M_{\ell}$ is the uniform bound on the score-matching loss given in
Lemma~\ref{lem_uniform_loss_bound_unbounded}. The constant
$C_{\mathrm{sg}}$ is independent of $k$, $N$, and $R$.
\end{lem}

\begin{thm}[Uniform bound of $\ell(\cdot, \cdot, s_w^o)$ and truncation bound of the mixture score $s_w^o$]
\label{thm_truncation_bound_mixture_score}

Suppose that the conditions of
Lemma~\ref{thm_mixture_distribution_properties} and \ref{lem_uniform_loss_bound_unbounded} hold, and let
$\ell^{\mathrm{tr}}$ be defined as in
(\ref{eq_truncated_loss}). Suppose further that the network envelope is chosen such that
\begin{equation}
    m_t
    =
    \frac{M_t}{\sqrt{\log N}}
    \asymp
    \frac{1}{\sigma_t},
    \qquad
    M_\ell
    =
    \int_{t_0}^T m_t^2\,dt.
    \label{eq_Mell_scaling_mixture_score}
\end{equation}
Then, we have
\begin{equation}
    \ell(x,z;s_w^o)
    \lesssim
    \|x\|_2^2
    +
    M_\ell, \label{eq_mixture_score_loss_envelope}
\end{equation}
and for sufficiently large $R$, it holds that 
\begin{align}
    &\mathbb E_{(X,Z)\sim P_w}
    \left[
        \left|
            \ell(X,Z;s_w^o)
            -
            \ell^{\mathrm{tr}}(X,Z;s_w^o)
        \right|
    \right]
    \lesssim
    \exp(-C_{\mathrm{sg}}R^2)
    R M_\ell, \quad \text{and}
    \label{eq_truncation_bound_mixture_score} \\
    & \mathbb E_{(X,Z)\sim P_k}
    \left[
        \left|
            \ell(X,Z;s_w^o)
            -
            \ell^{\mathrm{tr}}(X,Z;s_w^o)
        \right|
    \right]
    \lesssim
    \exp(-C_{\mathrm{sg}}R^2)
    R M_\ell, \quad k=0,\ldots, K.
\end{align}
where $C_{\mathrm{sg}}>0$ is independent of
$N$, $R$, and the mixture weights.
\end{thm}

\begin{proof}
By Lemma~\ref{thm_mixture_distribution_properties}, (\ref{eq_mixture_conditional_density_structure}),
\begin{equation}
    p_{w,0}(x| z)
    =
    \exp\left(
        -{C_x\|x\|_2^2}/{2}
    \right)
    \widetilde f_w(x,z),
\end{equation}
where
\[
    \widetilde f_w(x,z)
    \in
    \mathcal H^\beta
    \left(
        \mathbb R^{d_x} \times \mathbb R^{d_z},B
    \right)
\]
and
\[
    0< C_f \leq
    \widetilde f_w(x,z)
    \leq B < \infty
\]
uniformly over $(x,z)$ and the mixture weights.

Define the Gaussian reference density
\begin{equation}
\label{eq_def_gx}
    g(x)
    :=
    C_G^{-1}
    \exp\left(
        -\frac{C_x\|x\|_2^2}{2}
    \right),
    \qquad
    C_G
    =
    \left(
        \frac{2\pi}{C_x}
    \right)^{d_x/2}.
\end{equation}
Combining (\ref{eq_mixture_conditional_density_structure}) and
(\ref{eq_def_gx}), we have
\begin{equation}
    {C_f}{C_G}
    \leq
    \frac{p_{w,0}(x| z)}{g(x)}
    \leq
    {C_G}{B},
    \qquad
    (x,z)\in\mathbb R^{d_x} \times \mathbb R^{d_z}.
    \label{eq_conditional_gaussian_comparison_mixture}
\end{equation}
Thus, $g$ is the density of
$\mathcal N(0,C_x^{-1}I_{d_x})$ and has the same Gaussian tail
factor in $x$ as $p_{w,0}(x| z)$. 

Define the transition kernel of the forward process of diffusion model in Section \ref{sec_diff_model} as
\begin{equation}
    q_t(x_t|x) = \frac{1}{(2\pi\sigma_t^2)^{d_x/2}}
    \exp\left(
        -\frac{\|x_t-a_tx\|_2^2}{2\sigma_t^2}
    \right). \label{eq_transition_kernel}
\end{equation}

Let $\widetilde p_t$ denote the diffused density obtained by taking
$g$ as the initial density. By the forward process defined in Section \ref{sec_diff_model} we obtain
\begin{align}
    p_{w,t}(x_t| z)
    &=
    \int_{\mathbb R^{d_x}}
    q_t(x_t|x) 
    p_{w,0}(x| z)\,dx,
    \label{eq_pwt_conditional_integral}
    \\
    \widetilde p_t(x_t)
    &=
    \int_{\mathbb R^{d_x}}
    q_t(x_t|x) 
    g(x)\,dx.
    \label{eq_ptilde_integral}
\end{align}
Since the Gaussian kernel is nonnegative,
(\ref{eq_conditional_gaussian_comparison_mixture}) implies
\begin{equation}
    {C_f}{C_G} \widetilde p_t(x_t)
    \leq
    p_{w,t}(x_t| z)
    \leq
    {C_G}{B}\widetilde p_t(x_t).
\end{equation}
Therefore,
\begin{equation}
    {C_f}{C_G}
    \leq
    \frac{p_{w,t}(x_t| z)}
         {\widetilde p_t(x_t)}
    \leq
    {C_G}{B}.
    \label{eq_diffused_density_comparison_mixture}
\end{equation}

By Bayes' rule, (\ref{eq_conditional_gaussian_comparison_mixture}) (\ref{eq_diffused_density_comparison_mixture}), and the truth that $X_t \perp \!\!\! \perp Z |X_0$,
\begin{align}
    p_w(x| x_t,z)
    &=
    \frac{
        q_t(x_t| x)
        p_{w,0}(x| z)
    }{
        p_{w,t}(x_t| z)
    }
    \notag\\
    &\leq
    \frac{B}{C_f}
    \frac{
        q_t(x_t| x)g(x)
    }{
        \widetilde p_t(x_t)
    }.
\end{align}
Hence,
\begin{equation}
    p_w(x| x_t,z)
    \leq
    \frac{B}{C_f}
    \widetilde p(x| x_t),
    \label{eq_posterior_comparison_mixture}
\end{equation}
where $\widetilde p(x| x_t)$ denotes the posterior distribution
under the Gaussian reference model (\ref{eq_def_gx}), i.e. $\widetilde p(x| x_t) = q_t(x_t|x) g(x) / \widetilde p_t(x_t)$. The posterior comparison in
(\ref{eq_posterior_comparison_mixture}) allows us to control
posterior moments under the mixture distribution by the corresponding
moments under the Gaussian reference model. In particular, for any
nonnegative measurable function $h$,
\begin{equation}
    \mathbb E_{X_0 \sim p_w(\cdot| x_t,z)}
    \left[
        h(X_0)
    \right]
    \leq
    \frac{B}{C_f}
    \mathbb E_{X_0 \sim \widetilde p(\cdot| x_t)}
    \left[
        h(X_0)
    \right].
    \label{eq_posterior_moment_comparison}
\end{equation}

Under the reference model (\ref{eq_def_gx}),
\[
    X_0
    \sim
    \mathcal N
    \left(
        0,C_x^{-1}I_{d_x}
    \right),
\]
and
\[
    X_t
    =
    a_tX_0+\sigma_t\varepsilon,
    \qquad
    \varepsilon\sim\mathcal N(0,I_{d_x}).
\]
Therefore,
\begin{equation}
    X_0| X_t=x_t
    \sim
    \mathcal N
    \left(
        \frac{a_t}{a_t^2+C_x\sigma_t^2}x_t,\,
        \frac{\sigma_t^2}
        {a_t^2+C_x\sigma_t^2}
        I_{d_x}
    \right).
    \label{eq_reference_gaussian_posterior_mixture}
\end{equation}
It follows that
\begin{align}
    \mathbb E_{X_0 \sim \widetilde p(\cdot| x_t)}
    \left[
        \left\|
            \frac{x_t-a_tX_0}{\sigma_t^2}
        \right\|_2^2
    \right] =
    \frac{C_x^2}
    {(a_t^2+C_x\sigma_t^2)^2}
    \|x_t\|_2^2
    +
    \frac{
        d_xa_t^2
    }{
        \sigma_t^2
        (a_t^2+C_x\sigma_t^2)
    }.
    \label{eq_reference_residual_mixture}
\end{align}
For the Ornstein--Uhlenbeck forward process,
\[
    a_t^2+\sigma_t^2=1,
\]
and hence
\[
    a_t^2+C_x\sigma_t^2
    \geq
    \min\{1,C_x\}>0.
\]
Thus,
\begin{equation}
    \mathbb E_{X_0 \sim \widetilde p(\cdot| x_t)}
    \left[
        \left\|
            \frac{x_t-a_tX_0}{\sigma_t^2}
        \right\|_2^2
    \right]
    \lesssim
    \|x_t\|_2^2
    +
    \frac{1}{\sigma_t^2}.
    \label{eq_reference_residual_bound_mixture}
\end{equation}

By (\ref{eq_sw_conditional_score}) in Lemma~\ref{thm_mixture_distribution_properties} and the denoising score-matching identity \citep{Vincent_2011_score_mat},
\begin{equation}
    s_w^o(x_t,z,t)
    =
    -
    \mathbb E_{X_0 \sim p_w(\cdot| x_t,z)}
    \left[
        \frac{x_t-a_tX_0}{\sigma_t^2}
    \right].
\end{equation}
Therefore, Jensen's inequality,
(\ref{eq_posterior_comparison_mixture}) and
(\ref{eq_reference_residual_bound_mixture}) imply
\begin{align}
    \|s_w^o(x_t,z,t)\|_2^2
    &\leq
    \mathbb E_{X_0 \sim p_w(\cdot| x_t,z)}
    \left[
        \left\|
            \frac{x_t-a_tX_0}{\sigma_t^2}
        \right\|_2^2
    \right]
    \notag\\
    & {\lesssim}   \mathbb E_{X_0 \sim \widetilde p(\cdot| x_t)}
    \left[
        \left\|
            \frac{x_t-a_tX_0}{\sigma_t^2}
        \right\|_2^2
    \right] \notag\\
    &\lesssim
    \|x_t\|_2^2
    +
    \frac{1}{\sigma_t^2}.
    \label{eq_mixture_score_envelope}
\end{align}
Importantly, this bound is uniform over
$z\in\mathbb R^{d_z}$ and the mixture weights.

We next derive an envelope for the score-matching loss. For fixed
$(x,z)$,
\begin{align}
    &
    \mathbb E_{
        X_t\sim
        \mathcal N(a_tx,\sigma_t^2I_{d_x})
    }
    \left[
        \left\|
            s_w^o(X_t,z,t)
            +
            \frac{X_t-a_tx}{\sigma_t^2}
        \right\|_2^2
    \right]
    \notag\\
    &\quad\leq
    2
    \mathbb E_{
        X_t\sim
        \mathcal N(a_tx,\sigma_t^2I_{d_x})
    }
    \left[
        \|s_w^o(X_t,z,t)\|_2^2
    \right]
    +
    2
    \mathbb E_{
        X_t\sim
        \mathcal N(a_tx,\sigma_t^2I_{d_x})
    }
    \left[
        \left\|
            \frac{X_t-a_tx}{\sigma_t^2}
        \right\|_2^2
    \right]
    \notag\\
    &\quad\lesssim
    \mathbb E_{
        X_t\sim
        \mathcal N(a_tx,\sigma_t^2I_{d_x})
    }
    \left[
        \|X_t\|_2^2
    \right]
    +
    \frac{1}{\sigma_t^2}.
\end{align}
Since
\begin{equation}
    \mathbb E_{
        X_t\sim
        \mathcal N(a_tx,\sigma_t^2I_{d_x})
    }
    \left[
        \|X_t\|_2^2
    \right]
    =
    a_t^2\|x\|_2^2
    +
    d_x\sigma_t^2
    \lesssim
    \|x\|_2^2+1,
\end{equation}
we obtain
\begin{equation}
    \mathbb E_{
        X_t\sim
        \mathcal N(a_tx,\sigma_t^2I_{d_x})
    }
    \left[
        \left\|
            s_w^o(X_t,z,t)
            +
            \frac{X_t-a_tx}{\sigma_t^2}
        \right\|_2^2
    \right]
    \lesssim
    \|x\|_2^2
    +
    \frac{1}{\sigma_t^2}.
\end{equation}
Integrating over $t\in[t_0,T]$ and using
(\ref{eq_Mell_scaling_mixture_score}), 
\begin{equation}
    \boxed{\ell(x,z;s_w^o)
    \lesssim
    \|x\|_2^2
    +
    M_\ell,}
\end{equation}
which proves (\ref{eq_mixture_score_loss_envelope}) and this envelope is uniform in $z$. Now define the truncation region 
\begin{equation}
\label{eq_trunc_region}
    \mathcal A_R
    :=
    \left\{
        (x,z):
        \|x\|_\infty\leq R,\,
        \|z\|_\infty\leq R
    \right\}.
\end{equation}
Since
\[
    \ell^{\mathrm{tr}}(x,z;s_w^o)
    =
    \ell(x,z;s_w^o)
    \mathbf 1_{\mathcal A_R}(x,z),
\]
we have
\begin{align}
    \mathbb E_{(X,Z) \sim P_w}
    \left[
        \left|
            \ell(X,Z;s_w^o)
            -
            \ell^{\mathrm{tr}}(X,Z;s_w^o)
        \right|
    \right] &=
    \mathbb E_{P_w}
    \left[
        \ell(X,Z;s_w^o)
        \mathbf 1_{\mathcal A_R^c}(X,Z)
    \right]
    \notag\\
    &=
    \underbrace{
    \mathbb E_{P_w}
    \left[
        \ell(X,Z;s_w^o)
        \mathbf 1_{\{\|X\|_\infty>R\}}
    \right]
    }_{=:I_x}
    \notag\\
    &\quad+
    \underbrace{
    \mathbb E_{P_w}
    \left[
        \ell(X,Z;s_w^o)
        \mathbf 1_{\{
            \|X\|_\infty\leq R,\,
            \|Z\|_\infty>R
        \}}
    \right]
    }_{=:I_z}.
    \label{eq_mixture_truncation_decomposition}
\end{align}

By Lemma~\ref{thm_mixture_distribution_properties},
\begin{equation}
    p_{w,0}(x,z)
    \lesssim
    B
    \exp\left(
        -{C_x\|x\|_2^2}/{2}
        -{C_z\|z\|_2^2}/{2}
    \right).
\end{equation}
Combining this with
(\ref{eq_mixture_score_loss_envelope}),
\begin{align}
    I_x
    &\lesssim
    \int_{\|x\|_\infty>R}
    \int_{\mathbb R^{d_z}}
    \left(
        \|x\|_2^2+M_\ell
    \right)
    \exp\left(
        -{C_x\|x\|_2^2}/{2}
        -{C_z\|z\|_2^2}/{2}
    \right)
    dz\,dx
    \notag\\
    &\lesssim
    \exp(-C_x'R^2)
    R M_\ell
    \label{eq_mixture_x_tail}
\end{align}
for some $C_x'>0$.

Similarly,
\begin{align}
    I_z
    &\leq
    \mathbb E_{P_w}
    \left[
        \ell(X,Z;s_w^o)
        \mathbf 1_{\{\|Z\|_\infty>R\}}
    \right]
    \notag\\
    &\lesssim
    \int_{\|z\|_\infty>R}
    \int_{\mathbb R^{d_x}}
    \left(
        \|x\|_2^2+M_\ell
    \right)
    \exp\left(
        -{C_x\|x\|_2^2}/{2}
        -{C_z\|z\|_2^2}/{2}
    \right)
    dx\,dz
    \notag\\
    &\lesssim
    \exp(-C_z'R^2)
    R M_\ell
    \label{eq_mixture_z_tail}
\end{align}
for some $C_z'>0$.

The last inequalities follow from the Gaussian tail-moment bound;
any polynomial factor in $R$ can be absorbed into the exponential
term after decreasing the corresponding exponential constant.

Combining
(\ref{eq_mixture_truncation_decomposition}),
(\ref{eq_mixture_x_tail}), and
(\ref{eq_mixture_z_tail}),
\begin{equation}
\boxed{
\begin{aligned}
    E_{(X,Z) \sim P_w}
    \left[
        \left|
            \ell(X,Z;s_w^o)
            -
            \ell^{\mathrm{tr}}(X,Z;s_w^o)
        \right|
    \right] & \quad\lesssim
    \left[
        \exp(-C_x'R^2)
        +
        \exp(-C_z'R^2)
    \right]
    R M_\ell
    \\
    &\quad\lesssim
    \exp(-C_{\mathrm{sg}}R^2)
    R M_\ell, \label{eq_E_l_l_tr_bound}
\end{aligned}
}
\end{equation}
where
\begin{equation}
    C_{\mathrm{sg}}
    :=
    \min\{C_x',C_z'\}>0.
\end{equation}
The constants are independent of the mixture weights because
$B$, $C_f$, $C_F$, $C_x$, and $C_z$ are uniform over
$k=0,\ldots,K$.
This completes the proof.

The same truncation bound also holds uniformly under each
$P_k$, $k=0,\ldots,K$. Indeed, by
Assumption~\ref{assump_smoothness_tail},
\begin{equation}
    p_k(x,z)
    \lesssim
    B
    \exp\left(
        -{C_x\|x\|_2^2}/{2}
        -{C_z\|z\|_2^2}/{2}
    \right),
    \qquad
    k=0,\ldots,K.
\end{equation}
Moreover, the loss envelope
(\ref{eq_mixture_score_loss_envelope}),
\[
    \ell(x,z;s_w^o)
    \lesssim
    \|x\|_2^2+M_\ell,
\]
holds pointwise and uniformly in $z$ and the mixture weights.
Therefore, applying the same decomposition as (\ref{eq_mixture_truncation_decomposition}) gives
\begin{equation}
\boxed{
\begin{aligned}
    &\mathbb E_{(X,Z)\sim P_k}
    \left[
        \left|
            \ell(X,Z;s_w^o)
            -
            \ell^{\mathrm{tr}}(X,Z;s_w^o)
        \right|
    \right]
    \notag\\
    &\quad\lesssim
    \int_{\|x\|_\infty>R}
    \int_{\mathbb R^{d_z}}
    \left(
        \|x\|_2^2+M_\ell
    \right)
    \exp\left(
        -{C_x\|x\|_2^2}/{2}
        -{C_z\|z\|_2^2}/{2}
    \right)
    dz\,dx
    \notag\\
    &\qquad+
    \int_{\|z\|_\infty>R}
    \int_{\mathbb R^{d_x}}
    \left(
        \|x\|_2^2+M_\ell
    \right)
    \exp\left(
        -{C_x\|x\|_2^2}/{2}
        -{C_z\|z\|_2^2}/{2}
    \right)
    dx\,dz
    \notag\\
    &\quad\lesssim
    \exp(-C_{\mathrm{sg}}R^2)
    R M_\ell,
    \qquad
    k=0,\ldots,K.
    \label{eq_sw_mixture_truncation_bound_Pk}
\end{aligned}
}
\end{equation}
Thus, the same bound holds uniformly over all component
distributions $P_k$.

\end{proof}

\begin{thm}[Bounds of terms $A_1$ and $A_2$.]
Suppose that the conditions of Lemma \ref{lem_truncation_bound} and Theorem \ref{thm_truncation_bound_mixture_score} hold. We can bound the term $|A_1| + |A_2|$ in (\ref{eq_four_term_decomposition}) by
    \begin{equation}
    |A_1|+|A_2|
    =
    \mathcal{O}
    \left(
        \exp\left(-C_{\mathrm{sg}}R^2\right)
        R M_{\ell}
    \right).
    \label{eq_A1_A2_final_bound}
\end{equation}
\end{thm}

\begin{proof}
By the triangle inequality: 
\begin{align}
    \left|
        h_s(x,z)-h_s^{\mathrm{tr}}(x,z)
    \right|
    \leq
    \left|
        \ell(x,z;s)-\ell^{\mathrm{tr}}(x,z;s)
    \right| + 
    \left|
        \ell(x,z;s_w^o)
        -
        \ell^{\mathrm{tr}}(x,z;s_w^o)
    \right|.
\end{align}
Consequently, by (\ref{eq_lemma_loss_truncation_bound}) and (\ref{eq_sw_mixture_truncation_bound_Pk}),
\begin{equation}
    \mathbb{E}_{(x,z)\sim P_k}
    \left[
        \left|
            h_s(x,z)-h_s^{\mathrm{tr}}(x,z)
        \right|
    \right]
    \lesssim
    \exp\left(-C_{\mathrm{sg}}R^2\right)
    R M_{\ell}.
    \label{eq_lemma_excess_truncation_bound}
\end{equation}

We now apply Lemma~\ref{lem_truncation_bound} and Theorem \ref{thm_truncation_bound_mixture_score} to bound the two
truncation terms $A_1$ and $A_2$ in (\ref{eq_four_term_decomposition}).

Recall that
\begin{equation}
    A_1
    =
    \mathbb{E}_{D}
    \left[
        H_1^{\mathrm{tr}}-H_1
    \right].
\end{equation}
By the triangle inequality and
(\ref{eq_lemma_excess_truncation_bound}),
\begin{align}
    |A_1|
    &\leq
    \mathbb{E}_{D}
    \left[
        \left|
            H_1^{\mathrm{tr}}-H_1
        \right|
    \right]
    \notag\\
    &\leq
    \sum_{k=0}^K
    \pi_k
    \frac{1}{n_k}
    \sum_{i=1}^{n_k}
    \mathbb{E}_{D}
    \left[
        \left|
            h_{\widehat{s}_w}^{\mathrm{tr}}
            (X_{k,i},Z_{k,i})
            -
            h_{\widehat{s}_w}
            (X_{k,i},Z_{k,i})
        \right|
    \right]
    \notag\\
    &\lesssim
    \sum_{k=0}^K
    \pi_k
    \exp\left(-C_{\mathrm{sg}}R^2\right)
    R M_{\ell}
    \notag\\
    &=
    \mathcal{O}
    \left(
        \exp\left(-C_{\mathrm{sg}}R^2\right)
        R M_{\ell}
    \right),
    \label{eq_A1_bound}
\end{align}
where the last equality follows from
$\sum_{k=0}^K\pi_k=1$.

Next we bound $A_2
    =
    \mathbb{E}_{D}
    \mathbb{E}_{D'}
    \left[
        H_2-H_2^{\mathrm{tr}}
    \right]$ by the similar procedure. Again, by (\ref{eq_lemma_excess_truncation_bound}), we have:
\begin{align}
    |A_2|
    &\leq
    \mathbb{E}_{D}
    \mathbb{E}_{D'}
    \left[
        \left|
            H_2-H_2^{\mathrm{tr}}
        \right|
    \right]
    \notag\\
    &\leq
    \sum_{k=0}^K
    \pi_k
    \frac{1}{n_k}
    \sum_{i=1}^{n_k}
    \mathbb{E}_{D}
    \mathbb{E}_{D'}
    \left[
        \left|
            h_{\widehat{s}_w}
            (X_{k,i}',Z_{k,i}')
            -
            h_{\widehat{s}_w}^{\mathrm{tr}}
            (X_{k,i}',Z_{k,i}')
        \right|
    \right]
    \notag\\
    &\lesssim
    \sum_{k=0}^K
    \pi_k
    \exp\left(-C_{\mathrm{sg}}R^2\right)
    R M_{\ell}
    \notag\\
    &=
    \mathcal{O}
    \left(
        \exp\left(-C_{\mathrm{sg}}R^2\right)
        R M_{\ell}
    \right).
    \label{eq_A2_bound}
\end{align}

Combining (\ref{eq_A1_bound}) and (\ref{eq_A2_bound}), we obtain
\begin{equation}
    \boxed{
    |A_1|+|A_2|
    =
    \mathcal{O}
    \left(
        \exp\left(-C_{\mathrm{sg}}R^2\right)
        R M_{\ell}
    \right).
    }
\end{equation}

\end{proof}

\subsubsection{Bounding $C$}

Recall  (\ref{eq_C_definition}) that
\begin{equation}
    C
    :=
    \mathbb{E}_{D}\left[H_1\right],
\end{equation}
where $H_1$ is defined in (\ref{eq_H1}).
By the definition of $h_s$ in (\ref{eq_definition_hs}),
\begin{align}
    H_1
    &=
    \sum_{k=0}^K
    \pi_k
    \frac{1}{n_k}
    \sum_{i=1}^{n_k}
    \left[
        \ell(x_{k,i},z_{k,i};\widehat{s}_w)
        -
        \ell(X_{k,i},Z_{k,i};s_w^o)
    \right]
    \notag\\
    &=
    \widehat{\mathcal{L}}_w(\widehat{s}_w)
    -
    \widehat{\mathcal{L}}_w(s_w^o),
    \label{eq_H1_empirical_risk}
\end{align}
where 
\begin{equation}
    \widehat{\mathcal{L}}_w(s)
    :=
    \sum_{k=0}^K
    \pi_k
    \frac{1}{n_k}
    \sum_{i=1}^{n_k}
    \ell(x_{k,i},z_{k,i};s)
\end{equation}
denotes the weighted empirical risk, defined in (\ref{eq_weighted_erm_objective}). Since $\widehat{s}_w$ is an empirical risk minimizer over
$\mathcal{F}$,
\begin{equation}
    \widehat{\mathcal{L}}_w(\widehat{s}_w)
    \leq
    \widehat{\mathcal{L}}_w(s),
    \qquad
    \text{for every } s\in\mathcal{F}.
\end{equation}
Therefore, for any fixed $s\in\mathcal{F}$,
\begin{align}
    C
    =
    \mathbb{E}_{D}
    \left[
        \widehat{\mathcal{L}}_w(\widehat{s}_w)
        -
        \widehat{\mathcal{L}}_w(s_w^o)
    \right] \leq
    \mathbb{E}_{D}
    \left[
        \widehat{\mathcal{L}}_w(s)
        -
        \widehat{\mathcal{L}}_w(s_w^o)
    \right].
    \label{eq_C_erm_bound}
\end{align}
Since both $s$ and $s_w^o$ are fixed with respect to the training
samples, the empirical risks are unbiased estimators of their
corresponding population risks. Because of (\ref{eq_score_matching_risk_decomposition}),
\begin{align}
    \mathbb{E}_{D}
    \left[
        \widehat{\mathcal{L}}_w(s)
    \right]
    &=
    \sum_{k=0}^K
    \pi_k
    \mathcal{L}_k(s)
    =
    \sum_{k=0}^K
    \pi_k
    \left[ \mathcal{R}_k(s) - C_{\mathrm{DSM},k} \right] = \overline{\mathcal{R}}_w(s) - \sum_{k=0}^K \pi_k  C_{\mathrm{DSM},k}
    \\
    \mathbb{E}_{D}
    \left[
        \widehat{\mathcal{L}}_w(s_w^o)
    \right]
    &=
    \sum_{k=0}^K
    \pi_k
    \mathcal{L}_k(s_w^o)
    =
    \sum_{k=0}^K
    \pi_k
    \left[ \mathcal{R}_k(s_w^o) - C_{\mathrm{DSM},k} \right] = \overline{\mathcal{R}}_w(s_w^o) - \sum_{k=0}^K \pi_k  C_{\mathrm{DSM},k}.
\end{align}
Consequently,
\begin{align}
    C
    &\leq
    \overline{\mathcal{R}}_w(s)
    -
    \overline{\mathcal{R}}_w(s_w^o) =
    \overline{\mathcal{E}}_w(s) =
    \|s-s_w^o\|_{\mathcal{H}_w}^2,
    \qquad
    \forall s\in\mathcal{F}.
\end{align}
Taking the infimum over $s\in\mathcal{F}$ gives
\begin{equation}
    \boxed{
    C
    \leq
    \inf_{s\in\mathcal{F}}
    \overline{\mathcal{E}}_w(s) =
    \inf_{s\in\mathcal{F}} \|s-s_w^o\|_{\mathcal{H}_w}^2.
    }
    \label{eq_C_approximation_error}
\end{equation}

Thus, the term $C$ is controlled by the approximation error of the
network class $\mathcal{F}$ to the population-optimal score
$s_w^o$ under the $\mathcal{H}_w$ norm. Next, we invoke the score approximation result of
\cite{Fu2024UnveilCD}.

\begin{lem}[Conditional score approximation with unbounded covariates;
adapted from {\cite{Fu2024UnveilCD}, Proposition~C.4}]
\label{lem_conditional_score_approximation}

Suppose that $z\in\mathbb R^{d_z}$ and the conditional density
$p(x| z)$ satisfies
\begin{equation}
    p(x| z)
    =
    \exp\left(
        -{C_x\|x\|_2^2}/{2}
    \right)
    f(x,z),
\end{equation}
where
\begin{equation}
    f
    \in
    \mathcal H^\beta
    \left(
        \mathbb R^{d_x} \times \mathbb R^{d_z},B
    \right),
    \qquad
    \inf_{(x,z)\in\mathbb R^{d_x} \times \mathbb R^{d_z}}
    f(x,z)
    \geq
    C_f>0.
\end{equation}
Suppose further that the marginal density of $Z$ has a
sub-Gaussian tail, i.e.,
\begin{equation}
    p(z)
    \lesssim
    \exp\left(
        -{C_z\|z\|_2^2}/{2}
    \right)
\end{equation}
for some constant $C_z>0$. For $t>0$, let $p_t(x_t|z)$ denote the conditional density obtained by applying the forward diffusion process to $X$ while keeping $Z$ unchanged, i.e.,
\begin{equation}
p_t(x_t|z)
:=
\int_{\mathbb R^{d_x}}
q_t(x_t|x)
p(x|z),dx,
\label{eq_conditional_diffused_density}
\end{equation}
where $q_t(x_t|x)$ is defined in (\ref{eq_transition_kernel}).
Same with Lemma \ref{lem_uniform_loss_bound_unbounded}, for sufficiently large $N$ and constants
$C_{\sigma},C_{\alpha}>0$, let
\begin{equation}
    t_0=N^{-C_{\sigma}},
    \qquad
    T=C_{\alpha}\log N.
\end{equation}
Then there exists a ReLU network
$s\in\mathcal F(M_t,W,\kappa,L,S)$ such that, for every
$t\in[t_0,T]$,
\begin{align}
    \mathbb E_{Z \sim p(z)}
    \left[
        \int_{\mathbb R^{d_x}}
        \left\|
            s(x_t,Z,t)
            -
            \nabla_{x_t}
            \log p_t(x_t| Z)
        \right\|_2^2
        p_t(x_t| Z)
        \,dx_t
    \right] \lesssim
    \frac{B^2}{\sigma_t^2}
    N^{-\frac{2\beta}{d_x+d_z}}
    (\log N)^{\beta+1}.
    \label{eq_conditional_score_approximation}
\end{align}
The corresponding network parameters can be chosen such that
\begin{equation}
    M_t
    =
    \mathcal O
    \left(
        \frac{\sqrt{\log N}}{\sigma_t}
    \right),
    \qquad
    W
    =
    \mathcal O
    \left(
        N\log^7 N
    \right),
\end{equation}
and
\begin{equation}
    \kappa
    =
    \exp\left(
        \mathcal O(\log^4 N)
    \right),
    \qquad
    L
    =
    \mathcal O(\log^4 N),
    \qquad
    S
    =
    \mathcal O
    \left(
        N\log^9 N
    \right),
\end{equation}
where the choice of $M_t$ coincides with (\ref{eq_Mell_scaling_mixture_score}). 
\end{lem}

Next, we give the bound of term $C$ defined in (\ref{eq_four_term_decomposition}).

\begin{thm}[Approximation error of the weighted score]
\label{thm_approx_err_of_diff}
    Suppose that the conditions of Lemma \ref{thm_mixture_distribution_properties} and \ref{lem_conditional_score_approximation} hold. Same with Lemma \ref{lem_uniform_loss_bound_unbounded}, let
    \begin{equation*}
        t_0=N^{-C_{\sigma}}, \qquad T=C_{\alpha}\log N,
    \end{equation*}
    where $C_\sigma$ and $C_\alpha >0$ are constants. Then, for sufficiently large $N$, there exists $s_N\in\mathcal F(M_t,W,\kappa,L,S)$ such that 
    \begin{equation}
        \|s_N-s_w^o\|_{\mathcal H_w}^2
    \lesssim
    B^2
    N^{-\frac{2\beta}{d_x+d_z}}
    (\log N)^{\beta+1}.
    \end{equation}
    Consequently, 
    \begin{equation}
        \inf_{s\in\mathcal F}
    \|s-s_w^o\|_{\mathcal H_w}^2
    \lesssim
    B^2
    N^{-\frac{2\beta}{d_x+d_z}}
    (\log N)^{\beta+1}.
    \end{equation}
    Therefore, the approximation term $C$ defined in (\ref{eq_four_term_decomposition}) satisfies
    \begin{equation}
        C \lesssim
    B^2
    N^{-\frac{2\beta}{d_x+d_z}}
    (\log N)^{\beta+1}.
    \end{equation}
    Since $B$ is the radius of the H\"older ball and it is a constant, we have 
    \begin{equation}
        C = \mathcal{O} \left( 
    N^{-\frac{2\beta}{d_x+d_z}}
    (\log N)^{\beta+1} \right).
    \end{equation}
\end{thm}

\begin{proof}
As established in
Lemma~\ref{thm_mixture_distribution_properties},
the mixture conditional density $p_{w,t}(x_t| z)$ is defined in
(\ref{eq_pwt_conditional_mixture});
the fact that $p_{w,0}(x| z)$ satisfies the smoothness and tail
conditions required by
Lemma~\ref{lem_conditional_score_approximation}
is established in
(\ref{eq_mixture_conditional_density_structure})--(\ref{eq_pw_z_subgaussian});
and
\[
    s_w^o(x_t,z,t)
    =
    \nabla_{x_t}\log p_{w,t}(x_t\mid z)
\]
is proved in (\ref{eq_sw_conditional_score}).
Therefore,
Lemma~\ref{lem_conditional_score_approximation}
can be applied directly to the weighted conditional density $p_{w,0}(x|z)$.  For sufficiently large $N$, there exists
\[
    s_N\in\mathcal F(M_t,W,\kappa,L,S)
\]
such that, for every $t\in[t_0,T]$,
\begin{align}
    \mathbb E_{Z\sim p_w(z)}
    \left[
        \int_{\mathbb R^{d_x}}
        \left\|
            s_N(x_t,Z,t)
            -
            s_w^o(x_t,Z,t)
        \right\|_2^2
        p_{w,t}(x_t| Z)
        \,dx_t
    \right] \lesssim
    \frac{B^2}{\sigma_t^2}
    N^{-\frac{2\beta}{d_x+d_z}}
    (\log N)^{\beta+1}.
    \label{eq_weighted_score_approximation}
\end{align}
Equivalently, since
\[
    p_{w,t}(x_t,z)
    =
    p_{w,t}(x_t| z)p_w(z),
\]
we have
\begin{equation}
    \mathbb E_{(X_t,Z)\sim p_{w,t}}
    \left[
        \left\|
            s_N(X_t,Z,t)
            -
            s_w^o(X_t,Z,t)
        \right\|_2^2
    \right]
    \lesssim
    \frac{B^2}{\sigma_t^2}
    N^{-\frac{2\beta}{d_x+d_z}}
    (\log N)^{\beta+1}.
    \label{eq_weighted_score_approximation_joint}
\end{equation}

By the definition of the $\mathcal H_w$ norm in (\ref{eq_inner_product_Hw}),
\begin{align}
    \|s_N-s_w^o\|_{\mathcal H_w}^2
    &=
    \frac{1}{T-t_0}
    \int_{t_0}^{T}
    \mathbb E_{(X_t,Z)\sim p_{w,t}}
    \left[
        \left\|
            s_N(X_t,Z,t)
            -
            s_w^o(X_t,Z,t)
        \right\|_2^2
    \right]
    \,dt
    \notag\\
    &\lesssim
    B^2
    N^{-\frac{2\beta}{d_x+d_z}}
    (\log N)^{\beta+1}
    \frac{1}{T-t_0}
    \int_{t_0}^{T}
    \frac{1}{\sigma_t^2}
    \,dt.
    \label{eq_Hw_approximation_before_time}
\end{align}

To evaluate the remaining time integral, define
\begin{equation}
    I_{\sigma}
    :=
    \int_{t_0}^{T}
    \frac{1}{\sigma_t^2}
    \,dt \asymp M_\ell.
    \label{eq_Isigma_definition}
\end{equation}
Under the choices
\begin{equation}
    t_0=N^{-C_{\sigma}},
    \qquad
    T=C_{\alpha}\log N,
    \label{eq_t0_T_approximation}
\end{equation}
where $C_{\sigma},C_{\alpha}>0$ are constants, we next determine the
order of $I_{\sigma}$. 

For the Ornstein--Uhlenbeck forward process defined in
Section~\ref{sec_diff_model},
\[
    \sigma_t^2=1-e^{-t}.
\]
Hence,
\begin{align}
    I_{\sigma}
    &=
    \int_{t_0}^{T}
    \frac{1}{1-e^{-t}}
    \,dt
    \notag\\
    &=
    \left[
        \log(e^t-1)
    \right]_{t_0}^{T}
    \notag\\
    &=
    \log
    \left(
        \frac{e^T-1}{e^{t_0}-1}
    \right).
    \label{eq_Isigma_exact}
\end{align}

Since $T=C_{\alpha}\log N$, we have
\begin{align}
    \log(e^T-1)
    &=
    \log
    \left(
        N^{C_{\alpha}}-1
    \right)
    \notag\\
    &=
    C_{\alpha}\log N
    +
    \log
    \left(
        1-N^{-C_{\alpha}}
    \right)
    \notag\\
    &=
    C_{\alpha}\log N+o(1).
    \label{eq_upper_time_order}
\end{align}
On the other hand, since
$t_0=N^{-C_{\sigma}}\to0$, the expansion
\[
    e^{t_0}-1
    =
    t_0\left(1+\mathcal O(t_0)\right)
\]
gives
\begin{align}
    \log(e^{t_0}-1)
    &=
    \log t_0
    +
    \log
    \left(
        \frac{e^{t_0}-1}{t_0}
    \right)
    \notag\\
    &=
    -C_{\sigma}\log N+o(1).
    \label{eq_lower_time_order}
\end{align}
Combining
(\ref{eq_Isigma_exact})--(\ref{eq_lower_time_order}),
we obtain
\begin{equation}
    I_{\sigma}
    =
    (C_{\alpha}+C_{\sigma})\log N
    +
    o(\log N)
    =
    \mathcal O(\log N) \asymp M_\ell.
    \label{eq_Isigma_order}
\end{equation}

Meanwhile,
\begin{align}
    T-t_0
    &=
    C_{\alpha}\log N-N^{-C_{\sigma}}
    \notag\\
    &=
    C_{\alpha}\log N
    \left(1+o(1)\right)
    =
    \mathcal O(\log N).
    \label{eq_time_interval_order}
\end{align}
Therefore,
\begin{align}
    \frac{I_{\sigma}}{T-t_0}
    &=
    \frac{
        (C_{\alpha}+C_{\sigma})\log N
        +o(\log N)
    }{
        C_{\alpha}\log N
        +o(\log N)
    }
    \notag\\
    &=
    \frac{C_{\alpha}+C_{\sigma}}
         {C_{\alpha}}
    +o(1)
    \notag\\
    &=
    1+\frac{C_{\sigma}}{C_{\alpha}}
    +o(1)
    =
    \mathcal O(1).
    \label{eq_sigma_integral_order}
\end{align}

Substituting (\ref{eq_sigma_integral_order}) into
(\ref{eq_Hw_approximation_before_time}) yields
\begin{equation}
    \|s_N-s_w^o\|_{\mathcal H_w}^2
    \lesssim
    B^2
    N^{-\frac{2\beta}{d_x+d_z}}
    (\log N)^{\beta+1}.
    \label{eq_Hw_approximation_final}
\end{equation}
Since $s_N\in\mathcal F$, it follows that
\begin{equation}
    \boxed{
    \inf_{s\in\mathcal F}
    \|s-s_w^o\|_{\mathcal H_w}^2
    \leq
    \|s_N-s_w^o\|_{\mathcal H_w}^2
    \lesssim
    B^2
    N^{-\frac{2\beta}{d_x+d_z}}
    (\log N)^{\beta+1}.
    }
    \label{eq_inf_Hw_approximation}
\end{equation}

Combining this with
(\ref{eq_C_approximation_error}), we obtain
\begin{equation}
    \boxed{
    C
    \lesssim
    B^2
    N^{-\frac{2\beta}{d_x+d_z}}
    (\log N)^{\beta+1}.
    }
    \label{eq_C_approximation_bound}
\end{equation}
By
Lemma~\ref{thm_mixture_distribution_properties},
$B_w$ is uniformly bounded over the mixture weights and is independent
of the network complexity parameter $N$. Therefore,
\begin{equation}
    \boxed{
    C
    =
    \mathcal O
    \left(
        N^{-\frac{2\beta}{d_x+d_z}}
        (\log N)^{\beta+1}
    \right).
    }
    \label{eq_C_approximation_rate}
\end{equation}

\end{proof}

\subsubsection{Bounding $B$}
\paragraph{Population and empirical operators.}
For convenience, we introduce the population and empirical operators
used in the subsequent analysis. For any measurable function
$f:\mathbb{R}^{d_x}\times\mathbb{R}^{d_z}\rightarrow\mathbb{R}$,
define the population operator associated with the $k$-th domain by
\begin{equation}
    \mathbb{P}_k f
    :=
    \mathbb{E}_{(X,Z)\sim P_k}
    \left[
        f(X,Z)
    \right],
    \qquad
    k=0,\ldots,K.
    \label{eq_population_operator_domain}
\end{equation}
The corresponding empirical operator is defined as
\begin{equation}
    \mathbb{P}_{k,n_k}f
    :=
    \frac{1}{n_k}
    \sum_{i=1}^{n_k}
    f(x_{k,i},z_{k,i}).
    \label{eq_empirical_operator_domain}
\end{equation}
Recall that the weighted density defined in (\ref{eq_mixture_density_pw}) is
\[
    p_{w,0}(x,z) = \sum_{k=0}^{K} \pi_k\, p_{k,0}(x,z),
    \qquad
    \pi_k\geq 0,
    \qquad
    \sum_{k=0}^K\pi_k=1.
\]
Accordingly, define the weighted population operator by
\begin{align}
    \mathbb{P}_w f
    &:=
    \sum_{k=0}^K
    \pi_k\mathbb{P}_k f
    \notag\\
    &=
    \sum_{k=0}^K
    \pi_k
    \mathbb{E}_{(X,Z)\sim P_k}
    \left[
        f(X,Z)
    \right]
    \notag\\
    &=
    \mathbb{E}_{(X,Z)\sim P_w}
    \left[
        f(X,Z)
    \right].
    \label{eq_weighted_population_operator}
\end{align}
Similarly, define the weighted empirical operator by
\begin{align}
    \mathbb{P}_{w,\mathbf{n}}f
    &:=
    \sum_{k=0}^K
    \pi_k\mathbb{P}_{k,n_k}f
    \notag\\
    &=
    \sum_{k=0}^K
    \frac{\pi_k}{n_k}
    \sum_{i=1}^{n_k}
    f(x_{k,i},z_{k,i}).
    \label{eq_weighted_empirical_operator}
\end{align}
Hence, the weight assigned to each individual observation from the
$k$-th domain is
\begin{equation}
    \frac{\pi_k}{n_k}
    =
    \begin{cases}
        \displaystyle
        \frac{1}{n_0+W_N},
        & k=0,
        \\[3mm]
        \displaystyle
        \frac{w_k}{n_0+W_N},
        & k=1,\ldots,K.
    \end{cases}
    \label{eq_individual_sample_weight}
\end{equation}
Therefore, the weighted empirical operator can equivalently be written
as
\begin{equation}
    \mathbb{P}_{w,\mathbf{n}}f
    =
    \frac{
        \displaystyle
        \sum_{i=1}^{n_0}
        f(x_{0,i},z_{0,i})
        +
        \sum_{k=1}^K
        w_k
        \sum_{i=1}^{n_k}
        f(x_{k,i},z_{k,i})
    }{
        n_0+W_N
    }.
    \label{eq_weighted_empirical_operator_explicit}
\end{equation}
Define the ghost-sample empirical operator for the $k$-th domain by
\begin{equation}
    \mathbb{P}_{k,n_k}'f
    :=
    \frac{1}{n_k}
    \sum_{i=1}^{n_k}
    f(x_{k,i}',z_{k,i}'),
\end{equation}
and define the corresponding weighted ghost-sample empirical operator
as
\begin{align}
    \mathbb{P}_{w,\mathbf{n}}'f
    &:=
    \sum_{k=0}^K
    \pi_k
    \mathbb{P}_{k,n_k}'f
    \notag\\
    &=
    \sum_{k=0}^K
    \frac{\pi_k}{n_k}
    \sum_{i=1}^{n_k}
    f(x_{k,i}',z_{k,i}').
    \label{eq_weighted_ghost_operator}
\end{align}

Recall that
\begin{equation}
    h_s(x,z)
    :=
    \ell(x,z;s)
    -
    \ell(x,z;s_w^o),
\end{equation}
and
\begin{equation}
    h_s^{\mathrm{tr}}(x,z)
    :=
    \ell^{\mathrm{tr}}(x,z;s)
    -
    \ell^{\mathrm{tr}}(x,z;s_w^o).
\end{equation}
Then the weighted population excess risk in (\ref{eq_excess_risk_decomposed}) can be written as
\begin{equation}
    \overline{\mathcal{E}}_w(s)
    =
    \mathbb{P}_w h_s.
    \label{eq_excess_risk_operator}
\end{equation}
Moreover, the quantities introduced in the decomposition of the
expected excess risk (\ref{eq_H1})-(\ref{eq_H2_tr}) can be expressed compactly as
\begin{align}
    H_1
    &=
    \mathbb{P}_{w,\mathbf{n}}
    h_{\widehat{s}_w},
    &
    H_1^{\mathrm{tr}}
    &=
    \mathbb{P}_{w,\mathbf{n}}
    h_{\widehat{s}_w}^{\mathrm{tr}},
    \label{eq_H1_operator}
    \\
    H_2
    &=
    \mathbb{P}_{w,\mathbf{n}}'
    h_{\widehat{s}_w},
    &
    H_2^{\mathrm{tr}}
    &=
    \mathbb{P}_{w,\mathbf{n}}'
    h_{\widehat{s}_w}^{\mathrm{tr}}.
    \label{eq_H2_operator}
\end{align}

Conditioning on the training data $D$, we have
\begin{align}
    \mathbb{E}_{D'}
    \left[
        \mathbb{P}_{w,\mathbf{n}}'
        h_{\widehat{s}_w}
        \,\middle|\,D
    \right]
    &=
    \sum_{k=0}^K
    \pi_k
    \mathbb{P}_k
    h_{\widehat{s}_w}
    \notag\\
    &=
    \mathbb{P}_w
    h_{\widehat{s}_w}
    \notag\\
    &=
    \overline{\mathcal{E}}_w
    (\widehat{s}_w).
    \label{eq_ghost_population_identity}
\end{align}
\begin{align}
    \mathbb{E}_{D}
    \left[
        \overline{\mathcal{E}}_w(\widehat{s}_w)
    \right]
    &=
    \mathbb{E}_{D,D'}
    \left[
        \mathbb{P}_{w,\mathbf{n}}'
        h_{\widehat{s}_w}
    \right]
    \notag\\
    &=
    \underbrace{
    \mathbb{E}_{D,D'}
    \left[
        \mathbb{P}_{w,\mathbf{n}}'
        \left(
            h_{\widehat{s}_w}
            -
            h_{\widehat{s}_w}^{\mathrm{tr}}
        \right)
    \right]
    }_{A_2}
    \notag\\
    &\quad+
    \underbrace{
    \mathbb{E}_{D}
    \left[
        \mathbb{E}_{D'}
        \left[
            \mathbb{P}_{w,\mathbf{n}}'
            h_{\widehat{s}_w}^{\mathrm{tr}}
            \,\middle|\,D
        \right]
        -
        \mathbb{P}_{w,\mathbf{n}}
        h_{\widehat{s}_w}^{\mathrm{tr}}
    \right]
    }_{B}
    \notag\\
    &\quad+
    \underbrace{
    \mathbb{E}_{D}
    \left[
        \mathbb{P}_{w,\mathbf{n}}
        \left(
            h_{\widehat{s}_w}^{\mathrm{tr}}
            -
            h_{\widehat{s}_w}
        \right)
    \right]
    }_{A_1}
    \notag\\
    &\quad+
    \underbrace{
    \mathbb{E}_{D}
    \left[
        \mathbb{P}_{w,\mathbf{n}}
        h_{\widehat{s}_w}
    \right]
    }_{C}.
\end{align}
We aim to bound $B = \mathbb{E}_{D, D'}
        \left[ \left(\mathbb{P}_{w,\mathbf{n}}' - \mathbb{P}_{w,\mathbf{n}} \right) h_{\widehat{s}_w}^{\mathrm{tr}} \right]$.

\paragraph{Covering number of the loss function class.}
For a function class $\mathcal{G}$ equipped with a norm
$\|\cdot\|$, let
\begin{equation}
    \mathcal{N}
    \left(
        \delta,\mathcal{G},\|\cdot\|
    \right)
\end{equation}
denote its $\delta$-covering number, defined as
\begin{equation}
    \mathcal{N}
    \left(
        \delta,\mathcal{G},\|\cdot\|
    \right)
    :=
    \min
    \left\{
        M\in\mathbb{N}:
        \exists\, g_1,\ldots,g_M\in\mathcal{G}
        \text{ such that }
        \sup_{g\in\mathcal{G}}
        \min_{1\leq j\leq M}
        \|g-g_j\|
        \leq \delta
    \right\}.
    \label{eq_covering_number_definition}
\end{equation}

In (\ref{eq_trunc_region}), the joint truncation region of $x$ and $z$ is defined as
\begin{equation}
    \mathcal A_R
    =
    \left\{
        (x,z):
        \|x\|_\infty\leq R,\,
        \|z\|_\infty\leq R
    \right\}.
\end{equation}
We choose the
truncation radius as 
\begin{equation}
    R
    =
    \mathcal O(\sqrt{\log N}),
    \label{eq_truncation_radius_choice}
\end{equation}
where $N$ is same with (\ref{eq_mt_Ml}) and Theorem \ref{thm_approx_err_of_diff}. We also choose
$t_0=N^{-C_\sigma}$ and $T=C_\alpha\log N$, same with Theorem \ref{thm_approx_err_of_diff}. By (\ref{eq_Mell_scaling_mixture_score}) and (\ref{eq_Isigma_order}), we have
\begin{equation}
    M_\ell
    =
    \int_{t_0}^T m_t^2\,dt
    \asymp
    \int_{t_0}^T \frac{1}{\sigma_t^2}\,dt
    \asymp
    \log N.
\end{equation}
Hence,
\begin{equation}
    R^2
    \lesssim
    M_\ell.
    \label{eq_R2_Mell_relation}
\end{equation}
We define the loss function class restricted to
$\mathcal A_R$ as
\begin{equation}
    \mathcal{S}(R)
    :=
    \left\{
        \ell^{\mathrm{tr}}(\cdot,\cdot;s):
        s\in\mathcal{F}
    \right\}.
    \label{eq_loss_function_class_SR}
\end{equation}
Since
\[
    \ell^{\mathrm{tr}}(x,z;s)
    =
    \ell(x,z;s),
    \qquad
    (x,z)\in\mathcal{A}_R,
\]
the covering number of the truncated loss class on
$\mathcal{A}_R$ is identical to that of the original loss class
restricted to $\mathcal{A}_R$. For any function
$g:\mathcal{A}_R\rightarrow\mathbb{R}$, define
\begin{equation}
    \|g\|_{L^\infty(\mathcal{A}_R)}
    :=
    \sup_{(x,z)\in\mathcal{A}_R}
    |g(x,z)|.
    \label{eq_Linf_DR_norm}
\end{equation}
The next Lemma gives the $\delta$-covering number of of
$\mathcal{S}(R)$ with respect to
$\|\cdot\|_{L^\infty(\mathcal{A}_R)}$. 

\begin{lem}[Covering number of the loss function class;
{\cite{Fu2024UnveilCD}, Lemma~D.8 and Eqs.~(D.29)--(D.30)}]
\label{lem_loss_covering_number}

Let $\delta>0$ and $R>0$. For the ReLU network class
$\mathcal{F}(M_t,W,\kappa,L,S)$ defined in
(\ref{eq_relu_network_class}), the $\delta$-covering number of
$\mathcal{S}(R)$ with respect to
$\|\cdot\|_{L^\infty(\mathcal{A}_R)}$ satisfies
\begin{equation}
    \mathcal{N}
    \left(
        \delta,
        \mathcal{S}(R),
        \|\cdot\|_{L^\infty(\mathcal{A}_R)}
    \right)
    \lesssim
    \left(
        \frac{
            2L^2
            \bigl(W\max\{R,T\}+2\bigr)
            \kappa^L
            W^{L+1}
            \log N
        }{
            \delta
        }
    \right)^{2S}.
    \label{eq_loss_covering_number}
\end{equation}
Here, $N$ denotes the network complexity parameter controlling the
architecture of the ReLU network, whereas $S$ denotes the number of
nonzero network parameters.

In particular, under the network configuration associated with
Lemma~\ref{lem_conditional_score_approximation}, namely,
\begin{equation}
    W
    =
    \mathcal{O}
    \left(
        N\log^7 N
    \right),
    \qquad
    S
    =
    \mathcal{O}
    \left(
        N\log^9 N
    \right),
\end{equation}
together with
\begin{equation}
    L
    =
    \mathcal{O}(\log^4 N),
    \qquad
    \kappa
    =
    \exp
    \left(
        \mathcal{O}(\log^4 N)
    \right),
\end{equation}
the logarithm of the covering number satisfies
\begin{align}
    &
    \log
    \mathcal{N}
    \left(
        \delta,
        \mathcal{S}(R),
        \|\cdot\|_{L^\infty(\mathcal{A}_R)}
    \right)
    \notag\\
    &\qquad\lesssim
    N\log^9 N
    \Bigg(
        \operatorname{Poly}(\log\log N)
        +
        \operatorname{Poly}(\log\log N)
        \log N\log R +
        \log^8 N
        +
        \log\frac{1}{\delta}
    \Bigg)
    \notag\\
    &\qquad\lesssim
    N\log^9 N
    \left(
        \log^8 N
        +
        \log^2 N\log R
        +
        \log\frac{1}{\delta}
    \right).
    \label{eq_log_loss_covering_number}
\end{align}
\end{lem}

\paragraph{Finite-cover reduction of $B$.} Next, we discretize the truncated loss class using a finite
$\delta$-cover, thereby reducing the bound on $B$ to a finite empirical
process plus a covering approximation error. Recall that
\begin{equation}
    B
    =
    \mathbb{E}_{D,D'}
    \left[
        \left(
            \mathbb{P}_{w,\mathbf{n}}'
            -
            \mathbb{P}_{w,\mathbf{n}}
        \right)
        h_{\widehat{s}_w}^{\mathrm{tr}}
    \right].
    \label{eq_B_recall}
\end{equation}

Let
\begin{equation}
    \mathcal{J}
    :=
    \left\{
        \ell_1,\ldots,\ell_{\mathcal{N}}
    \right\}
    \subset \mathcal{S}(R)
\end{equation}
be a $\delta$-cover of $\mathcal{S}(R)$ with respect to the
$L^\infty(\mathcal{A}_R)$ norm, where
\begin{equation}
    \mathcal{N}
    :=
    \mathcal{N}
    \left(
        \delta,
        \mathcal{S}(R),
        \|\cdot\|_{L^\infty(\mathcal{A}_R)}
    \right).
\end{equation}
Since $\widehat{s}_w$ depends on the training sample $D$, there exists
a random index
\begin{equation}
    J=J(D)\in\{1,\ldots,\mathcal{N}\}
\end{equation}
such that
\begin{equation}
    \left\|
        \ell^{\mathrm{tr}}(\cdot,\cdot;\widehat{s}_w)
        -
        \ell_J
    \right\|_{L^\infty(\mathcal{A}_R)}
    \leq
    \delta.
    \label{eq_cover_hat_sw}
\end{equation}

For each $r=1,\ldots,\mathcal{N}$, define
\begin{equation}
    h_r(x,z)
    :=
    \ell_r(x,z)
    -
    \ell^{\mathrm{tr}}(x,z;s_w^o).
    \label{eq_hr_definition}
\end{equation}
Recall that
\begin{equation}
    h_{\widehat{s}_w}^{\mathrm{tr}}(x,z)
    =
    \ell^{\mathrm{tr}}(x,z;\widehat{s}_w)
    -
    \ell^{\mathrm{tr}}(x,z;s_w^o).
\end{equation}
Therefore,
\begin{align}
    h_{\widehat{s}_w}^{\mathrm{tr}}(x,z)
    -
    h_J(x,z)
    &=
    \ell^{\mathrm{tr}}(x,z;\widehat{s}_w)
    -
    \ell_J(x,z).
\end{align}
By (\ref{eq_cover_hat_sw}) and the definition of the truncated loss function class defined in (\ref{eq_loss_function_class_SR}), both $\ell^{\mathrm{tr}}$ and $\ell_J$ vanish outside $\mathcal{A}_R$. Hence, we have
\begin{equation}
    \left\|
        h_{\widehat{s}_w}^{\mathrm{tr}}
        -
        h_J
    \right\|_{L^\infty(\mathcal{A}_R)} = \left\|
        h_{\widehat{s}_w}^{\mathrm{tr}}
        -
        h_J
    \right\|_{\infty}
    \leq
    \delta.
    \label{eq_cover_h_hat}
\end{equation}
Using the linearity of the empirical operators, we decompose
\begin{align}
    \left(
        \mathbb{P}_{w,\mathbf{n}}'
        -
        \mathbb{P}_{w,\mathbf{n}}
    \right)
    h_{\widehat{s}_w}^{\mathrm{tr}} =
    \left(
        \mathbb{P}_{w,\mathbf{n}}'
        -
        \mathbb{P}_{w,\mathbf{n}}
    \right)
    h_J
    +
    \left(
        \mathbb{P}_{w,\mathbf{n}}'
        -
        \mathbb{P}_{w,\mathbf{n}}
    \right)
    \left(
        h_{\widehat{s}_w}^{\mathrm{tr}}
        -
        h_J
    \right).
    \label{eq_B_cover_decomposition}
\end{align}
For the second term, since both weighted empirical operators are
positive linear operators and
\begin{equation}
    \mathbb{P}_{w,\mathbf{n}}'1
    =
    \mathbb{P}_{w,\mathbf{n}}1
    =
    1,
\end{equation}
we have
\begin{align}
    \left|
        \left(
            \mathbb{P}_{w,\mathbf{n}}'
            -
            \mathbb{P}_{w,\mathbf{n}}
        \right)
        \left(
            h_{\widehat{s}_w}^{\mathrm{tr}}
            -
            h_J
        \right)
    \right| & \leq
    \mathbb{P}_{w,\mathbf{n}}'
    \left|
        h_{\widehat{s}_w}^{\mathrm{tr}}
        -
        h_J
    \right|
    +
    \mathbb{P}_{w,\mathbf{n}}
    \left|
        h_{\widehat{s}_w}^{\mathrm{tr}}
        -
        h_J
    \right|
    \notag\\
    &\leq
    \left(
        \mathbb{P}_{w,\mathbf{n}}'1
        +
        \mathbb{P}_{w,\mathbf{n}}1
    \right)
    \left\|
        h_{\widehat{s}_w}^{\mathrm{tr}}
        -
        h_J
    \right\|_{\infty}
    \notag\\
    &\leq
    2\delta.
    \label{eq_cover_remainder_bound}
\end{align}
Consequently,
\begin{equation}
    |B|
    \leq
    \mathbb{E}_{D,D'}
    \left[
        \left|
            \left(
                \mathbb{P}_{w,\mathbf{n}}'
                -
                \mathbb{P}_{w,\mathbf{n}}
            \right)
            h_J
        \right|
    \right]
    +
    2\delta.
    \label{eq_B_after_cover}
\end{equation}

\paragraph{Self-normalized control of the finite empirical process.} We next introduce a self-normalized empirical process. For an arbitrary
constant $a>0$ and each $r=1,\ldots,\mathcal{N}$, define
\begin{equation}
    U_r
    :=
    \max
    \left\{
        a,\,
        \sqrt{
            \left(
                \mathbb{P}_w h_r
            \right)_+
        }
    \right\},
    \label{eq_Ur_definition}
\end{equation}
where
\[
    x_+:=\max\{x,0\}.
\]
Furthermore, define
\begin{equation}
    D_w
    :=
    \max_{1\leq r\leq\mathcal{N}}
    \frac{
        \left|
            \left(
                \mathbb{P}_{w,\mathbf{n}}'
                -
                \mathbb{P}_{w,\mathbf{n}}
            \right)
            h_r
        \right|
    }{
        U_r
    }.
    \label{eq_Dw_definition}
\end{equation}
Since $J\in\{1,\ldots,\mathcal{N}\}$, it follows pointwise that
\begin{align}
    \left|
        \left(
            \mathbb{P}_{w,\mathbf{n}}'
            -
            \mathbb{P}_{w,\mathbf{n}}
        \right)
        h_J
    \right|
    &=
    U_J
    \frac{
        \left|
            \left(
                \mathbb{P}_{w,\mathbf{n}}'
                -
                \mathbb{P}_{w,\mathbf{n}}
            \right)
            h_J
        \right|
    }{
        U_J
    }
    \notag\\
    &\leq
    U_J D_w.
    \label{eq_hJ_self_normalized}
\end{align}
Substituting (\ref{eq_hJ_self_normalized}) into
(\ref{eq_B_after_cover}) gives
\begin{equation}
    |B|
    \leq
    \mathbb{E}_{D,D'}
    \left[
        U_JD_w
    \right]
    +
    2\delta.
    \label{eq_B_UJ_Dw}
\end{equation}
Finally, applying Young's inequality
$uv\leq (u^2+v^2)/2$, we obtain
\begin{align}
    |B|
    &\leq
    \frac{1}{2}
    \mathbb{E}_{D,D'}
    \left[
        U_J^2
    \right]
    +
    \frac{1}{2}
    \mathbb{E}_{D,D'}
    \left[
        D_w^2
    \right]
    +
    2\delta.
    \label{eq_B_self_normalized_bound}
\end{align}

\paragraph{Bounding $\mathbb{E}[U_J^2]$.}

We first bound $\mathbb{E}_{D,D'}[U_J^2]$. By the definition of $U_r$ in (\ref{eq_Ur_definition}), we have,
\begin{equation}
    U_J^2
    =
    \max
    \left\{
        a^2,\,
        \left(
            \mathbb{P}_w h_J
        \right)_+
    \right\}
    \leq
    a^2+
    \left(
        \mathbb{P}_w h_J
    \right)_+.
    \label{eq_UJ_square_initial}
\end{equation}
By~(\ref{eq_cover_h_hat}) and the positivity of $\mathbb P_w$,
\begin{equation}
    \left|
        \mathbb{P}_w h_J
        -
        \mathbb{P}_w h_{\widehat{s}_w}^{\mathrm{tr}}
    \right|
    \leq
    \mathbb{P}_w
    \left|
        h_J-h_{\widehat{s}_w}^{\mathrm{tr}}
    \right|
    \leq
    \left\|
        h_J-h_{\widehat{s}_w}^{\mathrm{tr}}
    \right\|_\infty
    \leq 
    \delta.
    \label{eq_population_cover_error}
\end{equation}
Consequently,
\begin{equation}
    \left(
        \mathbb{P}_w h_J
    \right)_+
    \leq
    \left(
        \mathbb{P}_w
        h_{\widehat{s}_w}^{\mathrm{tr}}
    \right)_+
    +
    \delta.
    \label{eq_positive_part_cover}
\end{equation}
Combining
(\ref{eq_UJ_square_initial}) and
(\ref{eq_positive_part_cover}) gives
\begin{align}
    \mathbb{E}_{D,D'}
    \left[
        U_J^2
    \right]
    &\leq
    a^2
    +
    \mathbb{E}_{D}
    \left[
        \left(
            \mathbb{P}_w
            h_{\widehat{s}_w}^{\mathrm{tr}}
        \right)_+
    \right]
    +
    \delta
    \notag\\
    &=
    a^2
    +
    \mathbb{E}_{D}
    \left[
        \left(
            \overline{\mathcal{E}}_w^{\mathrm{tr}}
            (\widehat{s}_w)
        \right)_+
    \right]
    +
    \delta,
    \label{eq_UJ_square_bound}
\end{align}
where
\begin{equation}
    \overline{\mathcal{E}}_w^{\mathrm{tr}}(s)
    :=
    \mathbb{P}_w h_s^{\mathrm{tr}}
\end{equation}
denotes the truncated population excess risk.

Moreover, by the truncation bound in
Lemma~\ref{lem_truncation_bound} and
Theorem~\ref{thm_truncation_bound_mixture_score},
\begin{align}
    \left|
        \overline{\mathcal{E}}_w^{\mathrm{tr}}
        (\widehat{s}_w)
        -
        \overline{\mathcal{E}}_w
        (\widehat{s}_w)
    \right|
    &\leq
    \mathbb{P}_w
    \left|
        \ell^{\mathrm{tr}}
        (\cdot,\cdot;\widehat{s}_w)
        -
        \ell
        (\cdot,\cdot;\widehat{s}_w)
    \right| +
    \mathbb{P}_w
    \left|
        \ell^{\mathrm{tr}}
        (\cdot,\cdot;s_w^o)
        -
        \ell
        (\cdot,\cdot;s_w^o)
    \right|
    \notag\\
    &\lesssim
    \exp(-C_{\mathrm{sg}}R^2)
    R M_{\ell}.
    \label{eq_truncated_to_population_excess}
\end{align}
Since
$\overline{\mathcal{E}}_w(\widehat{s}_w)\geq0$, it follows that
\begin{align}
    \mathbb{E}_{D,D'}
    \left[
        U_J^2
    \right]
    &\lesssim
    a^2
    +
    \mathbb{E}_{D}
    \left[
        \overline{\mathcal{E}}_w
        (\widehat{s}_w)
    \right]
    +
    \delta
    +
    \exp(-C_{\mathrm{sg}}R^2)
    R M_{\ell}.
    \label{eq_UJ_square_final}
\end{align}

\paragraph{Bounding $\mathbb{E}[D_w^2]$.}
We next control the second moment of the normalized empirical process $D_w$.
For each fixed covering index
$r\in\{1,\ldots,\mathcal{N}\}$, define
\begin{equation}
    Y_r
    :=
    \frac{
        \left(
            \mathbb{P}_{w,\mathbf{n}}'
            -
            \mathbb{P}_{w,\mathbf{n}}
        \right)
        h_r
    }{
        U_r
    }.
    \label{eq_Yr_definition}
\end{equation}
Let
\begin{equation}
    \lambda_k
    :=
    \frac{\pi_k}{n_k},
    \qquad
    \rho_w
    :=
    \max_{0\leq k\leq K}\lambda_k.
    \label{eq_rho_w_definition}
\end{equation}
Then
\begin{equation}
    Y_r
    =
    \sum_{k=0}^K
    \sum_{i=1}^{n_k}
    X_{k,i}^{(r)},
\end{equation}
where
\begin{equation}
    X_{k,i}^{(r)}
    :=
    \frac{\lambda_k}{U_r}
    \left[
        h_r(X_{k,i}',Z_{k,i}')
        -
        h_r(X_{k,i},Z_{k,i})
    \right].
    \label{eq_Xkir_definition}
\end{equation}
Since
$(X_{k,i}',Z_{k,i}')$ and
$(X_{k,i},Z_{k,i})$ are independent and identically distributed
according to $P_k$,
\begin{equation}
    \mathbb{E}
    \left[
        X_{k,i}^{(r)}
    \right]
    =0.
    \label{eq_Xkir_centered}
\end{equation}
Moreover, the random variables
$\{X_{k,i}^{(r)}\}_{k,i}$ are mutually independent for every fixed
$r$. By the uniform loss bound in
(\ref{eq_uniform_loss_bound_unbounded}), the mixture-score loss envelope
in (\ref{eq_mixture_score_loss_envelope}), and
$R^2\lesssim M_\ell$ in (\ref{eq_R2_Mell_relation}), there exists a universal constant
$C_M>0$ such that
\begin{equation}
    \|h_r\|_{\infty}
    \leq
    C_M M_{\ell},
    \qquad
    r=1,\ldots,\mathcal{N}.
    \label{eq_hr_uniform_bound}
\end{equation}
Since $U_r\geq a$, it follows that
\begin{equation}
    \left|
        X_{k,i}^{(r)}
    \right|
    \leq
    \frac{
        2C_M M_{\ell}\lambda_k
    }{
        a
    }
    \leq
    \frac{
        2C_M M_{\ell}\rho_w
    }{
        a
    }.
    \label{eq_Xkir_uniform_bound}
\end{equation}

For the total second moment, using
$(u-v)^2\leq2u^2+2v^2$, we obtain
\begin{align}
    \sum_{k=0}^K
    \sum_{i=1}^{n_k}
    \mathbb{E}
    \left[
        \left(
            X_{k,i}^{(r)}
        \right)^2
    \right]
    &\leq
    4
    \sum_{k=0}^K
    \frac{
        n_k\lambda_k^2
    }{
        U_r^2
    }
    \mathbb{P}_k
    h_r^2
    \notag\\
    &=
    4
    \sum_{k=0}^K
    \frac{
        \lambda_k\pi_k
    }{
        U_r^2
    }
    \mathbb{P}_k
    h_r^2
    \notag\\
    &\leq
    \frac{
        4\rho_w
    }{
        U_r^2
    }
    \sum_{k=0}^K
    \pi_k
    \mathbb{P}_k
    h_r^2
    \notag\\
    &=
    \frac{
        4\rho_w
    }{
        U_r^2
    }
    \mathbb{P}_w h_r^2.
    \label{eq_Yr_variance_general}
\end{align}

\paragraph{A Bernstein-type bound for the truncated excess loss.}
We next establish a Bernstein-type second-moment bound for the
functions $\{h_r\}_{r=1}^{\mathcal{N}}$.

Recall that each covering element $\ell_r\in\mathcal{S}(R)$
corresponds to some $s_r\in\mathcal{F}$ and (\ref{eq_hr_definition}), we have \begin{equation}
    h_r(x,z)
    =
    \ell^{\mathrm{tr}}(x,z;s_r)
    -
    \ell^{\mathrm{tr}}(x,z;s_w^o).
    \label{eq_hr_as_excess_loss}
\end{equation}
For notational simplicity, for a fixed $(x,z)$ define
\begin{align}
    A_r(x,z)
    &:=
    \frac{1}{T-t_0}
    \int_{t_0}^{T}
    \mathbb{E}_{
        X_t\sim
        \mathcal{N}(a_t x,\sigma_t^2 I_{d_x})
    }
    \left[
        \left\|
            s_r(X_t,z,t)-s_w^o(X_t,z,t)
        \right\|_2^2
    \right]
    dt,
    \label{eq_Ar_definition}
    \\
    B_r(x,z)
    &:=
    \frac{1}{T-t_0}
    \int_{t_0}^{T}
    \mathbb{E}_{
        X_t\sim
        \mathcal{N}(a_t x,\sigma_t^2 I_{d_x})
    }
    \Bigg[
        \Big\langle
            s_r(X_t,z,t)-s_w^o(X_t,z,t),
            \notag\\[-1mm]
    &\hspace{53mm}
            s_w^o(X_t,z,t)
            +
            \frac{X_t-a_t x}{\sigma_t^2}
        \Big\rangle
    \Bigg]
    dt.
    \label{eq_Br_definition}
\end{align}
By expanding the difference of the two squared losses, we obtain
\begin{equation}
    h_r(x,z)
    =
    \mathbf 1_{\mathcal A_R}(x,z)
    \left[
        A_r(x,z)+2B_r(x,z)
    \right].
    \label{eq_hr_AB_decomposition}
\end{equation}

By the Cauchy--Schwarz inequality,
\begin{align}
    |B_r(x,z)|^2
    &\leq
    A_r(x,z) \cdot 
    \frac{1}{T-t_0}
    \int_{t_0}^{T}
    \mathbb{E}_{
        X_t\sim
        \mathcal{N}(a_t x,\sigma_t^2 I_{d_x})
    }
    \left[
        \left\|
            s_w^o(X_t,z,t)
            + \frac{X_t - a_t x}{\sigma_t^2}
        \right\|_2^2
    \right]dt
    \notag\\
    &=
    A_r(x,z) \cdot 
    \ell(x,z;s_w^o).
    \label{eq_Br_CS}
\end{align}

Moreover, using
$\|u-v\|_2^2\leq2\|u-y\|_2^2+2\|v-y\|_2^2$,
we have
\begin{align}
    A_r(x,z)
    &\leq
    2\ell(x,z;s_r)
    +
    2\ell(x,z;s_w^o).
    \label{eq_Ar_loss_bound}
\end{align}
By the uniform loss bounds in (\ref{eq_uniform_loss_bound_unbounded}), (\ref{eq_mixture_score_loss_envelope}), and (\ref{eq_R2_Mell_relation}), there exists a constant $C_\ell>0$ such
that, uniformly over $r$ and $(x,z)\in\mathcal{A}_R$,
\begin{equation}
    \ell(x,z;s_r)
    \leq C_\ell M_\ell,
    \qquad
    \ell(x,z;s_w^o) \leq R^2 + M_\ell 
    \leq C_\ell M_\ell.
    \label{eq_loss_envelope_sr_swo}
\end{equation}
Consequently,
\begin{equation}
    A_r(x,z)
    \leq
    4C_\ell M_\ell,
    \qquad
    |B_r(x,z)|^2
    \leq
    C_\ell M_\ell A_r(x,z).
    \label{eq_Ar_Br_bounds}
\end{equation}

Applying $(a+b)^2\leq2a^2+2b^2$ to
(\ref{eq_hr_AB_decomposition}), we obtain
\begin{align}
    h_r^2(x,z)
    &\leq
    \mathbf{1}_{\mathcal A_R}(x,z)
    \left[
        2A_r^2(x,z)+8B_r^2(x,z)
    \right]
    \notag\\
    &\lesssim
    M_\ell
    \mathbf{1}_{\mathcal A_R}(x,z)
    A_r(x,z).
    \label{eq_hr_square_pointwise}
\end{align}
Taking expectation with respect to $P_w$ gives
\begin{equation}
    \mathbb{P}_w h_r^2
    \lesssim
    M_\ell
    \mathbb{P}_w
    \left[
        \mathbf{1}_{\mathcal A_R} A_r
    \right]
    \leq
    M_\ell
    \mathbb{P}_w A_r.
    \label{eq_hr_second_moment_Ar}
\end{equation}

We now identify $\mathbb{P}_w A_r$ with the population excess risk.
By the denoising-score identity \citep{Vincent_2011_score_mat} and (\ref{eq_weighted_population_operator}),
\begin{equation}
    \mathbb{E}_{
        X_0\sim p_w(\cdot| X_t=x_t,Z=z)
    }
    \left[
        \frac{x_t-a_tX_0}{\sigma_t^2}
    \right]
    =
    -s_w^o(x_t,z,t).
    \label{eq_denoising_score_identity_weighted}
\end{equation}
\begin{align}
    \mathbb{P}_w B_r
    &=
    \frac{1}{T-t_0}
    \int_{t_0}^{T}
    \mathbb{E}_{(X_0,Z)\sim P_w}
    \mathbb{E}_{
        X_t\sim
        \mathcal{N}(a_tX_0,\sigma_t^2I_{d_x})
    }
    \Bigg[
        \Big\langle
            s_r(X_t,Z,t)-s_w^o(X_t,Z,t),
            \notag\\[-1mm]
    &\hspace{38mm}
            s_w^o(X_t,Z,t)
            +
            \frac{X_t-a_tX_0}{\sigma_t^2}
        \Big\rangle
    \Bigg]
    \,dt
    \notag\\
    &=
    \frac{1}{T-t_0}
    \int_{t_0}^{T}
    \mathbb{E}_{(X_t,Z)\sim p_{w,t}}
    \Bigg[
        \Big\langle
            s_r(X_t,Z,t)-s_w^o(X_t,Z,t),
            \notag\\[-1mm]
    &\hspace{25mm}
            s_w^o(X_t,Z,t)
            +
            \mathbb{E}_{
                X_0\sim p_w(\cdot | X_t,Z)
            }
            \left[
                \frac{X_t-a_tX_0}{\sigma_t^2}
            \right]
        \Big\rangle
    \Bigg]
    \,dt
    \notag\\
    &=
    0.
    \label{eq_Br_orthogonality}
\end{align}
Hence, for the untruncated excess loss
\begin{equation}
    h_r^{\mathrm{full}}(x,z)
    :=
    \ell(x,z;s_r)-\ell(x,z;s_w^o),
\end{equation}
we have
\begin{align}
    \mathbb{P}_w h_r^{\mathrm{full}}
    =
    \mathbb{P}_w A_r =
    \|s_r-s_w^o\|_{\mathcal{H}_w}^2
    \geq 0.
    \label{eq_full_excess_equals_Ar}
\end{align}
Combining (\ref{eq_hr_second_moment_Ar}) and
(\ref{eq_full_excess_equals_Ar}), we obtain
\begin{equation}
    \mathbb{P}_w h_r^2
    \lesssim
    M_\ell
    \mathbb{P}_w h_r^{\mathrm{full}}.
    \label{eq_hr_Bernstein_full}
\end{equation}

By Lemma~\ref{lem_truncation_bound} and Theorem \ref{thm_truncation_bound_mixture_score},
\begin{align}
    \left|
        \mathbb{P}_w h_r^{\mathrm{full}}
        -
        \mathbb{P}_w h_r
    \right|
    &\leq
    \sum_{k=0}^K
    \pi_k
    \mathbb{E}_{P_k}
    \left[
        \left|
            h_r^{\mathrm{full}}-h_r
        \right|
    \right]
    \notag\\
    &\lesssim
    \exp(-C_{\mathrm{sg}}R^2)
    R M_\ell.
    \label{eq_full_truncated_excess_difference}
\end{align}
Define
\begin{equation}
    \tau_R
    :=
    C_{\mathrm{tr}}
    \exp(-C_{\mathrm{sg}}R^2)
    R M_\ell
\end{equation}
for a sufficiently large constant $C_{\mathrm{tr}}>0$. Since
$\mathbb{P}_w h_r^{\mathrm{full}}\geq0$, it follows that
\begin{equation}
    \mathbb{P}_w h_r^{\mathrm{full}}
    \leq
    (\mathbb{P}_w h_r)_+
    +
    \tau_R.
    \label{eq_full_excess_truncated_positive}
\end{equation}
Combining
(\ref{eq_hr_Bernstein_full}) and
(\ref{eq_full_excess_truncated_positive}), we arrive at
\begin{equation}
    \boxed{
    \mathbb{P}_w h_r^2
    \lesssim
    M_\ell
    \left[
        (\mathbb{P}_w h_r)_+
        +
        \tau_R
    \right].
    }
    \label{eq_Bernstein_truncated_hr}
\end{equation}
If $a>0$ is chosen such that
\begin{equation}
    a^2\geq\tau_R,
    \label{eq_a_dominates_truncation}
\end{equation}
then, by the definition of $U_r$ in (\ref{eq_Ur_definition}),
\begin{equation}
    (\mathbb{P}_w h_r)_+
    +
    \tau_R
    \leq 2 \max
    \left\{
        a^2,\, 
            \left(
                \mathbb{P}_w h_r
            \right)_+ 
    \right\} = 
    2U_r^2.
\end{equation}
Therefore,
\begin{equation}
    \boxed{
    \mathbb{P}_w h_r^2
    \lesssim
    M_\ell U_r^2,
    \qquad
    r=1,\ldots,\mathcal{N}.
    }
    \label{eq_Bernstein_normalized_hr}
\end{equation}
Equivalently,
\begin{equation}
    \boxed{
    \frac{
        \mathbb{P}_w h_r^2
    }{
        U_r^2
    }
    \lesssim
    M_\ell.
    }
    \label{eq_normalized_second_moment_hr}
\end{equation}

\paragraph{Bounding the normalized empirical process $D_w$.} In (\ref{eq_Yr_definition}), we defined
\begin{equation}
    Y_r
    =
    \frac{
        \left(
            \mathbb{P}_{w,\mathbf{n}}'
            -
            \mathbb{P}_{w,\mathbf{n}}
        \right)
        h_r
    }{
        U_r
    }
    =
    \sum_{k=0}^K
    \sum_{i=1}^{n_k}
    X_{k,i}^{(r)},
    \label{eq_Yr_sum}
\end{equation}
where $X_{k,i}^{(r)}$ is defined in (\ref{eq_Xkir_definition}).  
For every fixed $r$, the random variables
$\{X_{k,i}^{(r)}\}_{k,i}$ are mutually independent and centered:
\begin{equation}
    \mathbb{E}
    \left[
        X_{k,i}^{(r)}
    \right]
    =0.
\end{equation}

Let
$C_1 = 2 C_M >0$ and by (\ref{eq_Xkir_uniform_bound}), we have
\begin{equation}
    \left|
        X_{k,i}^{(r)}
    \right|
    \leq
    \frac{
        C_1 M_\ell\rho_w
    }{
        a
    },
    \label{eq_X_bound_Bernstein}
\end{equation}
By (\ref{eq_Yr_variance_general}) and the Bernstein-type bound
(\ref{eq_Bernstein_normalized_hr}), provided that $a^2 \geq \tau_R$ we have
\begin{align}
    \sum_{k=0}^K
    \sum_{i=1}^{n_k}
    \mathbb{E}
    \left[
        \left(
            X_{k,i}^{(r)}
        \right)^2
    \right]
    &\leq
    \frac{
        4\rho_w
    }{
        U_r^2
    }
    \mathbb{P}_w h_r^2
    \notag\\
    &\leq
    C_2 M_\ell\rho_w
    \label{eq_Yr_variance_proxy}
\end{align}
for some constant $C_2>0$. Therefore, Bernstein's inequality \citep{boucheron_2013_Concentration_ineq} gives, for every $t>0$,
\begin{align}
    \mathbb{P}
    \left(
        |Y_r|\geq t
    \right)
    &\leq
    2
    \exp
    \left\{
        -
        \frac{
            t^2/2
        }{
            C_2 M_\ell\rho_w
            +
            \dfrac{
                C_1 M_\ell\rho_w
            }{
                3a
            }t
        }
    \right\}
    \notag\\
    &\leq
    2
    \exp
    \left[
        -c
        \min
        \left\{
            \frac{t^2}{M_\ell\rho_w},
            \frac{a t}{M_\ell\rho_w}
        \right\}
    \right]
    \label{eq_Yr_Bernstein_tail}
\end{align}
for some universal constant $c>0$. Recall the definition of $D_w$ in (\ref{eq_Dw_definition}), $Y_r$ in (\ref{eq_Yr_definition}),  and (\ref{eq_Yr_Bernstein_tail}),
\begin{align}
    \mathbb{P}
    \left(
        D_w\geq t
    \right)
    &\leq
    \sum_{r=1}^{\mathcal{N}}
    \mathbb{P}
    \left(
        |Y_r|\geq t
    \right)
    \notag\\
    &\leq
    2\mathcal{N}
    \exp
    \left[
        -c
        \min
        \left\{
            \frac{t^2}{M_\ell\rho_w},
            \frac{a t}{M_\ell\rho_w}
        \right\}
    \right].
    \label{eq_Dw_tail}
\end{align}

Let
\begin{equation}
    L_{\mathcal{N}}
    :=
    \log(2\mathcal{N}).
    \label{eq_LN_definition}
\end{equation}

Set
\begin{equation}
    A_w:=M_\ell\rho_w,
    \qquad
    B_w^{\mathrm{emp}}
    :=
    \frac{M_\ell\rho_w}{a}.
\end{equation}
Then (\ref{eq_Dw_tail}) can be written as
\begin{equation}
    \mathbb{P}
    \left(
        D_w\geq t
    \right)
    \leq
    \exp
    \left[
        L_{\mathcal{N}}
        -
        c
        \min
        \left\{
            \frac{t^2}{A_w},
            \frac{t}{B_w^{\mathrm{emp}}}
        \right\}
    \right].
    \label{eq_Dw_subexp_tail}
\end{equation}
A standard integration of the sub-exponential tail bound yields
\begin{equation}
    \mathbb{E}_{D,D'}
    \left[
        D_w^2
    \right]
    \lesssim
    A_w L_{\mathcal{N}}
    +
    \left(
        B_w^{\mathrm{emp}}
    \right)^2
    L_{\mathcal{N}}^2.
    \label{eq_Dw_second_moment_general}
\end{equation}
Therefore,
\begin{equation}
    \mathbb{E}_{D,D'}
    \left[
        D_w^2
    \right]
    \lesssim
    M_\ell\rho_w L_{\mathcal{N}}
    +
    \frac{
        M_\ell^2\rho_w^2
    }{
        a^2
    }
    L_{\mathcal{N}}^2.
    \label{eq_Dw_second_moment}
\end{equation}

Indeed, using
\[
    \mathbb{E}[D_w^2]
    =
    2\int_0^\infty
    t\,
    \mathbb{P}(D_w\geq t)\,dt,
\]
and splitting the integral at a constant multiple of
\[
    \sqrt{
        A_w L_{\mathcal{N}}
    }
    +
    B_w^{\mathrm{emp}}L_{\mathcal{N}},
\]
gives (\ref{eq_Dw_second_moment_general}).

We choose $a>0$ such that
\begin{equation}
    a^2
    \asymp
    M_\ell\rho_w L_{\mathcal{N}}
    +
    \tau_R,
    \label{eq_a_choice}
\end{equation}
where
\begin{equation}
    \tau_R
    =
    C_{\mathrm{tr}}
    \exp(-C_{\mathrm{sg}}R^2)
    R M_\ell.
\end{equation}
In particular,
\begin{equation}
    a^2
    \gtrsim
    M_\ell\rho_w L_{\mathcal{N}},
    \qquad
    a^2\geq\tau_R.
\end{equation}

Under the choice (\ref{eq_a_choice}),
\begin{align}
    \frac{
        M_\ell^2\rho_w^2
    }{
        a^2
    }
    L_{\mathcal{N}}^2
    &\lesssim
    M_\ell\rho_wL_{\mathcal{N}}.
\end{align}
Consequently,
\begin{equation}
    \boxed{
    \mathbb{E}_{D,D'}
    \left[
        D_w^2
    \right]
    \lesssim
    M_\ell\rho_w
    L_{\mathcal{N}}.
    }
    \label{eq_Dw_second_moment_final}
\end{equation}

Recall from (\ref{eq_UJ_square_final}) that
\begin{align}
    \mathbb{E}_{D}
    \left[
        U_J^2
    \right]
    &\leq
    a^2
    +
    \mathbb{E}_{D}
    \left[
        \overline{\mathcal{E}}_w
        (\widehat{s}_w)
    \right]
    +
    \delta
    +
    \tau_R.
    \label{eq_UJ_second_moment_final}
\end{align}

Combining 
(\ref{eq_B_self_normalized_bound}),
(\ref{eq_Dw_second_moment_final}), and
(\ref{eq_UJ_second_moment_final}), we obtain
\begin{align}
    |B|
    &\leq
    \frac{1}{2}
    \mathbb{E}_{D}
    \left[
        \overline{\mathcal{E}}_w
        (\widehat{s}_w)
    \right]
    +
    \frac{1}{2}a^2
    +
    C_D
    M_\ell\rho_wL_{\mathcal{N}}
    +
    \frac{1}{2}\tau_R
    +
    \frac{5}{2}\delta,
    \label{eq_B_before_a}
\end{align}
for some constant $C_D>0$. By (\ref{eq_a_choice}),
\[
    a^2
    \lesssim
    M_\ell\rho_wL_{\mathcal N}
    +
    \tau_R.
\]
Therefore, absorbing the term $a^2/2$ and the numerical constants
into a constant $C_B>0$, we obtain
\begin{equation}
    \boxed{
    |B|
    \leq
    \frac{1}{2}
    \mathbb{E}_{D}
    \left[
        \overline{\mathcal{E}}_w
        (\widehat{s}_w)
    \right]
    +
    C_B
    \left[
        M_\ell\rho_wL_{\mathcal{N}}
        +
        \tau_R
        +
        \delta
    \right].
    }
    \label{eq_B_final_bound}
\end{equation}

\paragraph{Effective sample size and maximal observation weight.}

To characterize the variance of the weighted empirical risk, recall (\ref{eq_weighted_empirical_operator_explicit}) that
\begin{equation}
    \mathbb{P}_{w,\mathbf{n}} f
    =
    \sum_{k=0}^K
    \pi_k
    \mathbb{P}_{k,n_k}f
    =
    \sum_{k=0}^K
    \frac{\pi_k}{n_k}
    \sum_{i=1}^{n_k}
    f(X_{k,i},Z_{k,i}).
\end{equation}
Since the samples are independent across domains, for any measurable
function $f$ with finite second moments,
\begin{align}
    \operatorname{Var}
    \left(
        \mathbb{P}_{w,\mathbf{n}}f
    \right)
    &=
    \sum_{k=0}^K
    \frac{\pi_k^2}{n_k}
    \operatorname{Var}_{P_k}
    \left(
        f(X,Z)
    \right).
    \label{eq_weighted_empirical_variance}
\end{align}
Hence, the quantity
\begin{equation}
    \nu_w
    :=
    \sum_{k=0}^K
    \frac{\pi_k^2}{n_k}
    \label{eq_nu_w_definition}
\end{equation}
plays the role of the variance factor of the weighted empirical risk.
Motivated by the usual variance factor $1/n$ in standard ERM, we define
the effective sample size by
\begin{equation}
    N_{\mathrm{eff}}
    :=
    \nu_w^{-1}
    =
    \left(
        \sum_{k=0}^K
        \frac{\pi_k^2}{n_k}
    \right)^{-1} .
    \label{eq_Neff_definition}
\end{equation}
Recall (\ref{eq_define_pi}) that
\begin{equation}
    \pi_0
    =
    \frac{n_0}{n_0+W_N},
    \qquad
    \pi_k
    =
    \frac{w_kn_k}{n_0+W_N},
    \quad k=1,\ldots,K,
\end{equation}
where
    $W_N=\sum_{k=1}^K w_kn_k$.
Therefore,
\begin{align}
    \frac{1}{N_{\mathrm{eff}}}
    &=
    \frac{\pi_0^2}{n_0}
    +
    \sum_{k=1}^K
    \frac{\pi_k^2}{n_k}
    \notag\\
    &=
    \frac{
        n_0+\sum_{k=1}^K w_k^2n_k
    }{
        (n_0+W_N)^2
    },
\end{align}
or equivalently,
\begin{equation}
    \boxed{
    N_{\mathrm{eff}}
    =
    \frac{
        (n_0+W_N)^2
    }{
        n_0+\sum_{k=1}^K w_k^2n_k
    }.
    }
    \label{eq_Neff_explicit}
\end{equation}

Recall the definition of $\lambda_k$ and $\rho_w$ in (\ref{eq_rho_w_definition}), and by the definition of $\pi_k$, we have
\begin{equation}
    \rho_w
    =
    \frac{
        \max\{1,w_1,\ldots,w_K\}
    }{
        n_0+W_N
    }.
    \label{eq_rho_w_explicit}
\end{equation}
Moreover,
\begin{align}
    \frac{1}{N_{\mathrm{eff}}}
    =
    \sum_{k=0}^K
    \frac{\pi_k^2}{n_k}
    =
    \sum_{k=0}^K
    n_k\lambda_k^2 \leq
    \rho_w
    \sum_{k=0}^K
    n_k\lambda_k
    =
    \rho_w
    \sum_{k=0}^K\pi_k
    =
    \rho_w.
    \label{eq_Neff_rho_lower}
\end{align}
Thus,
\begin{equation}
    \frac{1}{N_{\mathrm{eff}}}
    \leq
    \rho_w.
\end{equation}

Using the assumption in (\ref{eq_weight_regularity_asymptotic}), we obtain
\begin{equation}
    \max\{1,w_1,\ldots,w_K\}
    \leq
    C_{\mathrm{reg}}
        \dfrac{n_0+\sum_{k=1}^K w_k^2n_k}{n_0+W_N},
    \label{eq_weight_regularity}
\end{equation}
where $C_{\mathrm{reg}} \geq 1$. 
Using (\ref{eq_Neff_explicit}) and (\ref{eq_rho_w_explicit}),
condition (\ref{eq_weight_regularity}) is equivalent to
\begin{equation}
    \rho_w
    \leq
    \frac{C_{\mathrm{reg}}}{N_{\mathrm{eff}}}.
\end{equation}
Together with (\ref{eq_Neff_rho_lower}), this yields
\begin{equation}
    \boxed{
    \rho_w
    \asymp
    \frac{1}{N_{\mathrm{eff}}}.
    }
    \label{eq_rho_Neff_equivalence}
\end{equation}

Consequently, under (\ref{eq_weight_regularity}),
\begin{equation}
    M_\ell\rho_w\log(2\mathcal{N})
    \lesssim
    \frac{
        M_\ell\log(2\mathcal{N})
    }{
        N_{\mathrm{eff}}
    }.
    \label{eq_empirical_complexity_Neff}
\end{equation}

\subsubsection{Proof of Theorem \ref{thm_score_risk_bound}}
\label{app_proof_of_th1}

\paragraph{Combining the bounds.}

Combining (\ref{eq_four_term_decomposition}),
(\ref{eq_A1_A2_final_bound}),
(\ref{eq_C_approximation_rate}), and
(\ref{eq_B_final_bound}), we obtain
\begin{align}
    \mathbb{E}_{D}
    \left[
        \overline{\mathcal{E}}_w(\widehat{s}_w)
    \right]
    &\leq
    |A_1|+|A_2|+|B|+C
    \notag\\
    &\leq
    \frac{1}{2}
    \mathbb{E}_{D}
    \left[
        \overline{\mathcal{E}}_w(\widehat{s}_w)
    \right]
    +
    C_B M_{\ell}\rho_w L_{\mathcal{N}}
    +
    C_B\tau_R
    +
    C_B\delta
    \notag\\
    &\quad+
    C_A
    \exp\left(-C_{\mathrm{sg}}R^2\right)
    R M_{\ell}
    +
    C_{\mathrm{app}}
    N^{-\frac{2\beta}{d_x+d_z}}
    (\log N)^{\beta+1},
    \label{eq_excess_risk_before_absorption}
\end{align}
where $C_A$, $C_B$, and $C_{\mathrm{app}}$ are positive constants
independent of the sample sizes and $N$.
Therefore, moving the first term on the right-hand side to the
left-hand side yields
\begin{align}
    \mathbb{E}_{D}
    \left[
        \overline{\mathcal{E}}_w(\widehat{s}_w)
    \right]
    &\lesssim
    N^{-\frac{2\beta}{d_x+d_z}}
    (\log N)^{\beta+1}
    +
    M_{\ell}\rho_w L_{\mathcal{N}}
    +
    \delta
    \notag\\
    &\quad+
    \tau_R
    +
    \exp\left(-C_{\mathrm{sg}}R^2\right)
    R M_{\ell}.
    \label{eq_excess_risk_combined}
\end{align}

Consequently, since
$\tau_R\lesssim
\exp(-C_{\mathrm{sg}}R^2)R M_{\ell}$,
we arrive at
\begin{equation}
    \boxed{
    \mathbb{E}_{D}
    \left[
        \overline{\mathcal{E}}_w(\widehat{s}_w)
    \right]
    \lesssim
    N^{-\frac{2\beta}{d_x+d_z}}
    (\log N)^{\beta+1}
    +
    M_{\ell}\rho_w L_{\mathcal{N}}
    +
    \delta
    +
    \exp\left(-C_{\mathrm{sg}}R^2\right)
    R M_{\ell}.
    }
    \label{eq_weighted_excess_risk_final}
\end{equation}

\paragraph{Choice of tuning parameters and convergence rate.}

Recall that in (\ref{eq_neff_def}), the effective sample size is defined as
\begin{equation}
    N_{\mathrm{eff}}
    =
    \frac{
        \left(
            n_0+\sum_{k=1}^K w_k n_k
        \right)^2
    }{
        n_0+\sum_{k=1}^K w_k^2 n_k
    }. 
\end{equation}
Moreover, recall that
\begin{equation}
    \rho_w
    =
    \frac{
        \max\{1,w_1,\ldots,w_K\}
    }{
        n_0+\sum_{k=1}^K w_k n_k
    }.
\end{equation}
Under Assumption~(\ref{eq_weight_regularity_asymptotic}),
\begin{align}
    \rho_w
    &\lesssim
    \frac{
        n_0+\sum_{k=1}^K w_k^2 n_k
    }{
        \left(
            n_0+\sum_{k=1}^K w_k n_k
        \right)^2
    }
    \notag\\
    &=
    \frac{1}{N_{\mathrm{eff}}}.
    \label{eq_rho_Neff_relation}
\end{align}

Following the parameter choices in
{\cite{Fu2024UnveilCD}, proof of Theorem~4.1}, we take
\begin{equation}
    \delta
    =
    N^{-\frac{2\beta}{d_x+d_z}},
    \qquad
    R^2
    =
    C_R\log N,
    \label{eq_delta_R_choice}
\end{equation}
where $C_R>0$ is chosen sufficiently large.
By (\ref{eq_LN_definition}) and  the covering-number bound in Lemma \ref{lem_loss_covering_number},
\begin{align}
    L_{\mathcal{N}}
    &=
    \log(2\mathcal{N})
    \notag\\
    &\lesssim
    N\log^9 N
    \left(
        \log^8 N
        +
        \log^2 N\log R
        +
        \log\frac{1}{\delta}
    \right)
    \notag\\
    &\lesssim
    N\log^{17}N.
    \label{eq_covering_rate_parameter_choice}
\end{align}
Moreover, in (\ref{eq_Mell_unbounded_order})
\begin{equation}
    M_{\ell}
    \lesssim
    \log\frac{1}{t_0}.
    \label{eq_Mell_rate}
\end{equation}
By choosing $C_R$ sufficiently large, the truncation remainder
$\exp(-C_{\mathrm{sg}}R^2)RM_\ell$ is of smaller order than the
approximation term and can therefore be absorbed into the leading
terms. The choice of $\delta$ in
(\ref{eq_delta_R_choice}) is also dominated by the approximation term.
Therefore, using (\ref{eq_rho_Neff_relation}), we obtain
\begin{align}
    \mathbb{E}_{D}
    \left[
        \overline{\mathcal{E}}_w(\widehat{s}_w)
    \right]
    &\lesssim
    N^{-\frac{2\beta}{d_x+d_z}}
    (\log N)^{\beta+1}
    +
    \frac{
        N
    }{
        N_{\mathrm{eff}}
    }
    \log\frac{1}{t_0}
    \log^{17}N.
    \label{eq_excess_risk_bias_variance}
\end{align}

Balancing the polynomial parts of the approximation and estimation
terms in (\ref{eq_excess_risk_bias_variance}), namely,
\[
    N^{-\frac{2\beta}{d_x+d_z}}
    \qquad\text{and}\qquad
    \frac{N}{N_{\mathrm{eff}}},
\]
we choose
\begin{equation}
    N
    \asymp
    N_{\mathrm{eff}}^{
        \frac{d_x+d_z}
        {d_x+d_z+2\beta}
    }.
    \label{eq_optimal_network_complexity}
\end{equation}
Under this choice,
\begin{equation}
    N^{-\frac{2\beta}{d_x+d_z}}
    \asymp
    \frac{N}{N_{\mathrm{eff}}}
    \asymp
    N_{\mathrm{eff}}^{
        -\frac{2\beta}
        {d_x+d_z+2\beta}
    },
    \label{eq_bias_variance_balance}
\end{equation}
and
\begin{equation}
    \log N
    \asymp
    \log N_{\mathrm{eff}}.
\end{equation}

Since $t_0=N^{-C_\sigma}$,
\begin{equation}
    \log\frac{1}{t_0}
    =
    C_\sigma\log N.
\end{equation}
Therefore,
\begin{align}
    N^{-\frac{2\beta}{d_x+d_z}}
    (\log N)^{\beta+1}
    &\lesssim
    \log\frac{1}{t_0}\,
    N^{-\frac{2\beta}{d_x+d_z}}
    (\log N)^{\beta}.
\end{align}
Consequently, substituting
(\ref{eq_optimal_network_complexity}) into
(\ref{eq_excess_risk_bias_variance}) yields
\begin{equation}
    \boxed{
    \mathbb{E}_{D}
    \left[
        \overline{\mathcal{E}}_w(\widehat{s}_w)
    \right]
    \lesssim
    \log\frac{1}{t_0}\,
    N_{\mathrm{eff}}^{
        -\frac{2\beta}
        {d_x+d_z+2\beta}
    }
    \left(
        \log N_{\mathrm{eff}}
    \right)^{\max\{17,\beta \}}.
    }
    \label{eq_weighted_score_estimation_rate}
\end{equation}

\paragraph{Target-domain risk bound.}

Combining
(\ref{eq_error_decomp}),
(\ref{eq_final_transfer_bias_general}), 
(\ref{eq_H0_Hw_estimation_error}),
(\ref{eq_weighted_score_estimation_rate}), and
, we obtain
\begin{align}
    \mathbb{E}_{D_0,\ldots,D_K}
    \left[
        \mathcal{R}_0(\widehat{s}_w)
    \right]
    &\lesssim
    \log\frac{1}{t_0}\,
    N_{\mathrm{eff}}^{
        -\frac{2\beta}
        {d_x+d_z+2\beta}
    }
    \left(
        \log N_{\mathrm{eff}}
    \right)^{\max\{17, \beta \}}
    \notag\\
    &\quad+
    \frac{W_N^2}{(n_0+W_N)^2} \left[ 
    \overline{w}^\top G \overline{w} + \operatorname{diag}(G)^\top
        \overline w \right].
    \label{eq_target_risk_final_general}
\end{align}
\hfill \qedsymbol

\subsection{The TV error bound of DA-Diff}
\label{app_proof_of_tv}
\subsubsection{Proof of Theorem \ref{thm_tv_target_bound}}
\begin{lem}[Target-domain total variation bound;
{\cite{Fu2024UnveilCD}, Appendix D.3}]
\label{lem_target_tv_bound}

Let $\widehat{P}_{w}(\cdot| z)$ denote the
conditional distribution generated by the backward diffusion process
using the estimated score network $\widehat{s}_w$, and let
$P_0(\cdot| z)$ denote the target conditional distribution.
Then
\begin{align}
    &\mathbb{E}_{D_0,\ldots,D_K}
    \mathbb{E}_{Z\sim p_0}
    \left[
        d_{\mathrm{TV}}
        \left(
            P_0(\cdot| Z),
            \widehat{P}_{w}(\cdot| z)
        \right)
    \right]
    \notag\\
    &\quad\lesssim
    \sqrt{t_0}
    \left(
        \log\frac{1}{t_0}
    \right)^{\frac{d_x+1}{2}}
    +
    \exp(-T)
    +
    \sqrt{
        {T}
    }\,
    \sqrt{
        \mathbb{E}_{D_0,\ldots,D_K}
        \left[
            \mathcal{R}_0(\widehat{s}_w)
        \right]
    }.
    \label{eq_target_tv_score_bound}
\end{align}

Consequently, by (\ref{eq_target_risk_final_general}),
\begin{align}
    &\mathbb{E}_{D_0,\ldots,D_K}
    \mathbb{E}_{Z\sim p_0}
    \left[
        d_{\mathrm{TV}}
        \left(
            P_0(\cdot| Z),
            \widehat{P}_{w}(\cdot| z)
        \right)
    \right]
    \notag\\
    &\quad\lesssim
    \sqrt{t_0}
    \left(
        \log\frac{1}{t_0}
    \right)^{\frac{d_x+1}{2}}
    +
    \exp(-T)
    \notag\\
    &\qquad+
    \sqrt{T}
    \Bigg[
        \log\frac{1}{t_0}\,
        N_{\mathrm{eff}}^{
            -\frac{2\beta}
            {d_x+d_z+2\beta}
        }
        \left(
            \log N_{\mathrm{eff}}
        \right)^{\max\{17, \beta \}}
        \notag\\
    &\hspace{19mm}
        +
        \frac{W_N^2}{(n_0+W_N)^2} \left[ 
    \overline{w}^\top G \overline{w} + \operatorname{diag}(G)^\top
        \overline w \right]
    \Bigg]^{1/2}.
    \label{eq_target_tv_final_general}
\end{align}
\end{lem}

\begin{cor}[Target-domain total variation convergence rate]
\label{cor_target_tv_rate}

Suppose the conditions of
Lemma~\ref{lem_target_tv_bound} and Theorem \ref{thm_score_risk_bound} hold. Let
\begin{equation}
    T
    =
    \frac{2\beta}
    {d_x+d_z+2\beta}
    \log N_{\mathrm{eff}},
    \qquad
    t_0
    =
    N_{\mathrm{eff}}^{
        -\frac{4\beta}{d_x+d_z+2\beta}-1
    }.
    \label{eq_T_t0_tv_choice}
\end{equation}
Then
\begin{equation}
    \log\frac{1}{t_0}
    \asymp
    \log N_{\mathrm{eff}},
    \qquad
    T
    \asymp
    \log N_{\mathrm{eff}}.
    \label{eq_T_t0_log_order}
\end{equation}
Moreover,
\begin{align}
    \exp(-T)
    &=
    N_{\mathrm{eff}}^{
        -\frac{2\beta}
        {d_x+d_z+2\beta}
    },
    \label{eq_terminal_tv_rate}
    \\
    \sqrt{t_0}
    \left(
        \log\frac{1}{t_0}
    \right)^{\frac{d_x+1}{2}}
    &=
    o\left(
        N_{\mathrm{eff}}^{
            -\frac{\beta}
            {d_x+d_z+2\beta}
        }
    \right).
    \label{eq_early_stopping_tv_rate}
\end{align}

Combining Lemma~\ref{lem_target_tv_bound} with
(\ref{eq_target_risk_final_general}), we therefore obtain
\begin{align}
    &
    \mathbb{E}_{D_0,\ldots,D_K}
    \mathbb{E}_{Z\sim p_0}
    \left[
        d_{\mathrm{TV}}
        \left(
            P_0(\cdot| Z),
            \widehat{P}_w
            (\cdot| Z)
        \right)
    \right]
    \notag\\
    &\quad\lesssim
    N_{\mathrm{eff}}^{
        -\frac{\beta}
        {d_x+d_z+2\beta}
    }
    \left(
        \log N_{\mathrm{eff}}
    \right)^{
        \max\left\{
            \frac{19}{2},
            \frac{\beta+2}{2}
        \right\}
    }
    \notag\\
    &\qquad\quad+
    \sqrt{\log N_{\mathrm{eff}}}
    \left[
        \frac{W_N^2}{(n_0+W_N)^2} \left[ 
    \overline{w}^\top G \overline{w} + \operatorname{diag}(G)^\top
        \overline w \right]
    \right]^{1/2} \notag \\
    &\quad\lesssim
    N_{\mathrm{eff}}^{
        -\frac{\beta}
        {d_x+d_z+2\beta}
    }
    \left(
        \log N_{\mathrm{eff}}
    \right)^{
        \max\left\{
            \frac{19}{2},
            \frac{\beta+2}{2}
        \right\}
    }
    \notag\\
    &\qquad\quad+
    \sqrt{\log N_{\mathrm{eff}}} \cdot
        \frac{W_N}{n_0+W_N} \sqrt{  \overline{w}^\top G \overline{w} + \operatorname{diag}(G)^\top \overline w } .
    \label{eq_target_tv_rate_transfer_general}
\end{align}
\end{cor}


Suppressing polylogarithmic factors in $N_{\mathrm{eff}}$ and
absorbing the lower-order remainder term, this can be written as
\begin{align}
    &\mathbb{E}_{D_0,\ldots,D_K}
    \mathbb{E}_{Z_0 \sim P_0}
    \left[
        d_{\mathrm{TV}}
        \left(
            P_0(\cdot| Z_0),
            \widehat{P}_w(\cdot| Z_0)
        \right)
    \right] \notag \\
    & \quad =
    \widetilde{\mathcal O}
    \left(
        N_{\mathrm{eff}}^{
            -\frac{\beta}{d_x+d_z+2\beta}
        }
        +
        \frac{W_N}{n_0+W_N}
        \sqrt{  \overline{w}^\top G \overline{w} + \operatorname{diag}(G)^\top \overline w }
    \right),
\end{align}
which proves Theorem~\ref{thm_tv_target_bound}.

\hfill \qedsymbol

\subsection{Type I error control of DA-CIT}
\label{subsec_da_cit_type1}

\begin{lem}[Type I error control of CRT; Theorem~4 of
\citet{berrett2020conditional}]
\label{lem_da_cit_type1}

Suppose that the null hypothesis $H_0$ holds. Let $D_0^{\mathrm{CRT}}
    :=
    \left\{
        (x_{0,i},y_{0,i},z_{0,i})
    \right\}_{i=1}^{n_0}$
denote the target samples used to conduct the CRT, and let $D_0^{\mathrm{Tr}}
    :=
    \left\{
        (x_{0,i},y_{0,i},z_{0,i})
    \right\}_{i=n_0+1}^{2n_0}$
denote the target samples used to train the conditional
diffusion model. Since the observations are independent,
$D_0^{\mathrm{CRT}}$ and $D_0^{\mathrm{Tr}}$ are independent.
Suppose that $\widehat P_0(\cdot| z)$ is estimated using
$D_0^{\mathrm{Tr}}$ and, possibly, the source datasets
$D_1,\ldots,D_K$, and is therefore independent of
$D_0^{\mathrm{CRT}}$. Then, for any $\alpha\in(0,1)$, the CRT $p$-value satisfies
\begin{align}
    \mathbb{P}(p\leq\alpha)
    \leq
    \alpha
    +
    \mathcal{O}
    \left(
        n_0\,
        \mathbb{E}_{D_0^{\mathrm{Tr}},D_1,\ldots,D_K}
        \mathbb{E}_{Z_0}
        \left[
            d_{\mathrm{TV}}
            \left(
                P_0(\cdot| Z_0),
                \widehat P_0(\cdot| Z_0)
            \right)
        \right]
    \right).
    \label{eq_crt_type1}
\end{align}
\end{lem}
Because $D_0^{\mathrm{CRT}}$ and $D_0^{\mathrm{Tr}}$ are independent, and $D_0^{\mathrm{CRT}}$ is also independent of $\widehat s_w$, we have $\mathbb{E}_{D_0^{\mathrm{Tr}},D_1,\ldots,D_K}=\mathbb{E}_{D_0,D_1,\ldots,D_K}$.
By Lemma~\ref{lem_da_cit_type1}, the excess Type I error of DA-CIT is controlled by TV error bound of $\widehat{P}_w (\cdot | Z_0)$ which is estimated by DA-Diff. Substituting the TV convergence rate in (\ref{eq_tv_target_bound}) into (\ref{eq_crt_type1}) therefore yields the Type I error bound of DA-CIT in (\ref{eq_da_cit_type1}).

\hfill \qedsymbol

\subsection{Estimation of source discrepancy matrix $G$}
\label{sec_estimate_G}

In this section, we show that the source discrepancy matrix $G$ can be
accurately estimated as mentioned in Remark~\ref{rem_estimate_G}.
To this end, we assume that an independent holdout validation sample
\begin{equation}
    D_0^{\mathrm{val}}
    :=
    \left\{
        (x_{0,\ell}^{\mathrm{val}},
        z_{0,\ell}^{\mathrm{val}})
    \right\}_{\ell=1}^{n_{\mathrm{val}}}
    \overset{\mathrm{i.i.d.}}{\sim} P_0
\end{equation}
is available, which is independent of all samples used to train
$\widehat s_w,\widehat s_1,\ldots,\widehat s_K$ and is used only for
estimating $G$. Recall the definition
$\Delta_{k,t}=s_k^*-s_0^*$ in~(\ref{eq_delta_k_t}) and the inner product
$\langle \cdot,\cdot\rangle_{\mathcal H_0}$ in
(\ref{eq_inner_product_Hk}). Then,
\begin{equation}
    G_{ij}
    =
    \langle
        \Delta_{i,t},
        \Delta_{j,t}
    \rangle_{\mathcal H_0}.
    \label{eq_G_inner_product}
\end{equation}
Define
\begin{equation}
    \widehat\Delta_{k,t}
    :=
    \widehat s_k-\widehat s_w,
    \qquad k=1,\ldots,K.
    \label{eq_G_hat_delta}
\end{equation}
We estimate $G_{ij}$ using the holdout validation sample as
\begin{align}
    \widehat G_{ij}
    &:=
    \frac{1}{n_{\mathrm{val}}}
    \sum_{\ell=1}^{n_{\mathrm{val}}}
    \frac{1}{T-t_0}
    \int_{t_0}^T
    \mathbb E_{
        X_t\sim
        \mathcal N(
            a_t x_{0,\ell}^{\mathrm{val}},
            \sigma_t^2I_{d_x}
        )
    }
    \Big[
        \big\langle
            \widehat s_w(
                X_t,z_{0,\ell}^{\mathrm{val}},t
            )
            -
            \widehat s_i(
                X_t,z_{0,\ell}^{\mathrm{val}},t
            ),
            \notag\\[-1mm]
    &\hspace{67mm}
            \widehat s_w(
                X_t,z_{0,\ell}^{\mathrm{val}},t
            )
            -
            \widehat s_j(
                X_t,z_{0,\ell}^{\mathrm{val}},t
            )
        \big\rangle
    \Big]
    dt .
    \label{eq_G_empirical_estimator}
\end{align}

We first establish the convergence rates of the score estimators
appearing in~(\ref{eq_G_empirical_estimator}). To achieve that, we need an extra overlap assumption on individual densities.

\begin{assumption}[Individual source overlap]
    For the source-specific estimators, we additionally assume that there
exists a constant $C_{\mathrm{ov}}>0$, independent of the sample sizes,
such that
\begin{equation}
    p_0(z)
    \leq
    C_{\mathrm{ov}}p_k(z),
    \qquad
    k=1,\ldots,K,
    \quad
    p_0\text{-almost everywhere}.
    \label{eq_individual_source_overlap_G}
\end{equation}
\end{assumption}
Together with Assumption~\ref{assump_smoothness_tail}, similar to Appendix \ref{sec_turn_H0_to_Hw}, we can derive
\begin{equation}
    \|g\|_{\mathcal H_0}^2
    \lesssim
    \|g\|_{\mathcal H_k}^2,
    \qquad
    k=1,\ldots,K.
    \label{eq_Hk_to_H0_G}
\end{equation}

\begin{lem}[Score estimation errors for estimating $G$]
\label{lem_score_error_G}

Suppose that Assumptions~\ref{assump_transf} and
\ref{assump_smoothness_tail} hold, together with
(\ref{eq_individual_source_overlap_G}), and that the conditions of
Theorem~\ref{thm_score_risk_bound} are satisfied.
For $k=1,\ldots,K$, let $\widehat s_k$ be the conditional score
estimator trained using $D_k$.

Define
\begin{equation}
    r_w
    :=
    \mathbb E_{D_0,\ldots,D_K}
    \left[
        \|\widehat s_w-s_0^*\|_{\mathcal H_0}^2
    \right],
    \qquad
    r_k
    :=
    \mathbb E_{D_k}
    \left[
        \|\widehat s_k-s_k^*\|_{\mathcal H_0}^2
    \right],
    \label{eq_G_score_errors}
\end{equation}
and
\begin{equation}
    \varepsilon_G^2
    :=
    r_w
    +
    \max_{1\leq k\leq K}r_k.
    \label{eq_epsilon_G}
\end{equation}
Let
\[
    n_{\min}
    :=
    \min_{1\leq k\leq K}n_k.
\]
Then,
\begin{align}
    \varepsilon_G^2
    \lesssim\;&
    \widetilde{\mathcal O}
    \left(
        N_{\mathrm{eff}}^{
        -\frac{2\beta}
        {d_x+d_z+2\beta}}
    \right)
    +
    \frac{W_N^2}{(n_0+W_N)^2}
    n_0^{-2\gamma} +
    \log\frac{1}{t_0}\,
    n_{\min}^{
        -\frac{2\beta}
        {d_x+d_z+2\beta}}
    (\log n_{\min})^{\max\{17,\beta\}}.
    \label{eq_epsilon_G_rate}
\end{align}
Under the feasible regime~(\ref{eq_neff_sandwich_condition}),
the first two terms on the right-hand side of
(\ref{eq_epsilon_G_rate}) converge to zero.
Moreover, since we choose
$t_0=N^{-C_\sigma}$ and
$N\asymp
N_{\mathrm{eff}}^{(d_x+d_z)/(d_x+d_z+2\beta)}$ in (\ref{eq_t0_T_approximation}) and (\ref{eq_optimal_network_complexity}),
the last term converges to zero under the additional growth condition
\begin{equation}
\label{eq_source_sample_growth_condition}
    N_{\mathrm{eff}}
    \ll
    \exp\left(n_{\min}^{c_s}\right),
    \qquad
    0<c_s<2\Gamma,
\end{equation}
where
$\Gamma=\beta/(d_x+d_z+2\beta)$ is defined in Section \ref{sec_da_cit} and $c_s$ is a constant. Therefore,
$\varepsilon_G=o(1)$
under~(\ref{eq_neff_sandwich_condition}) and (\ref{eq_source_sample_growth_condition}) when 
$n_0,n_{\min}\to\infty$.
\end{lem}

\begin{proof}
First,  Theorem \ref{thm_score_risk_bound} gives
\begin{align}
    r_w
    \lesssim\;&
    \widetilde{\mathcal O}
    \left(
        N_{\mathrm{eff}}^{
        -\frac{2\beta}
        {d_x+d_z+2\beta}}
    \right) +
    \frac{W_N^2}{(n_0+W_N)^2}
    \left[
        \overline w^\top G\overline w
        +
        \operatorname{diag}(G)^\top
        \overline w
    \right].
    \label{eq_rw_G_proof}
\end{align}

By Assumption~\ref{assump_transf},
\begin{equation}
    G_{kk}
    =
    \|\Delta_{k,t}\|_{\mathcal H_0}^2
    =
    \mathcal O(n_0^{-2\gamma}),
    \qquad
    k=1,\ldots,K.
    \label{eq_G_diag_rate}
\end{equation}
Since $G$ is a Gram matrix, the Cauchy--Schwarz inequality gives
\[
    |G_{ij}|
    \leq
    \sqrt{G_{ii}G_{jj}}
    =
    \mathcal O(n_0^{-2\gamma}).
\]
Moreover, since $\overline w$ belongs to the probability simplex,
\begin{align}
    \overline w^\top G\overline w
    &\leq
    \sum_{i=1}^K\sum_{j=1}^K
    \overline w_i\overline w_j|G_{ij}|
    =
    \mathcal O(n_0^{-2\gamma}),
    \\
    \operatorname{diag}(G)^\top\overline w
    &=
    \sum_{k=1}^K
    \overline w_kG_{kk}
    =
    \mathcal O(n_0^{-2\gamma}).
\end{align}
Substituting these bounds into~(\ref{eq_rw_G_proof}) yields
\begin{equation}
    r_w
    \lesssim
    \widetilde{\mathcal O}
    \left(
        N_{\mathrm{eff}}^{
        -\frac{2\beta}
        {d_x+d_z+2\beta}}
    \right)
    +
    \frac{W_N^2}{(n_0+W_N)^2}
    n_0^{-2\gamma}.
    \label{eq_rw_G_rate}
\end{equation}

We next consider the source-domain score estimators.
Lemma D.7 of \citet{Fu2024UnveilCD} shows that, under Assumption~\ref{assump_smoothness_tail}, the conditional score estimator
trained using $n_k$ observations satisfies
\begin{equation}
    \mathbb E_{D_k}
    \left[
        \|\widehat s_k-s_k^*\|_{\mathcal H_k}^2
    \right]
    \lesssim
    \log\frac{1}{t_0}\,
    n_k^{
        -\frac{2\beta}
        {d_x+d_z+2\beta}}
    (\log n_k)^{\max\{17,\beta\}}.
    \label{eq_source_score_rate_Hk}
\end{equation}
By (\ref{eq_Hk_to_H0_G}),
\begin{align}
    r_k
    &=
    \mathbb E_{D_k}
    \left[
        \|\widehat s_k-s_k^*\|_{\mathcal H_0}^2
    \right]
    \notag\\
    &\lesssim
    \mathbb E_{D_k}
    \left[
        \|\widehat s_k-s_k^*\|_{\mathcal H_k}^2
    \right]
    \notag\\
    &\lesssim
    \log\frac{1}{t_0}\,
    n_k^{
        -\frac{2\beta}
        {d_x+d_z+2\beta}}
    (\log n_k)^{\max\{17,\beta\}}.
    \label{eq_rk_G_rate}
\end{align}
Taking the maximum over $k=1,\ldots,K$ and combining
(\ref{eq_rw_G_rate}) and (\ref{eq_rk_G_rate}) gives
(\ref{eq_epsilon_G_rate}). 

\end{proof}

We next separate the score estimation error from the empirical
approximation error. Define the population plug-in counterpart of
(\ref{eq_G_empirical_estimator}) as
\begin{equation}
    \widetilde G_{ij}
    :=
    \left\langle
        \widehat\Delta_{i,t},
        \widehat\Delta_{j,t}
    \right\rangle_{\mathcal H_0},
    \qquad
    i,j=1,\ldots,K.
    \label{eq_G_population_estimator}
\end{equation}

\begin{lem}[Error of the population plug-in estimator]
\label{lem_G_plugin_consistency}

Under the conditions of Lemma~\ref{lem_score_error_G}, uniformly over
$i,j=1,\ldots,K$,
\begin{equation}
    \mathbb E_{D_0,\ldots,D_K}
    \left[
        \left|
            \widetilde G_{ij}-G_{ij}
        \right|
    \right]
    \lesssim
    n_0^{-2\gamma} 
    +
    \varepsilon_G^2.
    \label{eq_G_plugin_error}
\end{equation}
Under the feasible growth regime
(\ref{eq_neff_sandwich_condition}) and (\ref{eq_source_sample_growth_condition}), the right-hand side converges to zero.
\end{lem}

\begin{proof}
Let
\[
    e_0
    :=
    \widehat s_w-s_0^*,
    \qquad
    e_k
    :=
    \widehat s_k-s_k^*,
    \qquad
    u_k
    :=
    e_k-e_0.
\]
By the definitions of $\Delta_{k,t}$ and $\widehat\Delta_{k,t}$ in (\ref{eq_delta_k_t}) and (\ref{eq_G_hat_delta}),
\begin{align}
    \widehat\Delta_{k,t}
    &=
    \widehat s_k-\widehat s_w
    \notag\\
    &=
    (s_k^*+e_k)-(s_0^*+e_0)
    \notag\\
    &=
    \Delta_{k,t}+u_k.
    \label{eq_hat_delta_decomp_G}
\end{align}
Therefore,
\begin{align}
    \widetilde G_{ij}-G_{ij}
    &=
    \left\langle
        \Delta_{i,t}+u_i,
        \Delta_{j,t}+u_j
    \right\rangle_{\mathcal H_0}
    -
    \left\langle
        \Delta_{i,t},
        \Delta_{j,t}
    \right\rangle_{\mathcal H_0}
    \notag\\
    &=
    \langle
        u_i,\Delta_{j,t}
    \rangle_{\mathcal H_0}
    +
    \langle
        \Delta_{i,t},u_j
    \rangle_{\mathcal H_0}
    +
    \langle
        u_i,u_j
    \rangle_{\mathcal H_0}.
    \label{eq_G_plugin_decomp}
\end{align}

By Assumption~\ref{assump_transf},
\begin{equation}
    \|\Delta_{k,t}\|_{\mathcal H_0}
    =
    \sqrt{G_{kk}}
    =
    \mathcal O(n_0^{-\gamma}).
    \label{eq_delta_G_rate}
\end{equation}
Moreover,
\begin{align}
    &\mathbb E_{D_0,\ldots,D_K}
    \left[
        \|u_k\|_{\mathcal H_0}^2
    \right]
    \notag\\
    &\qquad=
    \mathbb E_{D_0,\ldots,D_K}
    \left[
        \|e_k-e_0\|_{\mathcal H_0}^2
    \right]
    \notag\\
    &\qquad\leq
    2
    \mathbb E_{D_0,\ldots,D_K}
    \left[
        \|e_0\|_{\mathcal H_0}^2
    \right]
    +
    2
    \mathbb E_{D_0,\ldots,D_K}
    \left[
        \|e_k\|_{\mathcal H_0}^2
    \right]
    \notag\\
    &\qquad=
    2r_w+2r_k
    \lesssim
    \varepsilon_G^2.
    \label{eq_uk_G_rate}
\end{align}

Applying the Cauchy--Schwarz inequality to the first term in
(\ref{eq_G_plugin_decomp}) gives
\begin{align}
    &\mathbb E_{D_0,\ldots,D_K}
    \left[
        \left|
        \left\langle
            u_i,\Delta_{j,t}
        \right\rangle_{\mathcal H_0}
        \right|
    \right]
    \notag\\
    &\qquad\leq
    \|\Delta_{j,t}\|_{\mathcal H_0}
    \left(
        \mathbb E_{D_0,\ldots,D_K}
        \left[
            \|u_i\|_{\mathcal H_0}^2
        \right]
    \right)^{1/2}
    \notag\\
    &\qquad \lesssim
    n_0^{-\gamma}\varepsilon_G \notag \\
    &\qquad \lesssim n_0^{-2\gamma} + \varepsilon_G^2
    \label{eq_G_plugin_cross_1}
\end{align}
Similarly,
\begin{equation}
    \mathbb E_{D_0,\ldots,D_K}
    \left[
        \left|
        \left\langle
            \Delta_{i,t},u_j
        \right\rangle_{\mathcal H_0}
        \right|
    \right]
    \lesssim
    n_0^{-\gamma}\varepsilon_G \lesssim n_0^{-2\gamma} + \varepsilon_G^2.
    \label{eq_G_plugin_cross_2}
\end{equation}
For the last term, another application of the Cauchy--Schwarz
inequality yields
\begin{align}
    &\mathbb E_{D_0,\ldots,D_K}
    \left[
        \left|
        \left\langle
            u_i,u_j
        \right\rangle_{\mathcal H_0}
        \right|
    \right]
    \notag\\
    &\qquad\leq
    \left(
        \mathbb E_{D_0,\ldots,D_K}
        \left[
            \|u_i\|_{\mathcal H_0}^2
        \right]
    \right)^{1/2}
    \left(
        \mathbb E_{D_0,\ldots,D_K}
        \left[
            \|u_j\|_{\mathcal H_0}^2
        \right]
    \right)^{1/2}
    \notag\\
    &\qquad\lesssim
    \varepsilon_G^2.
    \label{eq_G_plugin_cross_3}
\end{align}
Combining
(\ref{eq_G_plugin_cross_1})--(\ref{eq_G_plugin_cross_3})
proves~(\ref{eq_G_plugin_error}).
Under the feasible growth regime
(\ref{eq_neff_sandwich_condition}),
$\varepsilon_G\to0$, and hence the right-hand side of
(\ref{eq_G_plugin_error}) converges to zero.

\end{proof}

It remains to control the empirical approximation error
$\widehat G_{ij}-\widetilde G_{ij}$.
For fixed trained score networks, define
\begin{align}
    H_{ij}(x,z)
    &:=
    \frac{1}{T-t_0}
    \int_{t_0}^T
    \mathbb E_{
        X_t\sim
        \mathcal N(a_tx,\sigma_t^2I_{d_x})
    }
    \Big[
        \big\langle
            \widehat\Delta_{i,t}(X_t,z),
            \widehat\Delta_{j,t}(X_t,z)
        \big\rangle
    \Big]
    dt.
    \label{eq_G_Hij_def}
\end{align}
Recall (\ref{eq_G_hat_delta}) that
$\widehat\Delta_{k,t}=\widehat s_k-\widehat s_w$.
Then,
\begin{equation}
    \widehat G_{ij}
    =
    \frac{1}{n_{\mathrm{val}}}
    \sum_{\ell=1}^{n_{\mathrm{val}}}
    H_{ij}
    \left(
        x_{0,\ell}^{\mathrm{val}},
        z_{0,\ell}^{\mathrm{val}}
    \right),
    \qquad
    \widetilde G_{ij}
    =
    \mathbb E_{(X,Z)\sim P_0}
    \left[
        H_{ij}(X,Z)
    \right].
    \label{eq_G_Hij_emp_pop}
\end{equation}

\begin{lem}[Empirical approximation error of $\widehat G$]
\label{lem_G_empirical_error}

Suppose that the holdout sample $D_0^{\mathrm{val}}$ is independent of
$D_0,\ldots,D_K$, and that
$\widehat s_w,\widehat s_1,\ldots,\widehat s_K$ belong to the network
class~(\ref{eq_relu_network_class}). Define
\begin{equation}
    M_G
    :=
    \frac{4d_x}{T-t_0}
    \int_{t_0}^T
    M_t^2\,dt.
    \label{eq_MG_def}
\end{equation}
Then, uniformly over $i,j=1,\ldots,K$,
\begin{equation}
    \mathbb E_{
        D_0,\ldots,D_K,D_0^{\mathrm{val}}
    }
    \left[
        \left|
            \widehat G_{ij}-\widetilde G_{ij}
        \right|
    \right]
    \lesssim
    \frac{M_G}{\sqrt{n_{\mathrm{val}}}}.
    \label{eq_G_empirical_error}
\end{equation}
Under the network configuration used in Lemma \ref{lem_uniform_loss_bound_unbounded} Proposition C.4 of \citep{Fu2024UnveilCD}, and by the choose of $ \ T$ in Theorem \ref{thm_approx_err_of_diff}, we have 
\begin{equation}
    M_G
    =
    \mathcal O(\log N).
    \label{eq_MG_rate}
\end{equation}
Consequently,
\begin{equation}
    \mathbb E_{
        D_0,\ldots,D_K,D_0^{\mathrm{val}}
    }
    \left[
        \left|
            \widehat G_{ij}-\widetilde G_{ij}
        \right|
    \right]
    \lesssim
    \frac{\log N}{\sqrt{n_{\mathrm{val}}}}.
    \label{eq_G_empirical_error_rate}
\end{equation}
\end{lem}

\begin{proof}
Condition on the training datasets
$D_0,\ldots,D_K$.
Then the trained score networks, and hence $H_{ij}$, are fixed.
Since the holdout observations are i.i.d.\ from $P_0$ and independent
of $D_0,\ldots,D_K$, we have
\begin{equation}
    \mathbb E_{
        D_0^{\mathrm{val}}
    }
    \left[
        \widehat G_{ij}
        \,\middle|\,
        D_0,\ldots,D_K
    \right]
    =
    \widetilde G_{ij}.
    \label{eq_G_conditional_unbiased}
\end{equation}

By the definition of the network class
(\ref{eq_relu_network_class}),
\[
    \sup_{x,z}
    \|\widehat s_k(x,z,t)\|_\infty
    \leq
    M_t,
    \qquad
    k=1,\ldots,K,
\]
and the same bound holds for $\widehat s_w$.
Therefore,
\begin{equation}
    \|
        \widehat\Delta_{k,t}(x,z)
    \|_2
    =
    \|
        \widehat s_k(x,z,t)
        -
        \widehat s_w(x,z,t)
    \|_2
    \leq
    2\sqrt{d_x}M_t.
    \label{eq_hat_delta_uniform_bound}
\end{equation}
It follows from the Cauchy--Schwarz inequality that
\begin{align}
    |H_{ij}(x,z)|
    &\leq
    \frac{1}{T-t_0}
    \int_{t_0}^T
    \mathbb E
    \left[
        \|
            \widehat\Delta_{i,t}(X_t,z)
        \|_2
        \|
            \widehat\Delta_{j,t}(X_t,z)
        \|_2
    \right]
    dt
    \notag\\
    &\leq
    \frac{4d_x}{T-t_0}
    \int_{t_0}^T
    M_t^2\,dt
    =
    M_G.
    \label{eq_Hij_bound}
\end{align}
Thus, conditionally on $D_0,\ldots,D_K$,
\begin{equation}
    \operatorname{Var}_{(X,Z)\sim P_0}
    \left[
        H_{ij}(X,Z)
        \,\middle|\,
        D_0,\ldots,D_K
    \right]
    \leq
    M_G^2.
\end{equation}
Since $\widehat G_{ij}$ is the empirical average of
$n_{\mathrm{val}}$ independent copies of $H_{ij}(X,Z)$,
\begin{align}
    \mathbb E_{
        D_0^{\mathrm{val}}
    }
    \left[
        (\widehat G_{ij}-\widetilde G_{ij})^2
        \,\middle|\,
        D_0,\ldots,D_K
    \right] 
    &=
    \operatorname{Var}
    \left(
        \widehat G_{ij}
        \,\middle|\,
        D_0,\ldots,D_K
    \right) \notag \\
    & \leq
    \frac{M_G^2}{n_{\mathrm{val}}}.
    \label{eq_G_empirical_conditional_variance}
\end{align}
Hence, by the Cauchy--Schwarz inequality,
\begin{align}
    \mathbb E_{
        D_0^{\mathrm{val}}
    }
    \left[
        \left|
            \widehat G_{ij}-\widetilde G_{ij}
        \right|
        \,\middle|\,
        D_0,\ldots,D_K
    \right] &\leq
    \left\{
        \mathbb E_{
            D_0^{\mathrm{val}}
        }
        \left[
            (\widehat G_{ij}-\widetilde G_{ij})^2
            \,\middle|\,
            D_0,\ldots,D_K
        \right]
    \right\}^{1/2} \notag \\
    & \leq
    \frac{M_G}{\sqrt{n_{\mathrm{val}}}}.
\end{align}
Taking expectation over $D_0,\ldots,D_K$ proves
(\ref{eq_G_empirical_error}).
Finally, by (\ref{eq_Mell_scaling_mixture_score}),
\[
    M_{\ell}
    =
    \int_{t_0}^{T}
    \frac{M_t^2}{\log N}\,dt
    =
    \mathcal O(\log N).
\]
Therefore,
\begin{align}
    M_G
    &=
    \frac{4d_x}{T-t_0}
    \int_{t_0}^{T}
    M_t^2\,dt
    \notag\\
    &=
    \frac{4d_x\log N}{T-t_0}
    M_{\ell}.
\end{align}
Since 
$T-t_0\asymp\log N$ shown in (\ref{eq_time_interval_order}), we obtain
\begin{equation}
    M_G
    =
    \mathcal O(\log N),
\end{equation}
which completes the proof.
\end{proof}

Combining the population plug-in error with the empirical
approximation error yields the following estimation bound for $G$.

\begin{thm}[Estimation error of the source discrepancy matrix]
\label{thm_G_estimation_consistency}

Under the conditions of
Lemmas~\ref{lem_score_error_G},
\ref{lem_G_plugin_consistency}, and
\ref{lem_G_empirical_error}, uniformly over
$i,j=1,\ldots,K$,
\begin{align}
    \mathbb E_{
        D_0,\ldots,D_K,D_0^{\mathrm{val}}
    }
    \left[
        \left|
            \widehat G_{ij}-G_{ij}
        \right|
    \right] \lesssim
    n_0^{-2\gamma}
    +
    \varepsilon_G^2
    +
    \frac{\log N}{\sqrt{n_{\mathrm{val}}}}.
    \label{eq_G_final_error}
\end{align}

Furthermore,under the feasible regime
(\ref{eq_neff_sandwich_condition}) and (\ref{eq_source_sample_growth_condition}), provided that
\begin{equation}
    \frac{\log N }{\sqrt{n_{\mathrm{val}}}} \asymp \frac{\log N_{\mathrm{eff}} }{\sqrt{n_{\mathrm{val}}}}
    \to0,
    \label{eq_source_score_G_vanish}
\end{equation}
we have
\begin{equation}
    \mathbb E_{
        D_0,\ldots,D_K,D_0^{\mathrm{val}}
    }
    \left[
        \left|
            \widehat G_{ij}-G_{ij}
        \right|
    \right]
    \to0.
    \label{eq_G_final_convergence}
\end{equation}

\end{thm}

\begin{proof}
By the triangle inequality,
\begin{align}
    \mathbb E_{
        D_0,\ldots,D_K,D_0^{\mathrm{val}}
    }
    \left[
        \left|
            \widehat G_{ij}-G_{ij}
        \right|
    \right] &\leq
    \mathbb E_{
        D_0,\ldots,D_K,D_0^{\mathrm{val}}
    }
    \left[
        \left|
            \widehat G_{ij}-\widetilde G_{ij}
        \right|
    \right] \notag \\
    & \quad +
    \mathbb E_{
        D_0,\ldots,D_K
    }
    \left[
        \left|
            \widetilde G_{ij}-G_{ij}
        \right|
    \right].
\end{align}
Applying Lemmas~\ref{lem_G_empirical_error} and
\ref{lem_G_plugin_consistency} gives
\begin{align}
    \mathbb E_{
        D_0,\ldots,D_K,D_0^{\mathrm{val}}
    }
    \left[
        \left|
            \widehat G_{ij}-G_{ij}
        \right|
    \right] &\lesssim
    n_0^{-2\gamma} 
    +
    \varepsilon_G^2 +
    \frac{\log N}{\sqrt{n_{\mathrm{val}}}},
\end{align}
which proves~(\ref{eq_G_final_error}).

By Lemma~\ref{lem_score_error_G},
$\varepsilon_G^2\to0$ under the feasible regime
(\ref{eq_neff_sandwich_condition}) and (\ref{eq_source_sample_growth_condition}), together with
(\ref{eq_source_score_G_vanish}).
Since $n_0^{-2\gamma}\to0$, we further have
\[
    n_0^{-2\gamma} 
    +
    \varepsilon_G^2
    \to0. \quad \text{and} \quad \frac{\log N }{\sqrt{n_{\mathrm{val}}}} \asymp \frac{\log N}{\sqrt{n_{\mathrm{val}}}}
    \to0
\]
Therefore, every term on the right-hand side of
(\ref{eq_G_final_error}) converges to zero, which proves
(\ref{eq_G_final_convergence}). We can further take $N
    \asymp
    N_{\mathrm{eff}}^{
        \frac{d_x+d_z}
        {d_x+d_z+2\beta}
    }$  as specified in (\ref{eq_optimal_network_complexity}). Finally, we have 
\begin{align}
    \mathbb E_{
        D_0,\ldots,D_K,D_0^{\mathrm{val}}
    }
    \left[
        \left|
            \widehat G_{ij}-G_{ij}
        \right|
    \right] \lesssim
    n_0^{-2\gamma} 
    +
    \varepsilon_G^2
    +
    \frac{\log N_{\mathrm{eff}}}{\sqrt{n_{\mathrm{val}}}} \to 0 .
\end{align}
\end{proof}

\end{document}